\documentclass[11pt]{article}
\input{macros}
\allowdisplaybreaks[4]

\begin{document}

\title{Tractability from overparametrization:\\
The example of the negative perceptron}

\author{Andrea Montanari\thanks{Department of Statistics 
and Department of Electrical Engineering, Stanford University},\;\;
 Yiqiao Zhong\footnotemark[1], \;\; Kangjie Zhou\thanks{Department of Statistics, Stanford University}}
\date{\today}
\maketitle

\begin{abstract}
In the negative perceptron problem we are given $n$ data points 
$(\xx_i,y_i)$, where $\xx_i$ is a $d$-dimensional vector and $y_i\in\{+1,-1\}$
is a binary label. The data are not linearly separable and hence we content ourselves 
to find a linear classifier with the largest possible \emph{negative}
margin. In other words, we want to find
a unit norm vector $\btheta$ that maximizes $\min_{i\le n}y_i\<\btheta,\xx_i\>$.
This is a non-convex optimization problem (it is equivalent to finding a maximum norm
vector in a polytope), and we study its typical properties under two 
random models for the data. 

We consider the proportional asymptotics in which $n,d\to \infty$
with $n/d\to\delta$, and prove upper and lower bounds on the maximum margin $\kappa_{\sus}(\delta)$
or ---equivalently--- on its inverse function $\delta_{\sus}(\kappa)$.
In other words, $\delta_{\sus}(\kappa)$ is the overparametrization threshold:
for $n/d\le \delta_{\sus}(\kappa)-\eps$ a classifier achieving vanishing training error exists
with high probability, 
while for $n/d\ge \delta_{\sus}(\kappa)+\eps$ it does not.
Our bounds on $\delta_{\sus}(\kappa)$ match to the leading order as $\kappa\to -\infty$.
We then analyze a linear programming algorithm to find a solution, and characterize
the corresponding threshold $\delta_{\slin}(\kappa)$. We observe a gap
between the interpolation threshold $\delta_{\sus}(\kappa)$ and the linear programming threshold 
$\delta_{\slin}(\kappa)$, raising the question of the behavior of other algorithms.
\end{abstract}

\tableofcontents

\section{Introduction}
\label{sec:Introduction}

\subsection{Motivation: Overparametrized statistical models}
\label{sec:Motivation1}

In the last five years, an impressive amount of theoretical work has been devoted
to the study of overparametrized models in machine learning.
A small subsample of this literature includes \cite{jacot2018neural,du2018gradient,zou2018stochastic,
soudry2018implicit,allen2018convergence,chizat2019lazy,hastie2019surprises,liang2018just,belkin2018reconciling,
belkin2019two,huang2019dynamics,bartlett2020benign,montanari2019generalization,tsigler2020benign,chinot2022robustness}.

Overparametrized models are statistical models whose complexity 
is large enough to perfectly fit the training data, and that are indeed optimized
as to achieve vanishing training error. Classical statistical theory suggests that
using models that perfectly fit the training data can be dangerous if the data is noisy:
the model will adjust to interpolate the noise hence leading to poor predictions on unseen data.
Despite this theoretical prescription, overparametrization has become a 
widespread practice in deep learning \cite{zhang2021understanding}, thus
motivating a broad theoretical effort to understand its successes.

This theoretical effort has led to some important insights.
It was discovered that optimization algorithms of common use typically converge
to special models among the ones interpolating the data, a phenomenon known as 
``implicit regularization". Further, in high dimensions, a model can interpolate noise 
and yet be smooth on most of the space. As a consequence, interpolation or overfitting
can coexist with good generalization behavior.
We refer to \cite{bartlett2021deep,belkin2021review} for recent reviews on these
theoretical advances.

These mathematical insights were nearly all obtained in a
special setting: linear models under quadratic or convex loss. 
On the other hand, modern machine learning models  are highly 
non-linear and their optimization is a highly non-convex problem. In some
cases  \cite{jacot2018neural,du2018gradient,zou2018stochastic,allen2018convergence,chizat2019lazy},
it was possible to show that the linear setting is a good approximation of the 
non-linear one. However, in many cases of interest, the connection between
linear and non-linear models is just a loose analogy.

In particular, the optimization aspect of actual machine learning systems
is poorly modeled  within the linear/convex setting. 
Overparametrized models became popular because it was observed empirically that
they are easier to optimize than underparametrized ones, and they  can be optimized
using simple gradient-based methods, such as stochastic gradient descent
(SGD) \cite{Salakhutdinov}. A widespread rule of thumb among practitioners 
is to increase the number of model
 parameters until SGD converges to vanishing training error.
In some sense (which needs to be suitably formalized) optimization passes from
being  ``hard" to ``easy" as  as the number of parameters is increased.
Of course this transition is conspicuously absent in convex models which are tractable 
independently of the overparametrization.

A major motivation for the present paper is to initiate the rigorous study of a simple setting
in which the easy/hard phase transition in overparametrized models can be 
formalized and rigorously studied.

We will consider the standard setup of statistical learning, whereby
we are given $n$ data points $(\xx_i,y_i)\in \reals^d\times \{+1,-1\}$
(with $\xx_i$  a feature vector and $y_i$ a label). We want to learn
a model $f(\,\cdot\,;\btheta):\reals^d\to\reals$, parametrized by a vector $\btheta\in\Theta$. 
Empirical risk minimization  (ERM) proposes to find $\btheta$
 by minimizing the  training error (a.k.a. \emph{empirical risk}):
\begin{align}
\hR_n(\btheta) = \frac{1}{n}\sum_{i=1}^n\ell\big(y_i;f(\xx_i;\btheta)\big)\, ,\;\;\;\;\;
 \btheta\in\Theta\, .
\label{eq:GeneralERM}
\end{align}
Here $\ell:\{+1,-1\}\times\reals\to\reals$ is a loss function\footnote{If this matters
to the reader, they can assume that $\Theta\subseteq \reals^p$ is a closed set and
$\ell(\pm 1;\,\cdot\,)$  are lower semicontinuous fuctions, although we will
soon restrict ourselves to a very specific example.} such that
$\min_{x\in\reals}\ell(+1;x) = \min_{x\in\reals}\ell(-1;x)=0$.

By construction $\min_{\btheta\in\Theta}\hR_n(\btheta)\ge 0$. We will say that the model
is overparametrized if there is a choice of parameters that achieves vanishing training error,
i.e., $\min_{\btheta\in\Theta}\hR_n(\btheta) = 0$ \cite{neyshabur2015search}. 
%Of course this condition depends on the sample size $n$ and the data distribution 
%as well as on the model class $(f(\,\cdot\, ;\btheta):\btheta\in \Theta )$. 
When this happens, we define the set of interpolators (empirical risk minimizers):
\begin{align}
\ERM_0 :=\big\{\btheta\in\Theta:\;   \ell\big(y_i;f(\xx_i;\btheta)\big) =0 \;\;\forall i\le n\big\} 
\, .\label{eq:ERM0}
\end{align}
This definition can be of course generalized by considering the
set  $\ERM_\eps$ of vector parameters $\btheta$ such that $\hR_n(\btheta)\le \eps$.
In this paper we will focus on $\ERM_0$.

Classical statistical learning theory focuses on cases in which the empirical risk 
converges uniformly to the population (expected) risk  $R(\btheta):= \E\{\ell(y;f(\xx;\btheta))\}$ 
(possibly allowing for parameter spaces $\Theta(n)$ of growing complexity) 
\cite{geer2000empirical,shalev2014understanding}.
In this regime vanishing empirical risk can only be achieved if the population risk
 is vanishing:
\begin{align*}
&\lim_{n\to\infty}\sup_{\btheta\in\Theta(n)} \big|\hR_n(\btheta) -R(\btheta)\big| = 0 \;\;\mbox{and}
\;\; \ERM_0 \neq \emptyset,\\
&\phantom{AAAA} \Rightarrow \phantom{AAAA}\exists \theta_{*}(n)\in\Theta(n) \mbox{ s.t. } R(\btheta_*(n))\to 0\, .
\end{align*}
In contrast, we are interested here in cases in which labels are noisy (i.e. the conditional
distribution of $y_i$ given $\xx_i$ is non deterministic) and therefore
$\E\{\ell(y;f_*(\xx))\}>0$ holds strictly for any fixed $f_*$.
In these cases, vanishing training error is achieved because of a breakdown of uniform convergence.

A few important questions are:
\begin{itemize}
\item[{\sf Q1.}] \emph{Existence:} Is the set $\ERM_0$ non-empty?
\item[{\sf Q2.}] \emph{Tractability:} Can we find $\btheta\in \ERM_0$ using efficient
algorithms?
\item[{\sf Q3.}] \emph{Implicit regularization:} What are the properties of the 
special point $\hbtheta\in\ERM_0$ which is selected by a specific optimization algorithm of common 
use?
 \item[{\sf Q4.}] \emph{Generalization:} What is the error of the model $\hbtheta$
 thus selected on unseen data?
\end{itemize}
In this paper we study these questions in the simplest non-convex learning problem
we can think of: linear classification with a negative margin.

In linear classification, the function class is linear $f(\xx;\btheta) =\<\btheta,\xx\>$,
and we predict the label at a test point $\xx$ using $\sign(\<\btheta,\xx\>)$. Since the modulus of $\btheta$
is irrelevant, we will constrain throughout $\btheta\in\Theta=\S^{d-1}$ (the unit sphere in $d$
dimensions.)
If data are linearly separable (i.e., there exists $\btheta_*\in \S^{d-1}$ such that 
$y_i\<\xx_i,\btheta_*\> >0$  for all $i\le n$), a classical approach selects $\btheta$
by maximizing the margin. Namely, we define 
\begin{align}
\ERM_0(\kappa) :=\big\{\btheta\in\S^{d-1}:\;   y_i\<\xx_i,\btheta\>\ge \kappa  \;\;\forall i\le n\big\} 
\, .\label{eq:ERM0-Kappa}
\end{align}
and seek the largest $\kappa$ (the largest margin) such that this set is non-empty.
Although superficially non-convex, the problem of finding $\btheta\in\ERM_0(\kappa)$
is equivalent  to the one in which the constraint $\|\btheta\|_2=1$
is replaced by $\|\btheta\|_2\le 1$. Indeed, if a positive margin 
solution is found with $\|\btheta\|_2<1$, $\btheta$ can be rescaled to yield an even better margin. 
The set $\{\btheta: \, \|\btheta\|_2\le 1, \; 
y_i\<\xx_i,\btheta\>\ge \kappa \;\; \forall i\le n\}$ is convex, and hence is always computationally 
tractable.  
We refer to \cite{shalev2014understanding} for fundamental background on max-margin classification,
 and to \cite{montanari2019generalization} for a recent analysis in the overparametrized setting.

If data is not linearly separable, 
a popular method is to minimize the empirical risk  \eqref{eq:GeneralERM}, where $\ell$ 
is the hinge loss: $\ell(y;f) = (1-yf)_+$. However, an equally reasonable approach is to
seek $\hbtheta\in\ERM_0(\kappa)$ where now $\kappa$ is allowed to be negative 
(while $\ERM_0(\kappa)$ is still defined by Eq.~\eqref{eq:ERM0-Kappa}).
The remark below outlines a concrete statistical scenario in which selecting
a negative margin solution is a natural statistical procedure. 

We emphasize that our interest in this problem does not stem from 
a specific application. We are instead interested in understanding which tools and phenomena
can be expected in general non-convex overparametrized models.
\begin{rem}
We outline a concrete data distribution for which maximizing a negative margin arises
as a natural option. 
Pairs $(\bu_i,y_i)\in \reals^{d}\times\{+1,-1\}$, 
$i\le n$ are generated i.i.d. with joint distribution such that the two classes are 
linearly separable in $u$-space, with a positive magin $\kappa_0>0$. 
For instance, defining
the conditional distributions $\prob(\bu_i\in \,\cdot\, |y_i=\pm 1) = \rP_{\pm}(\,\cdot\,)$,
we could assume that
$\rP_{+}, \rP_-$ have disjoint and linearly separable supports. We denote by
$\btheta_0\in\S^{d-1}$ the direction along which the two classes are separable.

Instead of observing the latent vectors $\bu_i$, we observe a corrupted versions of the same
vectors $\bx_i=\bu_i+\beps_i$ with $\|\beps_i\|_2\le r_0$ with probability one. If $r_0 >\kappa_0$,
the resulting data $(\bx_i,y_i)$, $i\le n$ are no longer (in general) linearly separable.
On the other hand, we have $y_i\<\bx_i,\btheta_0\>\ge \kappa_0-r_0$, and hence
$\ERM_0(\kappa)$ in non-empty for all $\kappa\le \kappa_0-r_0$. It is reasonable to seek
$\hbtheta\in \ERM_0(\kappa)$ for such a negative $\kappa$.
\end{rem}

The present paper is devoted to the study of negative margin classification
under two simple models for the data distribution.

In a remarkable experiment,  \cite{zhang2021understanding} showed that state-of-the-art 
neural networks for computer vision could achieve vanishing 
training error even if actual labels $y_i$ were replaced by purely random ones.
In other words, the ability of such networks to interpolate data is not limited to
data that are nearly linearly separable at the population level.

The two data distributions that we will consider are partly motivated by these experiments.
In both cases, the feature vector $\xx_i$ have no structure: we take them to be isotropic Gaussians.
In the first data distribution, labels $y_i$ are pure noise,
i.e., uniformly random independently of $\xx_i$.
In the second they are noisy but not independent of the $\xx_i$'s
(more precisely, they depend on a one-dimensional projection of the $\xx_i$).

The two  distributions are dramatically different from the statistical viewpoint.
In particular, in the first one, no non-trivial learning is possible (because
no model can predict better than random guessing). However, we are interested in understanding 
whether they behave similarly from the optimization viewpoint as is the case 
in the experiment of  \cite{zhang2021understanding}.

\subsection{An additional motivation: Continuous constraint satisfaction problems}

In the overparametrized setting, learning is equivalent to solving a random constraint satisfaction 
problem (CSP) with continuous variables, as defined in Eq.~\eqref{eq:ERM0}.
We want to assign values to the variables $\btheta=(\theta_1,\dots,\theta_d)$
as to satisfy the $n$ constraints $\ell(y_i;f(\xx_i;\btheta))=0$, $i\le n$.
In the case studied here, each constraint removes a random spherical cap from the sphere $\S^{d-1}$.
Questions ${\sf Q1}$ (existence of solutions)  and ${\sf Q2}$ 
(computational tractability of the problem of finding solutions)
have been extensively studied for some canonical 
random CSPs over small alphabets (i.e., with decision variables $\btheta\in \{1,\dots, q\}^d$
for small constant $q$). Examples include random $K$-satisfiability, coloring of random graphs, and
independent sets of random graphs. We refer to
 \cite{achlioptas2006random,MezardMontanari,ding2022proof} and references therein. 

In contrast, much less is known rigorously for 
random CSPs with continuous variables, in particular when each constraint is
non-convex.

Among others, sphere packing problems can be formulated as constraint satisfaction
problems with continuous variables (the sphere centers).  These have attracted considerable 
attention within the statistical physics community.
Physicists derived a complete mean field theory of dense random packings in high 
dimension, which is however not proven rigorously,  see e.g.
\cite{parisi2020theory}.

Recently, Franz and Parisi \cite{franz2016simplest} studied the set
$\ERM_0(\kappa)$ of Eq.~\eqref{eq:ERM0-Kappa} (i.e., the linear classification 
for negative margin $\kappa<0$)
as toy model for the space of configurations of hard spheres in high 
dimension (they work within the `random labels' model that we also study here). 
One intuitive way to grasp the connection between  $\ERM_0(\kappa)$ and
hard spheres is as follows. $\ERM_0(\kappa)$ is 
obtained by carving $n$ spherical caps of geodesic radius $\arccos(|\kappa|)$
out of $\S^{d-1}$. This is the same as the set of 
allowed positions of the center of one hard sphere (of geodesic radius $r_*=\arccos(|\kappa|)/2$) 
in $\S^{d-1}$
when it is constrained not to overlap with 
 $n$  spheres of the same radius $r_*$ at fixed positions. Of course,
 the full sphere packing problem is more complex (since the other $n$ sphere are not
 centered at uniformly random positions) but \cite{franz2016simplest,franz2017universality} 
 argue that the two problems share some
 qualitative features. They use the replica method to derive the phase diagram of the
 negative margin problem, and the asymptotic law of the margins $y_i\<\xx_i,\btheta\>-\kappa$ as $\delta\uparrow \delta_{\sus}(\kappa)$.
 We refer to \cite{franz2019jamming, franz2019critical,sclocchi2022high} for further statistical physics 
 results on this problem.
% \am{Add refs}

Our work thus contributes to the rigorization of the results of this line of work.

\section{Main results}\label{sec:SumResults}

\subsection{Summary of results}
We consider the negative perceptron problem under two simple
data distributions.
\begin{description}
\item[Random labels.] We assume data to be isotropic Gaussian, and labels to be pure noise.
\begin{align}
 \xx_i\sim\normal(\bzero,\bI_d)\;\; \perp \;\;\; y_i\sim\Unif(\{+1,-1\})\, .
\end{align}	
\item[Labels correlated with a linear signal.] We assume the same simple model
for the feature vectors, but now  the labels depend on a one-dimensional projection
of the feature vectors:
\begin{equation*}
\xx_i\sim \normal(\bzero,\bI_d)\,\;\;\;\; \P( y_i = +1 \big\vert \xx_i ) = 
\varphi \big( \< \xx_i, \btheta_* \>\big)\, .
\end{equation*} 	
Here, without loss of generality $\|\btheta_*\|_2=1$.
Further, we will assume $\varphi:\reals\to [0,1]$ to be monotone increasing, 
with $\varphi(t)$ approaching 
$0$ or $1$ exponentially fast as $t\to\pm\infty$.
\end{description}
As mentioned above, the motivation for studying two different data model comes from 
recent experiments with deep learning systems \cite{zhang2021understanding}. In these experiments,
replacing true class labels 
(which contain a signal) with purely random labels led to some quantitative change
in the convergence of optimization algorithms, but no obvious qualitative change. 

In order to summarize our results,  
it is useful to define an interpolation threshold $\delta_{\sus}(\kappa)$
(the subscript `s' stands for satisfiability)
and an algorithmic threshold $\delta_{\salg}(\kappa)$ as follows.
The satisfiability threshold is the largest value of $n/d$ such that $\ERM_0$ is
non-empty with probability bounded away from zero.
Denoting by $\prob_{n,d}(\,\cdot\,)$  the probability distributions over instances
of negative perceptron:
\begin{align}
\delta_{\sus}(\kappa) := \sup\Big\{\delta\,: \; 
\liminf_{n\to\infty}\prob_{n,n/\delta}\Big(\ERM_0\neq \emptyset\Big)>0
\Big\}\, .\label{eq:Delta_S_def}
\end{align}

The algorithmic threshold is the largest value of $n/d$ such that there exists a 
polynomial-time algorithm  $\hbtheta^{\salg}$ that takes as input the data $\yy,\XX$
and outputs a global empirical risk minimizer (i.e., an element of $\ERM_0$), with 
probability bounded away from zero. Namely
\begin{align}
\delta_{\salg}(\kappa) := \sup\Big\{\delta\, : \;\exists \,\hbtheta^{\salg}\in{\sf Poly} 
\;\;\mbox{ s.t. }\;\;\liminf_{n\to\infty}\prob_{n,n/\delta}\Big(\hbtheta^{\salg}(\yy,\XX)\in \ERM_0\Big)>0
\Big\}\, ,\label{eq:Delta_Alg_def}
\end{align}
(We denote by {\sf Poly} the class of estimators that can be implemented by polynomial-time algorithms\footnote{This requires a slight modification of the problem in which 
a finite-precision  $\xx_i$'s are accepted as input.}.)
We expect $\delta_{\sus}(\kappa)$, $\delta_{\salg}(\kappa)$ to be sharp thresholds 
in the sense that for $\delta<\delta_{\sus}(\kappa)$ (or $\delta<\delta_{\salg}(\kappa)$),
the above probabilities are not only bounded away from zero but actually converge to one. 
We will show that this is indeed true for $\delta < \delta_{\rm lb} (\kappa)$, where 
$\delta_{\rm lb} (\kappa)$ is a lower bound on $\delta_{\sus} (\kappa)$ to be specified later, 
thus partially establishing the sharp threshold property.

\begin{figure}[t!]
    \centering
	\includegraphics[width=30em]{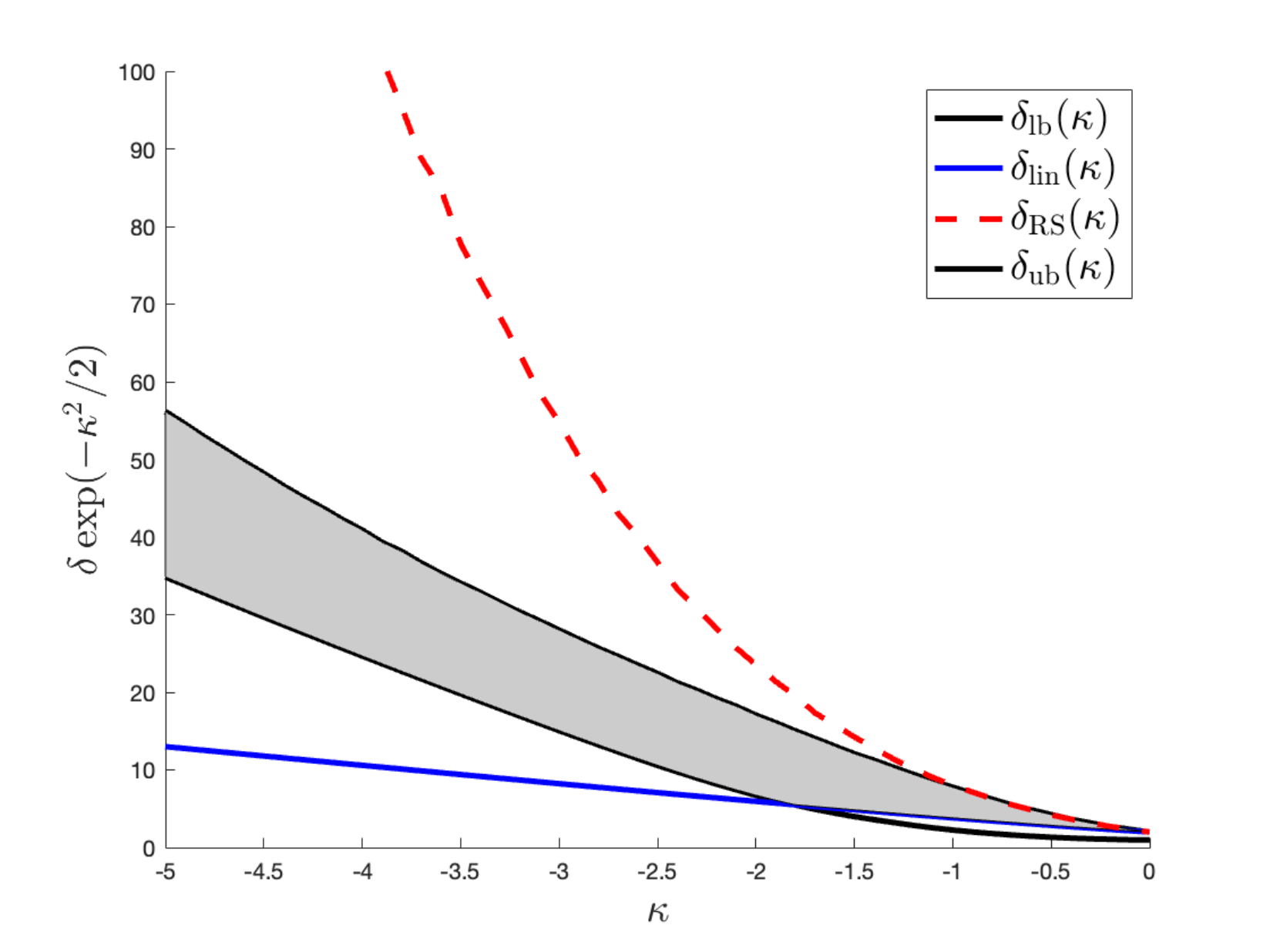}
	\caption{Phase diagram for the negative perceptron. The `replica symmetric' 
	prediction $\delta_{\mathrm{RS}} (\kappa)$ 
	coincides with the satisfiability threshold $\delta_{\sus}(\kappa)$ for $\kappa \ge 0$ but is only an upper bound for $\kappa < 0$. Our improved 
	$\delta_{\su} (\kappa)$ is strictly better than $\delta_{\mathrm{RS}} (\kappa)$ 
	for all negative values of $\kappa$. The lower bound $\delta_{\sl} (\kappa)$ 
	is inferior to the linear programming threshold $\delta_{\slin} (\kappa)$ (which is a lower bound on $\delta_{\salg} (\kappa)$; see its precise definition in Eq.~\eqref{cond:delta1}) when 
	$\vert \kappa \vert$ is small. As $\kappa$ decreases, $\delta_{\sl} (\kappa)$ 
	surpasses $\delta_{\slin} (\kappa)$. Finally, $\delta_{\sl}(\kappa)$,
	$\delta_{\sus}(\kappa)$ and $\delta_{\su}(\kappa)$ become asymptotically equivalent as 
	$\kappa \to -\infty$. The phase transition for the existence of $\kappa$-margin solution 
	occurs in the region delimited by 
	$\max \{\delta_{\sl} (\kappa),\delta_{\slin} (\kappa) \}$ and 
	$\delta_{\su} (\kappa)$ (gray area).}\label{fig:PhaseDiagram}
\end{figure}
We can now summarize our results following the blueprint of questions {\sf Q1}-- {\sf Q4}
of Section \ref{sec:Motivation1}.

\vspace{0.1cm}

\noindent{\sf Q1.} \emph{Existence.} 
We establish upper and lower bounds on the threshold
$\delta_{\sus}(\kappa)$.
These bounds are within a factor $1+o_{\kappa}(1)$ of each other for large 
negative $\kappa$. 

More in detail, denote by $\Phi(x) :=\int_{-\infty}^{x}\phi(t)\, \de t$ the Gaussian
distribution function and by $\phi(x) :=\exp(-x^2/2)/\sqrt{2\pi}$ the Gaussian density.
For \emph{random labels} we prove that, as $\kappa$ becomes large a negative (of course
after $n,d\to\infty$)
\begin{align}
\delta_{\sus}(\kappa) = \frac{\log|\kappa|}{\Phi(\kappa)}\big(1+o_{\kappa}(1)\big)\, .
\label{eq:ResultSummary_1}
\end{align}
(Recall that for large $\kappa<0$, $1/\Phi(\kappa)=(1+o_{\kappa}(1))|\kappa|/\phi(\kappa)$.)

For linear signals, let $\alpha>0$ be such that $\varphi(x) = 1 - C_{\tail} e^{-\alpha x}+o(e^{-\alpha x})$
as $x\to\infty$,  $\varphi(x) = C_{\tail} e^{\alpha x}+o(e^{\alpha x})$
as $x \to -\infty$, where $C_{\tail} > 0$ is a constant. Then we prove that
\begin{align}
\delta_{\sus}(\kappa) = \frac{e^{\alpha|\kappa|}\log|\kappa|}{4 C_{\tail} \Phi(\kappa)}\big(1+o_{\kappa}(1)\big)\, .
\end{align}

\noindent{\sf Q2.} \emph{Tractability.} We  obtain lower bounds on $\delta_{\salg}(\kappa)$
by constructing simple linear programming surrogates of the original non-convex problem.
For \emph{random labels}
\begin{align}
 \delta_{\salg}(\kappa) \ge \frac{1}{\Phi(\kappa)}\big(1+o_{\kappa}(1)\big)\, .
\label{eq:ResultSummary_1_ALG}
\end{align}
For linear signals, we prove that
\begin{align}
\delta_{\salg}(\kappa) \ge \frac{e^{\alpha|\kappa|}}{2 C_{\tail} \Phi(\kappa)}\big(1+o_{\kappa}(1)\big)\, .
\end{align}

\noindent{\sf Q3.} \emph{Implicit regularization.} As in 
the positive margin case \cite{soudry2018implicit,ji2018risk}, we construct a differentiable logistic-like loss
such that gradient flow with that loss converges to $\kappa$-margin
solutions, with negative $\kappa$, on non separable data, see Section \ref{sec:GD}.

\vspace{0.1cm}

\noindent{\sf Q4.} \emph{Generalization.} This question only makes sense within 
the linear signal model. 
In this case, we prove bounds on the values of correlation $\<\btheta,\btheta_*\>$
that are achieved by empirical risk minimizers $\btheta\in\ERM_0$,
and on the ones generated by our algorithm $\<\hbtheta^{\salg},\btheta_*\>$. These of 
course are equivalent to generalization bounds, as the generalization error is 
an explicit monotone function of  $\<\btheta,\btheta_*\>$.
We refer to Section \ref{sec:NonRandomLabels} for  precise statements.

\vspace{0.1cm}

While we summarized the behavior of our bounds for large negative $\kappa$,
we obtain explicit (although sometimes cumbersome) expressions for all $\kappa<0$. 
These bounds are plotted in Figure \ref{fig:PhaseDiagram} for the pure noise case,
and compared with the only previously known upper bound, which is
given by the so called `replica symmetric' formula.
The latter gives the correct phase transition for $\kappa>0$
and is an upper bound for all $\kappa$ \cite{stojnic2013another}.

It is perhaps useful to compare the results on $\delta_{\sus}(\kappa)$ 
with a natural `volume heuristics.'
$\Phi(\kappa)$ is the asymptotic normalized
volume of the spherical cap $\{\btheta\in\S^{d-1}:\; y_i\<\btheta,\xx_i\>\le \kappa\}$.
This is the fraction of the space of the parameters ruled out by each of
the $n$ interpolation constraints. 
Any fixed parameter vector  $\btheta\in\S^{d-1}$  
 satisfies all constraints with probability  $(1-\Phi(\kappa))^n$. 
One would guess the problem becomes unsatisfiable when this probability becomes
exponentially small in $d$, i.e., $n\log(1-\Phi(\kappa))\approx -Cd$. This heuristic
gives the rough order of the threshold in Eq.~\eqref{eq:ResultSummary_1}, 
namely $n/d\approx C'/\Phi(\kappa)$ but misses the $\log|\kappa|$ factor.

We next comment on the relevance of these results for our original statistical motivation.
\begin{description}
\item[Tractability from overparametrization.] The model becomes more overparametrized as $d/n$ increases (moving
down in Figure \ref{fig:PhaseDiagram}), or $\kappa$ decreases (moving left in Figure
 \ref{fig:PhaseDiagram}). Our results on $\delta_{\salg}(\kappa)$ confirm the empirical
 observation mentioned in the introduction, which is common wisdom in applied machine learning:
  if the model is sufficiently overparametrized, the problem of finding a global empirical risk minimizer becomes tractable. 

Our proofs use a very simple optimization algorithm, which 
replaces the original non-convex optimization problem by a linear programming surrogate;
see Section \ref{sec:AlgoRLabels} and Section \ref{sec:AlgoNRLabels}. We
expect a similar or potentially better behavior for gradient descent
or stochastic gradient descent, as suggested by numerical simulations of 
Section \ref{sec:numerical}. 
\item[Fundamental barriers for continuous optimization?] There 
is a large multiplicative gap, of order $\log|\kappa|$ between the satisfiability threshold
$\delta_{\sus}(\kappa)$ and the lower bound on the algorithmic threshold $\delta_{\salg}(\kappa)$. 
In words,
for most values of the overparametrization, solutions exist but our simple linear programming algorithm is not able to find them.

It is natural to ask whether better optimization algorithms can 
be constructed for this problem, or there is a fundamental computational barrier
and indeed $\delta_{\sus}(\kappa)\neq \delta_{\salg}(\kappa)$. Such gaps are indeed quite common 
in discrete random CSPs \cite{achlioptas2008algorithmic,coja2017walksat,bresler2021algorithmic}.
Precisely characterizing such gaps for overparametrized statistical learning problems
is a promising research problem. We took a first step in that direction.
\item[Beyond the neural tangent regime.] Over the last few years, considerable attention has been devoted to 
the remark that vanishing training error can be achieved generically in overparametrized
nonlinear models \cite{jacot2018neural,chizat2019lazy}.  In a nutshell, 
upon linearizing the model $f(\,\cdot\,;\btheta)$ 
around the initialization $\btheta_0$, interpolating the $n$ datapoints amounts to solving
the system of $n$ linear equations: 
$\bD f(\XX;\btheta_0)(\btheta-\btheta_0) = \yy-f(\XX;\btheta_0)$. 
(Here $f(\XX;\btheta)\in\reals^n$ is the vector with entries $f(\xx_i;\btheta)$.) 
For a non-degenerate Jacobian $\bD f(\XX;\btheta_0)$, this has a solution as soon as the number
of parameters exceeds the sample size. The analysis of neural networks based on this insight 
is also referred to as `neural tangent theory'.

Specializing this generic argument to the negative perceptron, we deduce that
 interpolation is possible for $n< d$. In other words, this argument implies 
 that $\delta_{\sus}(\kappa)\ge 1$. Explicitly, for $n<d$,
 we can solve the linear system $y_i\<\xx_i,\btheta\>=1$ for all $i\le n$, and
 hence obtain a linear separator by rescaling $\btheta$.
 
 Our results imply that (for negative $\kappa$) the interpolation threshold 
 $\delta_{\sus}(\kappa)$ is in fact much larger. Also, linearization arguments cannot 
 capture the gap between satisfiability and algorithmic threshold. 
 \item[Implicit regularization and generalization.] The generalization error 
 of a model with vanishing empirical risk, i.e., $\btheta\in\ERM_0(\kappa)$
 depends in general on the choice of the specific point. Section
 \ref{sec:Error} provides upper and lower bounds valid over the whole set $\ERM_0(\kappa)$
 under the linear signal model, as well as the asymptotic error for the linear programming algorithm. 
 
 Several features of these results are interesting from the statistics point of view:
 $(i)$~The set $\ERM_0(\kappa)$ appears to include models with a broad range of generalization errors;
 $(ii)$~The specific model selected by linear programming has good generalization properties;
 $(iii)$~Its behavior improves for smaller $\kappa$ (i.e., when the model is further overparametrized.)
 \end{description}

\subsection{Connection with random polytope geometry}
\label{sec:ConnectionsPolytope}

Given a dataset $\XX,\yy$ that is not linearly separable, the associated
maximum margin is 
\begin{align*}
\kappa_*(\XX,\yy):= \sup\big\{\kappa\in\reals:\; \exists \btheta \in \S^{d - 1} \;\; \text{s.t.} \;\; y_i\<\xx_i,\btheta\>\ge \kappa 
\;\; \forall i\le n\big\}\, ,
\end{align*}
This is of course related to the threshold $\delta_{\sus}(\kappa)$ introduced in 
Eq.~\eqref{eq:Delta_S_def}. If $\delta<\delta_{\sus}(\kappa)$ then $\kappa_*(\XX,\yy)\ge\kappa$
with probability bounded away from zero as $n,d\to\infty$, $n/d=\delta$.
If $\delta>\delta_{\sus}(\kappa)$ then $\kappa_*(\XX,\yy)\le \kappa$
with high probability.

By a simple change of variables,   $\kappa_*(\XX,\yy)$ is directly related to 
the radius of a certain polytope. Given $\ZZ\in\reals^{n\times d}$, a matrix
with rows $\zz_i$, denote by $\radius(\Polyt_\ZZ)$ the radius of the polytope
defined by the inequalities $\<\zz_i,\bxi\>\le 1$:
\begin{align}
\radius(\Polyt_\ZZ):= \max\big\{\|\bxi\|_2:\; \bxi\in\Polyt_{\ZZ}\big\}\, ,\;\;\;
\Polyt_\ZZ := \big\{ \bxi\in\reals^d: \<\zz_i,\bxi\>\le 1\;\; \forall i\le n\big\}\, .
\label{eq:RadiusDef}
\end{align}
Then 
\begin{align*}
\kappa_*(\XX,\yy)= -\frac{1}{\radius(\Polyt_{-\yy\odot \XX})}\, ,
\end{align*}
where $-\yy\odot \XX$ is the matrix with rows $-y_i\xx_i^\top$.
Hence, finding a maximum negative margin linear classifier
is equivalent to finding a maximum norm vector in a polytope (a non-convex problem). 
This can be contrasted with 
the case of positive margin that is equivalent to finding a minimum norm vector in a polytope
that does not contain the origin (a convex problem).

As a consequence, our results have implications on the geometry of random polytopes.
In particular, if $\yy,\XX$ are distributed according 
to either of the pure noise or linear signal models introduced above, then
our results imply that the radius satisfies the following
with probability bounded away from zero:
\begin{align}
\radius(\Polyt_{\ZZ})= \frac{1}{\sqrt{2\log\delta}} +\frac{\alpha}{2\log \delta}+o
\big((2\log \delta)^{-1}\big)\, .\label{eq:FirstRadius}
\end{align}
Here $\ZZ:=-\yy\odot \XX$ and it is understood that $\alpha=0$ for the pure noise model.
Note that, under the pure noise model, the rows of $\ZZ$ are i.i.d. $\normal(\bzero,\bI_d)$,
hence corresponding to an isotropic random polytope. We also point out that our actual
estimates on $\delta_{\sus}(\kappa)$ yield indeed a more precise characterization of
$\radius(\Polyt_{\ZZ})$ than in Eq.~\eqref{eq:FirstRadius}.
The best earlier result on $\radius(\Polyt_{\ZZ})$ (for the pure noise model)
was proved in the recent paper \cite{baldi2021theory}: We refer to Section~\ref{sec:Polytope} 
for a comparison.

The rest of this paper is organized as follows. The next section reviews 
related work. Sections \ref{sec:RandomLabels} and \ref{sec:NonRandomLabels} 
present our results for the random labels models and the linear signal model.
In Section \ref{sec:GD}, we discuss an alternative algorithmic approach based on gradient descent.
We then outline the proof techniques in Section \ref{sec:ProofIdea} 
and discuss the general picture and future directions in Section \ref{sec:Conclusion}.

\section{Further related work}
\label{sec:Related}

For $\kappa\ge 0$, the phase transition for the existence of $\kappa$-margin 
solutions has been well studied for several decades. For purely random 
labels, the maximum number of data points $n$ that can be classified correctly is also
known as the model memorization capacity. The model we study is sometimes
referred to as the `spherical perceptron', to emphasize the fact that $\btheta$ can take values on a sphere.

With an elegant combinatorial 
argument, Tom Cover \cite{cover1965geometrical} proved that $n$ random patterns in
$d$ dimensions can be linearly separated with high probability 
if $n < 2d-o(d)$, and can not if $n > 2d+o(d)$. In our language, this
implies that $\delta_{\sus}(\kappa = 0) = 2$. 
For $\kappa>0$, Elizabeth Gardner \cite{gardner1988space} employed the physicists' 
replica method to obtain the replica-symmetric prediction $\delta_{\mathrm{RS}} (\kappa)$ 
for the phase transition threshold. Gardner's prediction was rigorously confirmed 
in \cite{shcherbina2003rigorous, stojnic2013another}. In fact, 
the analysis of \cite{stojnic2013another} implies that 
$\delta_{\sus}(\kappa) \le\delta_{\mathrm{RS}} (\kappa)$ also for negative $\kappa$.

Recent work has extended the above results, and determined the phase transition boundary
$\delta_{\sus}(\kappa)$ for labels $y_i$ dependent on a linear signal $\<\btheta_*,\xx_i\>$, 
and for non-isotropic Gaussian covariates \cite{candes2020phase, sur2019modern,montanari2019generalization}.
Further, \cite{candes2020phase, sur2019modern} considered the underparametrized regime
$\delta>\delta_{\sus}(\kappa)$ and studied parameter estimation using logistic regression, while 
\cite{montanari2019generalization}  characterizes the generalization
 error of max-margin classification  within the over-parameterized regime $\delta<\delta_{\sus}(\kappa)$.

The negative margin problem $\kappa < 0$ has been much less studied 
by mathematicians or statisticians. From a technical point of view, it is more difficult than 
the original spherical perceptron problem since one cannot rely on convexity. 
  Talagrand conjectured in \cite{talagrand2010mean} that the replica symmetric formula
  $\delta_{\mathrm{RS}} (\kappa)$  is an upper bound for the phase transition 
  $\delta_{\sus}(\kappa)$ for all $\kappa\in \R$.  Stojnic \cite{stojnic2013another}
  used Gordon's  comparison inequality to prove that this is is indeed 
  the case.
  However, for $\kappa$ sufficiently negative, $\delta_{\sus} (\kappa)<\delta_{\mathrm{RS}} (\kappa)$ 
  strictly as also implied by our upper bound, cf. Fig.~\ref{fig:PhaseDiagram}, 
  due to the replica symmetry breaking (RSB) phenomenon. 
  Also note that Guerra's interpolation techniques from spin glass theory 
  \cite{guerra2003broken}
  cannot be applied to this case because the model is not symmetric
  (it is akin to a `multi-species' model \cite{panchenko2015free}).
  
  Our upper bound is also based on Gordon's inequality. However, following
  an approach introduced again by Stojnic \cite{stojnic2013negative},
  we apply the inequality to an exponential of the cost function.
  We prove that this yields an upper bound which captures the correct behavior at large negative $\kappa$.

  To conclude the survey of mathematical results on the negative spherical perceptron, the recent paper
  \cite{alaoui2020algorithmic} develops 
  algorithm  to find $\kappa$-margin solutions with $\kappa<0$. This algorithm 
  is based on the IAMP (incremental approximate message passing) strategy of 
  \cite{montanari2021optimization,alaoui2021optimization}
  and is  guaranteed to succeed under a certain no overlap gap condition on the order parameter.
  It is currently unknown for which values of $\kappa$ this condition holds.
  Finally \cite{baldi2021theory} proves a lower bound on the inner radius of 
  the convex hull of random points in high dimension, a problem which relates via duality to 
  the problem of Section \ref{sec:ConnectionsPolytope}. We provide further comparison
  in~Section \ref{sec:Polytope}.

As mentioned in the introduction, the negative spherical perceptron model has recently attracted attention
within the physics community as a toy model for the jamming transition
  \cite{franz2016simplest, franz2017universality,franz2019jamming, franz2019critical,sclocchi2022high}.

Our lower bound on $\delta_{\sus}(\kappa)$ is based on the   
``second moment method", which has been broadly applied to proving existence theorems 
in probability theory and computer science. Examples include the random $k$-satisfiability problem 
\cite{achlioptas2002asymptotic, achlioptas2003threshold, achlioptas2006random},
and the Ising perceptron \cite{ding2019capacity,aubin2019storage,abbe2021proof}, where we search for 
$\btheta \in \{-1, +1 \}^d$ instead of the unit sphere. Unlike in these applications, we deal
with a model with continuous decision variables. For this reason, we cannot upper bound
$\delta_{\sus}(\kappa)$ using the first moment method. As mentioned above,
we develop instead an alternative approach based 
on Gordon's comparison inequality.

Finally, our work is related to the question of determining the interpolation
threshold (or memorization capacity) of neural networks. Since the seminal work of Baum 
 \cite{baum1988capabilities}, several papers have been devoted to 
 estimating the interpolation threshold of neural networks, or showing that interpolators 
 can be constructed efficiently \cite{sakurai1992nh, kowalczyk1993counting, daniely2019neural, 
 daniely2020memorizing, bubeck2020network, montanari2020interpolation}.
 However earlier work is either based on constructing special networks that 
 are not typically produced by learning algorithms, or they apply to very high overparametrization
 ratios (e.g. in the linear regime).

\subsection{Notations}\label{sec:notation}
For any positive integer $n$, we let $[n] = \{1,2,\ldots, n\}$. For a scalar $a$, we 
write $a_+ = \max\{a, 0\}$ and $a_- = \max\{-a, 0\}$. We use $\| \uu \|$ or $\| \uu \|_2$ to denote the $\ell_2$ norm of a vector $\uu$, and we use $\| \bM \|_\op$ to denote the operator norm of a matrix $\bM$. We denote by $\S^{d-1}$ the unit sphere in $d$ dimensions.

We will use $O_n(\cdot)$ and $o_n(\cdot)$ for the standard big-$O$ and small-$o$ notation, 
where $n$ is the asymptotic variable. We occasionally write $a_n \gg b_n$ if $b_n = o_n(a_n)$. 
We write
$\xi_1(n) = o_{n,\P}(\xi_2(n))$ if $\xi_1(n) / \xi_2(n)$ converges to $0$ in probability.
 Additionally,  $\breve{o}_{\kappa} (1)$ denotes a
 quantity that vanishes as $\kappa \to -\infty$, uniformly in all other variables mentioned 
 in the context. Hence, 
 the meaning of $\breve{o}_{\kappa} (1)$ may change from line to line. 
 For example, if $(\rho, \kappa) \in [-1, 1] \times \R$ are the variables of interest, and 
 $s: [-1, 1] \times \R \to \R$ is a function, then
\begin{equation*}
	s (\rho, \kappa) = \breve{o}_{\kappa} (1) \iff \lim_{\kappa \to -\infty} \sup_{\rho \in [-1, 1]} \vert s(\rho, \kappa) \vert = 0.
\end{equation*}
Similarly, the notation $\breve{O}_{\kappa} (1)$ is defined via:
\begin{equation*}
	s(\rho, \kappa) = \breve{O}_{\kappa} (1) \iff \limsup_{\kappa \to -\infty} \sup_{\rho \in [-1, 1]} 
	\vert s(\rho, \kappa) \vert < +\infty.
\end{equation*}

For two vectors $\uu, \vv$ of the same dimension, we may write $\uu \ge \vv$ or $\uu \le \vv$ to mean entry-wise inequality. We always use $\Phi$ and $\phi$ to denote the c.d.f. and p.d.f. of standard normal distribution, respectively. 
We write $X \perp Y$ if $X$ and $Y$ are two independent random variables (or random vectors). Throughout this paper, we assume the asymptotics $\lim_{n \to +\infty} (n/d) = \delta \in (0, +\infty)$.

\section{The case of random labels}
\label{sec:RandomLabels}

Throughout this section, we assume that the data $\{(\xx_i,y_i)\}_{i \le n}$ are i.i.d.,
 with isotropic covariates $\xx_i \sim_{\iid} \normal(\bzero,\bI_d)$ and random labels 
 $y_i \sim_{\iid} \Unif(\{+1,-1\})$ independent of $\xx_i$. 

To the best of our knowledge, the best available rigorous result in this context 
is the replica symmetric upper bound in \cite{stojnic2013another}, which yields 
\begin{equation}
	\delta_{\sus} (\kappa) \le \delta_{\mathrm{RS}} (\kappa) =  \E\{(\kappa - G)_+^2\}^{-1} 
	= \frac{|\kappa|^2}{2\Phi(\kappa)}\big(1+o_{\kappa}(1)\big)\, .
\end{equation}
(Here expectation is with respect to $G\sim\normal(0,1)$.)
As demonstrated below, this upper bound has the wrong behavior for large negative $\kappa$.

\subsection{Existence of $\kappa$-margin classifiers}
\label{sec:Existence}
We begin by stating our lower bound.
\begin{defn}\label{def:pure_lower}
	For $\kappa < 0$, let $c_*=c_*(\kappa) > 0$ be the unique positive solution of the equation 
	\begin{align*}
	c (1 - \Phi(\kappa + c)) = \phi (\kappa + c)\, .
	\end{align*}
	  (Proof of existence and uniqueness is given in Lemma~\ref{lem:property_of_e_q}. $(a)$ in the appendix.)
	  
	 For $q \in [-1, 1]$, define
	\begin{equation}
	\Psi(q) = - \log \E_q \left[ \exp(-c_*(\kappa) (G_1 + G_2)) \bone \left\{ G_1 \ge \kappa, G_2 \ge \kappa \right\} \right] - \frac{1}{2\delta} \log \left( 1-q^2 \right),
	\end{equation}
	where $\E_q$ denotes expectation with respect to
	$(G_1, G_2)^\top \sim \normal \left( \bzero , \begin{bmatrix}
    	1 & q \\
    	q & 1
	\end{bmatrix} \right)$. Finally, we define 
\begin{align}
	\delta_{\sl} (\kappa):=\sup\Big\{
	\delta > 0 \mbox{ s.t.  } \Psi(q) \mbox{ is uniquely minimized at } 
	q = 0 \mbox{ and } \Psi''(0)> 0\Big\}\,.
\end{align}	
\end{defn}

Our first theorem provides a lower bound on $\delta_{\sus}(\kappa)$.
\begin{thm}\label{thm:pure_noise_lower_bound}
    If $\delta < \delta_{\sl}(\kappa)$, then with high probability there is a $\kappa$-margin solution,
	i.e.,
	\begin{equation}\label{eq:pure_noise_lower_bound}
		\delta_{\sl}(\kappa)\le \delta_{\sus}(\kappa)\, .
	\end{equation}
	Moreover, $\delta_{\sl}(\kappa)= (1+o_{\kappa}(1))(\log |\kappa|)/\Phi(\kappa)$ as 
	$\kappa \to -\infty$. Hence, for any $\veps > 0$, there exists 
	a $\underline \kappa = \underline \kappa(\veps) < 0$, such that for all 
	$\kappa < \underline \kappa$, 
	\begin{equation*}
		\delta < (1 - \veps) \frac{\log |\kappa| }{\Phi(\kappa)}\;\;
		\Rightarrow\;\; \lim_{n\to\infty}\prob(\ERM_0\neq\emptyset)=1\, .
	\end{equation*}
\end{thm}

\begin{defn}\label{def:pure_upper}
	For any $\kappa < 0$, define
	\begin{equation}
		\delta_{\su} (\kappa) = \inf \left\{ \delta > 0: \exists c>0, \ \emph{s.t.} \ \forall t>0, \frac{1}{\sqrt{c^2+4}+c} + \frac{1}{c} \log \frac{\sqrt{c^2+4} + c}{2} < \frac{1}{4t} - \frac{\delta}{c} \log \psi_{\kappa} (-ct) \right\},
	\end{equation}
	where, for $t\ge 0$,
	\begin{equation}
    	\psi_{\kappa} (-t) := \E \Big\{ \exp \left(-t \left( \kappa - G \right)_+^2 \right) 
    	\Big\}, \;\;\;\; G \sim \normal(0, 1).
	\end{equation}
\end{defn}

We next state our upper bound on $\delta_{\sus}(\kappa)$. The first part 
of this theorem was proved already in \cite{stojnic2013negative}, which however relied on
numerical evaluations to argue that the resulting bound was superior to the replica symmetric one.
\begin{thm}\label{thm:pure_noise_upper_bound}
	If $\delta > \delta_{\su} (\kappa)$, then with high probability there is no $\kappa$-margin 
	solution, i.e.,
	\begin{equation}\label{eq:pure_noise_upper_bound}
	\delta_{\su}(\kappa) \ge \delta_{\sus}(\kappa)\, .
	\end{equation}
	Moreover, $\delta_{\su}(\kappa)= (1+o_{\kappa}(1))(\log |\kappa|)/\Phi(\kappa)$ as 
	$\kappa \to -\infty$. Hence, for any $\veps > 0$, there exists 
	a $\underline \kappa = \underline \kappa(\veps) < 0$, such that for all 
	$\kappa < \underline \kappa$,
	\begin{equation*}
		\delta > (1 + \veps) \frac{\log |\kappa| }{\Phi(\kappa)}\;\;
		\Rightarrow \;\;\lim_{n\to\infty}\prob(\ERM_0\neq\emptyset)=0\, .
	\end{equation*}
\end{thm}

Figure \ref{fig:PhaseDiagram} reports the upper and lower bounds $\delta_{\su}(\kappa), 
\delta_{\sl}(\kappa)$ as  functions of $\kappa$.
The proofs for Theorem~\ref{thm:pure_noise_lower_bound} and Theorem~\ref{thm:pure_noise_upper_bound} 
are deferred to Appendix~\ref{sec:ExistencePureNoise} with outlines in Section \ref{sec:ProofIdea}.

\subsection{A linear programming algorithm}
\label{sec:AlgoRLabels}

As mentioned in Section \ref{sec:ConnectionsPolytope}, finding a 
$\kappa$-margin solution is equivalent to solving the following non-convex optimization
problem:
\begin{align}
\mbox{maximize} &\;\;\;\|\btheta\|_2^2\, ,\nonumber\\
\mbox{subject to} &\;\;\; y_i \langle \xx_i, \btheta \rangle \ge \kappa ~\text{for all}~i \in [n], 
\label{opt:non-convex-1}\\
&\;\;\; \| \btheta \|_2 \le 1\, . \nonumber
\end{align} 
If the solution to this problem has unit norm, then it provides a $\kappa$-margin solution.

Imagine now trying to maximize this cost function, starting at a random  initialization
near the origin, for instance $\btheta_0= \eps\vv$ with $\vv\sim\Unif(\S^{d-1})$.
Linearizing the objective around the initialization yields 
$\|\btheta\|_2^2= \|\btheta_0\|_2^2+2\<\btheta_0,\btheta-\btheta_0\>+
O(\|\btheta-\btheta_0\|_2^2)$. It is natural to try to maximize the linearized cost.
We are therefore led to the following algorithm:
\begin{enumerate}
\item Draw $\vv \sim \Unif(\S^{d-1})$.
\item Solve the following convex optimization problem
\begin{align}
\mbox{maximize} &\;\;\; \<\vv,\btheta\>\, , \nonumber\\
\mbox{subject to} &\;\;\; y_i \< \xx_i, \btheta \> \ge \kappa ~\text{for all}~i \in [n], 
\label{opt:convex-1}\\
&\;\;\; \| \btheta \|_2 \le 1. \nonumber
\end{align} 
\end{enumerate}
In optimization language, this algorithm executes the first step of 
Frank-Wolfe algorithm \cite{frank1956algorithm}.

We denote by $\hat \btheta$ the output of this algorithm.
(If the optimization problem \eqref{opt:convex-1} has more than one maximizer, one 
of them can be selected arbitrarily.)
Of course, if $\|\hbtheta\|_2=1$, then $\hbtheta$ is a $\kappa$-margin solution.
The next theorem establishes a phase transition boundary for this algorithm.
\begin{thm}\label{thm:algpure}
Under the pure noise model, assume $n,d\to\infty$ with $n/d\to\delta$. If 
\begin{equation}\label{cond:delta1}
\delta < \delta_{\slin}(\kappa) := \frac{1}{\Phi(\kappa)}\, ,
\end{equation}
then $\|\hbtheta\|_2=1$ with high probability and therefore $\hbtheta$
 is, with high probability, a $\kappa$-margin solution.
 On the other hand, if $\delta > \delta_{\slin}(\kappa)$ then,
 with high probability $\|\hbtheta\|_2<1$.
 
In particular, for any constant $\veps \in (0,1)$ there exists
$\underline \kappa = \underline \kappa(\veps) <0 $ such that for all
 $\kappa < \underline \kappa(\veps)$ the following holds. 
 If 
\begin{equation}\label{cond:delta2}
\delta \le (1-\veps) \sqrt{2\pi}\, |\kappa| \exp\left( \frac{\kappa^2}{2} \right),
\end{equation}
then $\hat \btheta$ is a $\kappa$-margin solution with high probability.
\end{thm}
Since the algorithm described above can be implemented in polynomial time, the above
result yields a lower bound on the polynomial threshold: 
$\delta_{\salg}(\kappa)\ge \delta_{\slin}(\kappa)$ for all $\kappa$.

\begin{rem}
A slightly different formulation of the same approach removes the norm constraint 
$\|\btheta\|_2\le 1$ in the convex program \eqref{opt:convex-1}. Denote the solution
of the resulting linear program by $\btheta_{\slin}$. If $\|\btheta_{\slin}\|_2\ge 1$,
then a $\kappa$ margin solution is given by $\hbtheta_{\slin}:=\btheta_{\slin}/\|\btheta_{\slin}\|_2$.

It is easy to see that this version of the algorithm succeeds (i.e., $\|\btheta_{\slin}\|_2\ge 1$)
if and only if the previous does (i.e., $\|\hbtheta\|_2= 1$). Indeed, if $\|\hbtheta\|_2=1$,
then $\|\btheta_{\slin}\|_2 \ge 1$ because $\hbtheta$ is feasible for the linear program.
On the other hand, we show that $\|\btheta_{\slin}\|_2\ge 1$ implies $\|\hbtheta\|_2 \ge 1$. Assume by contradiction $\|\hbtheta\|_2 < 1$.
Then defining $\hbtheta_t := (1-t)\hbtheta+t\btheta_{\slin}$, there exists 
$t_*\in(0,1]$ such that $\|\hbtheta_{t_*}\|_2=1$. This is a feasible point for
\eqref{opt:convex-1}, thus leading to a contradiction.
\end{rem}

\section{Labels correlated with a linear signal}
\label{sec:NonRandomLabels}

In this section, we assume that the labels are correlated with a one-dimensional projection of the data.
As before, $\{(\xx_i,y_i)\}_{i \le n}$ are i.i.d.~data, $\xx_i \sim \normal(\bzero,\bI_d)$ and the labels $y_i \in \{\pm 1 \}$ satisfy
\begin{equation*}
\P \left( y_i = 1 \big\vert \left\langle \xx_i, \btheta_* \right\rangle \right) = 1 - \P \left(y_i = -1 \big\vert \left\langle \xx_i  , \btheta_* \right\rangle \right) = \varphi \left( \left\langle \xx_i, \btheta_* \right\rangle \right),
\end{equation*}
where $\btheta_* \in \S^{d-1}$ is the true signal, and $\varphi: \R \to [0,1]$ is a monotone 
increasing function. We make the following assumption on the tail behavior 
of the link function $\varphi$.
\begin{ass}[Exponential tail]\label{ass:exponential_tail_link}
There exist constants $\alpha > 0$ and $C_{\tail}>0$, such that as $x \to - \infty$,
\begin{equation*}
	\frac{1 + \varphi(x) - \varphi(-x)}{2C_{\tail} \exp(\alpha x)} \to 1.
\end{equation*}	
\end{ass}

\begin{rem}
	The above assumption holds for any link function $\varphi(x)$ satisfying $\varphi(x) = 1 - C_{\tail}^{+} e^{-\alpha x}+o(e^{-\alpha x})$
	as $x\to\infty$, and $\varphi(x) = C_{\tail}^{-} e^{\alpha x}+o(e^{\alpha x})$
	as $x \to -\infty$, where $C_{\tail}^{+} + C_{\tail}^{-} = 2 C_{\tail}$. In particular, it holds for the logistic link function, i.e., $\varphi (x) = (1 + e^{- \alpha x})^{-1}$ with $C_{\tail} = 1$. Therefore, we see that Assumption~\ref{ass:exponential_tail_link} slightly generalizes our requirement on the link function in Section~\ref{sec:SumResults}. 
\end{rem}

Throughout this section, we denote by $(Y, G, W)$ a vector of random variables 
with joint distribution:
	\begin{equation}
	    \left\{\begin{array}{l}
	     G \sim \normal(0, 1),  \;\;\; \P(Y = 1 \vert G) = \varphi(G) = 1 - \P(Y = - 1 \vert G), \; Y\in\{+1,-1\}\, ,\\
	     (Y, G) \perp W\, ,\;\; W \sim \normal(0, 1) \, .
        \end{array}\right.
	\end{equation} 
	
\subsection{Existence of $\kappa$-margin classifiers}

We start with a definition that plays an important role in the phase transition lower bound.
\begin{defn}\label{def:sig_lower}
We define auxiliary function $\delta_{\mathrm{sec}}( \rho, \kappa, \kappa_{\mathrm{t}} )$ 
on the set
	\begin{equation*}
		\left\{ \left( \rho, \kappa, \kappa_{\mathrm{t}} \right): 0 < \rho < 1, \ 
		0 > \rho \kappa_{\mathrm{t}} > \kappa \right\}
	\end{equation*}
	through the following steps:
	\begin{enumerate}
		\item For $s \ge \kappa_{\mathrm{t}}$, define $\kappa (s) = (\kappa - \rho s)/\sqrt{1 - \rho^2}$.
		 Then we know $\kappa (s) < 0$, hence by Lemma~\ref{lem:property_of_e_q}.$(a)$ 
		 there exists $c (s) > 0$ such that
		\begin{equation}
			c (s) \left( 1 - \Phi (\kappa (s) + c (s)) \right) = \phi (\kappa (s) + c (s)).
		\end{equation}
	    \item For $q \in [-1, 1]$, define
	    \begin{equation}
	    	e (q, s) = \E_q \left[ \exp \left( - c (s) (G_1 + G_2) \right)
	    	 \bone \left\{ G_1 \ge \kappa (s), G_2 \ge \kappa (s) \right\} \right]
	    \end{equation}
	    and
	    \begin{equation}
	    	e (q) = 1 + \int_{\kappa_{\mathrm{t}}}^{+\infty} p_{YG} (s) 
	    	\left( \frac{e (q, s)}{e (0, s)} - 1 \right) \d s,
	    \end{equation}
	    where $p_{YG}$ denotes the density function of $YG$.
	    \item Now set $F (q) = - \log (e (q))$ and $I(q) = - \log (1-q^2)/2$. 
	    Similarly as in Definition~\ref{def:pure_lower}, let 
	    $\delta_{\mathrm{sec}}( \rho, \kappa, \kappa_{\mathrm{t}} )$ be the supremum of all
	     $\delta > 0$ such that $F''(0) + I''(0)/\delta > 0 $, and $F (q) + I(q)/\delta$ achieves 
	     unique minimum at $q=0$.
	\end{enumerate}
	Now we are in position to define $\delta_{\sl} (\kappa;\varphi)$ as the supremum of all $\delta > 0$, such that there exist parameters
	\begin{equation*}
		\rho \in (0, 1), \ 0 > \kappa_1 > \kappa_2 > \kappa_0 = \frac{\kappa}{\sqrt{1 - \rho^2}}
	\end{equation*}
	satisfying the inequality below:
	\begin{footnotesize}
	\begin{align*}
		& \frac{1}{\delta} > \min \Bigg\{ \max \left\{ \P \left( \rho YG + \sqrt{1 - \rho^2} W \le \kappa \right), \E \left[ \left( \kappa_0 - \frac{\rho YG}{\sqrt{1-\rho^2}} - W \right)_+^2 \right] \right\}, \ \delta_{\mathrm{sec}} \left( \rho, \sqrt{1 - \rho^2} \kappa_2, \frac{\sqrt{1-\rho^2}}{\rho} \kappa_1 \right)^{-1} \\
		+ & \inf_{c \ge 0} \left\{ \P \left(YG \ge \frac{\sqrt{1-\rho^2}}{\rho} \kappa_1 \right) \E \left[ \left( \frac{\kappa_0 - \kappa_2}{c} - W \right)_+^2 \right] + \frac{1}{c^2}\E \left[ \bone \left\{ YG < \frac{\sqrt{1-\rho^2}}{\rho} \kappa_1 \right\} \left( \kappa_0 - \frac{\rho YG}{\sqrt{1-\rho^2}} - \sqrt{1 + c^2} W \right)_+^2 \right] \right\} \Bigg\}.
	\end{align*}
	\end{footnotesize}
\end{defn}

The next theorem establishes that $\delta_{\sl}(\kappa;\varphi)$ 
is indeed a lower bound on the satisfiability threshold, and establishes
an approximation for this lower bound which holds for $\kappa$ large and negative.
\begin{thm}\label{thm:logistic_lower_bound}
Under the linear signal model, if $\delta < \delta_{\sl} (\kappa; \varphi)$, then with high probability, there is a $\kappa$-margin solution, i.e., we have
	\begin{equation}\label{eq:signal_lower_bound}
		 \delta_{\sus}(\kappa) \ge \delta_{\sl}(\kappa;\varphi)\, .
	\end{equation}
Moreover, under Assumption~\ref{ass:exponential_tail_link}, 
$\delta_{\sl}(\kappa;\varphi)=(1+o_{\kappa}(1))e^{\alpha|\kappa|}\log|\kappa|/(4C_{\tail}\Phi(\kappa))$
as $\kappa\to -\infty$. Hence, for any $\veps > 0$, there exists a 
$\underline \kappa = \underline \kappa(\veps; \varphi) < 0$, such that for all $\kappa < \underline \kappa$, as long as
	\begin{equation*}
		\delta < (1 - \veps) \frac{e^{\alpha|\kappa|}\log|\kappa|}{4C_{\tail}\Phi(\kappa)}
		\;\;
		\Rightarrow \;\;\lim_{n\to\infty}\prob(\ERM_0\neq\emptyset)=1\, .
	\end{equation*}
\end{thm}

\begin{figure}[t!]
\centering
\includegraphics[width=30em]{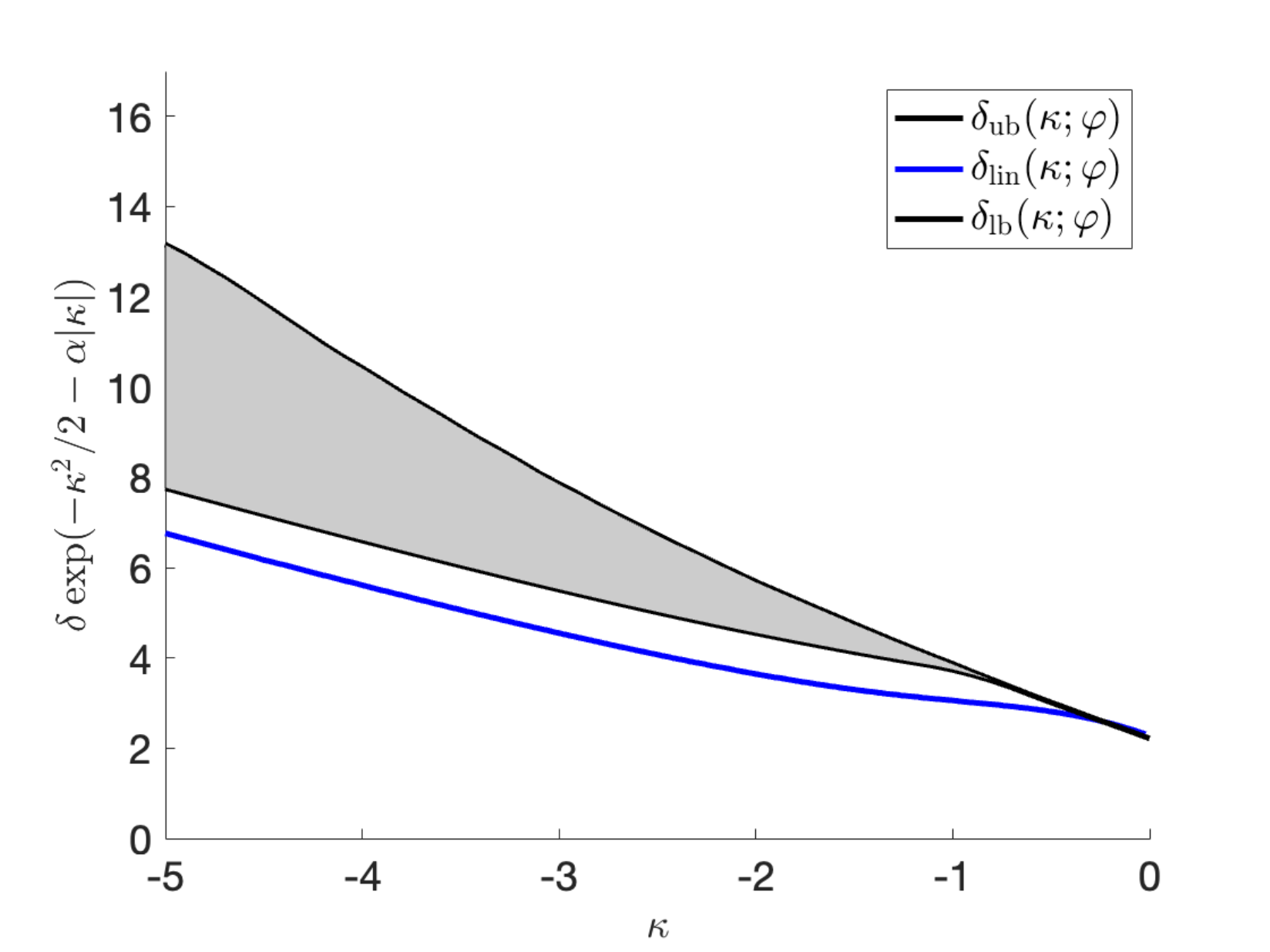}
\caption{Theoretical predictions of the phase transition thresholds and the linear programming lower bound
for labels correlated to a linear signal. Here the link function is
logistic: $\varphi(t) = (1+e^{-t})^{-1}$. The phase transition for the existence of $\kappa$-margin solution occurs in 
the region delimited by $\delta_{\sl}(\kappa;\varphi)$ and $\delta_{\su} (\kappa;\varphi)$ (gray area).}
\end{figure}
We next introduce and state our upper bound on the threshold $\delta_{\sus}(\kappa)$
in the linear signal model.
\begin{defn}\label{def:sig_upper}
	For any $\kappa < 0$ and $\rho \in [-1, 1]$, set
	\begin{scriptsize}
	\begin{equation*}
		\delta_{\su} ( \kappa, \rho ;\varphi) := \inf \left\{ \delta > 0: \inf_{c > 0} \left( \frac{\sqrt{1 - \rho^2}}{c \sqrt{1 - \rho^2} + \sqrt{c^2 (1 - \rho^2) + 4}} + \frac{1}{c} \log \frac{c \sqrt{1 - \rho^2} + \sqrt{c^2 (1 - \rho^2) + 4}}{2} - \inf_{u > 0} \left\{ \frac{c}{4 u} - \frac{\delta}{c} \log \psi_{\kappa, \rho} (-u) \right\} \right) < 0 \right\},
	\end{equation*}
	\end{scriptsize}
	where
	\begin{equation}
		\psi_{\kappa, \rho} (- u) = \E \left[ \exp \left( -u \left( \kappa - \rho YG - \sqrt{1 - \rho^2} W \right)_+^2 \right) \right].
	\end{equation}
	Then we define $\delta_{\su} (\kappa;\varphi) = \sup_{\rho \in [-1, 1]} \delta_{\su} ( \kappa, \rho;\varphi )$.
\end{defn}

\begin{thm}\label{thm:logistic_upper_bound}
	Under the linear signal model, if $\delta > \delta_{\su} (\kappa;\varphi)$, then with high 
	probability there is no $\kappa$-margin solution, i.e., we have
	\begin{equation}\label{eq:signal_upper_bound}
		 \delta_{\sus}(\kappa) \le \delta_{\su}(\kappa;\varphi)\, .
	\end{equation}
	Moreover, under Assumption~\ref{ass:exponential_tail_link}, 
	$\delta_{\su} (\kappa; \varphi)=(1+o_{\kappa}(1))e^{\alpha|\kappa|}\log|\kappa|/(4C_{\tail}Phi(\kappa))$ as $\kappa\to -\infty$.
	 That is to say, for any $\veps > 0$, there exists a $\underline \kappa = \underline \kappa(\veps; \varphi) < 0$, such that for all $\kappa < \underline \kappa$,
	\begin{equation*}
		\delta > (1 + \veps)\frac{e^{\alpha|\kappa|}\log|\kappa|}{4C_{\tail}\Phi(\kappa)}
		\;\;
		\Rightarrow \;\;\lim_{n\to\infty}\prob(\ERM_0\neq\emptyset)=0\, .
	\end{equation*}
\end{thm}
The proofs of Theorem~\ref{thm:logistic_lower_bound} and Theorem~\ref{thm:logistic_upper_bound} 
are deferred, respectively, to Appendix~\ref{sec:signal-lower} and~\ref{sec:signal-upper}.

\subsection{Linear programming algorithm}
\label{sec:AlgoNRLabels}

Our bounds imply that the interpolation threshold $\delta_{\sus}(\kappa)$ is
significantly higher in the linear signal model as compared to the case of purely random labels.
Indeed the threshold changed from  $\delta_{\sus}(\kappa)= \exp\{\kappa^2/2+O(\log|\kappa|)\}$
to  $\delta_{\sus}(\kappa)= \exp\{\kappa^2/2+\alpha|\kappa|+O(\log|\kappa|)\}$.
On the other hand, the algorithm of Section \ref{sec:AlgoRLabels},
if applied without modifications, has a similar threshold under the current model 
as for pure noise labels. The underlying reason is that the random direction $\vv$
is nearly orthogonal to the signal $\btheta_*$, and hence the algorithm is insensitive
to the change in data distribution.

In order to overcome this limitation, we introduce the following modification of the algorithm
of Section \ref{sec:AlgoRLabels}.
\begin{enumerate}
\item Compute
\begin{align}
\vv := \frac{1}{n}\sum_{i=1}^n y_i\xx_i \, .
\end{align}
\item Solve the following convex optimization problem
\begin{align}
\mbox{maximize} &\;\;\; \<\vv,\btheta\>\, , \nonumber\\
\mbox{subject to} &\;\;\; y_i \< \xx_i, \btheta \> \ge \kappa ~\text{for all}~i \in [n], 
\label{opt:convex-2}\\
&\;\;\; \| \btheta \|_2 \le 1. \nonumber
\end{align} 
\end{enumerate}
As in the case of purely random labels, we could remove the norm constraint $\|\btheta\|_2 \le 1$,
and normalize the solution of the resulting linear program.

\begin{figure}[t!]
    \centering
    \begin{subfigure}[t]{0.5\textwidth}
        \centering
        \includegraphics[width=\linewidth]{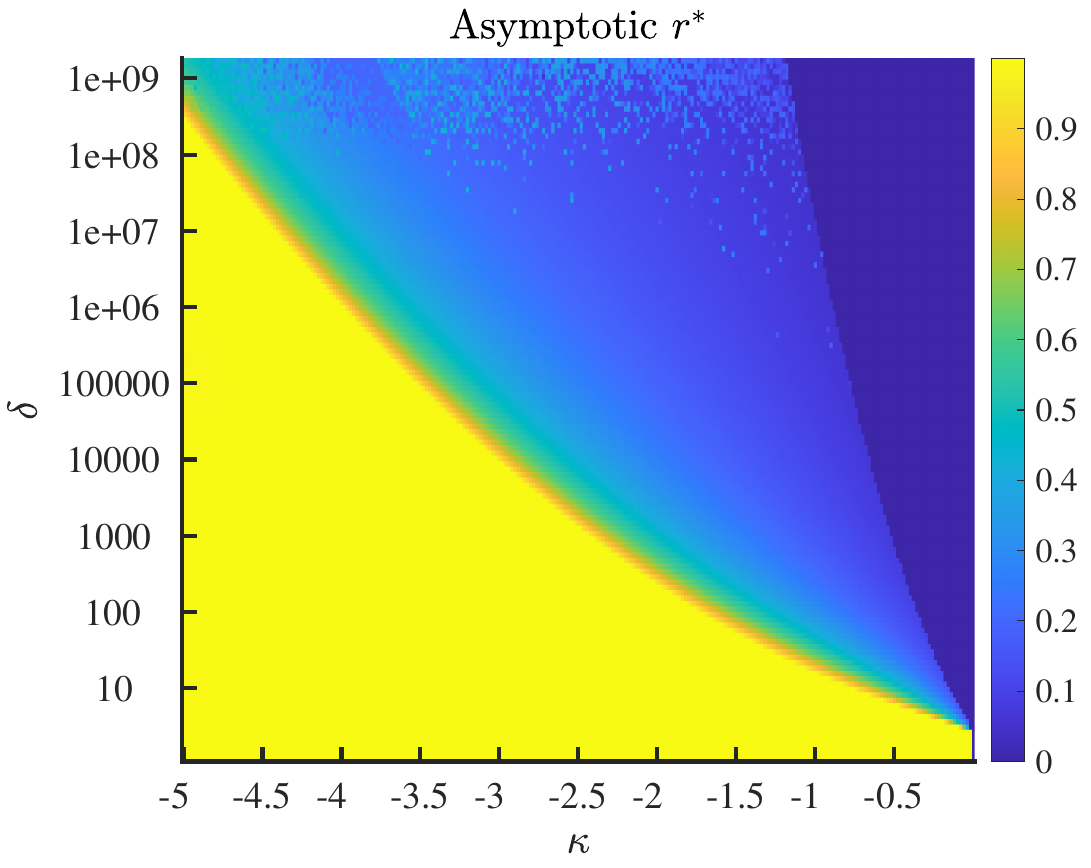}
    \end{subfigure}%
    ~ 
    \begin{subfigure}[t]{0.5\textwidth}
        \centering
        \includegraphics[width=\linewidth]{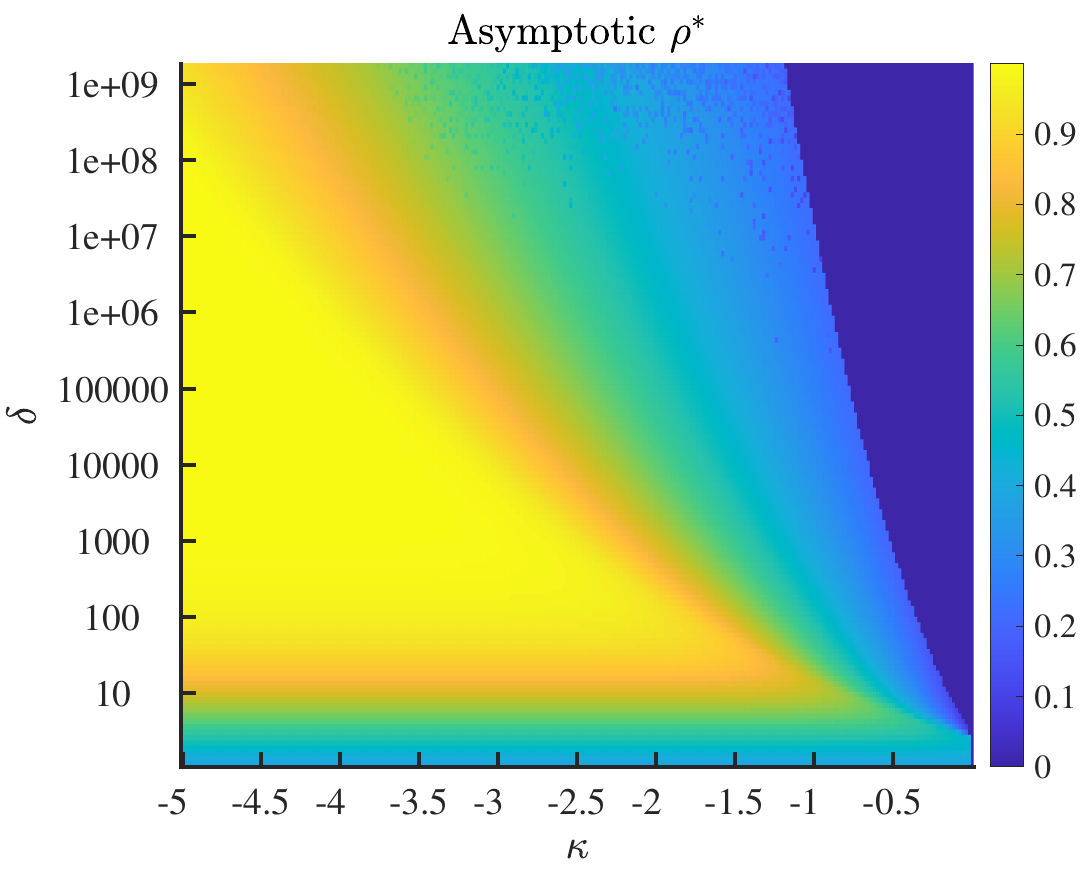}
    \end{subfigure}
    \caption{Maximizing the function $M(\rho,r)$ as defined in \eqref{def:Mrho} over $\rho \in [-1,1], r \in [0,1]$ numerically gives the maximizer $(\rho^*, r^*)$. The heatmaps show the values of $r^*$ (left) and $\rho^*$ (right) under varying $\kappa$ and $\delta$.  
\textbf{Left:} Yellow region indicates the regime where a $\kappa$-margin solution exists. \textbf{Right:} $\rho^*$ gives the asymptotic correlation $\langle \hat \btheta, \btheta_* \rangle$.}
\end{figure}

In order to state our results for this algorithm, recall the triple of random variables $(Y,G,W)$ 
introduced at the beginning of this section. It is useful to introduce the 
random variable $Z_{\rho,r}$ whose distribution depends on $\rho\in [-1,+1]$ and $r\ge 0$:
\begin{equation}\label{def:Z}
Z_{\rho,r} = \rho Y G + \sqrt{1-\rho^2}\, r W -   \kappa\, .
\end{equation}
We also define the domain $\Omega_{\ge}\subseteq \reals^2$, via:
\begin{align}\label{def:omegage}
\Omega_{\ge} = \Big\{(\rho, r) \in [-1,1] \times \reals_{\ge 0}: (1-\rho^2) r^2 \delta^{-1} \ge \E[Z_{\rho,r}^2 ; Z_{\rho,r} < 0 ] \Big\}\, .
\end{align}
We also denote by $\Omega_{>}$ the analogous set in which the inequality is satisfied strictly.

\begin{thm}\label{thm:signal}
Let $\varphi$ satisfy Assumption~\ref{ass:exponential_tail_link}. For $(\rho,r)\in \Omega_{\ge}$,
let $s_*(\rho,r)$ be the only non-negative solution of the equation
\begin{equation}\label{def:s}
(1-\rho^2) r^2 \delta^{-1} = \E[ \max\{s, -Z_{\rho,r}\}^2 ]\, ,
\end{equation}
and define 
\begin{equation}\label{def:Mrho}
M(\rho,r) = \E \big[ \big( Z_{\rho,r} + s_*(\rho,r)\big)_+ \big] + \kappa.
\end{equation}
Finally, for  $r_0>0$  define $M_*(r_0) = \max\{M(\rho,r): (\rho,r) \in \Omega_{\ge}, r\le r_0\}$ and
\begin{equation}\label{def:deltaalg}
\delta_{\mathrm{lin}}(\kappa;\varphi) :=\sup\Big\{ \bar \delta: \forall\, \delta < \bar \delta: 
\Omega_> \cap\{r\le 1\}\neq \emptyset, ~M_*(\infty) > M_*(1) \Big\}\, .
\end{equation}
If $\delta < \delta_{\mathrm{lin}}(\kappa;\varphi)$, then the linear programming 
algorithm outputs $\hbtheta$ such that $\|\hbtheta\|_2=1$ with high probability.
Therefore $\hbtheta$
 is, with high probability, a $\kappa$-margin solution.

In particular, for any constant $\veps \in (0,1)$ there exists
$\underline \kappa = \underline \kappa(\veps; \varphi) <0 $ such that for all
 $\kappa < \underline \kappa(\veps; \varphi)$ the following holds. 
 If 
\begin{equation}\label{cond:delta}
 \delta \le (1-\veps)\frac{e^{\alpha|\kappa|}}{2C_{\tail} \Phi(\kappa)} \, ,
\end{equation}
then $\hat \btheta$ is a $\kappa$-margin solution with high probability.
\end{thm}

\subsection{Asymptotic estimation error}
\label{sec:Error}

What is the geometry of the set of empirical risk minimizers? 
In this section we explore one specific aspect of this question, namely what are the possible 
distances between the signal $\btheta_*$ and $\btheta\in\ERM_0(\kappa)$,
and where in that spectrum is the solution $\hbtheta$ found by the linear programming 
algorithm. 

We begin by bounding the interval of values of $\<\btheta, \btheta_*\>$. 
%\am{Eliminated one supremum in the formulas below. This should be ok because of 
%(left/right) continuity of $\delta_{\su}$ wrt $\rho$. Modify if incorrect.}.
%
\begin{thm}\label{cor:rho_min_max}
Let Assumption \ref{ass:exponential_tail_link} hold.
Define $\rho_{\min},\rho_{\max}$ with $-1\le \rho_{\min}\le \rho_{\max}\le +1$ via
\begin{align}	
\rho_{\min}&:=\sup\big\{\overline{\rho}\in [-1,1]: \delta_{\su} (\kappa, \rho;\varphi) < \delta\, ,\;
\forall \rho\in [-1,\overline\rho]\big\}\, ,\\
\rho_{\max}&:=\inf\big\{\underline{\rho}\in [-1,1]: \delta_{\su} (\kappa, \rho;\varphi) < \delta\, ,\;
\forall \rho\in [\underline\rho,1]\big\}\, .
\end{align}		
(If  $\delta_{\su} (\kappa, \rho;\varphi) < \delta$ for all $\rho\in [-1,1]$, the statement
below is empty and therefore the definition of $\rho_{\min}, \rho_{\max}$ is immaterial.)

Then we have, with probability converging to one as $n,d\to\infty$, $n/d\to\delta$:
\begin{align}
\big\{\<\btheta,\btheta_*\> :\; \btheta\in\ERM_0(\kappa)\big\}\subseteq 
[\rho_{\min},\rho_{\max}]\, .
\end{align}
\end{thm}

In words, all solutions $\btheta\in\ERM_0(\kappa)$ lie in the annulus 
$\<\btheta,\btheta_*\>\in [\rho_{\min},\rho_{\max}]$. The next theorem
characterizes $\delta_{\su}$ for large negative $\kappa$. (Recall that the notation
 $\breve{o}_{\kappa} (1)$ was introduced in Section \ref{sec:notation}.)
\begin{thm}\label{thm:asymptotic_upper_bound}
Let Assumption \ref{ass:exponential_tail_link} hold.
As $\kappa \to - \infty$, the following holds uniformly over $\rho \in [-1, 1]$:
\begin{equation*}
    \delta_{\su} (\kappa, \rho;\varphi) = \frac{(1 + \breve{o}_{\kappa} (1)) \mathcal{D} \left( \kappa^2 (1 - \rho^2) \right)}{\P \left( \rho YG + \sqrt{1 - \rho^2} W < \kappa \right)},
\end{equation*}
where $\mathcal{D}: \R_{\ge 0} \to \R_{\ge 0}$ is  defined by:
\begin{equation*}
    \mathcal{D} (a) = \inf \left\{ b \ge 0: \exists c > 0, \frac{1}{c + \sqrt{c^2 + 4}} + \frac{1}{c} \log \frac{c + \sqrt{c^2 + 4}}{2} < \inf_{t > 0} \left\{ \frac{c}{4 t a} +  \frac{b}{c} \int_{0}^{+\infty} 2 t s \exp \left( -t s^2 - s \right) \d s \right\} \right\}.
\end{equation*}
\end{thm}

\begin{figure}[t!]
\centering
\includegraphics[width=30em]{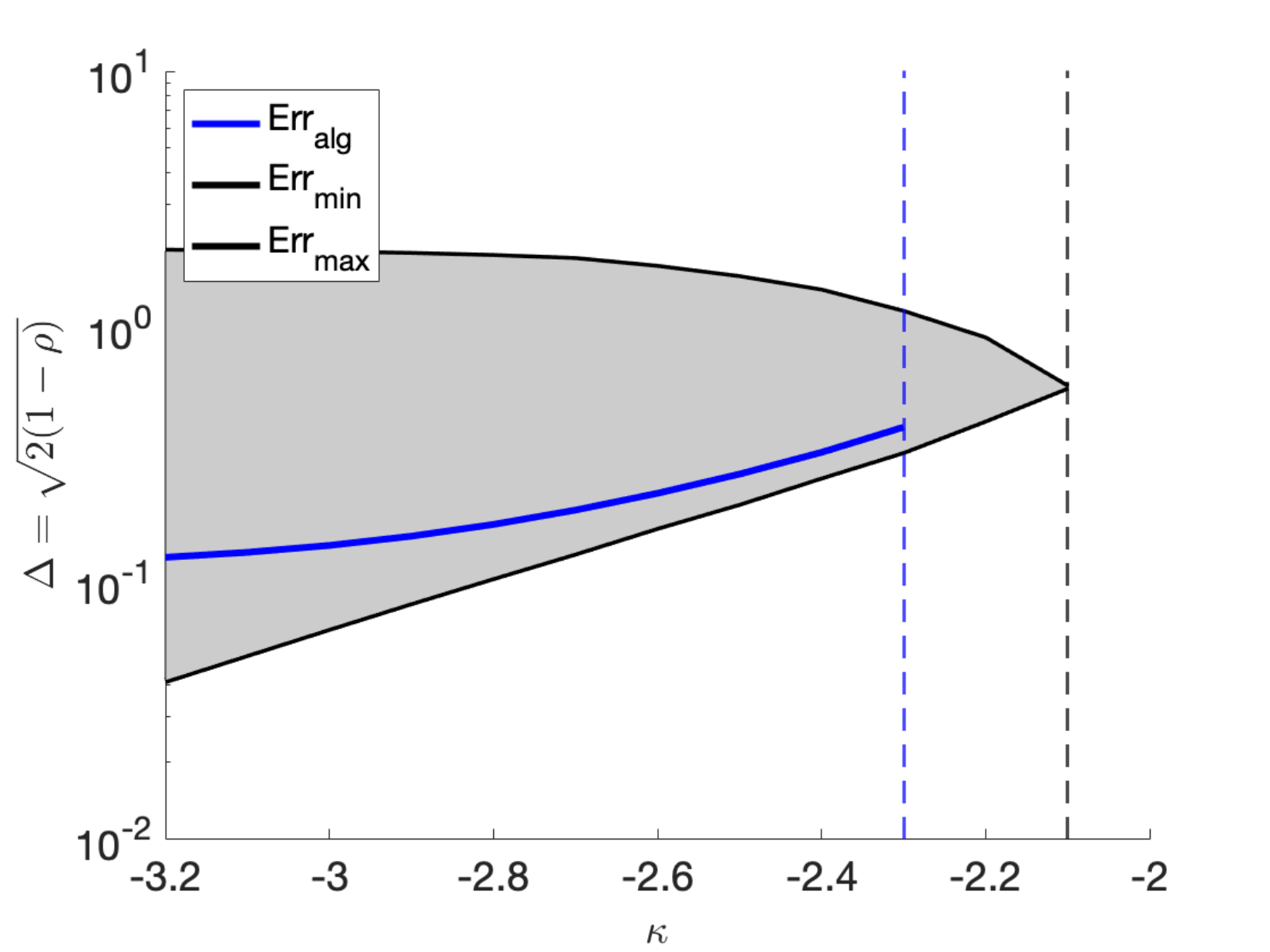}
\caption{Asymptotic predictions of the estimation error among $k$-margin solutions
$\btheta\in\ERM_0(\kappa)$, as a function of $\kappa$ for $\delta = 1.5^{15}$.
Here we use the logistic link function $\varphi(t)=(1+e^{-t})^{-1}$.
The estimation error of any $\kappa$-margin solution lies in the gray region.
The linear programming algorithm of Section \ref{sec:AlgoNRLabels} achieves 
estimation error reported as the blue curve (namely $\sqrt{2( 1 - \rho^*)}$ for $\rho^*$ introduced in Theorem~\ref{thm:err} (a)). Vertical dashed lines correspond to 
$\kappa_{\slin}(\delta) :=\sup\{\kappa:\; \delta<\delta_{\slin}(\kappa; \varphi)\}$
(left vertical line) and $\kappa_{\su}(\delta) :=\sup\{\kappa:\; \delta<\delta_{\su}(\kappa; \varphi)\}$
(right vertical line).
%\am{Not clear which formula is the blue curve precisely}
}\label{fig:Error}
\end{figure}
The next theorem provides a simpler expression for $\rho_{\max}$ below the 
linear programing threshold $\delta_{\slin}(\kappa;\varphi)$ which we introduced in
the previous  section.
\begin{thm}\label{thm:best_error}
Let Assumption \ref{ass:exponential_tail_link} hold.
 If $\delta < \delta_{\slin}(\kappa;\varphi)$, then with high probability $\ERM_0(\kappa)$ is not empty, and
\begin{equation}\label{eq:best_error_alg}
	\rho_{\max}\le  1- (1 + \breve{o}_{\kappa}(1)) \frac{\delta}{2} \E \left[ \left( \kappa - YG \right)_+^2 \right]
	\, .
\end{equation}
\end{thm}

We finally pass to studying the behavior of the linear programming algorithm
introduced in the previous section. The next result provides a simple deterministic
approximation for $\<\hbtheta,\btheta_*\>$ which is accurate for large negative $\kappa$.
\begin{thm}\label{thm:err}
Let Assumption~\ref{ass:exponential_tail_link} hold. 
We recall the definition of $M(\rho,r)$ in \eqref{def:Mrho} and denote $m = \E[YG]$.
Further recall that $\hbtheta$ denotes the $\kappa$-margin solution found by the linear programming algorithm.
Then, the following hold:
\begin{enumerate}
\item[(a)]{If $(\rho^*, r^*)$ is the unique maximizer of $M(\rho,r)$ over $[-1,1] \times [0,1]$, then $\langle \hat \btheta, \btheta_* \rangle \xrightarrow{p} \rho^*$ as $n \to \infty$.
}
\item[(b)]
 {Let  $\delta = \delta(\kappa)$ be such that $\lim_{\kappa \to -\infty} \delta(\kappa) = \infty$ and 
 condition \eqref{cond:delta} holds.  
 Then, there exists a sufficiently negative $\underline \kappa = \underline\kappa(\alpha, m, \veps)>0$
  such that the following holds for all $\kappa < \underline \kappa$:
\begin{align*}
&\lim_{n \to \infty} \P \Big( \Big| \<\hbtheta,\btheta_*\> - (1-\cE(\kappa)) \Big| \le \breve{o}_{\kappa}(1) \cdot \cE(\kappa) \Big) = 1, \\
&\cE(\kappa) :=  \frac{1}{2 m^2 \delta} + \frac{\delta}{\delta_0(\kappa)} \, ,
\;\;\;\;\;\;\; \delta_0(\kappa):= \frac{|\kappa|^2e^{\alpha|\kappa|}}{2C_{\tail}\Phi(\kappa)}\, .
\end{align*}
}
\end{enumerate}
\end{thm}

Figure \ref{fig:Error} reports the predictions for the $\ell_2$ estimation error
$\Delta:=\|\hbtheta-\btheta_*\|_2$. The relation between this error and the alignment 
$\rho=\<\hbtheta,\btheta_*\>$ is simply given by
\begin{align*}
\Delta=\sqrt{2(1-\rho)}\, .
\end{align*}
It might be useful to point out a couple of interesting features of Figure \ref{fig:Error}.
\begin{enumerate}
\item We can interpret $\kappa$ as tuning how constrained or overparametrized
is the model: as $\kappa$ becomes more negative, the problem becomes less constrained
or more overparametrized. We observe that ---somewhat counterintuitively--- 
the estimation error decreases as the problem
becomes less constrained.
\item The solutions selected by the linear programming 
algorithm appear to be close to the lower edge of the band of possible estimation errors
for $\btheta\in\ERM_0(\kappa)$. Further study is warranted  to 
understand precisely how atypical these solutions are.
\end{enumerate}

\subsection{Polytope geometry}
\label{sec:Polytope}

As anticipated in Section \ref{sec:ConnectionsPolytope}, our results have implications on the
geometry of random polytopes, which we spell out explicitly.
\begin{cor}\label{cor:Radius}
Let $\ZZ := -\yy\odot \XX\in\reals^{n\times d}$ be the matrix with rows
$-y_i\xx_i^\top$, and denote by $\radius(\Polyt_{\ZZ})$ the radius of the  polytope 
$\Polyt_{\ZZ}:=\{\btheta:\, \ZZ\btheta\le \bfone\}$, as defined in Eq.~\eqref{eq:RadiusDef}.
Then we have, for any $\veps > 0$, assuming $n,d\to\infty$, $n/d\to\delta$, 
\begin{align}
&\lim_{n\to\infty}\prob_{n,d}(\radius(\Polyt_{\ZZ})\le \rho_{+\veps}(\delta)\big) = 1\, ,\\
&\liminf_{n\to\infty}\prob_{n,d}(\radius(\Polyt_{\ZZ})\ge \rho_{-\veps}(\delta)\big) >0\, ,\\
&\rho_{\pm\eps}(\delta):= \frac{1}{\sqrt{2\log\delta}} +\frac{\alpha\pm \eps}{2\log \delta}\, .
\end{align}
(Here it is understood that $\alpha=0$ in the case of pure noise labels.)
\end{cor}

Given the matrix $\ZZ\in\reals^{n\times d}$, we can consider a different 
way of constructing a random polytope, by taking the convex hull of the rows
of $\ZZ$, namely $\zz_i = -y_i\xx_i^\top$. We then define the inner radius as the maximum radius
of a ball that is contained in this convex hull:
\begin{align}
\inner(\Polyt_{\ZZ}^{\circ}):= \max\big\{r\ge 0: \, \Ball^d(0;r)\subseteq\Polyt^{\circ}_{\ZZ}\big\}\, ,
\;\;\;\;
\Polyt^{\circ}_{\ZZ}:= \conv(\zz_1,\dots,\zz_n)\, .
\end{align}
The notation $\Polyt^{\circ}_{\ZZ}$ is due to the fact that 
$\conv(\zz_1,\dots,\zz_n)$ is dual to the polytope  $\Polyt_{\ZZ}$ defined in
Eq.~\eqref{eq:RadiusDef}. The outer radius of $\Polyt_{\ZZ}$ and inner radius of
its dual are related through the following standard duality relation.
\begin{prop}[\cite{gritzmann1992inner}]
Assume $\radius(\Polyt_{\ZZ})<\infty$. Then
\begin{align*}
\inner(\Polyt_{\ZZ}^{\circ})\radius(\Polyt_{\ZZ})=1\, .
\end{align*}
\end{prop}
\begin{proof}
The proof is elementary but we sketch it for the reader's convenience.
We note that $\vv\not\in \conv(\zz_1,\dots,\zz_n)$ if and only if there exists $\bxi\in\reals^d$
such that $\<\bxi,\vv\> > \<\bxi,\zz_i\>$ for all $i\le n$. Since $\bzero\in  \conv(\zz_1,\dots,\zz_n)$,
we have $\<\bxi,\vv\> > 0$, and hence can assume, without loss of generality, $\<\bxi,\vv\>=1$.
Therefore 
\begin{align*}
\inner(\Polyt_{\ZZ}^{\circ}) &= \inf\big\{\|\vv\|_2: \; \vv\not\in\Polyt_{\ZZ}^{\circ}\big\}\\
& = \inf\big\{\|\vv\|_2: \; \<\bxi,\vv\> = 1\, ,\;\<\bxi,\zz_i\> <1 \;\; \forall i\le n\big\}
\\
& \stackrel{(i)}{=} \inf\big\{1/\|\bxi\|_2: \; \<\bxi,\zz_i\> <1 \;\; \forall i\le n\big\}=\radius(\Polyt_{\ZZ})^{-1}\, ,
\end{align*}
where $(i)$ is due to $\inf \{\|\vv\|_2: \; \<\bxi,\vv\> = 1\} = 1/\|\bxi\|_2$. This completes the proof.
\end{proof} 
As a consequence of this duality, our results on the threshold $\delta_{\sus}(\kappa)$
have direct implications on $\inner(\Polyt_{\ZZ}^{\circ})$. Informally, with
probability bounded away from zero, 
\begin{align}
\inner(\Polyt^{\circ}_{\ZZ})= \sqrt{2\log\delta}-\alpha +o_{\delta}(1)\, .\label{eq:InnerEstimate}
\end{align}
Recently, Baldi and Vershynin \cite{baldi2021theory} considered
the problem of bounding  $\inner(\Polyt^{\circ}_{\ZZ})$. For the pure noise case, they
prove that, with probability at least $1-\exp(-n)$, 
$\inner(\Polyt^{\circ}_{\ZZ})\ge  \sqrt{2(1-\eps)\log\delta}$ for any $\delta>C(\eps)$.
Our results hold in an asymptotic setting, while the one of \cite{baldi2021theory} is
non-asymptotic. On the other hand, we provide a more precise characterization,
that applies also to the linear signal model, and a matching upper bound.

\begin{rem}
Neither Corollary \ref{cor:Radius} nor its consequence Eq.~\eqref{eq:InnerEstimate}
is as accurate as our main results in Theorems \ref{thm:pure_noise_lower_bound},
\ref{thm:pure_noise_upper_bound} (for the pure noise case) or Theorems 
\ref{thm:logistic_lower_bound}, \ref{thm:logistic_upper_bound} (for the linear signal case). In particular, these simpler formulas do not 
account for the  prefactors polynomial in $\kappa$ which separate between $\delta_{\slin}(\kappa)$
and $\delta_{\sus}(\kappa)$. 
\end{rem}

\section{Gradient descent}
\label{sec:GD}

Our algorithmic results are based on an extremely simple linear programming approach.
An interesting open question is to understand the behavior and relative advantages of other 
algorithms. In this section we introduce a differentiable loss function, 
whose minimization leads to a maximum negative margin solution, and
which does not require any tuning.

\subsection{A differentiable loss function without tuning parameters}
\label{sec:Loss}

 In the case of linearly separable data, it is known that 
gradient descent with respect to the logistic loss converges to
a maximum margin solution \cite{soudry2018implicit,ji2018risk}.
This approach has the advantage of providing a differentiable cost 
function without tuning parameters.
Is it possible to achieve the same for non-separable data (to achieve maximum \emph{negative}
margin)?

We are interested in finding $\btheta$ satisfying $y_i \<\btheta,\xx_i\>\ge \kappa$
for all $i\le n$. Let $\ell:\reals\to\reals_{\ge 0}$ be a monotone decreasing 
function with $\lim_{x\to +\infty}\ell(x)=0$.
We define the empirical risk function:
\begin{equation*}
	\hR_{n,\kappa}(\btheta) := 
	\frac{1}{n}\sum_{i=1}^{n} \ell \big( y_i\langle \xx_i, \btheta \rangle - \kappa \|\btheta\|_2 \big),
\end{equation*}
and consider using gradient descent (GD) to minimize it. Fixing a learning rate $\eta > 0$ 
(to be determined later), the GD iteration reads
\begin{equation}
	\btheta^{t+1} = \btheta^{t} - \eta \nabla \hR_{n,\kappa}(\btheta^t) =  \btheta^t - 
	\frac{\eta}{n} \sum_{i=1}^{n} 
	\ell' \big( y_i\langle \xx_i, \btheta^t \rangle - \kappa \norm{\btheta^t}_2 \big)
	 \left( y_i\xx_i - \kappa \frac{\btheta^t}{\norm{\btheta^t}_2} \right).
	 \label{eq:GD-kappa}
\end{equation}
The intuition behind our definition of $\hR_{n,\kappa}$ is the following.
 Assume $\hbtheta$ is a $(\kappa+\eps)$-margin solution for some $\eps>0$, then 
 $y_i \langle \xx_i, \hat{\btheta} \rangle - \kappa > 0$, $\forall i \in [n]$, and we have
\begin{equation*}
	\lim_{s \to +\infty} \hR_{n,\kappa} (s \hat{\btheta}) = \lim_{s \to +\infty} \frac{1}{n}
	\sum_{i=1}^{n} \ell \Big( s \big( y_i\< \xx_i, \hat{\btheta} \> - \kappa \big) \Big) 
	= 0\,. 
\end{equation*}
In words, $\hR_{n,\kappa}$ is minimized at infinity along the direction of $\hat{\btheta}$.
 Therefore, it would be reasonable to guess that if $\|\btheta^t\|_2 \to \infty$, 
 then $\btheta^t/\|\btheta^t\|_2$ should converge to a $\kappa$-margin solution.
We will prove that this is indeed the case under some additional conditions.
\begin{defn}\label{ass:tight_exp_tail}
	We say that $\ell$  has a tight exponential tail if 
	there exist positive constants $c, a, \mu_{+}, \mu_{-}, u_{+}$ and $u_{-}$ such that
\begin{align*}
& \forall u>u_{+}: -\ell'(u) \leq c\left(1+\exp \left(-\mu_{+} u\right)\right) e^{-a u}, \\
& \forall u>u_{-}: -\ell'(u) \geq c\left(1-\exp \left(-\mu_{-} u\right)\right) e^{-a u}.
\end{align*}
\end{defn}

\begin{thm}\label{thm:grad_descent}
Let $\ell:\reals\to\reals_{\ge 0}$ be  positive, differentiable, monotonically decreasing to zero 
at $+\infty$. Further assume $\ell'$ to be bounded and $\beta$-Lipschitz for some positive constant $\beta$.
Finally, assume $\ell$ to have a tight exponential tail.

There exists a constant $M > 0$, such that for all learning rate $\eta < 2/M$, 
as long as the gradient descent iterates diverge to infinity ($\norm{\btheta^t}_2 \to \infty$), 
then we have 
	\begin{equation*}
		\liminf_{t \to \infty} \min_{1 \le i \le n} y_i \langle \hbtheta^t ,\xx_i \rangle \ge \kappa\, ,
		\;\;\;\;
		\hbtheta^t:=\frac{\btheta^t}{\|\btheta^t\|_2}\, .
	\end{equation*}
\end{thm}

\subsection{Numerical experiments}
\label{sec:numerical}

In this section, we illustrate the behavior of the linear programming (LP) algorithm, and of
gradient descent (GD) on $\hR_{n,\kappa}(\btheta)$  via numerical simulations. 
We generate synthetic data from the pure noise model, namely $\xx_i \sim_{\iid} \cN (\bzero, \bI_d)$, 
and without loss of generality we set $y_i = 1$ for all $i \in [n]$. 
 We run the gradient descent iteration \eqref{eq:GD-kappa} with $\ell (x) = \log (1 + e^{- x})$, random initialization $\btheta^0 \sim \Unif (\S^{d - 1})$,
 and step size $\eta = 0.05$, and terminate it as soon as the data achieves a margin $\kappa$, i.e.,
\begin{equation*}
	\min_{1 \le i \le n} \langle \hbtheta^t ,\xx_i \rangle \ge \kappa\, ,
		\;\;\;\;
		\hbtheta^t:=\frac{\btheta^t}{\|\btheta^t\|_2}\, .
\end{equation*}
For the largest experiment, we replace GD by mini-batch stochastic gradient descent (SGD).

In the first experiment, we fix $\kappa = - 1.5$, and
$d \in \{50, 100, 200\}$. We choose various values of $n$ with $\delta=n/d$ ranging from below the linear programming
  threshold to the phase transition upper bound. For each pair of $(n, d)$ we 
  run both GD and LP (cf. Eq.~\eqref{opt:convex-1}) 
  for  $1000$ independent realizations of the data $(\yy,\XX)$, and report in Fig.~\ref{fig:GDvsLP_1} 
  the empirical  probability of finding a $\kappa$-margin solution. 

\begin{figure}[h!]
\centering
\includegraphics[width=45em]{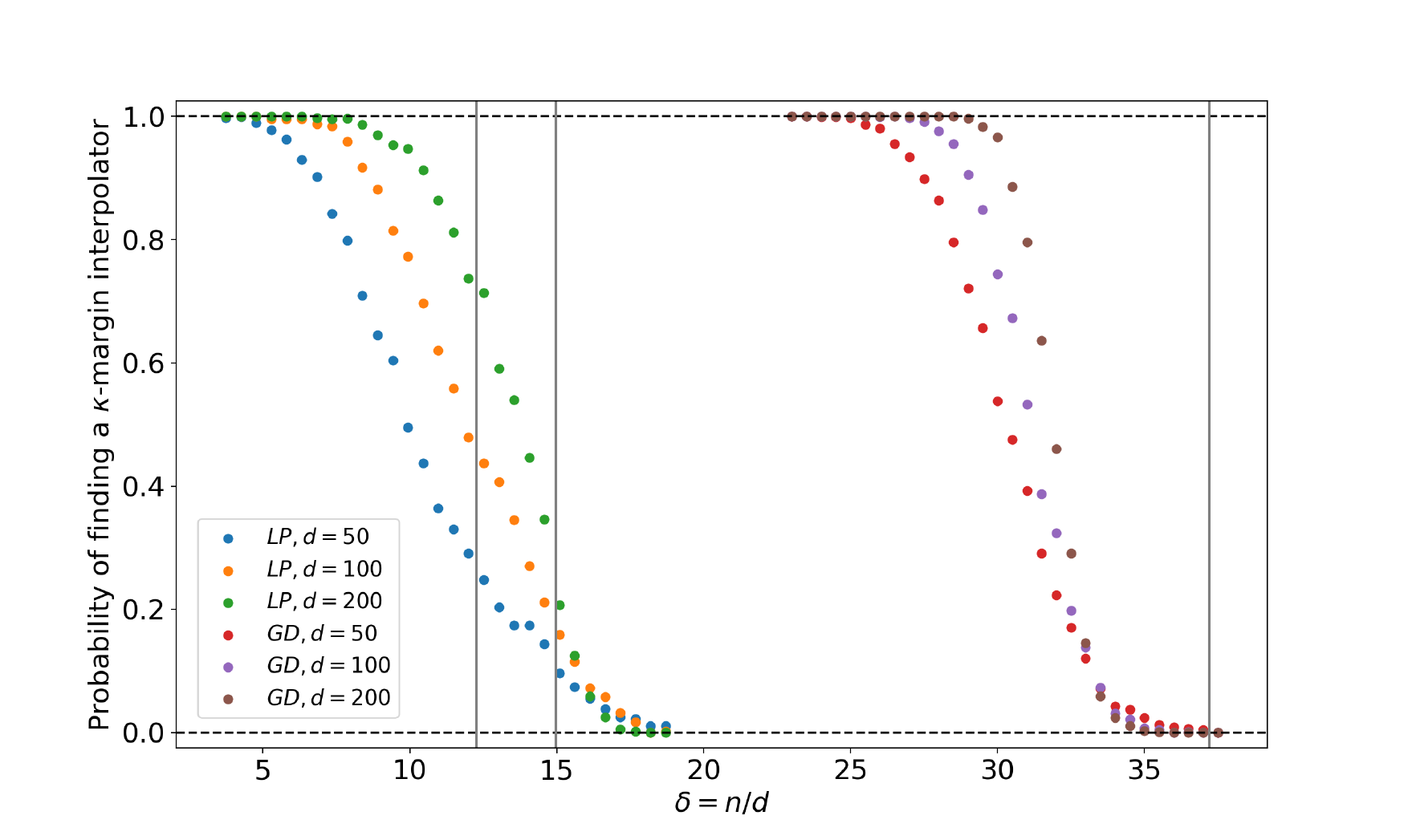}
\put(-312,43){{\small $\delta_{\sl}$}}
\put(-287,43){{\small $\delta_{\slin}$}}
\put(-79,43){{\small $\delta_{\su}$}}
\caption{Scatter plots of the empirical probability of finding a $\kappa$-margin interpolator 
using gradient descent (GD) and linear programming (LP), as a function of $\delta$. 
We fix $\kappa = -1.5$ and choose $d \in \{50, 100, 200\}$. The vertical solid lines represent 
$\delta_{\sl} (\kappa)$, $\delta_{\slin} (\kappa)$ and $\delta_{\su} (\kappa)$ (from left to right) respectively. Note that for $\kappa = -1.5$, we actually have $\delta_{\sl} (\kappa) < \delta_{\slin} (\kappa)$.}
\label{fig:GDvsLP_1}
\end{figure}

Figure~\ref{fig:GDvsLP_1} suggests the following remarks:
\begin{enumerate}
\item For both algorithms we see that the probability of finding a $\kappa$-margin
interpolator decreases from close to one to close to zero as $\delta$ crosses a threshold.
The transition region becomes steeper as $d$ grows larger. 
This finding confirms Theorem~\ref{thm:algpure}, which predicts a sharp
threshold for the LP algorithm. The location of this threshold is also consistent with 
the theoretical prediction.
\item  Figure~\ref{fig:GDvsLP_1} also suggests a phase transition 
 for GD under proportional asymptotics. A proof of this phenomenon is 
 at this moment an open problem.
 \item GD significantly outperforms LP.
 Notice that for $\kappa=-1.5$, our bounds $\delta_{\sl}(\kappa)$, $\delta_{\su} (\kappa)$ 
 are too wide to separate the linear programming threshold $\delta_{\slin}(\kappa)$
from the existence threshold $\delta_{\sus}(\kappa)$. However
the GD threshold appears to be strictly larger than $\delta_{\slin}(\kappa)$,
suggesting that such a separation exists.
 \end{enumerate}
 We note that an analysis of gradient-based algorithms using techniques 
 from physics was initiated in \cite{agoritsas2018out}.

As for the second experiment, we select $\kappa = - 3$,  and 
$d \in \{25, 50, 100 \}$. In this case we need to take $\delta\ge 1000$ to be close to
the satisfiability phase transition, resulting in $n$ of the order of $10^5$. 
Computing the full gradient is very slow in this case, and therefore 
we replace GD by mini-batch SGD with batch size $=1000$.
 As before, in Figure \ref{fig:GDvsLP_2} we plot the empirical probability of finding a $\kappa$-margin 
 interpolator
 using both SGD and LP calculated from $400$ independent realizations.

\begin{figure}[h!]
\centering
\includegraphics[width=45em]{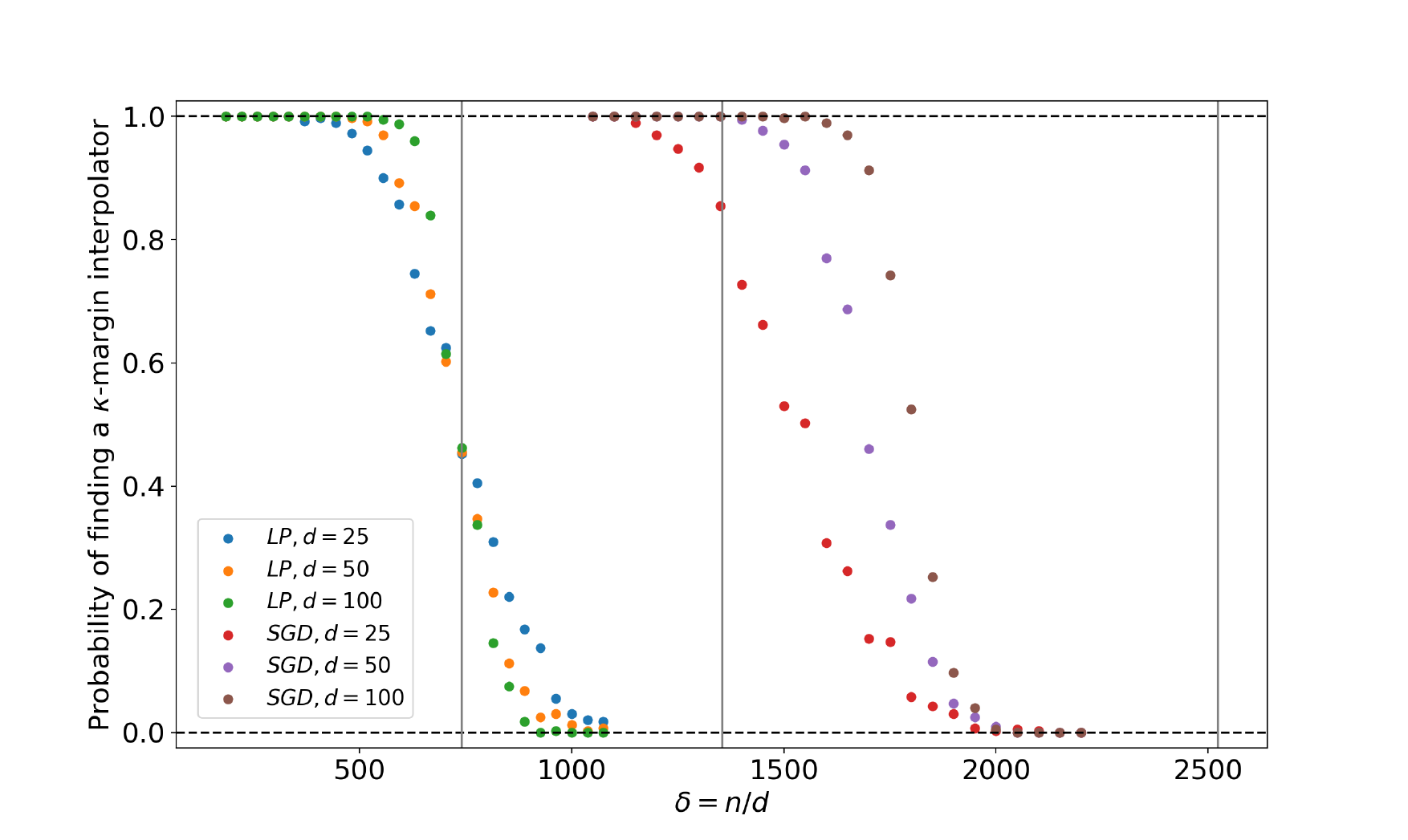}
\put(-317.5,43){{\small $\delta_{\slin}$}}
\put(-233.5,43){{\small $\delta_{\sl}$}}
\put(-78.5,43){{\small $\delta_{\su}$}}
\caption{Scatter plots of the empirical probability of finding a $\kappa$-margin interpolator using 
mini-batch stochastic gradient descent (SGD) and linear programming (LP), as a function of $\delta$. 
We fix $\kappa = - 3$ and choose $d \in \{25, 50, 100 \}$. The vertical solid lines represent
 $\delta_{\slin} (\kappa)$, $\delta_{\sl} (\kappa)$ and $\delta_{\su} (\kappa)$ (from left to right) respectively.}
\label{fig:GDvsLP_2}
\end{figure}
Note that for $\kappa=-3$, $\delta_{\sl}(\kappa)>\delta_{\slin}(\kappa)$.
In other words, our bounds are precise enough to separate the satisfiability and
linear programming threshold. However the empirical threshold for SGD seems to 
be  between $\delta_{\sl}(\kappa)$ and $\delta_{\su}(\kappa)$. Therefore we cannot exclude
that SGD finds solutions with probability bounded away from zero
 for all $\delta<\delta_{\sus}(\kappa)$.

\section{Proof ideas}
\label{sec:ProofIdea}

In the previous sections, we stated three main types of results:
$(i)$~Lower bounds on the interpolation threshold $\delta_{\sus}(\kappa)$;
$(ii)$~Upper bounds on the interpolation threshold $\delta_{\sus}(\kappa)$;
$(iii)$~Characterizations of the linear programming threshold $\delta_{\slin}(\kappa)$.
We briefly describe the approach we use for each of these proofs. 
%We will mostly focus this informal overview on the case of pure noise.  
%The linear signal model requires additional technical
%work, and we will only mention differences in the proof when they are substantial. 

\subsection{Lower bounds on $\delta_{\sus}(\kappa)$.}
A standard approach in the context of random CSPs is to apply the second moment method.
Namely, one constructs a non-negative random variable $Z_n=Z(\XX)\ge 0$ such that
$Z>0$ if and only if the problem has solutions. By the Paley-Zygmund inequality,
computing the first two moments of $Z$ yields a lower bound on the probability that
there exist solutions:
\begin{align}
\prob\big(\ERM_0(\kappa)\neq \emptyset\big)\ge \frac{\E\{Z\}^2}{\E\{Z^2\}}\, .
\end{align}
The simplest choice for $Z$ in discrete settings is the number of solutions \cite{achlioptas2006random},
and an obvious generalization would be the volume of the set $\ERM_0(\kappa)$. 
However, an explicit calculation shows that in this case  $\E\{Z^2\}$ is
exponentially larger than $\E\{Z\}^2$ for any $\delta>0$. 

\noindent \textbf{Pure noise model.} We obtain a non-trivial bound by weighting the constraints $y_i\<\xx_i,\btheta\>\ge \kappa$,
and defining
\begin{align}
Z := \int_{\S^{d-1}}\prod_{i=1}^nf(y_i\<\xx_i,\btheta\>)\, \mu(\de\btheta)\, ,
\end{align}
where $\mu$ is the uniform probability distribution over $\S^{d-1}$
and $f:\reals\to\reals$ is a non-negative measurable function such that $f(t) = 0$
for $t<\kappa$. Weighted solution counts were used already in the context of 
discrete CSPs, e.g. in \cite{achlioptas2005rigorous}.

Some calculation shows that $\E [Z^2]$ is determined by the integral
\begin{equation*}
\int_{-1}^1 \sqrt{\frac{d}{2\pi}} \exp \left( -n \Big( F(q) + \frac{I(q)}{\delta} \Big) \right) \; \d q
\end{equation*}
where $q$ is the inner product of a pair of vectors $ \btheta_1, \btheta_2 \in \S^{d-1}$. Here, functions $F$ and $I$ are defined by \begin{equation*}
F(q) = - \log \E_q[ f(G_1) f(G_2) ], \qquad I(q) = -\frac{1}{2} \log (1-q^2)
\end{equation*}
where $(G_1,G_2)$ follows $\cN \left( \bzero ,  \begin{bmatrix}
    	1 & q \\
    	q & 1
    \end{bmatrix} \right)$. 

We use $f(t) = e^{-ct}\bfone_{t\ge \kappa}$ and show that $c=c(\kappa)$
can be chosen so that $F(q) + I(q) / \delta$ has a unique minimizer at $q=0$. Once this is proved, we then use Laplace's method to obtain
\begin{equation*}
\liminf_{n \to + \infty} \frac{\E[Z]^2}{\E[Z^2]} \ge \sqrt{\frac{1}{\delta F''(0) + I''(0)}} > 0 \; .
\end{equation*}
Furthermore, we combine this estimate with a Gaussian concentrate argument to show that $\prob\big(\ERM_0(\kappa)\neq \emptyset\big)$ actually converges to one instead of being only bounded away from zero.\\

\noindent \textbf{Linear signal model.}
For labels correlated with a linear signal, the application of the second moment becomes more 
challenging, because the quantity $Z$ defined earlier has larger fluctuations, and its
moment can be dominated by exponentially rare realizations of the data $\{(y_i,\xx_i)\}_{i\le n}$.
In order to reduce this variability, we define $Z$ via a restricted integral
over vectors $\btheta$ such that $\<\btheta,\btheta_*\>=\rho$. 

To be more concrete, 
we decompose $\xx_i$ and $\btheta$ into components in $\spann\{\btheta^*\}$ and components in its orthogonal complement. Let $\zz_i \sim \cN(\bzero, \bI_{d-1})$, $G_i \sim \cN(0,1)$, and $\rho = \langle \btheta, \btheta^*\rangle$, we will show that for certain $\rho \in [-1,1]$ we choose,
\begin{equation}\label{ineq:problnd}
	\liminf_{n \to +\infty} \P \left( \exists \Vert \ww \Vert_2 \ge 1, \ \text{s.t.} \ \forall 1 \le i \le n, \rho y_i G_i + \sqrt{1-\rho^2} \langle \zz_i, \ww  \rangle \ge \kappa \right) > 0,
\end{equation}
where $(y_i, G_i) \bot \zz_i$ and the conditional distribution $y_i|G_i$ is given by the linear signal model. A similar concentration argument will further boost this probability to 1.

A technical difficulty lies in the additional random variable $\rho y_i G_i$. Ideally, we hope that $\rho (1-\rho^2)^{-1/2}\, y_i G_i \ge \kappa'$ for appropriately chosen $\kappa'$ so that the constraint reduces to
\begin{equation*}
\langle \zz_i, \ww  \rangle \ge \kappa   / \sqrt{1-\rho^2} - \kappa' 
\end{equation*}
and we could apply a similar second moment calculation to find a solution $\ww$. However, the lower bound on $\rho (1-\rho^2)^{-1/2}\, y_i G_i$ only holds for a subset of indices in $\{1,2,\ldots,n\}$, which we call ``good samples''. The remaining indices, which we call ``bad samples'', are only a small fraction of all indices but create obstacles when applying the second moment method.

To deal with bad samples, our strategy is to free up some coordinates of $\ww$ that are used to counter the bad samples. More specifically, we partition $\ww$ into subvectors $\ww = [\ww_1; \ww_2]$ where $\ww_1 \in \R^{d-1-d_0}$ and $\ww_2 \in \R^{d_0}$, and similarly partition $\ZZ = (\ZZ_1, \ZZ_2)$ where $\ZZ_1, \ZZ_2$ has $d-1-d_0$ and $d_0$ columns respectively. Our goal is equivalent to finding $[\ww_1; \ww_2]$ with 
\begin{equation*}
\begin{split}
	\uu^G + \ZZ_1^G \ww_1 + \ZZ_2^G \ww_2 \ge & \kappa \bone, \\ \uu^B + \ZZ_1^B \ww_1 + \ZZ_2^B \ww_2 \ge & \kappa \bone,
\end{split}
\end{equation*}
where the superscripts $G$ and $B$ denote the indices of good samples and those of bad samples. We will apply the second moment method to first find $\ww_1$ with $\uu^G + \ZZ_1^G \ww_1 \ge \tilde \kappa$, and then find $\ww_2$ such that 
	\begin{align*}
		\ZZ_2^G \ww_2 \ge & (\kappa - \tilde \kappa) \bone, \\
		\ZZ_2^B \ww_2 \ge & \kappa \bone - \uu^B - \ZZ_1^B \ww_1.
	\end{align*}
Note that (i) $\ww_1$ only relies on $\ZZ_1^G$ which is independent of $\ZZ_1^B$, $\ZZ_2^G$, $\ZZ_2^B$, so the RHS of the second inequality is fixed after conditioning on $\ZZ_1^G$; (ii) for a suitably large $d_0$, it is easy to find such $\ww_2$ and in the meantime not decrease the lower threshold $\delta_{\sl}(\kappa)$ much (still matching the upper threshold in the limit $\kappa \to -\infty$).

\subsection{Upper bounds on $\delta_{\sus}(\kappa)$.}
The simplest approach to obtain upper
bounds in discrete CSPs is to let $Z_0$ be the number of solutions and show
that (when the number of constraints per variable passes a threshold),
$\E\{Z_0\}\to 0$. This implies $\prob(Z_0>0)\to 0$ by Markov's inequality. 

In the case of continuous decision variables, this method cannot be applied, as 
the volume can be arbitrarily small and yet nonzero.
We consider instead the following minimax problem
\begin{equation*}
	\xi_{n, \kappa} = \min_{\norm{\btheta}_2 = 1} \max_{\norm{\blambda}_2 = 1, \blambda \ge \bzero} \frac{1}{\sqrt{d}} \blambda^\top \left( \kappa \bone - \XX \btheta \right).
\end{equation*}
It is easy to see that
\begin{align}
\P(\ERM_0(\kappa) \neq \emptyset) = \P \left( \xi_{n, \kappa} \le 0 \right)\, ,
\end{align}
and it is therefore sufficient to upper bound the latter probability.
Gordon's Gaussian comparison inequality \cite{gordon1985some} implies 
\begin{align}
&\xi_{n, \kappa} +\frac{1}{\sqrt{d}}z \succeq \xi^{\slin}_{n,\kappa}\, ,\label{eq:FirstGordon} \\
 &\xi^{\slin}_{n,\kappa} := \min_{\norm{\btheta}_2 = 1} \max_{\norm{\blambda}_2 = 1, \blambda \ge \bzero} 
 \Big\{
\frac{1}{\sqrt{d}} \blambda^\top \left( \kappa \bone - \norm{\btheta}_2 \hh\right) + 
\frac{1}{\sqrt{d}} \norm{\blambda}_2 \bgg^\top \btheta
\Big\}\, . \label{def:xilin}
\end{align}
Here $\succeq$ denotes stochastic domination\footnote{Namely, $X\succeq Y$ if $\E(\psi(X))\ge \E(\psi(Y))$ 
for any non-decreasing function $\psi$.} and $z\sim\normal(0,1)$, $\hh\sim\normal(\bzero,\bI_n)$,
$\bgg\sim\normal(\bzero,\bI_d)$ are mutually independent and independent of $\XX$. 
At this point, the standard approach is to estimate
$\E\xi^{\slin}_{n,\kappa}$, and upper bound $\delta_{\sus}(\kappa)$ by the infimum
of $\delta$ such that $\E\xi^{\slin}_{n,\kappa}$ is positive bounded away from zero.
In the present setting, this yields the replica symmetric upper bound of  \cite{stojnic2013another},
which is substantially suboptimal.

We follow a different approach, first appearing in \cite{stojnic2013negative}, which
is based on estimating
the moment generating function of $\xi_{n, \kappa}$. 
Namely, we use Eq.~\eqref{eq:FirstGordon} to deduce that, for any $c\ge 0$,
$\E\{\exp(-c \xi_{n, \kappa})\}\E\{\exp(-cz/\sqrt{d})\}\le \E\{\exp(-c \xi^{\slin}_{n, \kappa})\}$,
whence
\begin{align}\label{ineq:ERMubd}
\P(\ERM_0(\kappa) \neq \emptyset) \le  \frac{\E\{\exp(-c \xi^{\slin}_{n, \kappa})\}}
{\E\{\exp(-cz/\sqrt{d})\}}\, .
\end{align}
The proof is obtained by estimating the right-hand side and optimizing this bound over $c\ge 0$.
Remarkably, the exponential weighting method yields an upper bound that is significantly 
tighter than more standard applications of Gordon's inequality.\\
 
\noindent \textbf{Pure noise model.} The min-max form of $\xi_{n,\kappa}^{\slin}$ is actually equivalent to two separate optimization problems that admit simple solutions: under the condition $\min_{i \le n} h_i \le \kappa$ (whose effect is negligible), we have
\begin{equation*}
\xi_{n,\kappa}^{\slin} = \frac{1}{\sqrt{d}} \big\| \kappa \bone - \hh \big\|_2 - \frac{1}{\sqrt{d}} \| \bgg \|_2
\end{equation*}
Since $\hh$ and $\bgg$ are independent, we can simplify $\E\{\exp(-c \xi^{\slin}_{n, \kappa})\}$ and choose the optimal $c\ge 0$ on the RHS of \eqref{ineq:ERMubd}, which leads to the claimed upper bound.\\

\noindent \textbf{Linear signal model.} Similar to before, we decompose $\xx_i$ and $\btheta$ into components in $\spann\{\btheta^*\}$ and those in its orthogonal complements. Consider applying Gordon's theorem, for each fixed $\rho \in [-1,1]$, to
\begin{equation*}
\xi_{n,\kappa, \rho} =  \min_{\| \ww \|_2 = 1} \max_{ \|\blambda \|_2 = 1, \blambda \ge 0} \frac{1}{\sqrt{d}} \left(\kappa \bone - \rho \yy \odot \bG - \sqrt{1 - \rho^2}\, \ZZ \ww \right)\, .
\end{equation*}
Some calculations yield a $\rho$-dependent threshold $\delta_{\su}(\kappa, \rho; \varphi)$, above which no solutions correlated with the signal, namely $\{\btheta \in \S^{d-1}: \langle \btheta, \btheta^* \rangle = \rho \}$,  exist with high probability. We then use a covering argument to exclude solutions for all possible $\rho$ above the largest threshold. This results in the final upper bound $\delta_{\su}(\kappa; \varphi)$ in Definition~\ref{def:sig_upper}.

\subsection{Characterization of $\delta_{\slin}(\kappa)$.}

Notice that the convex problem \eqref{opt:convex-2} is simply a linear program without the norm constraint $\| \btheta \|_2 \le 1$. We hope to show that whenever $\delta < \delta_{\slin}(\kappa)$, solving the problem \eqref{opt:convex-2} produces a unit-norm solution $\hat \btheta \in \S^{d-1}$.\\

\noindent \textbf{Pure noise model.}
The idea is to verify that $\| \btheta \|_2 \le 1$ is an active constraint: if we replace $\| \btheta \|_2 \le 1$ with $\| \btheta \|_2 \le 1+\veps$, then the optimal value of the objective function would increase. The optimal value of the objective function is tightly characterized by an
 application of Gordon's inequality analogous to \cite{thrampoulidis2018precise,montanari2019generalization}, so it is easy to compare the optimal values and verify that the constraint is active.\\

\noindent \textbf{Linear signal model.} We follow our recipe again: decompose $\xx_i$ and $\btheta$ into orthogonal components and apply Gordon's theorem for each fixed $\rho$. To obtain the existence of solutions, we check that the constraint $\| \btheta \|_2 \le 1$ is active. Further, we obtain asymptotic estimation error by checking that an additional linear constraint of the form $\underline{\rho} \le \langle \btheta, \btheta^* \rangle \le \overline{\rho}$ is active.

\section{Conclusion}
\label{sec:Conclusion}

Given data $(\XX,\yy) = \{(\xx_i,y_i)\}_{i\le n}$, $\xx_i\in\reals^d$, $y_i\in\{+1,-1\}$ that are not 
linearly separable, the negative perceptron attempts to find a linear classifier
that minimizes the maximum distance of any misclassified sample from the decision boundary.
This leads to a non-convex empirical risk minimization problem, whose global minimizers
are $\kappa$-margin classifiers, with $\kappa<0$. 

We studied two simple models for the data $(\XX,\yy)$: purely random labels, and labels 
correlated with a linear signal. We focused on the high-dimensional regime
$n,d\to\infty$, $n/d\to\delta$.
For each of these models, we proved bounds on the 
maximum margin $\kappa_{\sus}(\delta)$ or, equivalently, on the maximum number
of samples per dimension $\delta_{\sus}(\kappa)$ such that a $\kappa$-margin solution 
exists for  $\delta<\delta_{\sus}(\kappa)$. 

Our main result is that, despite non-convexity, very simple algorithms
can find $\kappa$ margin solutions in a large portion of the region in which such solutions 
exist. In particular, a single step of Frank-Wolfe succeeds with high probability for
$n/d\le \delta_{\slin}(\kappa)(1-\eps)$, for a threshold $\delta_{\slin}(\kappa)$
that we characterize. In other words the non-convex learning problem becomes easy at sufficient 
overparametrization.

It is worth emphasizing that this phenomenon is very different from the one arising in 
certain high-dimensional statistics problems, whereby the problem becomes easy
in a high signal-to-noise ratio regime. Example of this have been studied in matrix completion
\cite{keshavan2010matrix,ge2016matrix}, phase retrieval \cite{candes2015phase,chen2015solving,sun2018geometric}, mixture of Gaussians \cite{xu2016global,daskalakis2017ten}, 
and so on.
In those setting, additional samples make the problem even easier, and the basic mechanism 
is uniform concentration around a bowl-shaped population risk that is uniquely minimized 
at the true signal \cite{mei2018landscape}.
In contrast, in the present example, the problem becomes easy because
it is underconstrained and larger sample size can translate into a harder problem.

Our work leads to a number of interesting open questions. In particular, it would be 
interesting to obtain a more precise characterization of the existence phase transition
$\delta_{\sus}(\kappa)$ and prove that it is a sharp threshold. Also, it would be interesting to
analyze gradient descent algorithms, which empirically seem to have a superior behavior
to the linear programming approach. 

Finally, we used the negative perceptron mainly as a toy model for more complex
non-convex empirical risk minimization problems. It would be interesting to understand 
whether negative margin methods can be useful from a statistical viewpoint.

\section*{Acknowledgements}

This work was supported by NSF through award DMS-2031883 and from the Simons Foundation through
Award 814639 for the Collaboration on the Theoretical Foundations of Deep Learning.
It was also supported by the NSF grant CCF-2006489 and the ONR grant N00014-18-1-2729.

\newpage

\bibliographystyle{alpha}

\newcommand{\etalchar}[1]{$^{#1}$}

\newpage

\appendix

\section{$\kappa$-margin classifiers in the pure noise model:\\ Proofs of Theorems
\ref{thm:pure_noise_lower_bound} and \ref{thm:pure_noise_upper_bound}}

\label{sec:ExistencePureNoise}

\subsection{Phase transition lower bound: Proof of Theorem~\ref{thm:pure_noise_lower_bound}}\label{sec:pure-lower}

We first state the following crucial lemma, whose proof will be deferred to Appendix~\ref{sec:append-pure-lemma}.

\begin{lem}\label{lem:property_of_e_q} The followings hold:
\hspace{2em}
\begin{enumerate}
		\item[(a)] For any $\kappa < 0$, there exists a unique $c_* = c_* (\kappa) > 0$ satisfying $c \left( 1 - \Phi (\kappa +c) \right) = \phi (\kappa + c)$. Moreover, $c_* = \breve{O}_{\kappa} (\exp(- \kappa^2/2))$ as $\kappa \to -\infty$.
		
        \item[(b)] Let $\kappa$ and $c_*$ be as described in (a), set $f (x) = \exp (- c_*(\kappa) x) \bone \{ x \ge \kappa \}$, and define
        \begin{equation}\label{def:F(q)_and_I(q)}
	        F(q)=-\log \E_q \left[f(G_1) f(G_2) \right], \quad I(q)=-\frac{1}{2} \log \left(1-q^2 \right)
        \end{equation}
        as per Definition~\ref{def:pure_lower}. We further define
	\begin{equation*}
		e (q) = \E_q [f(G_1) f(G_2)] = \exp \left(- F(q) \right),
	\end{equation*}
    then it follows that $e'(0) = F'(0) = 0$, $e(0) = 1 + \breve{o}_{\kappa} (1)$, and $e(1) - e(0) = (1 + \breve{o}_{\kappa} (1)) \Phi (\kappa)$.
    
        \item[(c)] For $q \in [-1, 1]$, we have
    \begin{equation}\label{eq:asymptotics_for_e'(q)}
    \begin{split}
    e'(q) =& \frac{1}{2 \pi \sqrt{1-q^2}} \exp \left(-2 \kappa c_*-\frac{\kappa^2}{1+q}\right)+c_*^2 \exp \left(c_*^2(1+q)\right) \P_q \left(\min \left(G_1, G_2 \right) \ge \kappa+c_*(1+q)\right) \\
    	&- \sqrt{\frac{2}{\pi}} \left(1-\Phi\left(\sqrt{\frac{1-q}{1+q}}(\kappa+c_*(1+q))\right)\right) c_* \exp \left(c_*^2(1+q)-\frac{1}{2}(\kappa+c_*(1+q))^2\right) \\
    	=& \frac{1+\breve{o}_{\kappa}(1)}{2 \pi \sqrt{1-q^{2}}} \exp \left(-\frac{\kappa^{2}}{1+q}\right)+\breve{O}_{\kappa} \left(\exp (-\kappa^2) \right).
    	\end{split}
    \end{equation}
    Here, $\P_q$ denotes the probability distribution $\cN \left( \bzero , \begin{bmatrix}
    	1 & q \\
    	q & 1
    \end{bmatrix} \right)$.
    Further, for any $q \in [0, 1)$ there exists a constant $C \in (0, +\infty)$ such that
    \begin{equation}\label{eq:bound_on_e''(q)}
        \sup_{u \in [-q, q]} \vert e''(u) \vert \le \vert \kappa \vert^C \exp \left( - \frac{\kappa^2}{1+q} \right).
    \end{equation}

	\end{enumerate}

    \end{lem}

We will prove Theorem~\ref{thm:pure_noise_lower_bound} via the second moment method and a Gaussian concentration argument. Let $\mu$ denote the uniform probability measure on $\S^{d-1}$, and $f$ be as defined in Lemma~\ref{lem:property_of_e_q} \textit{(b)}. Then we know that $f$ is supported on the interval $[\kappa, +\infty)$, i.e., $f(x) > 0$ if and only if $x \ge \kappa$. We define the random variable
\begin{equation*}
	Z = \int_{\S^{d-1}} \prod_{i=1}^{n} f \left( \left\langle \btheta, \xx_i \right\rangle \right) \mu(\d \btheta),
\end{equation*}
then we know that $Z > 0$ if and only if $\exists \btheta \in \S^{d-1}, \forall 1 \le i \le n, \left\langle \xx_i, \btheta \right\rangle \ge \kappa$. Therefore,
\begin{equation}\label{eq:Z_to_lower_bound}
	\P \left( \exists \btheta \in \S^{d-1}, \forall 1 \le i \le n, \left\langle \xx_i, \btheta \right\rangle \ge \kappa \right) = \P \left( Z > 0 \right).
\end{equation}
Now, according to the Paley-Zygmund inequality,
\begin{equation}\label{ineq:paley_zygmund}
	\P \left( Z > 0 \right) = \E \left[ \bone \left\{ Z > 0 \right\} \right] \ge \frac{\E \left[ Z \bone \left\{ Z > 0 \right\} \right]^2}{\E \left[ Z^2 \right]} = \frac{\E \left[ Z \right]^2}{\E \left[ Z^2 \right]},
\end{equation}
where the intermediate bound follows from the Cauchy-Schwarz inequality. Our aim is to show that $\delta < \delta_{\sl} (\kappa)$ implies $\E \left[ Z^2 \right] = O_d ( \E \left[ Z \right]^2 )$, then combining Eq.~\eqref{eq:Z_to_lower_bound} and Eq.~\eqref{ineq:paley_zygmund} yields
\begin{equation}\label{eq:pos_prob_pure}
	\liminf_{n \to \infty} \P \left( \exists \btheta \in \S^{d-1}, \forall 1 \le i \le n, \left\langle \xx_i, \btheta \right\rangle \ge \kappa \right) > 0.
\end{equation}

To this end we calculate and compare the first and second moments of $Z$. Note that for any $\btheta \in \S^{d-1}$, $\left\langle \btheta, \xx_i \right\rangle \sim_{\iid} \cN (0, 1)$, hence
\begin{equation}\label{eq:1st_moment}
	\E [Z] = \int_{\S^{d-1}} \E \left[ \prod_{i=1}^{n} f \left( \left\langle \btheta, \xx_i \right\rangle \right) \right] \mu(\d \btheta) = \E \left[ f(G) \right]^n = \exp \left( n \log \E \left[ f(G) \right] \right), \quad G \sim \cN (0, 1).
\end{equation}
To calculate $\E [Z^2]$, we first write $Z^2$ as a double integral:
\begin{equation*}
	Z^2=\int_{\S^{d-1} \times \S^{d-1}} \prod_{i=1}^{n} f\left(\left\langle \btheta_1, \xx_i \right\rangle\right) f\left(\left\langle \btheta_2, \xx_i \right\rangle\right) \mu(\d \btheta_1) \mu\left(\d \btheta_2 \right).
\end{equation*}
Next, denote $q = \left\langle \btheta_1, \btheta_2 \right\rangle$, we observe the following two useful facts:
\begin{enumerate}
	\item For $\btheta_1, \btheta_2 \sim_{\iid} \Unif (\S^{d-1})$, the probability density function of $q$ is given as
	\begin{equation*}
		p (q) = \frac{\Gamma\left( d/2 \right)}{\sqrt{\pi} \Gamma\left( (d-1)/2 \right)}\left(1-q^{2}\right)^{(d-3)/2} \bone_{[-1, 1]} (q).
	\end{equation*}

    \item For $1 \le i \le n$,
    \begin{equation*}
    	\left(\left\langle \btheta_1, \xx_i \right\rangle, \left\langle \btheta_2, \xx_i \right\rangle\right)^{\top} \sim_{\iid} \cN \left( \bzero , \begin{bmatrix}
    		1 & q \\
    		q & 1
    	\end{bmatrix} \right).
    \end{equation*}
\end{enumerate}
Therefore, it follows that
\begin{align*}
	\E [Z^2] &= \int_{\S^{d-1} \times \S^{d-1}} \E \left[ \prod_{i=1}^{n} f\left(\left\langle \btheta_1, \xx_i \right\rangle\right) f\left(\left\langle \btheta_2, \xx_i \right\rangle\right) \right] \mu(\d \btheta_1) \mu\left(\d \btheta_2 \right) \\
	&= \int_{-1}^{1} \E_q \left[ f(G_1)f(G_2) \right]^n p(q) \d q = \left( 1 + o_d (1) \right) \int_{-1}^{1} \sqrt{\frac{d}{2 \pi}} \E_q \left[f(G_1) f(G_2) \right]^{n}\left(1-q^{2}\right)^{(d-3)/2} \d q,
\end{align*}
where $\E_q$ denotes the expectation taken under $(G_1, G_2)^\top \sim \cN \left( \bzero , \begin{bmatrix}
    1 & q \\
    q & 1
\end{bmatrix} \right)$, and the last equality is due to $\Gamma (d/2)/\Gamma((d-1)/2) = (1 + o_d (1)) \sqrt{d/2}$, which can be easily derived from Stirling's formula:
\begin{equation*}
	\Gamma(t+1) = \left( 1 + o_t (1) \right) \sqrt{2 \pi t} \left( \frac{t}{e} \right)^t, \quad t \to +\infty.
\end{equation*} 
According to Eq.~\eqref{def:F(q)_and_I(q)}, we can then write
\begin{align*}
	\E [Z^2] &= (1+o_d (1)) \int_{-1}^{1} \sqrt{\frac{d}{2 \pi}} \exp \left( -n F(q) - (d-3) I(q) \right) \d q \\
	&= (1+o_d (1)) \int_{-1}^{1} \sqrt{\frac{d}{2 \pi}} \exp \left( -n \left( F(q) + \frac{d-3}{n} I(q) \right) \right) \d q.
\end{align*}
Since $\lim_{n \to +\infty} n/(d-3) = \delta < \delta_{\sl} (\kappa)$, for any $\delta'$ such that $\delta < \delta' < \delta_{\sl} (\kappa)$, we have
\begin{equation*}
	\E [Z^2] \le (1 + o_d (1)) \int_{-1}^{1} \sqrt{\frac{d}{2 \pi}} \exp \left( -n \left( F(q) + \frac{I(q)}{\delta'} \right) \right) \d q.
\end{equation*}

By Definition~\ref{def:pure_lower}, $ F(q) + I(q)/\delta'$ is uniquely minimized at $q = 0$, and $F''(0) + I''(0)/\delta' > 0$. Utilizing Laplace's method (Chap 4.2 of \cite{de1981asymptotic}), we obtain that
\begin{align}
    \E [Z^2] &\le (1 + o_d (1)) \int_{-1}^{1} \sqrt{\frac{d}{2 \pi}} \exp \left( -n \left( F(q) + \frac{I(q)}{\delta'} \right) \right) \d q \tag*{} \\
	&= (1 + o_d (1)) \sqrt{\frac{2 \pi}{n \left( F''(0) + I''(0)/\delta' \right)}} \sqrt{\frac{d}{2 \pi}} \exp \left( -n \left( F(0) + I(0)/\delta' \right) \right) \tag*{} \\
	&\stackrel{(i)}{=} (1+o_d(1)) \sqrt{\frac{1}{\delta \left( F''(0) + I''(0)/\delta' \right)}} \exp \left( -nF(0) \right) \tag*{}  \\
	\label{eq:2nd_moment}
	&\stackrel{(ii)}{=} (1+o_d(1)) \sqrt{\frac{1}{\delta \left( F''(0) + I''(0)/\delta' \right)}} \exp \left( 2 n \log \E [f(G)] \right), \quad G \sim \cN (0, 1),
\end{align}
where $(i)$ is due to the fact that $I(0) = 0$, and $(ii)$ is due to the definition of $F(q)$. Now combining Eq.~\eqref{eq:1st_moment} and Eq.~\eqref{eq:2nd_moment} gives
\begin{equation*}
	\liminf_{n \to +\infty} \P (Z > 0) \ge \liminf_{n \to +\infty} \frac{\E [Z]^2}{\E [Z^2]} \ge \sqrt{\frac{1}{\delta \left( F''(0) + I''(0)/\delta' \right)}} > 0,
\end{equation*}
thus proving Eq.~\eqref{eq:pos_prob_pure}. Now, we define the following function:
\begin{equation*}
	S_\kappa (\XX) = \inf_{\btheta \in \S^{d-1}} \left\{ \frac{1}{\sqrt{n}} \norm{(\kappa \bone - \XX \btheta)_+}_2 \right\}, \quad \XX \in \R^{n \times d}.
\end{equation*}
Then we see that for $\XX, \XX' \in \R^{n \times d}$,
\begin{align*}
	\left\vert S_\kappa (\XX) - S_\kappa (\XX') \right\vert \le & \sup_{\btheta \in \S^{d-1}} \left\vert \frac{1}{\sqrt{n}} \norm{(\kappa \bone - \XX \btheta)_+}_2 - \frac{1}{\sqrt{n}} \norm{(\kappa \bone - \XX' \btheta)_+}_2 \right\vert \\
	\stackrel{(i)}{\le} & \sup_{\btheta \in \S^{d-1}} \left\{ \frac{1}{\sqrt{n}} \norm{(\kappa \bone - \XX \btheta)_+ - (\kappa \bone - \XX' \btheta)_+}_2 \right\} \\
	\stackrel{(ii)}{\le} & \sup_{\btheta \in \S^{d-1}} \left\{ \frac{1}{\sqrt{n}} \norm{\XX \btheta - \XX' \btheta}_2 \right\} = \frac{1}{\sqrt{n}} \norm{\XX - \XX'}_{\op} \le \frac{1}{\sqrt{n}} \norm{\XX - \XX'}_{\rm F},
\end{align*}
where $(i)$ follows from triangle inequality, and $(ii)$ is due to the $1$-Lipschitzness of $x \mapsto (\kappa - x)_+$. We thus obtain that $S_\kappa (\XX)$ is an $(1/\sqrt{n})$-Lipschitz function of $\XX$. According to Gaussian concentration inequality \cite[Thm. 5.2.2]{vershynin2018high}, we know that $\forall t > 0$,
\begin{equation}\label{eq:pure_noise_concentrate}
	\P \left( \left\vert S_\kappa (\XX) - \E [S_\kappa (\XX)] \right\vert \ge t \right) \le 2 \exp \left( - \frac{n t^2}{2} \right).
\end{equation}
Now we are in position to prove Eq.~\eqref{eq:pure_noise_lower_bound}. From the proof of Lemma~\ref{lem:property_of_e_q}, we know that the function $\Psi(q)$ (which also depends on $\kappa$) in Definition~\ref{def:pure_lower} is continuous in $(\kappa, q)$. As a result, the threshold $\delta_{\rm lb} (\kappa)$ is continuous in $\kappa$. Since $\delta < \delta_{\rm lb} (\kappa)$, we can find some $\eta > 0$ such that $\delta < \delta_{\rm lb} (\kappa + \eta)$, thus leading to
\begin{equation*}
	\liminf_{n \to \infty} \P \left( S_{\kappa + \eta} (\XX) = 0 \right) > 0.
\end{equation*}
Using Eq.~\eqref{eq:pure_noise_concentrate} with $\kappa$ replaced by $\kappa + \eta$, we know that with high probability, $S_{\kappa + \eta} (\XX) = o(\eta)$, i.e., there exists $\btheta \in \S^{d-1}$ such that $\norm{((\kappa+\eta) \bone - \XX \btheta)_+}_2 = o(\eta \sqrt{n})$. Applying the (rounding) Lemma 11 in \cite{alaoui2020algorithmic}, we finally deduce that there exists some $\btheta' \in \S^{d-1}$, such that
\begin{equation*}
	\XX \btheta' \ge (\kappa + \eta - o(\eta)) \bone \ge \kappa \bone.
\end{equation*}
This completes the proof of Eq.~\eqref{eq:pure_noise_lower_bound}. In order to conclude the second part of Theorem~\ref{thm:pure_noise_lower_bound}, it suffices to show that if $\vert \kappa \vert$ is sufficiently large, and 
\begin{equation*}
	\delta < (1 - \veps) \sqrt{2 \pi} |\kappa| \log |\kappa| \exp \left( \frac{\kappa^2}{2} \right),
\end{equation*}
then $ F(q) + I(q)/\delta$ is uniquely minimized at $q = 0$, and $F''(0) + I''(0)/\delta > 0$. To verify these two conditions, we proceed with the following three steps.

\noindent \textbf{Step 1. Simplify the first condition} 

\noindent Recall that $F (q) = - \log e(q)$, therefore the condition ``\textit{$ F(q) + I(q)/\delta$ is uniquely minimized at $q = 0$}" is equivalent to
\begin{align}
	& \forall q \in [-1, 1] \backslash \{ 0 \}, \quad F(q) + \frac{I(q)}{\delta} > F (0) + \frac{I(0)}{\delta} = F(0) \tag*{} \\
	\iff & \forall q \in [-1, 1] \backslash \{ 0 \}, \quad \delta \log \left( \frac{e(q)}{e(0)} \right) < I (q) \tag*{} \\
	\stackrel{(i)}{\impliedby} & \forall q \in [-1, 1] \backslash \{ 0 \}, \quad \delta \frac{e(q) - e(0)}{e(0)} < I (q) \tag*{} \\
	\label{eq:minimum_at_q=0} \iff & \forall q \in [-1, 1] \backslash \{ 0 \}, \quad \frac{\delta}{e(0)} \int_{0}^{q} e'(s) \d s < I(q) = - \frac{1}{2} \log \left( 1 - q^2 \right),
\end{align}
where in $(i)$ we use the inequality $\log x \le x-1$.

\noindent \textbf{Step 2. Prove the second condition} 

\noindent Note that Eq.~\eqref{eq:bound_on_e''(q)} implies $e''(0) \le \vert \kappa \vert^C \exp (- \kappa^2)$, hence we have
\begin{equation*}
	\vert F''(0) \vert = \left\vert \frac{e'(0)^2 - e''(0) e(0)}{e(0)^2} \right\vert \stackrel{(i)}{=} \left\vert \frac{e''(0)}{e(0)}  \right\vert \stackrel{(ii)}{\le} \vert \kappa \vert^C \exp \left( -\kappa^2 \right),
\end{equation*}
where $(i)$ is due to $e'(0) = 0$, $(ii)$ is due to $e(0) = 1 + \breve{o}_{\kappa} (1)$ (Lemma~\ref{lem:property_of_e_q} \textit{(b)}). Now since $1/\delta \ge \vert \kappa \vert^{-C} \exp (-\kappa^2 / 2)$, $I''(0) = 1$, it immediately follows that $\exists M_1 > 0$, $\forall \kappa < -M_1$, $F''(0) + I''(0)/\delta > 0$.

\noindent \textbf{Step 3. Prove the first condition} 

\noindent To show Eq.~\eqref{eq:minimum_at_q=0}, consider the following three cases:

(1) $q \in [-1/2, 1/2] \backslash \{ 0 \}$. Following Eq.~\eqref{eq:bound_on_e''(q)}, and using the facts $e(0) = 1 + \breve{o}_{\kappa} (1)$, $e'(0) = 0$ from Lemma~\ref{lem:property_of_e_q} \textit{(b)}, we obtain the estimate below:
\begin{align*}
	& \frac{\delta}{e(0)} \int_{0}^{q} e'(s) \d s = (1 + \breve{o}_{\kappa} (1)) \delta \int_{0}^{q} e'(s) \d s = (1+\breve{o}_{\kappa} (1)) \delta \int_{0}^{q} \d s \int_{0}^{s} e''(u) \d u \\
	\le & \delta \vert \kappa \vert^C \exp \left( - \frac{\kappa^2}{1+q} \right) \int_{0}^{q} \d s \int_{0}^{s} \d u \\
	\le & (1 - \veps) \sqrt{2 \pi} |\kappa|^{C+1} \log |\kappa| \exp \left( \frac{\kappa^2}{2} \right) \exp \left( - \frac{2 \kappa^2}{3} \right)\frac{q^2}{2} \\
	\le & \vert \kappa \vert^C \exp \left( - \frac{\kappa^2}{6} \right) \frac{q^2}{2} = \breve{o}_{\kappa} (1) q^2.
\end{align*}
Since $I (q) \ge q^2/2$ for all $q \in [-1, -1]$, we conclude that $\exists M_2 > 0$, as if $\kappa < -M_2$,
\begin{equation*}
	\frac{\delta}{e(0)} \int_{0}^{q} e'(s) \d s \le \breve{o}_{\kappa} (1) q^2 < I (q), \quad q \in \left[ - \frac{1}{2}, \frac{1}{2} \right] \big\backslash \{ 0 \}.
\end{equation*}

(2) $q \in [-1, -1/2]$. In this case, with the aid of Eq.~\eqref{eq:asymptotics_for_e'(q)} one gets that
\begin{align*}
	\frac{\delta}{e(0)} \int_{0}^{q} e'(s) \d s = & (1 + \breve{o}_{\kappa} (1))\delta \int_{0}^{q} \left( \frac{1+\breve{o}_{\kappa}(1)}{2 \pi \sqrt{1-s^{2}}} \exp \left(-\frac{\kappa^{2}}{1+s}\right)+\breve{O}_{\kappa} \left(\exp (-\kappa^2) \right) \right) \d s \\
	\stackrel{(i)}{\le} & C \delta \left( \int_{0}^{\vert q \vert} \frac{\d s}{\sqrt{1 - s^2}} + \vert q \vert \right) \exp \left( - \kappa^2 \right) \stackrel{(ii)}{\le} C \delta \exp \left( - \kappa^2 \right) \\
	\le & C |\kappa| \log |\kappa| \exp \left( - \frac{\kappa^2}{2} \right) = \breve{o}_{\kappa} (1),
\end{align*}
where $(i)$ holds since $\exp ( - \kappa^2/(1+s) ) \le \exp ( - \kappa^2 )$ when $s \le 0$, and $(ii)$ follows from the fact
\begin{equation*}
	\int_{0}^{\vert q \vert} \frac{\d s}{\sqrt{1 - s^2}} \le \int_{0}^{1} \frac{\d s}{\sqrt{1 - s^2}} = \frac{\pi}{2}.
\end{equation*}
Therefore, we can find some $M_3 > 0$ such that when $\kappa < - M_3$,
\begin{equation*}
	\frac{\delta}{e(0)} \int_{0}^{q} e'(s) \d s \le \breve{o}_{\kappa} (1) < I \left( - \frac{1}{2} \right) \le I(q), \quad q \in \left[ -1, - \frac{1}{2} \right].
\end{equation*}

(3) $q \in [1/2, 1]$. On the one hand, we observe that Eq.~\eqref{eq:asymptotics_for_e'(q)} implies that $e'(q) \ge 0$ when $\vert \kappa \vert$ is large enough, and also note that $e(1) - e(0) = (1 + \breve{o}_{\kappa} (1)) \Phi (\kappa)$ (Lemma~\ref{lem:property_of_e_q} \textit{(b)}), thus leading to
\begin{align*}
	& \frac{\delta}{e(0)} \int_{0}^{q} e'(s) \d s \le \delta \frac{e(1) - e(0)}{e(0)} = (1 + \breve{o}_{\kappa} (1)) \delta \Phi (\kappa) \\
	\le & (1 + \breve{o}_{\kappa} (1)) (1 - \veps) \sqrt{2 \pi} |\kappa| \log |\kappa| \exp \left( \frac{\kappa^2}{2} \right) \Phi (\kappa) \\
	\stackrel{(i)}{\le} & (1 + \breve{o}_{\kappa} (1)) (1 - \veps) \log |\kappa| < \left( 1 - \frac{\veps}{2} \right) \log \vert \kappa \vert \\
	\le & - \frac{1}{2} \log \left( 1 - q^2 \right) = I (q),
\end{align*}
if $\vert \kappa \vert$ is sufficiently large and $1 - q \le \vert \kappa \vert^{-2 + \veps}/2$. Here, $(i)$ is due to the inequality $\Phi(\kappa) \le \vert \kappa \vert \phi(\kappa)$. On the other hand, if $1 - q \ge \vert \kappa \vert^{-2 + \veps}/2$, then for any $s \le q$, we have
\begin{equation*}
	- \frac{\kappa^2}{1+s} \le - \frac{\kappa^2}{1+q} = - \frac{\kappa^2}{2} - \frac{(1-q) \kappa^2}{2(1+q)} \le - \frac{\kappa^2}{2} - \frac{\vert \kappa \vert^\veps}{8},
\end{equation*}
which further implies (use Eq.~\eqref{eq:asymptotics_for_e'(q)})
\begin{equation*}
	e'(s) \le \frac{1+\breve{o}_{\kappa}(1)}{2 \pi \sqrt{1-s^{2}}} \exp \left( -\frac{\kappa^2}{2} - \frac{\vert \kappa \vert^\veps}{8} \right)+\breve{O}_{\kappa}\left(\exp (-\kappa^2) \right).
\end{equation*}
Hence it follows that
\begin{align*}
	\frac{\delta}{e(0)} \int_{0}^{q} e'(s) \d s \le & (1 + \breve{o}_{\kappa}(1)) \delta \left( \frac{1 + \breve{o}_{\kappa} (1)}{2 \pi} \exp \left( -\frac{\kappa^2}{2} - \frac{\vert \kappa \vert^\veps}{8} \right) \int_{0}^{q} \frac{\d s}{\sqrt{1-s^2}}  + \breve{O}_{\kappa} \left(\exp (-\kappa^2) \right) q \right) \\
	\le & (1 + \breve{o}_{\kappa} (1)) (1 - \veps) \sqrt{2 \pi} |\kappa| \log |\kappa| \exp \left( \frac{\kappa^2}{2} \right) \left( \frac{1}{4} \exp \left( -\frac{\kappa^2}{2} - \frac{\vert \kappa \vert^\veps}{8} \right) + \breve{O}_{\kappa} \left(\exp (-\kappa^2) \right) \right) \\
	\le & \vert \kappa \vert^C \exp \left( - \frac{\vert \kappa \vert^\veps}{8} \right) = \breve{o}_{\kappa} (1) < - \frac{1}{2} \log \left( 1-q^2 \right) = I(q), \quad q \in \left[ \frac{1}{2}, 1 - \frac{\vert \kappa \vert^{-2 + \veps}}{2} \right].
\end{align*} 
To conclude, there exists some $M_4 > 0$, such that for all $\kappa < - M_4$, 
\begin{equation*}
	\frac{\delta}{e(0)} \int_{0}^{q} e'(s) \d s < I(q), \quad q \in \left[ \frac{1}{2}, 1 \right].
\end{equation*}
Now we choose $M = \max_{1 \le j \le 4} M_j$ and $\underline{\kappa} = - M$. Then for $\kappa < \underline{\kappa}$, and
\begin{equation*}
\delta < (1 - \veps) \sqrt{2 \pi} |\kappa| \log |\kappa| \exp \left( \frac{\kappa^2}{2} \right),
\end{equation*}
Eq.~\eqref{eq:minimum_at_q=0} will be valid and $F''(0) + I''(0)/\delta > 0$. The conclusion of Theorem~\ref{thm:pure_noise_lower_bound} follows naturally.

\subsection{Phase transition upper bound: Proof of Theorem~\ref{thm:pure_noise_upper_bound}}\label{sec:pure-upper}

To begin with, let us define the following random variable:
\begin{equation*}
	\xi_{n, \kappa} = \min_{\norm{\btheta}_2 = 1} \max_{\norm{\blambda}_2 = 1, \blambda \ge \bzero} \frac{1}{\sqrt{d}} \blambda^\top \left( \kappa \bone - \XX \btheta \right).
\end{equation*}
Note that    	
\begin{align*}
	& \P \left( \xi_{n, \kappa} \le 0 \right) = \P \left( \exists \btheta \in \S^{d-1}, \max_{\norm{\blambda}_2 = 1, \blambda \ge \bzero} \frac{1}{\sqrt{d}} \blambda^\top \left( \kappa \bone - \XX \btheta \right) \le 0 \right) \\
	= & \P \left( \exists \btheta \in \S^{d-1}, \kappa \bone - \XX \btheta \le \bzero \right) = \P \left( \exists \btheta \in \S^{d-1}, \forall 1 \le i \le n, \left\langle \xx_i, \btheta \right\rangle \ge \kappa \right),
\end{align*}
hence it suffices to prove $\P \left( \xi_{n, \kappa} \le 0 \right) \to 0$ under the assumption of Theorem~\ref{thm:pure_noise_upper_bound}. To this end we will use a modified version of Gordon's comparison theorem, namely Theorem 2.3 of \cite{gordon1985some}. We state its adaption to our case below for readers' convenience:

\begin{thm}[Theorem 2.3 of \cite{gordon1985some}]
	Assume $\psi: \R \to \R$ is an increasing function, then we have
    \begin{align}\label{eq:modified_Gordon}
        \begin{split}
        	& \E \left[ \min_{\norm{\btheta}_2 = 1} \max_{\norm{\blambda}_2 = 1, \blambda \ge \bzero} \psi \left( \frac{1}{\sqrt{d}} \blambda^\top \left( \kappa \bone - \XX \btheta \right) + \frac{1}{\sqrt{d}} \norm{\blambda}_2 \norm{\btheta}_2 z \right) \right] \\
	    \ge & \E \left[ \min_{\norm{\btheta}_2 = 1} \max_{\norm{\blambda}_2 = 1, \blambda \ge \bzero} \psi \left( \frac{1}{\sqrt{d}} \blambda^\top \left( \kappa \bone - \norm{\btheta}_2 \hh\right) + \frac{1}{\sqrt{d}} \norm{\blambda}_2 \bgg^\top \btheta \right) \right],
        \end{split}
    \end{align}
    where $\bgg \sim \cN (\bzero, \bI_d)$, $\hh \sim \cN (\bzero, \bI_n)$ and $z \sim \cN (0, 1)$ are mutually independent, and further independent of the standard Gaussian matrix $\XX$.
\end{thm}

Now for some $c > 0$ which may depend on $\kappa$ (to be determined later), if we choose $\psi (x) = - \exp (-cnx)$, then Eq.~\eqref{eq:modified_Gordon} becomes
\begin{align}
    & - \E \left[ \exp \left( - c n \xi_{n, \kappa} \right) \right] \E \left[ \exp \left( - \frac{c n z}{\sqrt{d}} \right) \right] \tag*{} \\
    = & - \E \left[ \exp \left( -c n \left( \xi_{n, \kappa} + \frac{z}{\sqrt{d}} \right) \right) \right] \tag*{} \\
	\ge & - \E \left[ \exp \left( -c n \left( \min_{\norm{\btheta}_2 = 1} \frac{1}{\sqrt{d}} \bgg^\top \btheta + \max_{\norm{\blambda}_2 = 1, \blambda \ge \bzero} \frac{1}{\sqrt{d}} \blambda^\top \left( \kappa \bone - \hh \right) \right) \right) \right] \tag*{} \\
	\label{eq:exponential_Gordon}
	=& - \E \left[ \exp \left( - \frac{cn}{\sqrt{d}} \min_{\norm{\btheta}_2 = 1} \bgg^\top \btheta \right) \right] \E \left[ \exp \left( - \frac{cn}{\sqrt{d}} \max_{\norm{\blambda}_2 = 1, \blambda \ge \bzero} \blambda^\top \left( \kappa \bone - \hh \right) \right) \right],
\end{align}
where the first and last equality are due to the independence between $\bgg, \hh, z, \XX$. Since
\begin{equation*}
    \min_{\norm{\btheta}_2 = 1} \bgg^\top \btheta = - \norm{\bgg}_2, \quad 
	\max_{\norm{\blambda}_2 = 1, \blambda \ge \bzero} \blambda^\top \left( \kappa \bone - \hh \right) = 
	\begin{cases} 
	\displaystyle \norm{\left( \kappa \bone - \hh \right)_+}_2 , & \min_{1 \le i \le n} h_i \le \kappa \\
    \displaystyle \max_{1 \le i \le n} \left( \kappa - h_i \right), & \min_{1 \le i \le n} h_i > \kappa
    \end{cases},
\end{equation*}
it follows that
\begin{equation*}
	\E \left[ \exp \left( - \frac{cn}{\sqrt{d}} \min_{\norm{\btheta}_2 = 1} \bgg^\top \btheta \right) \right] = \E \left[ \exp \left( \frac{cn}{\sqrt{d}} \norm{\bgg}_2 \right) \right],
\end{equation*}
and
\begin{align}
	& \E \left[ \exp \left( - \frac{cn}{\sqrt{d}} \max_{\norm{\blambda}_2 = 1, \blambda \ge \bzero} \blambda^\top \left( \kappa \bone - \hh \right) \right) \right] \tag*{} \\
	= & \E \left[ \exp \left( - \frac{cn}{\sqrt{d}} \norm{\left( \kappa \bone - \hh \right)_+}_2 \right) \bone \left\{ \min_{1 \le i \le n} h_i \le \kappa \right\} \right] + \E \left[ \exp \left( - \frac{cn}{\sqrt{d}} \max_{1 \le i \le n} \left( \kappa - h_i \right) \right) \bone \left\{ \min_{1 \le i \le n} h_i > \kappa
        \right\} \right] \tag*{} \\
    \label{eq:term_I_and_II}
    \le & \underbrace{\E \left[ \exp
        \left( - \frac{cn}{\sqrt{d}} \norm{\left( \kappa \bone - \hh
        \right)_+}_2 \right) \right]}_{\text{I}} + \underbrace{\E \left[ \exp \left( - \frac{cn}{\sqrt{d}} \max_{1 \le i \le n} \left( \kappa - h_i \right) \right) \bone \left\{ \min_{1 \le i \le n} h_i > \kappa
        \right\} \right] }_{\text{II}}.
\end{align}
According to Markov's inequality, we obtain that $\P(\xi_{n, \kappa} \le 0) \le \E [\exp(- c n \xi_{n, \kappa})]$. Combining this estimate with Eq.~\eqref{eq:modified_Gordon} and Eq.~\eqref{eq:exponential_Gordon} yields
\begin{align}
	& \frac{1}{n} \log \P \left( \xi_{n, \kappa} \le 0 \right) \le \frac{1}{n} \log \E \left[ \exp \left(- c n \xi_{n, \kappa} \right) \right] \tag*{} \\
	\le & - \frac{1}{n} \log \E \left[ \exp \left( - \frac{c n z}{\sqrt{d}} \right) \right] + \frac{1}{n} \log \E \left[ \exp \left( - \frac{cn}{\sqrt{d}} \min_{\norm{\btheta}_2 = 1} \bgg^\top \btheta \right) \right] \tag*{} \\
	&+ \frac{1}{n} \log \E \left[ \exp \left( - \frac{cn}{\sqrt{d}} \max_{\norm{\blambda}_2 = 1, \blambda \ge \bzero} \blambda^\top \left( \kappa \bone - \hh \right) \right) \right] \tag*{} \\
	\label{eq:xi_non_positive_probability_bound}
	\le & - \frac{c^2 n}{2 d} + \frac{1}{n} \log \E \left[ \exp \left( \frac{cn}{\sqrt{d}} \norm{\bgg}_2 \right) \right] + \frac{1}{n} \log (\text{I} + \text{II}).
\end{align}
In what follows we focus on the right-hand side of the above inequality, the argument is divided into three parts.

\noindent \textbf{Part 1. Find an upper bound for $\text{II}$} 

\noindent Note that $- \max_{1 \le i \le n} (\kappa - h_i) = \min_{1 \le i \le n} (h_i - \kappa )$ has probability density
\begin{equation*}
	p(u) = n \left( 1 - \Phi \left( u + \kappa \right) \right)^{n-1} \phi \left( u + \kappa \right),
\end{equation*}
thus leading to
\begin{align}\label{eq:term_II_estimate}
\begin{split}
	\text{II} =& \int_{0}^{+ \infty} n \left( 1 - \Phi \left( u + \kappa \right) \right)^{n-1} \phi \left( u + \kappa \right) \exp \left( \frac{cnu}{\sqrt{d}} \right) \d u \\
	=& n \int_{0}^{+ \infty} \exp \left( (n-1) \log \left( 1 - \Phi \left( u + \kappa \right) \right) + \frac{cnu}{\sqrt{d}} \right) \phi \left( u + \kappa \right) \d u \\
	\stackrel{(i)}{\le} & n \int_{0}^{+ \infty} \exp \left( (n-1) \left( \log \left( 1 - \Phi(\kappa) \right) - \frac{\phi(\kappa) u}{1 - \Phi(\kappa)} \right) + \frac{cnu}{\sqrt{d}} \right) \phi \left( u + \kappa \right) \d u \\
	\le & n \left( 1 - \Phi(\kappa) \right)^{n-1} \int_{0}^{+\infty} \exp \left( - \left( \frac{(n-1) \phi(\kappa)}{1 - \Phi(\kappa)} - \frac{cn}{\sqrt{d}} \right) u \right) \phi \left( u + \kappa \right) \d u \\
	\stackrel{(ii)}{=} & o_d \left( n \left( 1 - \Phi(\kappa) \right)^{n-1} \right),
	\end{split}
\end{align}
where $(i)$ is due to the fact that $u \mapsto \log (1 - \Phi (u + \kappa))$ is concave, and
\begin{equation*}
	\frac{\d}{\d u} \log \left( 1 - \Phi (u + \kappa) \right) \Big\vert_{u = 0} = - \frac{\phi(\kappa)}{1 - \Phi(\kappa)},
\end{equation*} 
$(ii)$ is due to our assumption $n/d \to \delta$, which further implies that
\begin{equation*}
	\frac{(n-1) \phi(\kappa)}{1 - \Phi(\kappa)} - \frac{cn}{\sqrt{d}} \to + \infty \implies \int_{0}^{+\infty} \exp \left( - \left( \frac{(n-1) \phi(\kappa)}{1 - \Phi(\kappa)} - \frac{cn}{\sqrt{d}} \right) u \right) \phi \left( u + \kappa \right) \d u \to 0.
\end{equation*}

\noindent \textbf{Part 2. Calculate the limit of other terms as $n \to +\infty$}

\noindent We use the following consequence of Varadhan's Integral lemma (Theorem 4.3.1 of \cite{Dembo_2010}):
\begin{lem}\label{lem:varadhan}
	Let $X_i$ be $\iid$ non-negative random variables, $S_n = X_1 + \cdots + X_n$. Assume $\psi (t) = \E [\exp(t X_1)]$ is finite for some $t > 0$, then we have
	\begin{equation*}
	\lim_{n \to +\infty} \frac{1}{n} \log \E \left[ \exp \left( c \sqrt{n S_n} \right) \right] = 
	\begin{cases} 
	\displaystyle \inf_{t > 0} \left( \frac{c^2}{4 t} + \log \psi (t) \right), \ c > 0 \\
    \displaystyle \sup_{t < 0} \left( \frac{c^2}{4 t} + \log \psi (t) \right), \ c < 0
    \end{cases}.
	\end{equation*}
\end{lem}
The proof of Lemma~\ref{lem:varadhan} will be provided in Appendix~\ref{sec:append-pure-lemma}. Since
\begin{equation*}
	\E \left[ \exp \left( t g_1^2 \right) \right] = \frac{1}{\sqrt{1 - 2t}}, \quad 0 < t < \frac{1}{2},
\end{equation*}
it follows from the above lemma that
\begin{align}
	& \lim_{n \to +\infty} \frac{1}{n} \log \E \left[ \exp \left( \frac{cn}{\sqrt{d}} \norm{\bgg}_2 \right) \right] \tag*{} \\
	= & \frac{1}{\delta} \lim_{d \to + \infty} \frac{1}{d} \log \E \left[ \exp \left( c \delta \sqrt{d \sum_{i=1}^{d} g_i^2} \right) \right] \tag*{} \\
    = & \frac{1}{\delta} \inf_{t > 0} \left( \frac{c^2 \delta^2}{4 t} + \log \left( \frac{1}{\sqrt{1 - 2t}} \right) \right) \tag*{} \\
    \label{eq:varadhan_1st_limit}
    = & \frac{c^2 \delta + c \sqrt{c^2 \delta^2 + 4}}{4} - \frac{1}{\delta} \log \left( \frac{\sqrt{c^2 \delta^2 +4} - c \delta}{2} \right),
\end{align}
where the last equality just results from direct computation. Similarly, we have ($\text{I}$ is defined in Eq.~\eqref{eq:term_I_and_II})
\begin{align}
	& \lim_{n \to +\infty} \frac{1}{n} \log \left( \text{I} \right) = \lim_{n \to +\infty} \frac{1}{n} \log \E \left[ \exp \left( - \frac{cn}{\sqrt{d}} \norm{\left( \kappa \bone - \hh \right)_+}_2 \right) \right] \tag*{} \\
	=& \lim_{n \to +\infty} \frac{1}{n} \log \E \left[ \exp \left( - c \sqrt{\delta} \sqrt{n \sum_{i=1}^{n} \left( \kappa - h_i \right)_{+}^2} \right) \right] \tag*{} \\
	\label{eq:term_I_limit}
	=& \sup_{t < 0} \left( \frac{c^2 \delta}{4 t} + \log \psi_{\kappa} (t) \right) = - \inf_{u > 0} \left( \frac{c^2 \delta}{4 u} - \log \psi_{\kappa} (-u) \right),
\end{align}
where $\psi_{\kappa} (-u) = \E [ \exp (-u (\kappa - h)_+^2)]$, with $h \sim \cN (0, 1)$.

\noindent \textbf{Part 3. Put everything together}
	
\noindent Note that Eq.~\eqref{eq:term_II_estimate} implies ($\text{II}$ is defined in Eq.~\eqref{eq:term_I_and_II})
\begin{align*}
	& \limsup_{n \to +\infty} \frac{1}{n} \log \left( \text{II} \right) \le \log \left( 1 - \Phi(\kappa) \right) = - \lim_{u \to +\infty} \left( \frac{c^2 \delta}{4 u} - \log \psi_{\kappa} (-u) \right) \\
		\le & - \inf_{u > 0} \left( \frac{c^2 \delta}{4 u} - \log \psi_{\kappa} (-u) \right) = \lim_{n \to +\infty} \frac{1}{n} \log \left( \text{I} \right).
\end{align*}
Combining this calculation with Eq.~\eqref{eq:term_I_limit} leads to
\begin{equation}\label{eq:varadhan_2nd_limit}
	\lim_{n \to +\infty} \frac{1}{n} \log (\text{I} + \text{II}) = \lim_{n \to +\infty} \frac{1}{n} \log \left( \text{I} \right) = - \inf_{u > 0} \left( \frac{c^2 \delta}{4 u} - \log \psi_{\kappa} (-u) \right).
	\end{equation}
We can now upper bound $\P (\xi_{n, \kappa} \le 0)$. Using Eq.~\eqref{eq:xi_non_positive_probability_bound}, Eq.~\eqref{eq:varadhan_1st_limit} and Eq.~\eqref{eq:varadhan_2nd_limit} we conclude that for all $ c > 0$,
\begin{equation*}
	\limsup_{n \to +\infty} \frac{1}{n} \log \P \left( \xi_{n, \kappa} \le 0 \right) \le - \frac{c^2 \delta}{2} + \frac{c^2 \delta + c \sqrt{c^2 \delta^2 + 4}}{4} - \frac{1}{\delta} \log \left( \frac{\sqrt{c^2 \delta^2 +4} - c \delta}{2} \right) - \inf_{u > 0} \left( \frac{c^2 \delta}{4 u} - \log \psi_{\kappa} (-u) \right).
\end{equation*}

Therefore, if there exists $c > 0$ such that the right-hand side of the above inequality is negative, then $\P (\xi_{n, \kappa} \le 0)$ would be exponentially small, and the first part of Theorem~\ref{thm:pure_noise_lower_bound} follows immediately. After dividing both sides by $c$ and making a change of variable $c \mapsto c/\delta$ here, we see that it suffices to show
\begin{align}
	& \exists c > 0, \frac{\sqrt{c^2+4} - c}{4} - \frac{1}{c} \log \frac{\sqrt{c^2+4} - c}{2} - \inf_{u>0} \left( \frac{c}{4u} - \frac{\delta}{c} \log \psi_{\kappa} (-u) \right) < 0 \tag*{} \\
	\stackrel{u = ct}{\iff}& \exists c > 0, \frac{1}{\sqrt{c^2+4}+c} + \frac{1}{c} \log \frac{\sqrt{c^2+4} + c}{2} - \inf_{t>0} \left( \frac{1}{4t} - \frac{\delta}{c} \log \psi_{\kappa} (-ct) \right) < 0 \tag*{} \\
	\label{eq:pure_noise_final_minimax}
	\iff & \exists c>0, \ \text{s.t.} \ \forall t>0, \frac{1}{\sqrt{c^2+4}+c} + \frac{1}{c} \log \frac{\sqrt{c^2+4} + c}{2} < \frac{1}{4t} - \frac{\delta}{c} \log \psi_{\kappa} (-ct),
\end{align}
where (by integration by part)
\begin{equation}\label{eq:expression_for_psi_kappa}
    \psi_{\kappa} (-u) = \E \left[ \exp \left(-u \left( \kappa - h \right)_+^2 \right) \right] = \Phi \left( \vert \kappa \vert \right) + \frac{\exp \left( - u \vert \kappa \vert^2/(1+2u) \right)}{\sqrt{1+2u}} \left( 1 - \Phi \left( \frac{\vert \kappa \vert}{\sqrt{1+2u}} \right) \right).
\end{equation}
According to Definition~\ref{def:pure_upper}, if $\delta > \delta_{\mathrm u} (\kappa)$, then Eq.~\eqref{eq:pure_noise_final_minimax} holds and we complete the proof of the first part of Theorem~\ref{thm:pure_noise_upper_bound}. Now we are in position to finish the proof. Assume $\vert \kappa \vert$ is sufficiently large, and
\begin{equation*}
	\delta > (1 + \veps) \sqrt{2 \pi} \vert \kappa \vert \log \vert \kappa \vert \exp \left( \frac{\kappa^2}{2} \right).
\end{equation*}
 Choose $\eta > 0$ such that $(1+\eta)^2 < 1 + \veps$, and choose $c > 0$ such that
\begin{equation*}
    \frac{c^2}{4 (1+\eta) \log c} = \vert \kappa \vert^{2(1+\eta)},
\end{equation*}
then for large $\vert \kappa \vert$, we must have $\log c = (1 + \breve{o}_{\kappa} (1)) (1+\eta) \log \vert \kappa \vert$ and
\begin{equation*}
	\frac{1}{\sqrt{c^2+4}+c} + \frac{1}{c} \log \frac{\sqrt{c^2+4} + c}{2} < (1+\eta) \frac{\log c}{c}.
\end{equation*}
Consider the following two situations:

(1) $t \le c/(4 (1+\eta) \log c)$, then we get that
\begin{equation*}
	\frac{1}{4t} \ge (1+\eta) \frac{\log c}{c} > \frac{1}{\sqrt{c^2+4}+c} + \frac{1}{c} \log \frac{\sqrt{c^2+4} + c}{2},
\end{equation*}
therefore, Eq.~\eqref{eq:pure_noise_final_minimax} holds.

(2) $t > c/(4 (1+\eta) \log c)$. Based on Eq.~\eqref{eq:expression_for_psi_kappa}, we obtain that
\begin{align*}
	1 - \psi_{\kappa} \left( - \vert \kappa \vert^{2(1+\eta)} \right) = & 1 - \Phi \left( \vert \kappa \vert \right) - \frac{\exp \left( - \vert \kappa \vert^{2(1+\eta)} \vert \kappa \vert^2/\left(1+2\vert \kappa \vert^{2(1+\eta)} \right) \right)}{\sqrt{1+2\vert \kappa \vert^{2(1+\eta)}}} \left( 1 - \Phi \left( \frac{\vert \kappa \vert}{\sqrt{1+2\vert \kappa \vert^{2(1+\eta)}}} \right) \right) \\
	=& 1 - \Phi \left( \vert \kappa \vert \right) - \breve{O}_{\kappa} \left(\frac{\exp \left( - \vert \kappa \vert^2/2 \right)}{\vert \kappa \vert^{1+\eta}} \exp \left( - \frac{\vert \kappa \vert^2}{2 \left( 1 + 2 \vert \kappa \vert^{2 (1+\eta)} \right)} \right) \right) \\
	=& 1 - \Phi \left( \vert \kappa \vert \right) - \breve{o}_{\kappa} \left(\frac{\exp \left( - \vert \kappa \vert^2/2 \right)}{\vert \kappa \vert} \right) \stackrel{}{=} (1+\breve{o}_{\kappa}(1)) \frac{\exp(- \vert \kappa \vert^2 /2)}{\sqrt{2 \pi} \vert \kappa \vert}.
\end{align*}
We also note that
\begin{equation*}
	1 - \Phi(\kappa) \le \psi_{\kappa} (-u) \le 1 \implies 1 - \psi_{\kappa} (-u) = \breve{o}_{\kappa} (1) \implies - \log \psi_{\kappa} (-u) = (1+\breve{o}_{\kappa} (1)) (1 - \psi_{\kappa} (-u)).
\end{equation*}    
Now by our assumption, $\delta > (1 + \veps) \sqrt{2 \pi} |\kappa| \log |\kappa| \exp \left( \kappa^2/2 \right)$, thus leading to	\begin{align*}
	& - \frac{\delta}{c} \log \psi_{\kappa} (-ct) > - \frac{\delta}{c} \log \psi_{\kappa} \left( - \frac{c^2}{4 (1+\eta) \log c} \right) = - \frac{\delta}{c} \log \psi_{\kappa} \left( - \vert \kappa \vert^{2(1+\eta)} \right) \\
	=& (1+\breve{o}_{\kappa}(1)) \frac{\delta}{c} \left( 1 - \psi_{\kappa} \left( - \vert \kappa \vert^{2(1+\eta)} \right) \right) = (1+\breve{o}_{\kappa}(1)) \frac{\delta}{c}  \frac{\exp(- \vert \kappa \vert^2 /2)}{\sqrt{2 \pi} \vert \kappa \vert} \\
	> & (1+\breve{o}_{\kappa}(1)) \frac{(1+\veps) \log \vert \kappa \vert}{c} \ge (1+\breve{o}_{\kappa}(1)) \frac{(1+\veps) \log c}{(1+\eta)c},
\end{align*}
where the last inequality is due to $\log c = (1 + \breve{o}_{\kappa} (1)) (1+\eta) \log \vert \kappa \vert$. Since $(1 + \veps)/(1+\eta) > 1 + \eta$, for $\vert \kappa \vert$ large enough we know that the right-hand side of the above inequality is greater than
\begin{equation*}
	(1 +\eta) \frac{\log c}{c} > \frac{1}{\sqrt{c^2+4}+c} + \frac{1}{c} \log \frac{\sqrt{c^2+4} + c}{2},
\end{equation*}
which implies Eq.~\eqref{eq:pure_noise_final_minimax}.

In conclusion, there exists a $\underline{\kappa} < 0$, such that if $\kappa < \underline{\kappa}$ and
\begin{equation*}
	\delta > (1 + \veps) \sqrt{2 \pi} |\kappa| \log |\kappa| \exp \left( \frac{\kappa^2}{2} \right),
\end{equation*}
then we have Eq.~\eqref{eq:pure_noise_final_minimax} and as a consequence, Eq.~\eqref{eq:pure_noise_upper_bound} holds.

\subsection{Additional proofs}\label{sec:append-pure-lemma}

\begin{proof}[\bf Proof of Lemma~\ref{lem:property_of_e_q}]
(a): Denote the Mills ratio of standard normal distribution by $R(u) = (1 - \Phi(u)) / \phi(u)$, and define an auxiliary function $A_\kappa(c) = c - 1/R(\kappa + c)$, then we have
\begin{equation*}
	A_\kappa' (c) = 1 + \frac{R'(\kappa + c)}{R(\kappa + c)^2} = 1 + \frac{(\kappa + c) R(\kappa + c) - 1}{R(\kappa + c)^2} =  \frac{R(\kappa + c)^2 + (\kappa + c) R(\kappa + c) - 1}{R(\kappa + c)^2},
\end{equation*}
since $R'(u) = u R(u) - 1$. According to \cite{birnbaum1942inequality}, we have for all $u \in \R$,
\begin{equation*}
	R(u) > \frac{\sqrt{u^2 + 4} - u}{2} \implies R(u)^2 + u R(u) - 1 > 0,
\end{equation*}
thus leading to $A_\kappa'(c) > 0$. Hence $A_\kappa (c)$ is increasing. Note that $A_\kappa (0) < 0$.
Choose $c > \vert \kappa \vert + 1/\vert \kappa \vert$, we deduce that
\begin{equation*}
	A_{\kappa} (c) \stackrel{(i)}{>} c -  \frac{(\kappa + c)^2 + 1}{\kappa + c} = \frac{c \vert \kappa \vert - \kappa^2 - 1}{\kappa + c} > 0,
\end{equation*}
where in $(i)$ we use the Gaussian tail bound:
\begin{equation*}
	R(x) = \frac{1 - \Phi(x)}{\phi (x)} >\frac{x}{ x^2 + 1}.
\end{equation*}
The existence and uniqueness of $c_*$ is thus established. For sufficiently large $\vert \kappa \vert$, one has
\begin{equation*}
	A_\kappa \left( \frac{1}{\kappa^2} \right) = \frac{1}{\kappa^2} - (1+\breve{o}_{\kappa} (1)) \phi (\kappa) > 0,
\end{equation*}
which implies $ c_* \in (0, 1/\kappa^2)$. Furthermore,
\begin{equation*}
	c_* = \frac{\phi(\kappa + c_*)}{1 - \Phi (\kappa + c_*)} = (1 + \breve{o}_{\kappa} (1)) \phi (\kappa) = \breve{O}_{\kappa} \left( \exp \left( - \frac{\kappa^2}{2} \right) \right).
\end{equation*}

\noindent (b) \& (c): We first derive the expression for $e'(q)$, namely Eq.~\eqref{eq:asymptotics_for_e'(q)}. Let $G, G' \sim_{\iid} \cN(0, 1)$, then we have
\begin{equation}\label{eq:e(q)_stein}
\begin{split}
	e'(q) = & \frac{\d}{\d q} \E_q [f(G_1) f(G_2)] = \frac{\d}{\d q} \E \left[ f(G) f (q G + \sqrt{1 - q^2} G') \right] \\
	= & \E \left[ f(G) f' (q G + \sqrt{1 - q^2} G') \left( G - \frac{q}{\sqrt{1 - q^2}} G' \right) \right] \\
	\stackrel{(i)}{=} & \E \left[ f' (G) f' (q G + \sqrt{1 - q^2} G') \right] + q \E \left[ f(G) f'' (q G + \sqrt{1 - q^2} G') \right] \\
	& - \frac{q}{\sqrt{1 - q^2}} \E \left[ f(G) f'' (q G + \sqrt{1 - q^2} G') \cdot \sqrt{1-q^2} \right] \\
	= & \E \left[ f' (G) f' (q G + \sqrt{1 - q^2} G') \right] = \E_q [f'(G_1) f'(G_2)],
	\end{split}
\end{equation}
where $(i)$ is due to Stein's identity. Note that $f$ is weakly differentiable, with
\begin{equation*}
	f'(x) = -c_* \exp(-c_*x) \bone \{ x \ge \kappa \} + \exp (-c_* \kappa) \delta_{\kappa} (x) = -c_* f(x) + \exp (-c_* \kappa) \delta_{\kappa} (x),
\end{equation*}
where $\delta_{\kappa}$ denotes the Dirac's delta measure at $x = \kappa$. It then follows that
\begin{align*}
	e'(q) = & c_*^2 \E_q \left[ f(G_1) f(G_2) \right] - 2 c_* \exp(- c_* \kappa) \E_q \left[ f(G_1) \delta_{\kappa} (G_2) \right] + \exp(-2 c_* \kappa) \E_q \left[ \delta_{\kappa} (G_1) \delta_{\kappa} (G_2) \right] \\
	= & c_*^2 e(q) - \frac{c_* \exp(- c_* \kappa)}{\pi \sqrt{1 - q^2}} \int_{\kappa}^{+\infty} \exp \left( - \frac{1}{2 (1 - q^2)} (x^2 + \kappa^2 - 2 q \kappa x) - c_*x \right) \d x + \frac{\exp(- 2 c_* \kappa)}{2 \pi \sqrt{1 - q^2}} \exp \left( - \frac{\kappa^2}{1 + q} \right) \\
	:= & T_1 - T_2 + T_3.
\end{align*}
As for the first term, we notice that
\begin{align*}
	e(q) = & \E_q \left[ \exp(-c_* (G_1 + G_2)) \bone \{ \min(G_1, G_2) \ge \kappa \} \right] \\
	= & \int_{[\kappa, +\infty)^2} \frac{1}{2 \pi \sqrt{1 - q^2}} \exp \left( - \frac{1}{2 (1 - q^2)} \left( x^2 + y^2 - 2 q xy \right) - c_* (x + y) \right) \d x \d y \\
	\stackrel{(i)}{=} & \frac{\exp(c_*^2 (1 + q))}{2 \pi \sqrt{1 - q^2}} \int_{[\kappa, +\infty)^2} \exp \left( - \frac{(x+y)^2}{4(1+q)} - \frac{(x-y)^2}{4(1-q)} - c_* (x + y) - c_*^2 (1 + q) \right) \d x \d y \\
	= & \frac{\exp(c_*^2 (1 + q))}{2 \pi \sqrt{1 - q^2}} \int_{[\kappa, +\infty)^2} \exp \left( - \frac{(x+y + 2 c_*(1 + q))^2}{4(1+q)} - \frac{(x-y)^2}{4(1-q)} \right) \d x \d y \\
	\stackrel{(ii)}{=} & \frac{\exp(c_*^2 (1 + q))}{2 \pi \sqrt{1 - q^2}} \int_{[\kappa + c_*(1 + q), +\infty)^2} \exp \left( - \frac{1}{2 (1 - q^2)} \left( x^2 + y^2 - 2 q xy \right) \right) \d x \d y \\
	= & \exp \left(c_*^2(1+q)\right) \P_q \left(\min \left(G_1, G_2 \right) \ge \kappa+c_*(1+q)\right),
\end{align*}
where in $(i)$ and $(ii)$ we use the identity
\begin{equation*}
	\frac{1}{2 (1 - q^2)} \left( x^2 + y^2 - 2 q xy \right) = \frac{(x+y)^2}{4(1+q)} + \frac{(x-y)^2}{4(1-q)}.
\end{equation*}
This proves that
\begin{equation*}
	T_1 = c_*^2 \exp \left(c_*^2(1+q)\right) \P_q \left(\min \left(G_1, G_2 \right) \ge \kappa+c_*(1+q)\right).
\end{equation*}
Moreover, by direct calculation we obtain that
\begin{align*}
	T_2 = & \frac{c_* \exp(- c_* \kappa)}{\pi \sqrt{1 - q^2}} \int_{\kappa}^{+\infty} \exp \left( - \frac{1}{2 (1 - q^2)} (x^2 + \kappa^2 - 2 q \kappa x) - c_*x \right) \d x \\
	= & \sqrt{\frac{2}{\pi}} \left(1-\Phi\left(\sqrt{\frac{1-q}{1+q}}(\kappa+c_*(1+q))\right)\right) c_* \exp \left(c_*^2(1+q)-\frac{1}{2}(\kappa+c_*(1+q))^2\right),
\end{align*}
thus proving Eq.~\eqref{eq:asymptotics_for_e'(q)}. Now we have
\begin{align*}
	e'(0) = & \frac{1}{2 \pi} \exp (-2 \kappa c_* - \kappa^2 )+c_*^2 \exp ( c_*^2 ) \left(1-\Phi (\kappa+c_*)\right)^2 - \sqrt{\frac{2}{\pi}} \left(1-\Phi (\kappa+c_*)\right) c_* \exp \left(c_*^2 -\frac{1}{2}(\kappa+c_*)^2\right) \\
	= & \exp(c_*^2) \left( c_* \left(1-\Phi (\kappa+c_*)\right) - \phi(\kappa + c_*) \right)^2 = 0,
\end{align*}
and consequently $F'(0) = - e'(0)/e(0) = 0$. Next we prove $e(0) = 1 + \breve{o}_{\kappa} (1)$, and $e(1) - e(0) = (1 + \breve{o}_{\kappa} (1)) \Phi (\kappa)$. It's straightforward that
\begin{align*}
	& \E [f(G)] = \int_{\kappa}^{+\infty} \frac{1}{\sqrt{2 \pi}} \exp \left( - c_*x - \frac{x^2}{2} \right) \d x = \exp \left( \frac{c_*^2}{2} \right) \int_{\kappa}^{+\infty} \frac{1}{\sqrt{2 \pi}} \exp \left( - \frac{(x+c_*)^2}{2} \right) \d x \\
	= & \exp \left( \frac{c_*^2}{2} \right) \left( 1 - \Phi(\kappa + c_*) \right) = \exp(\breve{o}_{\kappa}(1)) (1+\breve{o}_{\kappa} (1)) = 1+\breve{o}_{\kappa} (1),
\end{align*}
since $c_* = \breve{o}_{\kappa} (1)$. Hence we deduce that $ e(0) = \E [f(G)]^2 = 1+\breve{o}_{\kappa}(1)$.
Similarly, we calculate
\begin{equation*}
	e(1) = \E \left[ f(G)^2 \right] = \int_{\kappa}^{+\infty} \frac{1}{\sqrt{2 \pi}} \exp \left( - 2c_*x - \frac{x^2}{2} \right) \d x = \exp \left( 2 c_*^2 \right) \left( 1 - \Phi(\kappa + 2c_*) \right),
\end{equation*}
which implies that
\begin{align*}
	& e(1) - e(0) = (1 +\breve{o}_{\kappa}(1)) \frac{e(1) - e(0)}{e(0)} = (1+\breve{o}_{\kappa}(1)) \log \left( \frac{e(1)}{e(0)} \right) \\
	= & (1+\breve{o}_{\kappa}(1)) \log \left( \exp(c_*^2) \frac{ 1 - \Phi(\kappa + 2c_*)}{ (1 - \Phi(\kappa + c_*))^2} \right) \\
	= & (1+\breve{o}_{\kappa}(1)) \left( c_*^2 + \log \left( 1 - \Phi(\kappa + 2c_*) \right) - 2 \log \left( 1 - \Phi(\kappa + c_*) \right) \right) \\
	= & (1+\breve{o}_{\kappa}(1)) \left( \breve{O}_{\kappa} (\exp (-\kappa^2)) - (1+\breve{o}_{\kappa}(1)) \Phi(\kappa+c_*) + 2 (1+\breve{o}_{\kappa}(1)) \Phi(\kappa+2c_*) \right) \\
	\stackrel{(i)}{=} & (1+\breve{o}_{\kappa}(1)) \left( \breve{O}_{\kappa} (\exp (-\kappa^2)) + (1+\breve{o}_{\kappa}(1)) \Phi(\kappa) \right) = (1 + \breve{o}_{\kappa} (1)) \Phi (\kappa),
\end{align*}
as desired. Here the approximation $(i)$ is valid because $c_* = \breve{o}_{\kappa} (1/\kappa)$. Hence, (b) is proved.

Now we turn to show Eq.~\eqref{eq:bound_on_e''(q)}, i.e., the uniform bound on $e''(q)$. Using Eq~\eqref{eq:asymptotics_for_e'(q)}, we derive
\begin{align*}
	e''(q) =& \frac{1}{2 \pi \sqrt{1-q^2}} \exp \left(-2 \kappa c_*-\frac{\kappa^2}{1+q}\right) \left( \frac{q}{1-q^2} + \frac{\kappa^2}{(1+q)^2} + c_*^2 (1-2q) - \frac{2 \kappa c_*}{1+q} \right) \\
	& + c_*^4 \exp \left(c_*^2(1+q)\right) \P_q \left(\min \left(G_1, G_2 \right) \ge \kappa+c_*(1+q)\right) \\
    & + \sqrt{\frac{2}{\pi}} \left(1-\Phi\left(\sqrt{\frac{1-q}{1+q}}(\kappa+c_*(1+q))\right)\right) c_*^2 (\kappa - c_*(1-q) ) \exp \left(c_*^2(1+q)-\frac{1}{2}(\kappa+c_*(1+q))^2\right) \\
    := & A(q) +B(q) +C(q).
\end{align*}
Again noting that $c_* = \breve{O}_{\kappa} (\exp(-\kappa^2/2))$, we get the following upper bounds:
\begin{align*}
	\sup_{u \in [-q, q]} \vert A(u) \vert \le & \frac{1}{2 \pi \sqrt{1-q^2}} \exp \left(-2 \kappa c_*-\frac{\kappa^2}{1+q}\right) \left( \frac{q}{1-q^2} + \frac{\kappa^2}{(1-q)^2} + c_*^2 - \frac{2 \kappa c_*}{1-q} \right) \\
	= & \frac{1 + \breve{o}_{\kappa}(1)}{2 \pi \sqrt{1-q^2}} \exp \left(-\frac{\kappa^2}{1+q}\right) \left( \frac{q}{1-q^2} + \frac{\kappa^2}{(1-q)^2} + c_*^2 - \frac{2 \kappa c_*}{1-q} \right)\le \vert \kappa \vert^C \exp \left( - \frac{\kappa^2}{1+q} \right), \\
	\sup_{u \in [-q, q]} \vert B(u) \vert \le & c_*^4 \exp \left(c_*^2(1+q)\right) = (1+\breve{o}_{\kappa}(1)) c_*^4 = \breve{O}_{\kappa} (\exp(- 2 \kappa^2)), \\
	\sup_{u \in [-q, q]} \vert C(u) \vert \le & \sqrt{\frac{2}{\pi}}  c_*^2 (\vert \kappa \vert + c_*(1+q) ) \exp \left(c_*^2(1+q)-\frac{1}{2}(\kappa+c_*(1+q))^2\right) \\
	= & \sqrt{\frac{2}{\pi}}  c_*^2 (\vert \kappa \vert + c_*(1+q) ) (1+\breve{o}_{\kappa}(1)) \exp\left( - \frac{\kappa^2}{2} \right) \le \vert \kappa \vert^C \breve{O}_{\kappa} \left( \exp\left( - \frac{3 \kappa^2}{2} \right) \right),
\end{align*}
which immediately implies
\begin{equation*}
        \sup_{u \in [-q, q]} \vert e''(u) \vert \le \vert \kappa \vert^C \exp \left( - \frac{\kappa^2}{1+q} \right).
\end{equation*}
We finish the proof here.
\end{proof}

\begin{proof}[\bf Proof of Lemma~\ref{lem:varadhan}]
This proof is based on the original Varadhan's lemma, i.e., Theorem 4.3.1 in \cite{Dembo_2010}. According to Cramer's theorem (Theorem 2.2.3 in \cite{Dembo_2010}), we know that $S_n / n$ satisfies the LDP (Large deviations principle) with the following convex rate function:
\begin{equation*}
	I (x) = \sup_{t \in \R} \left \{ t x - \log \psi(t) \right\}.
\end{equation*}
Note that by our assumption, $\psi(t_0) < +\infty$ for some $t_0 > 0$, and $\psi(t)$ is increasing ($X_1$ is non-negative), therefore $\psi(t) < +\infty$ on $(- \infty, t_0)$, and Lemma 2.2.20 of \cite{Dembo_2010} implies $I(x)$ is a good rate function. Now we verify the condition of Varadhan's lemma, i.e., there exists $\gamma > 1$ such that
\begin{equation*}
	\limsup_{n \to +\infty} \frac{1}{n} \log \E \left[ \exp \left( \gamma c \sqrt{n S_n} \right) \right] < +\infty.
\end{equation*}
The above inequality automatically holds when $c \le 0$. For $c > 0$, we have
\begin{equation*}
	\limsup_{n \to +\infty} \frac{1}{n} \log \E \left[ \exp \left( \gamma c \sqrt{n S_n} \right) \right] \stackrel{(i)}{\le} \limsup_{n \to +\infty} \frac{1}{n} \log \E \left[ \exp \left( t_0 S_n + \frac{n \gamma^2 c^2}{4 t_0} \right) \right] = \log \psi(t_0) + \frac{\gamma^2 c^2}{4 t_0},
\end{equation*}
where $(i)$ follows from AM-GM inequality. Hence, Varadhan's lemma is applicable, and we obtain that
\begin{equation*}
    \lim_{n \to +\infty} \frac{1}{n} \log \E \left[ \exp \left( c \sqrt{n S_n} \right) \right] = \sup_{x > 0} \left\{ c \sqrt{x} - I(x) \right\}.
\end{equation*}	
It suffices to show the following relationship:
\begin{equation*}
    \sup_{x > 0} \left\{ c \sqrt{x} - I(x) \right\} =
    \begin{cases} 
	\displaystyle \inf_{t > 0} \left( \frac{c^2}{4 t} + \log \psi (t) \right), \ \text{for} \ c > 0, \\
    \displaystyle \sup_{t < 0} \left( \frac{c^2}{4 t} + \log \psi (t) \right), \ \text{for} \ c < 0.
    \end{cases}
\end{equation*}
To this end we first consider the case when $c < 0$, note that $c \sqrt{x} = \sup_{t < 0} \{ c^2/(4t) + tx \}$, therefore
\begin{equation*}
	\sup_{x > 0} \left\{ c \sqrt{x} - I(x) \right\} = \sup_{x > 0} \sup_{t < 0} \left\{ \frac{c^2}{4 t} + tx - I(x) \right\} = \sup_{t < 0} \sup_{x > 0} \left\{ \frac{c^2}{4 t} + tx - I(x) \right\} = \sup_{t < 0} \left\{ \frac{c^2}{4 t} + \log \psi(t) \right\},
\end{equation*}
where the last equality is due to the property of Legendre transform. Now assume $c > 0$, it is well-known that if $x \le \psi'(+\infty)/\psi(+\infty)$, then $I(x) = t_x x - \log \psi(t_x)$, where $t_x$ satisfies ($t_x$ can be $+\infty$)
\begin{equation*}
	\frac{\psi'(t_x)}{\psi(t_x)} = x \iff t_x = \left( \frac{\psi'}{\psi} \right)^{-1} (x),
\end{equation*}
otherwise $I(x) = + \infty$. Since $x \mapsto t_x$ is a bijection from $[0, +\infty)$ to $\R$, we obtain that
\begin{equation*}
	\sup_{x > 0} \left\{ c \sqrt{x} - I(x) \right\} = \sup_{x > 0} \left\{ c \sqrt{\frac{\psi'(t_x)}{\psi(t_x)}} - t_x \frac{\psi'(t_x)}{\psi(t_x)} + \log \psi (t_x) \right\} =  \sup_{t \in \R} \left\{ c \sqrt{\frac{\psi'(t)}{\psi(t)}} - t \frac{\psi'(t)}{\psi(t)} + \log \psi (t) \right\}.
\end{equation*}
We calculate the derivative of the above function:
\begin{equation*}
	\frac{\d}{\d t} \left( c \sqrt{\frac{\psi'(t)}{\psi(t)}} - t \frac{\psi'(t)}{\psi(t)} + \log \psi (t) \right) = \sqrt{\frac{\psi(t)}{\psi'(t)}} \left( \frac{c}{2} - t \sqrt{\frac{\psi'(t)}{\psi(t)}} \right) \frac{\d}{\d t} \left( \frac{\psi'(t)}{\psi(t)} \right).
\end{equation*}
Since $t \sqrt{\psi'(t)/\psi(t)}$ is increasing on $(0, +\infty)$, and approaches $0$ and $+\infty$ as $t \to 0$ and $t \to +\infty$ respectively, there exists a unique $t^* > 0$ such that
\begin{align*}
	& \frac{c}{2} = t^* \sqrt{\frac{\psi'(t^*)}{\psi(t^*)}}, \ \text{and further} \\
	& t > (<) t^* \iff \frac{c}{2} < (>) t^* \sqrt{\frac{\psi'(t^*)}{\psi(t^*)}}.
\end{align*}
Moreover, we have
\begin{align*}
	& \sup_{t \in \R} \left\{ c \sqrt{\frac{\psi'(t)}{\psi(t)}} - t \frac{\psi'(t)}{\psi(t)} + \log \psi (t) \right\} \\
	= & c \sqrt{\frac{\psi'(t^*)}{\psi(t^*)}} - t^* \frac{\psi'(t^*)}{\psi(t^*)} + \log \psi (t^*) \\
	= & c \cdot \frac{c}{2 t^*} - t^* \cdot \left( \frac{c}{2 t^*} \right)^2 + \log \psi(t^*) \\
	= & \frac{c^2}{4 t^*} + \log \psi(t^*) \stackrel{(i)}{=} \inf_{t > 0} \left( \frac{c^2}{4 t} + \log \psi (t) \right),
\end{align*}
where $(i)$ just results from the fact that
\begin{equation*}
	\frac{\d}{\d t} \left( \frac{c^2}{4 t} + \log \psi (t) \right) = - \frac{c^2}{4 t^2} + \frac{\psi'(t)}{\psi(t)} = \frac{1}{t^2} \left(t \sqrt{\frac{\psi'(t)}{\psi(t)}} - \frac{c}{2} \right) \left(t \sqrt{\frac{\psi'(t)}{\psi(t)}} + \frac{c}{2} \right)
\end{equation*}
is negative when $t < t^*$, and is positive when $t > t^*$. In conclusion, we have
\begin{equation*}
		\sup_{x > 0} \left\{ c \sqrt{x} - I(x) \right\} = \inf_{t > 0} \left( \frac{c^2}{4 t} + \log \psi (t) \right).
\end{equation*}
This completes the proof.
\end{proof}

\section{Linear programming algorithm in the pure noise model:\\ Proof of Theorem
\ref{thm:algpure}}

\label{sec:AlgorithmPureNoise}

Recall that we assume $y_i = 1$ for all $i$ without loss of generality. In this proof, we will also assume that $\veps_1>0$ is a sufficiently small constant such that $\delta \Phi(\kappa / (1+\veps_1)) < 1$ if $\delta < 1/\Phi(\kappa)$ holds and $\delta \Phi(\kappa / (1-\veps_1)) > 1$ if $\delta > 1/\Phi(\kappa)$ holds.

The optimization problem \eqref{opt:convex-1} can be rewritten into the max-min form.
\begin{align}
M &= \max_{\| \btheta \| \le 1} \min_{\balpha \ge \bzero} \left\{ \langle \btheta, \vv \rangle + \sum_{i=1}^n \alpha_i \big( \langle \xx_i, \btheta \rangle -  \kappa \big) \right\} \notag \\
&=   \max_{\| \btheta \| \le 1} \min_{\balpha \ge \bzero} \left\{   \langle \btheta, \vv \rangle +   \langle \balpha, \XX \btheta  \rangle - \kappa\langle \balpha , \bone_n \rangle \right\} \label{def:M}
\end{align}
We will apply Gordon's theorem. To this end, we define
%\footnote{Rigorously, we should use the notation $\sup$ and $\inf$. But we will soon simplify this double optimization and get an explicit formula.}
\begin{align}
\tilde M &= \max_{\| \btheta \| \le 1} \min_{\balpha \ge \bzero} \Big\{  \langle \btheta, \vv \rangle +  \| \btheta \| \langle \balpha, \bgg  \rangle +  \| \balpha \| \langle \btheta, \hh  \rangle  - \kappa\langle \balpha , \bone_n \rangle \Big\}. \label{def:Mtilde}
\end{align}
where $\bgg \sim \cN(\bzero, \bI_n)$ and $\hh \sim \cN(\bzero, \bI_d)$ are independent Gaussian vectors. We use a variant of Gordon's theorem due to \cite[Thm.~3]{thrampoulidis2015regularized} which is also used in \cite{montanari2019generalization}. 

\begin{thm}\label{thm:Gordon}
Let $\cC_1 \in \R^n$ and $\cC_2 \in \R^d$ be two compact sets and let $T: \cC_1 \times \cC_2 \to \R$ be a continuous function. Let $\XX = (X_{i,j}) \sim_\iid \cN(0,1) \in \R^{n \times d}, \bgg \sim \cN(\bzero, \bI_n)$, and $\hh \sim \cN(\bzero, \bI_d)$ be independent vectors and matrices. Define, 
\begin{align}
&Q_1(\XX) = \min_{\ww_1 \in \cC_1} \max_{\ww_2 \in \cC_2} \ww_1^\top \XX \ww_2 + T(\ww_1, \ww_2), \label{def:Q1} \\
&Q_2(\bgg, \hh) = \min_{\ww_1 \in \cC_1} \max_{\ww_2 \in \cC_2} \| \ww_2 \| \bgg^\top \ww_1 + \| \ww_1\| \hh^\top \ww_2 + T(\ww_1, \ww_2). \label{def:Q2}
\end{align}
Then, for all $t \in \R$, 
\begin{enumerate}
\item[(a)]{we have
\begin{equation*}
\P\left( Q_1(\XX) \le t \right) \le 2 \P\left( Q_2(\bgg, \hh) \le t \right);
\end{equation*}
}
\item[(b)]{
if further, $\cC_1$ and $\cC_2$ are convex and $T$ is convex concave in $(\ww_1, \ww_2)$, then 
\begin{equation*}
\P\left( Q_1(\XX) \ge t \right) \le 2 \P\left( Q_2(\bgg, \hh) \ge t \right).
\end{equation*}
}
\end{enumerate}
\end{thm}

When applying this theorem, we first condition on $\vv$ and then taking expectations in the end. A subtlety is that the constraint $\balpha \ge \bzero$ does not produce a compact feasible region. We nevertheless obtain the desired inequalities \eqref{ineq:MMtilde} below; proofs are deferred to the end of this section. %Section~\ref{sec:append-pure-alg}.

We apply this theorem twice: one with the original constraints in the optimization problem \eqref{opt:convex-1} and the other with the norm constraint $\| \btheta \| \le 1$ replaced by $\| \btheta\| \le 1+\veps_1$. The problem with the modified constraint is only used for the proof.

\begin{cor}\label{cor:Gordon}
Let $M$ and $\tilde M$ be defined as in \eqref{def:M} and \eqref{def:Mtilde}. We replace the constraint $\| \btheta \| \le 1$ by $\| \btheta\| \le 1+\veps_1$ and similarly define $M'$ and $\tilde M'$. Then, we have
\begin{align}\label{ineq:MMtilde}
\begin{cases}
\P\big( M   \le t \big) \le 2\P\big( \tilde M \le t \big), \\
\P\big( M   \ge t \big) \le 2 \P \big( \tilde M \ge t \big).
\end{cases}
\quad 
\begin{cases}
\P\big( M'   \le t \big) \le 2\P\big( \tilde M' \le t \big), \\
\P\big( M'   \ge t \big) \le 2 \P \big( \tilde M' \ge t \big).
\end{cases}
\end{align}
\end{cor}

In order to analyze $\tilde M$, first we fix some $ r_1 \ge 0$ and minimize our objective over $\balpha \ge \bzero$ with $\| \balpha \| = r_1$. Then we can rewrite $\tilde M$ as
\begin{equation*}
\tilde M = \max_{\| \btheta \| \le 1} \min_{r_1 \ge 0} \Big[  \langle \btheta, \vv \rangle + r_1 \langle \btheta, \hh \rangle + \min_{\substack{\| \balpha \| = r_1 \\ \balpha \ge \bzero}} \big\{ \langle \balpha , \| \btheta \| \bgg -  \kappa \bone_n \rangle \big\} \Big]. % := \max_{\| \btheta \| \le 1}\zeta(\btheta)
\end{equation*}
%where we defined by $\zeta(\btheta)$ the inner minimization problem.  
For any $r \in [0,1+\veps_1]$, we define
\begin{equation*}
\tilde M_r = \max_{\| \btheta \| = r} \min_{r_1 \ge 0} \Big[  \langle \btheta, \vv \rangle + r_1 \langle \btheta, \hh \rangle + \min_{\substack{\| \balpha \| = r_1 \\ \balpha \ge \bzero}} \big\{ \langle \balpha , \| \btheta \| \bgg -  \kappa \bone_n \rangle \big\} \Big]
\end{equation*}
so that $\tilde M = \max_{r \in [0,1]} \tilde M_r$ and $\tilde M' = \max_{r \in [0,1+\veps_1]} \tilde M_r$ under this notation. Notice that the following holds with probability $1-o_d(1)$: for all $\btheta$ with $\veps' \le \| \btheta \| \le 1+\veps_1$ where $\veps' \in (0,1)$ is any constant, there is at least one negative element in the vector $\| \btheta \| \bgg -  \kappa \bone_n$. In fact, we have
\begin{equation*}
\P\left(\exists\, i \in [n]~\text{such that}~ \veps' g_i -  \kappa < 0 \right) = 1 - (1 - \Phi((\veps')^{-1}  \kappa) )^n \to 1 \quad \text{as} \quad n \to \infty.
\end{equation*}
For $r \ge \veps'$ and on this event, we have
\begin{align*}
\min_{\substack{\| \balpha \| = r_1 \\ \balpha \ge \bzero}} \left\{ \langle \balpha, \| \btheta \|  \bgg   - \kappa \bone_n \rangle\right\} &\stackrel{(i)}{=} \min_{\substack{\| \balpha \| = r_1 \\ \balpha \ge \bzero}} \left\{ \left\langle \balpha,  - \left[ \| \btheta \|  \bgg   - \kappa \bone_n  \right]_- \right\rangle\right\}  \\
&\stackrel{(ii)}{=} - r_1 \big\| \left[ \| \btheta \|  \bgg   - \kappa \bone_n \right]_- \big\|
\end{align*}
where \textit{(i)} is because we always assign zero to $\alpha_i$ if $\| \btheta \| g_i -  \kappa$ is positive, and \textit{(ii)} is due to the Cauchy-Schwarz inequality (particularly the condition where equality holds). Therefore, 
\begin{align*}
\tilde M_r &\stackrel{(i)}{=} \max_{\| \btheta \| \le r} \min_{r_{1} \ge 0} \Big\{  \langle \btheta, \vv \rangle + r_1 \langle \btheta, \hh \rangle - r_1 \big\| \left[ r  \bgg   - \kappa \bone_n \right]_- \big\|  \Big\} \\
&\stackrel{(ii)}{=}  \min_{r_1 \ge 0} \max_{\| \btheta \| \le r}  \Big\{  \langle \btheta, \vv \rangle + r_1 \langle \btheta, \hh \rangle - r_1 \big\| \left[ r  \bgg   - \kappa \bone_n \right]_- \big\|  \Big\} \\
&\stackrel{(iii)}{=}  \min_{r_1 \ge 0} \Big\{  r\sqrt{1 + r_1^2 \| \hh \|^2}  - r_1 \big\| \left[ r  \bgg   - \kappa \bone_n \right]_- \big\|  \Big\}
\end{align*}
Here, in \textit{(i)} we replaced the constraint $\| \btheta \| = r$ by $\| \btheta \| \le r$ since the maximum can be always attained at the boundary, \textit{(ii)} is because of Sion's minimax theorem, and \textit{(iii)} is because of the Cauchy-Schwarz inequality. We then simplify the inner minimization problem using the following lemma.

\begin{lem}\label{lem:simopt}
Suppose $a,b \ge 0$. Then we have
\begin{align*}
&\inf_{x \ge 0} \left\{ a\sqrt{1+x^2}  -bx \right\} = \begin{cases} \sqrt{a^2 - b^2} , & \text{if}~a \ge b \\ -\infty , & \text{if}~a < b
\end{cases} \\
&\underset{x \ge 0}{\mathrm{argmin}} \left\{ a\sqrt{1+x^2}  -bx \right\} = \begin{cases}  b/\sqrt{a^2 - b^2} , & \text{if}~a > b \\ \infty , & \text{if}~a \le b
\end{cases}
\end{align*}
\end{lem}
\begin{proof}[{\bf Proof of Lemma~\ref{lem:simopt}}]
It is easy to check that the function $f(x) = a\sqrt{1+x^2}  -bx$ is convex in $x$. Moreover,
\begin{equation*}
f'(x) = \frac{ax}{\sqrt{1+x^2}} - b.
\end{equation*}
If $a \le b$, the function $f(x)$ is monotone decreasing, so the infimum is taken at $\infty$; if $a > b$, the function $f(x)$ has a unique minimizer at $x$ such that $f'(x) = 0$, which gives $x = b /\sqrt{a^2-b^2}$.
\end{proof}
Denote $\cT(x) = x \bone \{ x \ge 0 \} + (-\infty) \bone \{ x < 0 \}$. We apply this lemma (by viewing $\| \hh \| r_1$ as $x$) and obtain
\begin{align*}
\tilde M_r  &=  \cT\left[ \Big( r^2 - \frac{1}{\| \hh \|^2} \Big\| \big[ r  \bgg   - \kappa \bone_n \big]_- \Big\|^2 \Big)^{1/2} \right].
\end{align*}
\begin{lem}\label{lem:lln}
Recall that $n/d \to \delta$. Denote $G \sim \cN(0,1)$. We have 
\begin{equation*}
\sup_{r \in [0,1+\veps_1]} \Big| \frac{1}{\| \hh \|^2} \Big\| \big[ r  \bgg   - \kappa \bone_n \big]_- \Big\|^2 - \delta \E \left\{ \big[ r G - \kappa \big]_-^2 \right\} \Big| = o_{n,\P}(1).
\end{equation*}
\end{lem}
\begin{proof}[{\bf Proof of Lemma~\ref{lem:lln}}]
First, we use uniform law of large numbers (ULLN) \cite{newey1994large}[Lemma\ 2.4] to obtain
\begin{equation*}
\sup_{r \in [0,1+\veps_1]} \Big| \frac{1}{n} \Big\| \big[ r  \bgg   - \kappa \bone_n \big]_- \Big\|^2 - \E \left\{ \big[ r G - \kappa \big]_-^2 \right\} \Big| = o_{n,\P}(1).
\end{equation*}
(Note that $[r g_i - \kappa]_-^2 \le (1+\veps_1)^2 g_i^2$ holds for all $r \in [0,1]$ and $\E [g_i^2] < \infty$, so we can apply the ULLN.) Denote $f_1(r) = n^{-1} \Big\| \big[ r  \bgg   -  \kappa \bone_n \big]_- \Big\|^2 $ and $f_2(r) = \E \left\{ \big[ r G -  \kappa \big]_-^2 \right\}$. We notice that $\sup_{r \in [0,1+\veps_1]} f_2(r) \le \E [G^2]$ and $\frac{n}{\| \hh \|^2} = \delta + o_{d,\P}(1)$ and conclude that
\begin{equation*}
\sup_{r \in [0,1+\veps_1]}\Big| \frac{n}{\| \hh \|^2} f_1(r) - \delta f_2(r) \Big| \le \frac{n}{\| \hh \|^2} \sup_{r \in [0,1+\veps_1]} \big| f_1(r) - f_2(r) \big|  + \Big| \frac{n}{\| \hh \|^2} - \delta \Big| \sup_{r \in [0,1+\veps_1] } f_2(r) = o_{d,\P}(1),
\end{equation*}
which finishes the proof.
\end{proof}

Now let us define, for $r \in [0,1+\veps_1]$,  
\begin{equation*}
F(r) = r^2 - \delta \E \left\{ \big[ r G -  \kappa \big]_-^2 \right\}.
\end{equation*}
Using this definition, we can express $\tilde M_r$ as, for all $r \in [\veps', 1+\veps_1]$,
\begin{equation*}
\tilde M_r = \cT\left[\Big(F(r) + o_{d,\P}(1)\Big)^{1/2}\right].
\end{equation*}
Here, $o_{d,\P}(1)$ is uniform over all $r$.

\begin{lem}\label{lem:F}
Consider the function $F(r)$ defined above. For all $r\in(0,1+\veps_1)$, the function $F(r)$ is differentiable, and its derivative is given by
\begin{align*}
F'(r) =  2r\Big( 1 - \delta \P\big( rG  - \kappa < 0 \big) \Big).
\end{align*}
If $\delta < \Phi( \kappa / (1+\veps_1))^{-1}$, then $F'(r)>0$ and $F(r)>0$ for all $r\in(0,1+ \veps_1]$. If $\delta > \Phi( \kappa / (1-\veps_1))^{-1}$, then there exists $r_1 \in (0, 1-\veps_1)$ such that $F'(r) > 0$ on $r \in [0,r_1)$ and $F'(r) < 0$ on $r \in (r_1, 1)$
\end{lem}
\begin{proof}[{\bf Proof of Lemma~\ref{lem:F}}]
The differentiability of $F(r)$ follows from the fact that $\big[ r G -  \kappa \big]_-^2$ is differentiable in $r$ and the dominated convergence theorem. Taking the derivative of $F(r)$, we obtain
\begin{align*}
F'(r) = 2r + 2 \delta \E  \left\{ G \big[ r G -  \kappa \big]_- \right\} = 2r\Big( 1 - \delta \P\big( rG  -  \kappa < 0 \big) \Big)
\end{align*}
where we used Stein's identity in the second equality. Note that $\P\big( rG  - \kappa < 0 \big) = \Phi(r^{-1}   \kappa)$ is monotone increasing in $r$, where recall that $\Phi$ is the Gaussian CDF. Thus we have %Denote $\Phi(r^{-1} b^{-1} \kappa)$ by $\Phi$ and $\Phi'(r^{-1} b^{-1} \kappa)$ by $\Phi'$. Taking the derivative, we have
\begin{equation*}
F'(r) > 2r\big(1 - \delta \Phi( \kappa / (1+\veps_1)) \big) >0
\end{equation*} 
in the case $\delta < \Phi( \kappa / (1+\veps_1))^{-1}$. Since $F(0) = 0$, it follows that $F(r) > 0$. If  $\delta > \Phi( \kappa / (1-\veps_1))^{-1}$, then there the equation $1 - \delta \P\big( rG  - \kappa < 0 \big) = 0$ has a unique zero $r_1 \in (0, 1-\veps_1)$. 
\end{proof}

This lemma states that maximizing $F(r)$ requires $r$ to be as large as possible. Thus, we obtain the following result. 

\begin{prop}\label{prop:noise}
Suppose that $\delta < \Phi( \kappa / (1+\veps_1))^{-1}$ holds. Then, we have 
\begin{equation}\label{eq:tildeMr}
\max_{r \in [0,1]} \tilde M_r = \sqrt{F(1)} + o_{d,\P}(1), \qquad \max_{r \in [0,1+\veps_1]} \tilde M_r = \sqrt{F(1+\veps_1)} + o_{d,\P}(1).
\end{equation}
Consequently, with probability $1-o_d(1)$, we have $M' > M$.
%which implies $\| \hat \btheta_1 \| > 1$ with high probability.
\end{prop}

\begin{proof}[{\bf Proof of Proposition~\ref{prop:noise}}]
First, we notice that $\tilde M_r \le \max_{\| \btheta \|=r} \langle \btheta, \vv \rangle = r$ from the definition of $\tilde M_r$ (choosing $r_1=0$). This gives a simple bound $\max_{r \in [0, \veps']} \tilde M_r \le \veps'$. We can choose $\veps'$ so that $\veps' < \sqrt{F(1)}$. From Lemma~\ref{lem:F} we know that $\max_{r \in [0,1]} F(r) = F(1) > 0$ so $\max_{r \in [0,1]} \tilde M_r = \sqrt{F(1)} + o_{d,\P}(1)$. The second equality in \eqref{eq:tildeMr} is obtained similarly. Let $\veps_0 = \sqrt{F(1+\veps_1)} - \sqrt{F(1)}$, which does not depend on $d$. Then the following holds with probability $1-o_d(1)$.
\begin{equation*}
\max_{r \in [0,1+\veps_1]}  M_r \ge \max_{r \in [0,1+\veps_1]} \tilde M_r - \frac{\veps_0}{3}, \quad \max_{r \in [0,1]}  \tilde M_r \ge \max_{r \in [0,1]} M_r - \frac{\veps_0}{3}, \quad \max_{r \in [0,1+\veps_1]}  \tilde M_r > \max_{r \in [0,1]} \tilde M_r + \frac{2\veps_0}{3}.
\end{equation*}
The first two inequalities are due to Corollary~\ref{cor:Gordon} and the third is due to \eqref{eq:tildeMr}. Combining these inequalities leads to $M' > M$. 
\end{proof}

\begin{proof}[{\bf Proof of Theorem~\ref{thm:algpure}}]
First, we consider the case $\delta \Phi(\kappa/(1+\veps_1)) < 1$. Recall that $\hat \btheta$ is a maximizer to the optimization problem \eqref{opt:convex-1} so by feasibility we must have $y_i \langle \xx_i, \hat \btheta \rangle \ge \kappa$. Let $\hat \btheta'$ be any maximizer to the modified optimization problem where the constraint $\| \btheta \| \le 1$ is replaced by $\| \btheta \| \le 1+ \veps_1$. It must also satisfy $y_i \langle \xx_i, \hat \btheta' \rangle \ge \kappa$. For any $\lambda \in [0,1]$, the interpolant
\begin{equation*}
\hat \btheta_\lambda := \lambda \hat \btheta + (1-\lambda) \hat \btheta'
\end{equation*}
also satisfies $y_i \langle \xx_i, \hat \btheta_\lambda \rangle \ge \kappa$ by linearity. Proposition~\ref{prop:noise} implies that the modified optimization problem has a strictly larger maximum, so we must have $\| \hat \btheta' \| >1$ and $\langle \hat \btheta', \vv \rangle > \langle \hat \btheta, \vv \rangle$. If $\| \hat \btheta\| <1$, then we can choose appropriate $\lambda \in (0,1)$ such that $\| \hat \btheta_\lambda \| = 1$ and $\langle \hat \btheta_\lambda, \vv \rangle > \langle \hat \btheta, \vv \rangle$, which contradicts the definition of $\hat \btheta$. Therefore, we conclude that $\| \hat \btheta \|=1$, so the first claim of the theorem is proved.

Next, we consider the case $\delta \Phi(\kappa/(1-\veps_1)) > 1$. In \eqref{def:M}, we replace the constraint $\| \btheta \| \le 1$ by a non-convex constraint $1-\veps_1 \le \| \btheta \| \le 1$ and denote the optimal values by $M''$. Lemma~\ref{lem:F} implies that $\max_{r \in [0,1]} \tilde M_r > \max_{r \in [1-\veps_1,1]} \tilde M$ with probability $1-o_d(1)$. Despite the non-convex constraint, we are still able to apply Theorem~\ref{thm:Gordon} (a) and obtain $M > M''$ with  probability $1-o_d(1)$. This implies that the optimal solution $\hat \btheta$ must satisfy $\| \hat \btheta \| \le 1-\veps_1$ with  probability $1-o_d(1)$, thus proving the second claim.

%In the optimization problem \eqref{} we replace the constraint $\| \btheta \| \le 1$ by a non-convex constraint $1-\veps_1 \le \| \btheta \| \le 1$. We are still able to apply Theorem~\ref{thm:Gordon} (a). Lemma~\ref{lem:F} implies that $\max_{r \in [0,1]} \tilde M_r > \max_{r \in [1-\veps_1,1]} \tilde M$ with probability $1-o_d(1)$, and Theorem~\ref{thm:Gordon} (a) further gives 

The final claim follows from the Gaussian tail probability asymptotics:
\begin{equation*}
\Phi(\kappa) = \frac{1+o_{\kappa}(1)}{\sqrt{2\pi}\, |\kappa|} \exp\left(-\frac{\kappa^2}{2} \right).
\end{equation*}
Now the proof is complete.
\end{proof}

Finally, we prove Corollary~\ref{cor:Gordon} as claim before.
%\subsection{Proofs for Section~\ref{sec:pure-alg}}\label{sec:append-pure-alg}

\begin{proof}[{\bf Proof of Corollary~\ref{cor:Gordon}}]
We will only prove the inequalities for $M$ and $\tilde M$ as the others can be derived similarly. Let $(\tau_k)_{k\ge 1}$ be an increasing sequence with $\tau_k > 0$ and $\lim_{k \to \infty} \tau_k = \infty$. Let us define
\begin{align*}
M_k &=  \max_{\| \btheta \| \le 1} \min_{\substack{\balpha \ge \bzero\\\| \balpha \| \le \tau_k}} \left\{   \langle \btheta, \vv \rangle +   \langle \balpha, \XX \btheta  \rangle - \kappa\langle \balpha , \bone_n \rangle \right\} \\
\tilde M_k &= \max_{\| \btheta \| \le 1} \min_{\substack{\balpha \ge \bzero\\\| \balpha \| \le \tau_k}} \Big\{  \langle \btheta, \vv \rangle +  \| \btheta \| \langle \balpha, \bgg  \rangle +  \| \balpha \| \langle \btheta, \hh  \rangle  - \kappa\langle \balpha , \bone_n \rangle \Big\}.
\end{align*}
Note that we can flip $\min$ and $\max$ by multiplying $-1$ on both sides, i.e.,
\begin{equation*}
-M_k = \min_{\| \btheta \| \le 1} \max_{\substack{\balpha \ge \bzero\\\| \balpha \| \le \tau_k}} \left\{   \langle \btheta, -\vv \rangle +   \langle \balpha, -\XX \btheta  \rangle + \kappa\langle \balpha , \bone_n \rangle \right\}
\end{equation*}
and $-\tilde M_k$ has a similar expression. We can then apply Theorem~\ref{thm:Gordon} and obtain, for every $k \ge 1$ and $t$
\begin{equation*}
\P(M_k \le t) \le 2\P(\tilde M_k \le t), \qquad \P(M_k \ge t) \le 2\P(\tilde M_k \ge t).
\end{equation*}
It suffices to prove that $\lim_{k \to \infty} M_k = M$ and $\lim_{k \to \infty} \tilde M_k = \tilde M$. Clearly $M \ge 0$ and $\tilde M \ge 0$ (because they are lower bounded at $\btheta = 0$). By definition of $M_k$, we must have $ M_k \ge M$ and $M_k$ is non-increasing. It is similar for $\tilde M_k$. So clearly we have $\lim_k M_k \ge M$ and $\lim_k \tilde M_k \ge \tilde M$. Below we shall prove that $\lim_k M_k \le M$ and $\lim_k \tilde M_k \le \tilde M$.

Let $\btheta_k$ be a maximizer associated with $M_k$. Without loss of generality we assume $\btheta_k$ converges to a limit $\btheta^*$ where $\| \btheta^* \| \le 1$, because otherwise we can pass to a subsequence. We claim that for every $i$, the sequence $( \langle \xx_i, \btheta_k \rangle -  \kappa )_{k \ge 1}$ does not have a negative accumulation point. Indeed, if it has a negative accumulation point, then we could assign $\alpha_i = -\tau_k, \alpha_j = 0$ ($j \neq i$), which would lead to $\lim_k M_k = -\infty$, which is a contradiction. Thus $\langle \xx_i, \btheta^* \rangle -  \kappa \ge 0$ for every $i$ and therefore $M \ge \langle \btheta^*, \vv \rangle$. On the other hand, we have $M_k \le \langle \btheta_k, \vv \rangle$, which implies $\lim_k M_k \le \langle \btheta^*, \vv \rangle \le M$.

Next, let $\tilde \btheta_k$ be a maximizer associated with $\tilde M_k$. Without loss of generality we assume $\tilde \btheta_k$ converges to a limit $\tilde \btheta^*$ where $\| \tilde \btheta^* \| \le 1$. We rewrite $\tilde M_k$ into 
\begin{equation*}
\tilde M_k = \max_{\| \btheta \| \le 1} \min_{\substack{\balpha \ge \bzero\\\| \balpha \| \le \tau_k}} \Big\{  \langle \btheta, \vv \rangle +  \big\langle \balpha, \| \btheta \| \bgg   - \kappa \bone_n \rangle +  \| \balpha \| \langle \btheta, \hh  \rangle \Big\}.
\end{equation*}
We claim that the sequence $(\eta_k)_{k \ge 1}$ does not have a negative accumulation point, where
\begin{equation*}
\eta_k = \min_{\balpha \ge \bzero, \| \balpha \| = 1} \langle \balpha, \|  \tilde \btheta_k \| \bgg -  \kappa \bone_n \rangle + \langle \tilde \btheta_k, \hh \rangle.
\end{equation*}
Otherwise, we could choose $\balpha$ such that $\| \balpha \| = \tau_k$ and then $\lim_k \tilde M_k = -\infty$, leading to a contradiction. Thus, by inspecting $\tilde \btheta^*$, we have
\begin{equation*}
\tilde M \ge \langle \tilde \btheta^*, \vv \rangle + \| \balpha \| \cdot \Big(  \langle \frac{\balpha}{\| \balpha \|}, \|  \tilde \btheta^* \| \bgg - \kappa \bone_n \rangle + \langle \tilde \btheta^*, \hh \rangle \Big) \ge \langle \tilde \btheta^*, \vv \rangle.
\end{equation*}
On the other hand, we have $\tilde M_k \le \langle \tilde \btheta_k, \vv \rangle$, which implies $\lim_k \tilde M_k \le \langle \tilde \btheta^*, \vv \rangle \le \tilde M$.
\end{proof}

\section{$\kappa$-margin classifiers in the linear signal model:\\ Proofs of Theorems
\ref{thm:logistic_lower_bound}, \ref{thm:logistic_upper_bound}, \ref{cor:rho_min_max},
\ref{thm:asymptotic_upper_bound} and \ref{thm:best_error}}

Without loss of generality we will assume $C_{\tail}=1$ in the proofs of our main results for the linear signal model, since the constant $C_{\tail}$ in our exponential tail assumption interacts with these results only through Lemma~\ref{lem:tail} below. For generic $C_{\tail}$, the tail probability estimates in Lemma~\ref{lem:tail} only differs by a multiplicative factor $C_{\tail}$ from the special case $C_{\tail} = 1$, and the proof proceeds similarly.

\label{sec:ExistenceSignal}

We state below a useful lemma that characterizes the tail probability of a key random variable, whose proof will be deferred to Appendix~\ref{sec:append-signal-lemma}. 

\begin{lem}\label{lem:tail}
Let $\eta_0 \in (0,0.1)$ be any constant. For $t > 0$, define 
\begin{align*}
A_{\rho,t} = \begin{cases} \displaystyle \frac{1}{t} \sqrt{\frac{2}{\pi}} \exp\Big(- \frac{t^2}{2} - \alpha \rho t + \frac{(1-\rho^2)\alpha^2}{2} \Big), & \rho \in [\eta_0, 1] \\
\displaystyle \frac{1}{t} \sqrt{\frac{2}{\pi}} \exp\Big(- \frac{t^2}{2} \Big) , & \rho \in [-1, -\eta_0] \\
\displaystyle \frac{1}{t} \sqrt{\frac{2}{\pi}} \exp\Big(- \frac{t^2}{2} \Big) \cdot a_{\rho, t} , & \rho \in (-\eta_0, \eta_0)
\end{cases}
\end{align*}
where $a_{\rho,t}$ is given by 
\begin{equation}\label{def:a}
a_{\rho,t} = \frac{1}{2\sqrt{2\pi(1-\rho^2)}} \int_\R \left[ \varphi(u - \rho t) + 1 - \varphi(u+\rho t) \right] \exp\left(-\frac{u^2}{2(1-\rho^2)} \right) \; \d u, 
\end{equation}
and it satisfies
\begin{equation*}
	\frac{1 + \breve{o}_t(1)}{2}\min\big\{1, \exp(-\alpha\eta_0t)\big\} \le a_{\rho,t} \le 1.
\end{equation*}
Then, we have
\begin{equation*}
\lim_{t \to \infty} \max_{\rho \in [-1,1]} \Big |A_{\rho,t}^{-1} \,\P\left( \rho Y G + \sqrt{1-\rho^2}\, W < - t\right) - 1 \Big| = 0.
\end{equation*}
As a consequence, 
\begin{equation*}
\lim_{\kappa \to -\infty} \max_{\rho \in [-1,1]} \Big |(2|\kappa|^{-2}A_{\rho,|\kappa|})^{-1} \,\E\left[ \big( \rho Y G + \sqrt{1-\rho^2}\, W - \kappa \big)_-^2 \right] - 1 \Big| = 0.
\end{equation*}
\end{lem}

\subsection{Phase transition lower bound: Proof of Theorem~\ref{thm:logistic_lower_bound}}\label{sec:signal-lower}

Without loss of generality, we can assume that the label $y_i$ only depends on the first coordinate of $\xx_i$. To see this, let $\bP_{\btheta^*} = \btheta^* (\btheta^*)^\top \in \R^{d \times d}$ be the orthogonal projection onto $\spann \{ \btheta^* \}$ and $\bP_{\btheta^*}^\perp = \bI_d - \bP_{\btheta^*}$ be the projection matrix onto the orthogonal complement of $\btheta^*$. Then we have the following decomposition:
\begin{equation*}
y_i\langle \xx_i, \btheta \rangle = y_i \langle \btheta, \btheta^* \rangle \langle \xx_i, \btheta^* \rangle + y_i \langle \xx_i, \bP_{\btheta*}^\perp\btheta \rangle = y_i \rho G_i +y_i \sqrt{1-\rho^2} \langle \xx_i, \bar \btheta \rangle,
\end{equation*}
where $G_i=\langle \xx_i, \btheta^* \rangle$, $\rho = \langle \btheta, \btheta^* \rangle$, and $\bar \btheta = \bP_{\btheta*}^\perp\btheta / \| \bP_{\btheta*}^\perp\btheta \| \in \S^{d-1}$. Since $(y_i, G_i) \bot \langle \xx_i, \bar \btheta \rangle$ and $\langle \xx_i, \bar \btheta \rangle$ has a symmetric distribution, we can write
\begin{equation*}
y_i\langle \xx_i, \btheta \rangle \stackrel{d}{=} \rho y_i G_i + \sqrt{1-\rho^2} \langle \xx_i, \bar \btheta \rangle \stackrel{d}{=} \rho y_i G_i + \sqrt{1 - \rho^2} \langle \zz_i, \ww \rangle,
\end{equation*}
where $\ww \in \S^{d-2}$, $\{(y_i, G_i, \zz_i)\}_{1 \le i \le n}$ are i.i.d., each have joint distribution:
\begin{equation*}
	(y_i, G_i) \bot \zz_i, \ \zz_i \sim \cN (\bzero, \bI_{d-1}), \ G_i \sim \cN (0, 1), \ \P(y_i = 1 \vert G_i) = \varphi (G_i) = 1 - \P (y_i = -1 \vert G_i).
\end{equation*}
Note that we can actually search for $\ww$ with $\Vert \ww \Vert_2 \ge 1$, which will result a $\kappa$-margin classifier $\btheta$ lying outside the unit ball. Since $\kappa < 0$, the orthogonal projection of $\btheta$ onto $\S^{d-1}$ is a $\kappa$-margin classifier as well. Consequently, we aim to show that under the conditions of Theorem~\ref{thm:logistic_lower_bound}, for properly chosen $\rho \in [-1, 1]$,
\begin{equation}\label{eq:logistic_lower_bound}
	\liminf_{n \to +\infty} \P \left( \exists \Vert \ww \Vert_2 \ge 1, \ \text{s.t.} \ \forall 1 \le i \le n, \rho y_i G_i + \sqrt{1-\rho^2} \langle \zz_i, \ww  \rangle \ge \kappa \right) > 0.
\end{equation}
Then, we will use a similar concentration argument to that in the proof of Theorem~\ref{thm:pure_noise_lower_bound} to show that the above probability converges to $1$ as $n \to \infty$. Since $\delta < \delta_{\sl}(\kappa;\varphi)$, by Definition~\ref{def:sig_lower} we know that (at least) one of the following assumptions holds:

\begin{ass}\label{cond:sig_lower_a}
	There exists $\rho \in (0, 1)$ such that
	\begin{equation}
		\delta < \max \left\{ \P \left( \rho YG + \sqrt{1 - \rho^2} W \le \kappa \right), \E \left[ \left( \kappa_0 - \frac{\rho YG}{\sqrt{1-\rho^2}} - W \right)_+^2 \right] \right\}^{-1}.
	\end{equation}
\end{ass}

\begin{ass}\label{cond:sig_lower_b}
	There exist parameters
    \begin{equation}
	    \rho \in (0, 1), \ 0 > \kappa_1 > \kappa_2 > \kappa_0 = \frac{\kappa}{\sqrt{1 - \rho^2}}, \ c \ge 0,
    \end{equation}
    such that
    \begin{equation}\label{eq:cond_b}
        \begin{split}
        	\frac{1}{\delta} > & \delta_{\mathrm{sec}} \left( \rho, \sqrt{1 - \rho^2} \kappa_2, \frac{\sqrt{1-\rho^2}}{\rho} \kappa_1 \right)^{-1} + \P \left(YG \ge \frac{\sqrt{1-\rho^2}}{\rho} \kappa_1 \right) \E \left[ \left( \frac{\kappa_0 - \kappa_2}{c} - W \right)_+^2 \right] \\
    	& + \frac{1}{c^2}\E \left[ \bone \left\{ YG < \frac{\sqrt{1-\rho^2}}{\rho} \kappa_1 \right\} \left( \kappa_0 - \frac{\rho YG}{\sqrt{1-\rho^2}} - \sqrt{1 + c^2} W \right)_+^2 \right].
        \end{split}
    \end{equation}

\end{ass}

\begin{lem}\label{lem:sig_lower_gordon}
	Under Assumption~\ref{cond:sig_lower_a}, Eq.~\eqref{eq:logistic_lower_bound} holds, which further implies Eq.~\eqref{eq:signal_lower_bound}.
\end{lem}

The proof of Lemma~\ref{lem:sig_lower_gordon} is deferred to Appendix~\ref{sec:append-signal-lemma}. Now assume \ref{cond:sig_lower_b} and denote
\begin{equation*}
	\yy = (y_i)_{1 \le i \le n}, \ \bG = (G_i)_{1 \le i \le n}, \ \ZZ = (\zz_i)_{1 \le i \le n}^\top,
\end{equation*}
then we have
\begin{align*}
	& \rho \yy \odot \bG + \sqrt{1 - \rho^2} \ZZ \ww \ge \kappa \bone \iff \frac{\rho \yy \odot \bG}{\sqrt{1 - \rho^2}} + \ZZ \ww \ge \frac{\kappa}{\sqrt{1 - \rho^2}} \bone \\
	\iff & \uu + \ZZ \ww \ge \kappa_0 \bone, \quad \where \ \uu = \frac{\rho \yy \odot \bG}{\sqrt{1 - \rho^2}}, \ \kappa_0 = \frac{\kappa}{\sqrt{1 - \rho^2}}.
\end{align*}
Let $\kappa_1$, $\kappa_2$ be as defined in Assumption~\ref{cond:sig_lower_b}. We further define the ``good sample" as those satisfying $u_i \ge \kappa_1$, i.e.,
\begin{equation*}
	S_G = \{ 1 \le i \le n: u_i \ge \kappa_1 \}, \ \text{and} \ \uu^G = (u_i)_{i \in S_G}, \ \ZZ^G = (\zz_i)_{i \in S_G}^\top.
\end{equation*}
Similarly, the set of ``bad sample" is defined as $S_B = [n] \backslash S_G$, and $\uu^B = (u_i)_{i \in S_B}$, $\ZZ^B = (\zz_i)_{i \in S_B}^\top$. For future convenience, we also denote
\begin{equation*}
	p_B = \P (u < \kappa_1) = \P \left( YG < \frac{\sqrt{1 - \rho^2}}{\rho} \kappa_1 \right).
\end{equation*}
A major ingredient of our proof is the ``second moment method" applied to the good sample, where a small fraction of the columns of $\ZZ$ plays a special role. The moment calculation will be based on the randomness in the remaining majority of columns of $\ZZ$. According to Eq.~\eqref{eq:cond_b}, there exists a sequence of positive integers $\{ d_0 = d_0 (n) \}_{n \in \mathbb{N}}$ such that
\begin{equation}\label{eq:Z_partition_2}
\begin{split}
	\lim_{n \to +\infty} \frac{d_0}{n} > & \P \left(YG \ge \frac{\sqrt{1-\rho^2}}{\rho} \kappa_1 \right) \E \left[ \left( \frac{\kappa_0 - \kappa_2}{c} - W \right)_+^2 \right] \\
	& + \frac{1}{c^2}\E \left[ \bone \left\{ YG < \frac{\sqrt{1-\rho^2}}{\rho} \kappa_1 \right\} \left( \kappa_0 - \frac{\rho YG}{\sqrt{1-\rho^2}} - \sqrt{1 + c^2} W \right)_+^2 \right],
	\end{split}
\end{equation}
and
\begin{equation}\label{eq:Z_partition_1}
	\begin{split}
	\lim_{n \to +\infty} \frac{d - 1 - d_0}{n} > & \delta_{\mathrm{sec}} \left( \rho, \sqrt{1 - \rho^2} \kappa_2, \frac{\sqrt{1-\rho^2}}{\rho} \kappa_1 \right)^{-1}.
	\end{split}
\end{equation}
With this choice of $d_0$, we partition
\begin{equation*}
	\ZZ = \left( \ZZ_1, \ZZ_2 \right), \quad \where \ \ZZ_1 \in \R^{n \times (d - 1 - d_0)}, \ \ZZ_2 \in \R^{n \times d_0}.
\end{equation*}
Then Eq.~\eqref{eq:logistic_lower_bound} reduces to proving that with probability bounded away from $0$, there exists $\ww_1 \in \R^{d-1-d_0}$, $\ww_2 \in \R^{d_0}$, satisfying $\norm{\ww_1}_2 = 1$, $\norm{\ww_2}_2 \le c$, and
\begin{equation}\label{ineq:bad_and_good_sample}
\begin{split}
	\uu^G + \ZZ_1^G \ww_1 + \ZZ_2^G \ww_2 \ge & \kappa_0 \bone, \\ \uu^B + \ZZ_1^B \ww_1 + \ZZ_2^B \ww_2 \ge & \kappa_0 \bone,
\end{split}
\end{equation}
which is achieved by the following two lemmas:

\begin{lem}\label{lem:good_sample_2nd_moment_method}
	If \eqref{eq:Z_partition_1} holds, then with probability bounded away from $0$, there exists $\ww_1 \in \R^{d - 1 - d_0}$, $\norm{\ww_1}_2 = 1$, such that
	\begin{equation}\label{eq:good_sample_2nd_moment_method}
		\uu^G + \ZZ_1^G \ww_1 \ge \kappa_2 \bone.
	\end{equation}
\end{lem}

\begin{lem}\label{lem:bad_sample_gordon}
	Assume \eqref{eq:Z_partition_2}, for any fixed $\ww_1 \in \R^{d - 1 - d_0}$, $\norm{\ww_1}_2 = 1$, with high probability there exists some $\ww_2 \in \R^{d_0}$, $\norm{\ww_2}_2 \le c$, such that the following inequalities hold:
	\begin{align*}
		\ZZ_2^G \ww_2 \ge & (\kappa_0 - \kappa_2) \bone, \\
		\ZZ_2^B \ww_2 \ge & \kappa_0 \bone - \uu^B - \ZZ_1^B \ww_1.
	\end{align*}
\end{lem}
Note that in Lemma~\ref{lem:bad_sample_gordon} we can assume that $\ww_1$ obtained by Lemma~\ref{lem:good_sample_2nd_moment_method} is fixed, since $\ww_1$ is only determined by $\uu^G$ and $\ZZ_1^G$, and is independent of other random variables. Hence, the validity of Lemma~\ref{lem:bad_sample_gordon} is preserved after conditioning on any given $\ww_1$. Now, combining the results of the above lemmas yields Eq.~\eqref{ineq:bad_and_good_sample}, thus proving that 
\begin{equation}
    \delta < \delta_{\sl} (\kappa; \varphi) \Rightarrow \liminf_{n \to \infty} \left( \exists \btheta \in \S^{d-1}, \ \text{s.t.} \ y_i \langle \xx_i, \btheta \rangle \ge \kappa, \ \forall i \in [n] \right) > 0.
\end{equation}
We now proceed to show that the above probability converges to one, using the same argument as in the proof of the pure noise case. Recall that the function $S_{\kappa}$ defined in the proof of Theorem~\ref{thm:pure_noise_lower_bound} is $(1/\sqrt{n})$-Lipschitz, and that by our assumption, $\yy \odot \XX \stackrel{d}{=} (\yy \odot \bG, \ZZ)$. Combining Lemma~\ref{lem:transport_yG} below and Theorem 4.31 in \cite{van2014probability}, we know that there exists some constant $C > 0$, such that
\begin{equation*}
	\P \left( \left\vert S_\kappa (\yy \odot \XX) - \E \left[ S_\kappa (\yy \odot \XX) \right] \right\vert \ge t \right) \le 2 \exp \left( - \frac{n t^2}{2 C} \right), \quad \forall t > 0.
\end{equation*}
Similarly, one can show that the function $\kappa \mapsto \delta_{\sl} (\kappa; \varphi)$ is continuous. Therefore, using the same argument as in the proof of Theorem~\ref{thm:pure_noise_lower_bound}, we deduce that there exists $\btheta \in \S^{d-1}$ such that $\norm{((\kappa + \eta) \bone - \yy \odot \XX \btheta)_+}_2 = o(\eta \sqrt{n})$ for some small $\eta > 0$. Note that the rounding lemma in \cite{alaoui2020algorithmic} is still applicable here since its proof only uses the fact that $s_{\rm min} (\XX) / \sqrt{n}$ is bounded away from $0$ with high probability, which is still true for $\yy \odot \XX$ (see Lemma~\ref{lem:norm_bd_yX}, here $s_{\rm min}$ represents the minimum singular value). Hence, we finally conclude that with high probability, there exists $\btheta' \in \S^{d-1}$, such that
\begin{equation*}
	\yy \odot \XX \btheta' \ge (\kappa + \eta - o(\eta)) \bone \ge \kappa \bone.
\end{equation*}
This proves the first part of Theorem~\ref{thm:logistic_lower_bound}.
 
\begin{lem}\label{lem:transport_yG}
	Let $(y, G)$ have the joint distribution $G \sim \normal (0, 1)$, and $\P (y=1 \vert G) = \varphi(G) = 1 - \P (y=-1 \vert G)$. Then, the law of $yG$ satisfies the $T_2$-inequality, i.e., there exists a constant $C > 0$ such that
	\begin{equation*}
		W_2 \left( \operatorname{Law}(yG), \nu \right) \le C \sqrt{D_{\rm KL} \left( \nu \Vert \operatorname{Law}(yG) \right)}
	\end{equation*}
	for all one-dimensional probability measure $\nu$.
\end{lem}

\begin{lem}\label{lem:norm_bd_yX}
	Assume $\delta > 1$, then for any $\veps > 0$, with high probability we have $s_{\rm min} (\yy \odot \XX) \ge (\sqrt{\delta} - 1 - \veps) \sqrt{d}$.
\end{lem}

\begin{lem}\label{lem:sig_lower_asymptotic}
	For any $\veps > 0$, there exists $\underline{\kappa} = \underline{\kappa} (\veps) <0$, such that for all $\kappa < \underline{\kappa}$ and
\begin{equation*}
	\delta < (1 - \veps) \frac{\sqrt{\pi}}{2 \sqrt{2}} \vert \kappa \vert \log \vert \kappa \vert \exp \left( \frac{\kappa^2}{2} + \alpha \vert \kappa \vert \right),
\end{equation*}
Assumption~\ref{cond:sig_lower_b} is true. Therefore, $\delta < \delta_{\sl}(\kappa;\varphi)$.
\end{lem}
Applying Lemma~\ref{lem:sig_lower_asymptotic} concludes the proof of Theorem~\ref{thm:logistic_lower_bound}. The proofs of Lemmas~\ref{lem:good_sample_2nd_moment_method}, \ref{lem:bad_sample_gordon}, \ref{lem:transport_yG}, \ref{lem:norm_bd_yX}, and \ref{lem:sig_lower_asymptotic} are deferred to Appendix~\ref{sec:append-signal-lemma}.

\subsection{Phase transition upper bound: Proof of Theorem~\ref{thm:logistic_upper_bound}}\label{sec:signal-upper}

The first part of Theorem~\ref{thm:logistic_upper_bound} is a direct consequence of Theorem~\ref{cor:rho_min_max}, and can be deduced by taking $J = [-1, 1]$ in the proof of Theorem~\ref{cor:rho_min_max}. Now we focus on the second part. We claim that, if $\vert \kappa \vert$ is sufficiently large, then with a proper choice of $c$ (the same for all $\rho \in [-1, 1]$), and
\begin{equation*}
	\delta > (1 + \veps) \frac{\sqrt{\pi}}{2 \sqrt{2}} \vert \kappa \vert \log \vert \kappa \vert \exp \left( \frac{\kappa^2}{2} + \alpha \vert \kappa \vert \right),
\end{equation*}
the following inequality holds:
\begin{small}
\begin{equation}\label{eq:uniform_negative_minimax}
	\sup_{\rho \in [-1, 1]} \left( \frac{\sqrt{1 - \rho^2}}{c \sqrt{1 - \rho^2} + \sqrt{c^2 (1 - \rho^2) + 4}} + \frac{1}{c} \log \left( \frac{c \sqrt{1 - \rho^2} + \sqrt{c^2 (1 - \rho^2) + 4}}{2} \right) - \inf_{u > 0} \left\{ \frac{c}{4 u} - \frac{\delta}{c} \log \psi_{\kappa, \rho} (-u) \right\} \right) < 0,
\end{equation}
\end{small}
which implies that for all $\rho \in [-1, 1]$,
\begin{small}
\begin{align}
	\inf_{c > 0} \left( \frac{\sqrt{1 - \rho^2}}{c \sqrt{1 - \rho^2} + \sqrt{c^2 (1 - \rho^2) + 4}} + \frac{1}{c} \log \left( \frac{c \sqrt{1 - \rho^2} + \sqrt{c^2 (1 - \rho^2) + 4}}{2} \right) - \inf_{u > 0} \left\{ \frac{c}{4 u} - \frac{\delta}{c} \log \psi_{\kappa, \rho} (-u) \right\} \right) < 0.
\end{align}
\end{small}
Hence, $\delta \ge \delta_{\su}(\kappa;\varphi)$ for all $\rho \in [-1, 1]$, leading to $\delta' \ge \delta_{\su}(\kappa, \rho;\varphi)$ as long as
\begin{equation*}
	\delta > \delta' > (1 + \veps) \frac{\sqrt{\pi}}{2 \sqrt{2}} \vert \kappa \vert \log \vert \kappa \vert \exp \left( \frac{\kappa^2}{2} + \alpha \vert \kappa \vert \right).
\end{equation*}
As a consequence, we obtain that $\delta > \delta_{\su}(\kappa;\varphi)$. To conclude the proof, it suffices to prove the claim~\eqref{eq:uniform_negative_minimax}. Let $\nu > 0$ be such that $(1+\nu)(1 + 3\nu)/(1 - \nu) < 1+\veps$, choose $c = \vert \kappa \vert^{1 + \nu}$, we first focus on the term
\begin{equation*}
	\inf_{u > 0} \left\{ \frac{c}{4 u} - \frac{\delta}{c} \log \psi_{\kappa, \rho} (-u) \right\}.
\end{equation*}
Consider the following two cases:

(1) $u \le \vert \kappa \vert^{2 + \nu}$, then we have
\begin{equation*}
	\frac{c}{4 u} - \frac{\delta}{c} \log \psi_{\kappa, \rho} (-u) \ge \frac{c}{4 u} \ge \frac{1}{4 \vert \kappa \vert}.
\end{equation*}

(2) $u > \vert \kappa \vert^{2 + \nu}$, then it follows that
\begin{equation*}
	\frac{c}{4 u} - \frac{\delta}{c} \log \psi_{\kappa, \rho} (-u) \ge - \frac{\delta}{c} \log \psi_{\kappa, \rho} (-u) \ge - \frac{\delta}{c} \log \psi_{\kappa, \rho} \left( - \vert \kappa \vert^{2 + \nu} \right) \stackrel{(i)}{=} (1 + \breve{o}_{\kappa} (1))\frac{\delta}{c} \left( 1 - \psi_{\kappa, \rho} \left( - \vert \kappa \vert^{2 + \nu} \right) \right),
\end{equation*}
where the approximation $(i)$ is valid, since Lemma~\ref{lem:tail} implies
\begin{equation*}
	0 \le 1 - \psi_{\kappa, \rho} \left( - \vert \kappa \vert^{2 + \nu} \right) \le \P \left( \rho YG + \sqrt{1 - \rho^2} W \le \kappa \right) = \breve{o}_{\kappa} (1). 
\end{equation*}
Now we want to further simplify the right hand side of the above inequality. Note that
\begin{align*}
	& 0 \le \P \left( \rho YG + \sqrt{1 - \rho^2} W \le \kappa \right) - \left( 1 - \psi_{\kappa, \rho} \left( - \vert \kappa \vert^{2 + \nu} \right) \right) \\ 
	= & \E \left[ \exp \left( - \vert \kappa \vert^{2 + \nu} \left( \kappa - \rho YG - \sqrt{1 - \rho^2} W \right)_+^2 \right) \bone \left\{ \rho YG + \sqrt{1 - \rho^2} W \le \kappa \right\} \right] \\
	= &  \E \left[ \exp \left( - \vert \kappa \vert^{2 + \nu} \left( \kappa - \rho YG - \sqrt{1 - \rho^2} W \right)_+^2 \right) \bone \left\{ \kappa - \frac{1}{\vert \kappa \vert^{1 + \nu/4}} < \rho YG + \sqrt{1 - \rho^2} W \le \kappa \right\} \right] \\
	& + \E \left[ \exp \left( - \vert \kappa \vert^{2 + \nu} \left( \kappa - \rho YG - \sqrt{1 - \rho^2} W \right)_+^2 \right) \bone \left\{ \rho YG + \sqrt{1 - \rho^2} W \le \kappa - \frac{1}{\vert \kappa \vert^{1 + \nu/4}} \right\} \right] \\
	\le & \P \left( \kappa - \frac{1}{\vert \kappa \vert^{1 + \nu/4}} < \rho YG + \sqrt{1 - \rho^2} W \le \kappa \right) + \exp \left( - \vert \kappa \vert^{\nu/2} \right) \P \left( \rho YG + \sqrt{1 - \rho^2} W \le \kappa - \frac{1}{\vert \kappa \vert^{1 + \nu/4}} \right).
\end{align*}
From Lemma~\ref{lem:tail}, we also know that uniformly for all $\rho \in [-1, 1]$,
\begin{equation*}
	\P \left( \rho YG + \sqrt{1 - \rho^2} W \le \kappa - \frac{1}{\vert \kappa \vert^{1 + \nu/4}} \right) = (1+\breve{o}_{\kappa} (1)) \P \left( \rho YG + \sqrt{1 - \rho^2} W \le \kappa \right),
\end{equation*}
thus leading to
\begin{align*}
	& \P \left( \rho YG + \sqrt{1 - \rho^2} W \le \kappa \right) - \left( 1 - \psi_{\kappa, \rho} \left( - \vert \kappa \vert^{2 + \nu} \right) \right) = \breve{o}_{\kappa} ( 1 ) \cdot \P \left( \rho YG + \sqrt{1 - \rho^2} W \le \kappa \right) \\
	\implies & 1 - \psi_{\kappa, \rho} \left( - \vert \kappa \vert^{2 + \nu} \right) = (1 + \breve{o}_{\kappa} (1)) \P \left( \rho YG + \sqrt{1 - \rho^2} W \le \kappa \right) \\
	\implies & \frac{c}{4 u} - \frac{\delta}{c} \log \psi_{\kappa, \rho} (-u) \ge (1 - \nu) \frac{\delta}{c} \P \left( \rho YG + \sqrt{1 - \rho^2} W \le \kappa \right)
\end{align*}
when $\vert \kappa \vert$ is large enough. To summarize, we have
\begin{equation}\label{eq:part_0}
	\inf_{u > 0} \left\{ \frac{c}{4 u} - \frac{\delta}{c} \log \psi_{\kappa, \rho} (-u) \right\} \ge \min \left\{ \frac{1}{4 \vert \kappa \vert}, (1 - \nu) \frac{\delta}{c} \P \left( \rho YG + \sqrt{1 - \rho^2} W \le \kappa \right) \right\}.
\end{equation}
It will be useful to remember the fact that
\begin{equation*}
	\frac{\sqrt{1 - \rho^2}}{c \sqrt{1 - \rho^2} + \sqrt{c^2 (1 - \rho^2) + 4}} + \frac{1}{c} \log \left( \frac{c \sqrt{1 - \rho^2} + \sqrt{c^2 (1 - \rho^2) + 4}}{2} \right)
\end{equation*}
is an increasing function of $\sqrt{1 - \rho^2}$, which we will apply for several times later. To show Eq.~\eqref{eq:uniform_negative_minimax}, we proceed with the following two parts:

\noindent \textbf{Part 1.} 
Since $c = \vert \kappa \vert^{1 + \nu}$, for sufficiently large $\vert \kappa \vert$ we obtain that
\begin{equation}\label{eq:part_1}
\begin{split}
	& \sup_{\rho \in [-1, 1]} \left( \frac{\sqrt{1 - \rho^2}}{c \sqrt{1 - \rho^2} + \sqrt{c^2 (1 - \rho^2) + 4}} + \frac{1}{c} \log \left( \frac{c \sqrt{1 - \rho^2} + \sqrt{c^2 (1 - \rho^2) + 4}}{2} \right) - \frac{1}{4 \vert \kappa \vert} \right) \\
	\le & \frac{1}{c + \sqrt{c^2 + 4}} + \frac{1}{c} \log \left( \frac{c + \sqrt{c^2 + 4}}{2} \right) - \frac{1}{4 \vert \kappa \vert} \le (1 + \nu) \frac{\log c}{c} - \frac{1}{4 \vert \kappa \vert} = (1 + \nu)^2 \frac{\log \vert \kappa \vert}{\vert \kappa \vert^{1 + \nu}} - \frac{1}{4 \vert \kappa \vert} < 0.
	\end{split}
\end{equation}

\noindent \textbf{Part 2.}
Also, when $\vert \kappa \vert$ is large, using Lemma~\ref{lem:tail} implies that uniformly for all $\rho < 1 - \vert \kappa \vert^{-1 + \nu}$, 
\begin{align*}
    & (1 - \nu) \frac{\delta}{c} \P \left( \rho YG + \sqrt{1 - \rho^2} W \le \kappa \right) \\
    \ge & (1 + \breve{o}_{\kappa} (1)) \frac{(1 + \veps) (1 - \nu)}{2 c} \log \vert \kappa \vert \exp \left( \alpha \vert \kappa \vert^{-1 + \nu} \vert \kappa \vert \right) \\
	\ge & (1 + \nu)^2 \frac{\log \vert \kappa \vert}{c} = (1 + \nu) \frac{\log c}{c} \\
	\ge & \frac{1}{c + \sqrt{c^2 + 4}} + \frac{1}{c} \log \left( \frac{c + \sqrt{c^2 + 4}}{2} \right) \\
	\ge & \frac{\sqrt{1 - \rho^2}}{c \sqrt{1 - \rho^2} + \sqrt{c^2 (1 - \rho^2) + 4}} + \frac{1}{c} \log \left( \frac{c \sqrt{1 - \rho^2} + \sqrt{c^2 (1 - \rho^2) + 4}}{2} \right).
\end{align*}
Similarly, if $\rho \ge 1 - \vert \kappa \vert^{-1 + \nu}$, one gets that $ \sqrt{1 - \rho^2} \le \sqrt{2} \vert \kappa \vert^{-1/2 + \nu/2}$, hence the following inequalities hold uniformly for all such $\rho$ and sufficiently large $\vert \kappa \vert$:
\begin{align*}
	& \frac{\sqrt{1 - \rho^2}}{c \sqrt{1 - \rho^2} + \sqrt{c^2 (1 - \rho^2) + 4}} + \frac{1}{c} \log \left( \frac{c \sqrt{1 - \rho^2} + \sqrt{c^2 (1 - \rho^2) + 4}}{2} \right) \\
	\le & \sqrt{2} \vert \kappa \vert^{-1/2 + \nu/2} \left( \frac{1}{\sqrt{2} \vert \kappa \vert^{(1+3 \nu)/2} + \sqrt{2 \vert \kappa \vert^{1+3 \nu} + 4}} + \frac{1}{\sqrt{2} \vert \kappa \vert^{(1+3 \nu)/2}} \log \left( \frac{\sqrt{2} \vert \kappa \vert^{(1+3 \nu)/2} + \sqrt{2 \vert \kappa \vert^{1+3 \nu} + 4}}{2} \right) \right) \\
	\le & \sqrt{2} \vert \kappa \vert^{-1/2 + \nu/2} (1 + \nu) \frac{\log \left( \sqrt{2} \vert \kappa \vert^{(1+3 \nu)/2} \right)}{\sqrt{2} \vert \kappa \vert^{(1+3 \nu)/2}} = \frac{(1 + \nu) (1+3 \nu)}{2} \frac{(1 + \breve{o}_{\kappa} (1)) \log \vert \kappa \vert}{\vert \kappa \vert^{1 + \nu}} \\
	\stackrel{(i)}{<} & (1 + \breve{o}_{\kappa} (1)) \frac{(1 + \veps) (1 - \nu)}{2 c} \log \vert \kappa \vert \stackrel{(ii)}{\le} (1 - \nu) \frac{\delta}{c} \P \left( \rho YG + \sqrt{1 - \rho^2} W \le \kappa \right),
\end{align*}
where $(i)$ follows from our assumption $(1+\nu)(1 + 3\nu)/(1 - \nu) < 1+\veps$, $(ii)$ follows from Lemma~\ref{lem:tail}. Now we have proved that
\begin{footnotesize}
\begin{equation*}
	\sup_{\rho \in [-1, 1]} \left( \frac{\sqrt{1 - \rho^2}}{c \sqrt{1 - \rho^2} + \sqrt{c^2 (1 - \rho^2) + 4}} + \frac{1}{c} \log \left( \frac{c \sqrt{1 - \rho^2} + \sqrt{c^2 (1 - \rho^2) + 4}}{2} \right) - (1 - \nu) \frac{\delta}{c} \P \left( \rho YG + \sqrt{1 - \rho^2} W \le \kappa \right) \right) < 0,
\end{equation*}
\end{footnotesize}
combining this with equations~\eqref{eq:part_0} and \eqref{eq:part_1} gives Eq.~\eqref{eq:uniform_negative_minimax} and finishes proving Theorem~\ref{thm:logistic_upper_bound}.

\subsection{Solution space geometry: Proofs of Theorems~\ref{cor:rho_min_max}, \ref{thm:asymptotic_upper_bound} and \ref{thm:best_error}}\label{sec:signal-geom}

\begin{proof}[\bf Proof of Theorem~\ref{cor:rho_min_max}]
Let $J = [- 1, \rho_{\min}] \cup [\rho_{\max}, 1]$, then we know that $\delta > \delta_{\su}(\kappa, \rho;\varphi)$ for all $\rho \in J$. As in the proof of Theorem~\ref{thm:logistic_lower_bound}, we may assume that $\btheta^* = (1, 0, \cdots, 0)^\top$ and write $y_i \langle \xx_i, \btheta \rangle$ as $\rho y_i G_i + \sqrt{1 - \rho^2} \langle \zz_i, \ww \rangle$, where $\{(y_i, G_i, \zz_i)\}_{1 \le i \le n}$ are i.i.d., each have joint distribution:
\begin{equation*}
	(y_i, G_i) \bot \zz_i, \ \zz_i \sim \cN (\bzero, \bI_{d-1}), \ G_i \sim \cN (0, 1), \ \P(y_i = 1 \vert G_i) = \varphi (G_i) = 1 - \P (y_i = -1 \vert G_i).
\end{equation*}
Now for any fixed $\rho \in [-1, 1]$, define
\begin{equation*}
	\xi_{n, \kappa, \rho} = \min_{\norm{\ww}_2 = 1} \max_{\norm{\blambda}_2 = 1, \blambda \ge \bzero} \frac{1}{\sqrt{d}} \blambda^\top \left( \kappa \bone - \rho \yy \odot \bG - \sqrt{1 - \rho^2} \ZZ \ww \right),
\end{equation*}
then we know that
\begin{equation*}
	\exists \btheta \in \S^{d-1}, \left\langle \btheta, \btheta^* \right\rangle \in J, \ \text{s.t.} \ \forall 1 \le i \le n, y_i \left\langle \xx_i, \btheta \right\rangle \ge \kappa \iff \exists \rho \in J, \ \text{s.t.} \ \xi_{n, \kappa, \rho} \le 0.
\end{equation*}
Hence, it suffices to prove that
\begin{equation*}
	\lim_{n \to +\infty} \P \left( \min_{\rho \in J} \xi_{n, \kappa, \rho} \le 0 \right) = 0.
\end{equation*}
To this end, we proceed with the following two steps:

\noindent \textbf{Step 1. Control the probability $\P (\xi_{n, \kappa, \rho} \le \eta)$ for $\eta > 0$.}

\noindent The argument in this part mainly uses the modified Gordon's inequality, i.e., Eq.~\eqref{eq:modified_Gordon}, and is almost identical to what appears in the proof of Theorem~\ref{thm:pure_noise_upper_bound}, so we will omit some technical details here. Set $\psi(x) = - \exp(- cnx)$ for some $c > 0$ (may depend on $\kappa$ and $\rho$), and combining Markov's inequality with Eq.~\eqref{eq:exponential_Gordon} yields
\begin{align*}
    & \P \left( \xi_{n, \kappa, \rho} \le \eta \right) \le \exp \left( c n \eta \right) \E \left[ \exp \left( - c n \xi_{n, \kappa, \rho} \right) \right] \\
    \le & \exp \left( cn \eta \right) \E \left[ \exp \left( - \frac{cn \sqrt{1 - \rho^2}}{\sqrt{d}} z \right) \right]^{-1} \E \left[ \exp \left( \frac{cn \sqrt{1 - \rho^2}}{\sqrt{d}} \norm{\bgg}_2 \right) \right] \\
    & \times \E \left[ \exp \left( - \frac{cn}{\sqrt{d}} \max_{\norm{\blambda}_2 = 1, \blambda \ge \bzero} \blambda^\top \left( \kappa \bone - \rho \yy \odot \bG - \sqrt{1-\rho^2} \hh \right) \right) \right],
\end{align*}
where $z \sim \cN (0, 1)$, $\bgg \sim \cN (\bzero, \bI_{d-1})$, $\hh \sim \cN (\bzero, \bI_n)$ are independent, and further independent of $\yy \odot \bG$.

Now we calculate the asymptotics of the right hand side in the above inequality. First, similarly as before, we apply Lemma~\ref{lem:varadhan} to obtain that
\begin{align*}
	& \lim_{n \to +\infty} \frac{1}{n} \log \E \left[ \exp \left( \frac{cn \sqrt{1 - \rho^2}}{\sqrt{d}} \norm{\bgg}_2 \right) \right] \\
	= & \frac{1}{\delta} \left( \frac{c^2 \delta^2 (1 - \rho^2) + c \delta \sqrt{1 - \rho^2} \sqrt{c^2 \delta^2 (1 - \rho^2) + 4}}{4} - \log \left( \frac{\sqrt{c^2 \delta^2 (1 - \rho^2) + 4} - c \delta \sqrt{1 - \rho^2}}{2} \right) \right).
\end{align*}
Next, with a slight modification on the argument in the proof of Theorem~\ref{thm:pure_noise_upper_bound} (conditioning on $\yy \odot \bG$, since it's independent of $\hh$), we can show that
\begin{equation*}
	\lim_{n \to +\infty} \frac{1}{n} \log \E \left[ \exp \left( - \frac{cn}{\sqrt{d}} \max_{\norm{\blambda}_2 = 1, \blambda \ge \bzero} \blambda^\top \left( \kappa \bone - \rho \yy \odot \bG - \sqrt{1-\rho^2} \hh \right) \right) \right] = - \inf_{u > 0} \left\{ \frac{c^2 \delta}{4 u} - \log \psi_{\kappa, \rho} (-u) \right\},
\end{equation*}
where for $u > 0$,
\begin{equation*}
	\psi_{\kappa, \rho} (- u) = \E \left[ \exp \left( -u \left( \kappa - \rho YG - \sqrt{1 - \rho^2} W \right)_+^2 \right) \right].
\end{equation*}
Therefore, after making a change of variable ($c \mapsto c/\delta$), it finally follows that
\begin{small}
	\begin{align*}
	& \limsup_{n \to +\infty} \frac{1}{n} \log \P \left( \xi_{n, \kappa, \rho} \le \eta \right) \le \frac{c \eta}{\delta} - \frac{c^2 (1 - \rho^2)}{2 \delta} + \frac{1}{\delta} \frac{c^2 (1 - \rho^2) + c \sqrt{1 - \rho^2} \sqrt{c^2 (1 - \rho^2) + 4}}{4} \\
	& - \frac{1}{\delta} \log \left( \frac{\sqrt{c^2 (1 - \rho^2) + 4} - c \sqrt{1 - \rho^2}}{2} \right) - \inf_{u > 0} \left\{ \frac{c^2}{4 \delta u} - \log \psi_{\kappa, \rho} (-u) \right\} \\
	= & \frac{c}{\delta} \left( \eta + \frac{\sqrt{1 - \rho^2}}{c \sqrt{1 - \rho^2} + \sqrt{c^2 (1 - \rho^2) + 4}} + \frac{1}{c} \log \left( \frac{c \sqrt{1 - \rho^2} + \sqrt{c^2 (1 - \rho^2) + 4}}{2} \right) - \inf_{u > 0} \left\{ \frac{c}{4 u} - \frac{\delta}{c} \log \psi_{\kappa, \rho} (-u) \right\} \right).
\end{align*}
\end{small}
Now we define for $\eta > 0$ and $\rho \in [-1, 1]$:
\begin{footnotesize}
\begin{equation*}
	f_{\delta} (\eta, \rho) = \inf_{c > 0} \left\{ \frac{c}{\delta} \left( \eta + \frac{\sqrt{1 - \rho^2}}{c \sqrt{1 - \rho^2} + \sqrt{c^2 (1 - \rho^2) + 4}} + \frac{1}{c} \log \frac{c \sqrt{1 - \rho^2} + \sqrt{c^2 (1 - \rho^2) + 4}}{2} - \inf_{u > 0} \left\{ \frac{c}{4 u} - \frac{\delta}{c} \log \psi_{\kappa, \rho} (-u) \right\} \right) \right\},
\end{equation*}
\end{footnotesize}
then the above argument implies that
\begin{equation*}
	\limsup_{n \to +\infty} \frac{1}{n} \log \P \left( \xi_{n, \kappa, \rho} \le \eta \right) \le f_\delta (\eta, \rho).
\end{equation*} 
Moreover, by definition of $\delta_{\su}(\kappa, \rho;\varphi)$ we know that
\begin{equation*}
	\delta > \delta_{\su}(\kappa, \rho;\varphi) \implies \exists \eta_\rho > 0, f_\delta (\eta_\rho, \rho) < 0,
\end{equation*}
therefore $\forall \rho \in J$, there exists a $\eta_\rho > 0$ such that
\begin{equation*}
	\limsup_{n \to +\infty} \frac{1}{n} \log \P \left( \xi_{n, \kappa, \rho} \le \eta_\rho \right) \le f_\delta (\eta_\rho, \rho) < 0.
\end{equation*}

We will show that under the conditions of Theorem~\ref{cor:rho_min_max},
\begin{equation}\label{eq:uniform_negative_f_delta}
	\exists \eta > 0, \ \text{s.t.} \ \sup_{\rho \in J} f_\delta (\eta, \rho) < 0.
\end{equation}
In order to prove \eqref{eq:uniform_negative_f_delta}, we introduce the following useful property:

\begin{lem}\label{lem:continuity_of_f_delta}
For any fixed $\eta > 0$, $f_\delta (\eta, \rho)$ is upper semicontinuous in $\rho \in [-1, 1]$.
\end{lem}

The proof of Lemma~\ref{lem:continuity_of_f_delta} is provided in Appendix~\ref{sec:append-signal-lemma}. Assuming its correctness, we are now in position to prove Eq.~\eqref{eq:uniform_negative_f_delta} by contradiction. If for any $\eta > 0$, $\sup_{\rho \in J} f_\delta (\eta, \rho) \ge 0$, then according to the compactness of $J$ and Lemma~\ref{lem:continuity_of_f_delta}, there exists $\rho_\eta \in J$ such that $f_\delta (\eta, \rho_\eta) \ge 0$. Hence, we can find a sequence of $(\eta_n, \rho_n)$ satisfying $\eta_n \to 0^+$, $\rho_n \in J$, and $f_\delta (\eta_n, \rho_n) \ge 0$. Now again since $J$ is compact, $\{ \rho_n \}$ has a limit point $\rho \in J$ (WLOG we may assume $\rho_n \to \rho$, otherwise just recast the convergent subsequence as $\{ \rho_n \}$), and we can find a $\eta_\rho > 0$ such that $f_\delta (\eta_\rho, \rho) < 0$, thus leading to
\begin{equation*}
	0 > f_\delta (\eta_\rho, \rho) \stackrel{(i)}{\ge} \limsup_{n \to +\infty} f_\delta (\eta_\rho, \rho_n) \stackrel{(ii)}{\ge} \limsup_{n \to +\infty} f_\delta (\eta_n, \rho_n) \ge 0,
\end{equation*}
where $(i)$ is due to Lemma~\ref{lem:continuity_of_f_delta}, $(ii)$ is due to $\eta_n \to 0^+$ and the fact that $f_\delta (\eta, \rho)$ is an increasing function of $\eta$ for any fixed $\rho \in [-1, 1]$. Therefore, a contradiction occurs and Eq.~\eqref{eq:uniform_negative_f_delta} follows.

\noindent \textbf{Step 2. Covering argument.}

\noindent By definition of $\xi_{n, \kappa, \rho}$, one has for any $\rho_1, \rho_2 \in [-1, 1]$,
\begin{align*}
	& \vert \xi_{n, \kappa, \rho_1} - \xi_{n, \kappa, \rho_2} \vert \le \max_{\norm{\ww}_2 = 1} \max_{\norm{\blambda}_2 = 1, \blambda \ge \bzero} \bigg\vert \frac{1}{\sqrt{d}} \blambda^\top (\rho_2 - \rho_1 ) \yy \odot \bG - \frac{1}{\sqrt{d}} \blambda^\top \left( \sqrt{1 - \rho_1^2} - \sqrt{1 - \rho_2^2} \right) \ZZ \ww \bigg\vert \\
	\le & \frac{1}{\sqrt{d}} \vert \rho_2 - \rho_1 \vert \norm{\yy \odot \bG}_2 + \frac{1}{\sqrt{d}} \bigg\vert \sqrt{1 - \rho_1^2} - \sqrt{1 - \rho_2^2} \bigg\vert \norm{\ZZ}_{\mathrm{op}},
\end{align*}
where the last line just follows from the definition of operator norm. According to Cramer's theorem and Theorem 4.4.5 of \cite{vershynin2018high}, we know there exist constants $C_0, C_1, C_2 > 0$, such that (note $n/d \to \delta$)
\begin{equation*}
	\max \left\{ \P \left( \norm{\yy \odot \bG}_2 > C_0 \sqrt{d} \right), \P \left( \norm{\ZZ}_{\mathrm{op}} > C_0 \sqrt{d} \right) \right\} < C_1 \exp(-C_2 n),
\end{equation*}
Now we choose a finite covering $\{ \rho_1, \cdots, \rho_N \}$ of $J$, satisfying that
\begin{equation*}
	\forall \rho \in J, \ \exists 1 \le i \le N, \ \vert \rho - \rho_i \vert + \bigg\vert \sqrt{1 - \rho^2} - \sqrt{1 - \rho_i^2} \bigg\vert < \frac{\eta}{C_0},
\end{equation*}
hence if $\xi_{n, \kappa, \rho} \le 0$, and $\max \{ \norm{\yy \odot \bG}_2, \norm{\ZZ}_{\mathrm{op}} \} \le C_0 \sqrt{d}$, then we must have $\xi_{n, \kappa, \rho_i} \le \eta$.

As a consequence, we deduce that
\begin{align*}
	\P \left( \min_{\rho \in J} \xi_{n, \kappa, \rho} \le 0 \right) \le & \P \left( \max \left\{ \norm{\yy \odot \bG}_2, \norm{\ZZ}_{\mathrm{op}} \right\} > C_0 \sqrt{d} \right) + \P \left( \min_{1 \le i \le N} \xi_{n, \kappa, \rho_i} \le \eta \right) \\
	\le & 2 C_1 \exp(-C_2 n) + N \max_{1 \le i \le N} \P \left( \xi_{n, \kappa, \rho_i} \le \eta \right) \to 0,
\end{align*}
where $\max_{1 \le i \le N} \P \left( \xi_{n, \kappa, \rho_i} \le \eta \right) \to 0$ follows from Eq.~\eqref{eq:uniform_negative_f_delta}. This concludes the proof.
\end{proof}

\begin{proof}[\bf Proof of Theorem~\ref{thm:asymptotic_upper_bound}]
To begin with, recall the definition of $\delta_{\su}(\kappa;\varphi)$, which is the infimum of all $\delta$ such that $\exists c > 0$, such that
\begin{equation*}
	\frac{1}{c + \sqrt{c^2 + 4}} + \frac{1}{c} \log \frac{c + \sqrt{c^2 + 4}}{2} < \inf_{u > 0} \left\{ \frac{c}{4 u (1-\rho^2)} - \frac{\delta}{c} \log \psi_{\kappa, \rho} (-u) \right\}.
\end{equation*}
Now we rescale by
\begin{equation*}
	1 - \rho^2 = \frac{a}{\kappa^2}, \ u = t \kappa^2, \ \delta = \frac{b}{\P ( \rho YG + \sqrt{1 - \rho^2} W < \kappa )},
\end{equation*}
it follows that
\begin{align*}
	& \frac{c}{4 u (1-\rho^2)} - \frac{\delta}{c} \log \psi_{\kappa, \rho} (-u) = \frac{c}{4 t a} - \frac{b}{c} \frac{\log \psi_{\kappa, \rho} (- t \kappa^2)}{\P ( \rho YG + \sqrt{1 - \rho^2} W < \kappa )} \\
	\stackrel{(i)}{=} &  \frac{c}{4 t a} + (1 + \breve{o}_{\kappa}(1)) \frac{b}{c} \frac{1 - \psi_{\kappa, \rho} (- t \kappa^2)}{\P ( \rho YG + \sqrt{1 - \rho^2} W < \kappa )},
\end{align*}
where $(i)$ is due to $\log \psi_{\kappa, \rho} (- t \kappa^2) = - (1 + \breve{o}_{\kappa} (1)) (1 - \psi_{\kappa, \rho} (- t \kappa^2))$, since $0 \le 1 - \psi_{\kappa, \rho} (- t \kappa^2) \le \P (\rho YG + \sqrt{1 - \rho^2}W < \kappa) = \breve{o}_{\kappa} (1)$.
Using integration by parts, we obtain that
\begin{align*}
	& 1 - \psi_{\kappa, \rho} \left(- t \kappa^2 \right) = 1 - \E \left[ \exp \left( -t \kappa^2 \left( \kappa - \rho YG - \sqrt{1 - \rho^2} W \right)_+^2 \right) \right] \\
	= & 1 - \left( 1 - \int_{\R} 2 t \kappa^2 (\kappa - x)_+ \exp \left(- t \kappa^2 (\kappa - x)_+^2 \right) \P \left( \rho YG + \sqrt{1 - \rho^2} W < x \right) \d x \right) \\
	= & \int_{0}^{+\infty} 2 t \kappa^2 x \exp \left( -t \kappa^2 x^2 \right) \P \left( \rho YG + \sqrt{1 - \rho^2} W < \kappa - x \right) \d x,
\end{align*}
leading to
\begin{align*}
	& \frac{1 - \psi_{\kappa, \rho} (- t \kappa^2)}{\P ( \rho YG + \sqrt{1 - \rho^2} W < \kappa )} = \int_{0}^{+\infty} 2 t \kappa^2 x \exp \left( -t \kappa^2 x^2 \right) \frac{\P \left( \rho YG + \sqrt{1 - \rho^2} W < \kappa - x \right)}{\P \left( \rho YG + \sqrt{1 - \rho^2} W < \kappa \right)} \d x \\
	\stackrel{(i)}{=} & (1 + \breve{o}_{\kappa}(1)) \int_{0}^{+\infty} 2 t \kappa^2 x \exp \left( -t \kappa^2 x^2 \right) \frac{\vert \kappa \vert}{\vert \kappa - x \vert} \exp \left( - \frac{x^2}{2} + \kappa x - \alpha \rho_+ x \right) \d x \\
	\stackrel{(ii)}{=} & (1 + \breve{o}_{\kappa}(1)) \int_{0}^{+\infty} 2 t s \exp \left( -t s^2 \right) \frac{\kappa^2}{\kappa^2 + s} \exp \left( - \frac{s^2}{2 \kappa^2} - s - \frac{\alpha \rho_+ s}{\vert \kappa \vert} \right) \d s,
\end{align*}
where $\rho_+ = \max (\rho, 0)$, $(i)$ follows from Lemma~\ref{lem:tail}, and $(ii)$ follows from the change of variable $s = \vert \kappa \vert x$. Now we claim that uniformly for all $t > 0$ and $\rho \in [-1, 1]$,
\begin{equation*}
	\int_{0}^{+\infty} 2 t s \exp \left( -t s^2 \right) \frac{\kappa^2}{\kappa^2 + s} \exp \left( - \frac{s^2}{2 \kappa^2} - s - \frac{\alpha \rho_+ s}{\vert \kappa \vert} \right) \d s = (1 + \breve{o}_{\kappa}(1)) \int_{0}^{+\infty} 2 t s \exp \left( -t s^2 - s \right) \d s,
\end{equation*}
which is equivalent to
\begin{equation*}
	\lim_{\kappa \to -\infty} \sup_{\rho \in [-1, 1], t > 0} \left\vert \frac{\int_{0}^{+\infty} 2 t s \exp \left( -t s^2 \right) \frac{\kappa^2}{\kappa^2 + s} \exp \left( - \frac{s^2}{2 \kappa^2} - s - \frac{\alpha \rho_+ s}{\vert \kappa \vert} \right) \d s}{\int_{0}^{+\infty} 2 t s \exp \left( -t s^2 - s \right) \d s} - 1 \right\vert = 0.
\end{equation*}
In what follows we prove the above claim. Consider the following two cases:

(1) $0 < t \le 1$. We obtain that
\begin{align*}
	& \left\vert \int_{0}^{+\infty} 2 t s \exp \left( -t s^2 \right) \frac{\kappa^2}{\kappa^2 + s} \exp \left( - \frac{s^2}{2 \kappa^2} - s - \frac{\alpha \rho_+ s}{\vert \kappa \vert} \right) \d s - \int_{0}^{+\infty} 2 t s \exp \left( -t s^2 - s \right) \d s \right\vert \\
	= & \left\vert \int_{0}^{+\infty} 2 t s \exp \left( -t s^2 - s \right) \left( \frac{\kappa^2}{\kappa^2 + s} \exp \left( - \frac{s^2}{2 \kappa^2} - \frac{\alpha \rho_+ s}{\vert \kappa \vert} \right) - 1 \right) \d s \right\vert \\
	\le & 2 t \int_{0}^{+\infty} s \exp \left( - s \right) \left\vert \frac{\kappa^2}{\kappa^2 + s} \exp \left( - \frac{s^2}{2 \kappa^2} - \frac{\alpha \rho_+ s}{\vert \kappa \vert} \right) - 1 \right\vert \d s \\
	\stackrel{(i)}{=} & 2t \breve{o}_{\kappa} (1) = 2t \breve{o}_{\kappa} (1) \int_{0}^{+\infty} s \exp \left( - s^2 - s \right) \d s \le \breve{o}_{\kappa} (1) \int_{0}^{+\infty} 2 t s \exp \left( -t s^2 - s \right) \d s,
\end{align*}
where $(i)$ is a consequence of the Dominated Convergence Theorem.

(2) $t > 1$. Then making the change of variable $x = \sqrt{t} s$ gives
\begin{align*}
	& \left\vert \int_{0}^{+\infty} 2 t s \exp \left( -t s^2 \right) \frac{\kappa^2}{\kappa^2 + s} \exp \left( - \frac{s^2}{2 \kappa^2} - s - \frac{\alpha \rho_+ s}{\vert \kappa \vert} \right) \d s - \int_{0}^{+\infty} 2 t s \exp \left( -t s^2 - s \right) \d s \right\vert \\
	= & \left\vert \int_{0}^{+\infty} 2 x \exp \left( - x^2 \right) \frac{\sqrt{t} \kappa^2}{\sqrt{t} \kappa^2 + x} \exp \left( - \frac{x^2}{2 t \kappa^2} - \frac{x}{\sqrt{t}} - \frac{\alpha \rho_+ x}{\sqrt{t} \vert \kappa \vert} \right) \d x - \int_{0}^{+\infty} 2 x \exp \left( -x^2 - \frac{x}{\sqrt{t}} \right) \d x \right\vert \\
	\le & \left\vert \int_{0}^{+\infty} 2 x \exp \left( -x^2 - \frac{x}{\sqrt{t}} \right) \left( \frac{\sqrt{t} \kappa^2}{\sqrt{t} \kappa^2 + x} \exp \left( - \frac{x^2}{2 t \kappa^2} - \frac{\alpha \rho_+ x}{\sqrt{t} \vert \kappa \vert} \right) - 1 \right) \d x \right\vert \\
	\le & \int_{0}^{+\infty} 2 x \exp \left( -x^2 \right) \left\vert \frac{\sqrt{t} \kappa^2}{\sqrt{t} \kappa^2 + x} \exp \left( - \frac{x^2}{2 t \kappa^2} - \frac{\alpha \rho_+ x}{\sqrt{t} \vert \kappa \vert} \right) - 1 \right\vert \d x \\
	\stackrel{(i)}{=} & \breve{o}_{\kappa} (1) = \breve{o}_{\kappa} (1) \int_{0}^{+\infty} 2 x \exp \left( -x^2 - x \right) \d x \le \breve{o}_{\kappa}(1) \int_{0}^{+\infty} 2 x \exp \left( -x^2 - \frac{x}{\sqrt{t}} \right) \d x \\
	= & \breve{o}_{\kappa}(1) \int_{0}^{+\infty} 2 t s \exp \left( -t s^2 - s \right) \d s,
\end{align*}
where $(i)$ again results from the Dominated Convergence Theorem, and the fact that $\sqrt{t} \kappa \to - \infty$.

Combining Case 1 and 2 immediately yields our claim, hence we deduce that
\begin{equation*}
	\frac{c}{4 u (1-\rho^2)} - \frac{\delta}{c} \log \psi_{\kappa, \rho} (-u) = \frac{c}{4 t a} + (1 + \breve{o}_{\kappa} (1)) \frac{b}{c} \int_{0}^{+\infty} 2 t s \exp \left( -t s^2 - s \right) \d s.
\end{equation*}
Now according to the definition of $\mathcal{D} (a)$ in Theorem~\ref{thm:asymptotic_upper_bound}, we finally obtain that
\begin{equation*}
	\delta_{\su}(\kappa;\varphi) = \frac{(1 + \breve{o}_{\kappa} (1)) \mathcal{D} (a)}{\P ( \rho YG + \sqrt{1 - \rho^2} W < \kappa )} = \frac{(1 + \breve{o}_{\kappa} (1)) \mathcal{D} (\kappa^2 (1 - \rho^2))}{\P ( \rho YG + \sqrt{1 - \rho^2} W < \kappa )}.
\end{equation*}
This concludes the proof.
\end{proof}

\begin{proof}[\bf Proof of Theorem~\ref{thm:best_error}]
	To begin with, we collect some important properties of $\mathcal{D} (a)$ in the following proposition (proved in Appendix~\ref{sec:append-signal-lemma}), which allows us to determine the asymptotic behavior of $\delta_{\su}(\kappa;\varphi)$ explicitly within some regime.

\begin{prop}\label{prop:properties_of_d(a)}
\hspace{2em}
\begin{enumerate}
	\item[(a)] $\mathcal{D} (a)$ is asymptotic to $(\log a)/2$ as $a \to +\infty$. Consequently, when $\kappa^2 (1 - \rho^2) \to +\infty$ we have
\begin{equation*}
	\delta_{\su}(\kappa;\varphi) = (1 + \breve{o}_{\kappa} (1)) \frac{\log(\vert \kappa \vert \sqrt{1 - \rho^2})}{\P ( \rho YG + \sqrt{1 - \rho^2} W < \kappa )}.
\end{equation*}

    \item[(b)] $\mathcal{D} (a) = a/2$ when $a \le 2$. Therefore, if $1 - \rho^2 \le 2/\kappa^2$, then we get that
\begin{equation*}
	\delta_{\su}(\kappa;\varphi) = (1 + \breve{o}_{\kappa}(1)) \frac{\kappa^2 (1 - \rho^2)}{2 \P ( \rho YG + \sqrt{1 - \rho^2} W < \kappa )}.
\end{equation*}
\end{enumerate}

\end{prop}
Now we turn to the proof of Theorem~\ref{thm:best_error}. Note that $\mathcal{D}^{-1}$ is an increasing function, we first show that for any $\veps > 0$, there exists a $\underline{\kappa} = \underline{\kappa} (\veps) < 0$, such that as long as $\kappa <  \underline{\kappa}$, with high probability we have
\begin{equation}\label{eq:max_corr}
	\rho_{\max} \le 1 - \frac{1}{2 \kappa^2} \mathcal{D}^{-1} \left( (1 - \veps) \delta \P (YG < \kappa) \right).
\end{equation}
Assume that
\begin{equation*}
	\rho \ge 1 - \frac{1}{2 \kappa^2} \mathcal{D}^{-1} \left( (1 - \veps) \delta \P (YG < \kappa) \right),
\end{equation*}
then we have the following estimate:
\begin{equation*}
	\kappa^2 (1 - \rho^2) \le 2 \kappa^2 (1 - \rho) \le \mathcal{D}^{-1} \left( (1 - \veps) \delta \P (YG < \kappa) \right).
\end{equation*}
Using the conclusion of Theorem~\ref{thm:asymptotic_upper_bound}, it follows that
\begin{align*}
	\delta_{\su}(\kappa, \rho;\varphi) = & \frac{(1 + \breve{o}_{\kappa} (1)) \mathcal{D} (\kappa^2 (1 - \rho^2))}{\P ( \rho YG + \sqrt{1 - \rho^2} W < \kappa )} \le \frac{(1 + \breve{o}_{\kappa} (1)) (1 - \veps) \delta \P (YG < \kappa)}{\P ( \rho YG + \sqrt{1 - \rho^2} W < \kappa )} \\
	\stackrel{(i)}{=} & (1 + \breve{o}_{\kappa} (1)) (1 - \veps) \delta \exp \left( \alpha (1 - \rho) \kappa - \frac{\alpha^2 (1 - \rho^2)}{2} \right) \\
	\le & \left( 1 - \frac{\veps}{2} \right) \delta < \delta
\end{align*}
if $\vert \kappa \vert$ is large, where in $(i)$ we use Lemma~\ref{lem:tail}. Applying Theorem~\ref{cor:rho_min_max} then yields Eq.~\eqref{eq:max_corr}.

To show Eq.~\eqref{eq:best_error_alg}, we first notice that if $\delta < \delta_{\slin}(\kappa;\varphi)$, then using the asymptotics of $\delta_{\slin}(\kappa;\varphi)$ from Theorem~\ref{thm:signal} implies
\begin{equation*}
	(1 + \breve{o}_{\kappa} (1)) \delta \P (YG < \kappa) \le 1 + \breve{o}_{\kappa} (1),
\end{equation*}
which further implies that with high probability,
\begin{align*}
	\rho_{\max} \le & 1 - \frac{1}{2 \kappa^2} \mathcal{D}^{-1} \left( (1 + \breve{o}_{\kappa} (1)) \delta \P (YG < \kappa) \right) \\
	\stackrel{(i)}{=} & 1 - \frac{(1 + \breve{o}_{\kappa} (1)) \delta \P (YG < \kappa)}{\vert \kappa \vert^2} \\
	\stackrel{(ii)}{=} & 1 - \frac{\delta}{2}(1 + \breve{o}_{\kappa}(1)) \E \left[ \left( \kappa - YG \right)_+^2 \right],
\end{align*}
where $(i)$ and $(ii)$ both follow from Lemma~\ref{lem:tail}. This concludes the proof.
\end{proof}

\subsection{Additional proofs}\label{sec:append-signal-lemma}

\begin{proof}[{\bf Proof of Lemma~\ref{lem:tail}}]
Recall the $\breve{o}$ notation we introduced before: for a function $s(\rho,r)$, we write $s(\rho, t) = \breve{o}_{t}(1)$ to mean
\begin{equation*}
\lim_{t \to \infty} \max_{\rho \in [-1,1]} \vert s(\rho, t) \vert = 0.
\end{equation*}
For convenience, sometimes in this proof we may write $o_t(1)$ instead of $\breve{o}_t(1)$.

\textbf{Step 1.}  First, we consider the case $\rho \ge \eta_0$. We calculate 
\begin{align*}
&~~~~\P\left( \rho Y G <  -t - \sqrt{1-\rho^2}\, W \right) \\
&= \int_\R \P \left( \rho Y G < -t - \sqrt{1-\rho^2}\, x \right) \frac{1}{\sqrt{2\pi}} \exp\big(-\frac{x^2}{2} \big) \; \d x \\
&= o_{t}(1)\cdot \exp\Big( -\frac{t^2(1+\veps_0)^2}{2}\Big) +  \int_{|x| \le (1+\veps_0) t } \P \left( \rho Y G < -t - \sqrt{1-\rho^2}\, x \right) \frac{1}{\sqrt{2\pi}} \exp\Big(-\frac{x^2}{2} \Big) \; \d x 
\end{align*}
where $\veps_0 := \veps_0(\eta_0)>0$ is sufficiently small such that $\sqrt{1-\eta_0^2}\, (1+\veps_0) < 1 - \frac{\eta_0^2}{4}$, and where we used the Gaussian tail probability (Lemma~\ref{lem:Gtail}) bound in the last expression. %Note that $o_{\bar \kappa}(1)$ is also $\breve{o}_{\kappa}(1)$ uniformly over $t\ge 0$. 
We denote $\tilde t = t + \sqrt{1-\rho^2}\, x $, which satisfies $\tilde t \ge t - \sqrt{1-\rho^2} (1+\veps_0) t \ge \eta_0^2 t / 4$. We observe that
\begin{align*}
\P \left( \rho Y G < - \tilde t \right) &= \int_{-\infty}^{- \rho^{-1} \tilde t} \P(Y = 1 | G = z ) \frac{1}{\sqrt{2\pi}} \exp\Big(-\frac{z^2}{2} \Big) \;  \d z  \\
&+  \int_{\rho^{-1} \tilde t}^{\infty} \P(Y = -1 | G = z ) \frac{1}{\sqrt{2\pi}} \exp\Big(-\frac{z^2}{2} \Big) \; \d z  \\
& = \int_{\rho^{-1} \tilde t}^{\infty} \big[ \varphi(- z) + 1 - \varphi(z) \big] \frac{1}{\sqrt{2\pi}} \exp\Big(-\frac{z^2}{2} \Big) \; \d z\\
&\stackrel{(i)}{=} \frac{2(1+o_t(1))}{\sqrt{2\pi}} \int_{\rho^{-1}  \tilde t  }^\infty \exp\Big(-\frac{z^2}{2} - \alpha z \Big) \; \d z \\
&\stackrel{(ii)}{=} \frac{2(1+o_t(1)) \rho }{\sqrt{2\pi}\, | \tilde t | } \exp\Big( - \frac{\tilde t^2}{2 \rho^2} - \frac{\alpha\tilde t}{\rho}  \Big). 
\end{align*}
Here \textit{(i)} is because by the tail probability assumption we have
\begin{equation*}
\big| 1 - \varphi(z) + \varphi(-z) - 2e^{-\alpha z} \big| \le \beta_t e^{-\alpha z}, \qquad \text{for }z \ge t
\end{equation*}
where $\beta_t = o_t(1)$;
%due to $\varphi(x) = (1+\breve{o}_{\kappa}(1))\exp(x)$ if $x \le c\kappa$ and $1-\varphi(x) = (1+\breve{o}_{\kappa}(1))\exp(x)$ if $x \ge -c\kappa$ (where $c>0$ is a constant), and 
\textit{(ii)} is due to the Gaussian tail probability inequality (Lemma~\ref{lem:Gtail}). Thus, we obtain
\begin{equation*}
\P\left( \rho Y G < - t - \sqrt{1-\rho^2}\, W\right) = \frac{(1+\breve{o}_{\kappa}(1)) \rho }{\pi\, } \int_{|x| \le (1+\veps_0) t } \frac{1}{| \tilde t |} \exp\Big( - \frac{\tilde t^2}{2 \rho^2} - \frac{\alpha\tilde t}{\rho}  -\frac{x^2}{2} \Big) \; \d x + o_{t}(1)\cdot \exp\Big( -\frac{t^2(1+\veps_0)^2}{2}\Big). 
\end{equation*}
We rearrange the exponent
\begin{align*}
- \frac{\tilde t^2}{2 \rho^2} - \frac{\alpha\tilde t}{\rho}  -\frac{x^2}{2} &= -\frac{x^2}{2\rho^2} - \frac{\alpha \sqrt{1-\rho^2}}{\rho} x - \frac{t \sqrt{1-\rho^2}}{\rho^2}x - \frac{t^2}{2\rho^2} - \frac{\alpha t}{\rho} \\
&= -\frac{1}{2\rho^2}\left( x + \sqrt{1-\rho^2}\, t + \alpha \rho \sqrt{1-\rho^2} \right)^2 + \frac{1}{2\rho^2} \left( \sqrt{1-\rho^2}\,  t + \alpha\rho \sqrt{1-\rho^2} \right)^2 - \frac{\alpha t}{\rho} - \frac{t^2}{2 \rho^2} \\
&= - \frac{1}{2\rho^2}\left( x + \sqrt{1-\rho^2}\, \bar \kappa +  \alpha \rho \sqrt{1-\rho^2} \right)^2 - \frac{t^2}{2} - \alpha \rho t+ \frac{(1-\rho^2)\alpha^2}{2}.
\end{align*}
Note that $\sqrt{1-\rho^2}\, t + \alpha\rho \sqrt{1-\rho^2} \le t$ when $t$ is sufficiently large, so denoting $b = \sqrt{1-\rho^2}\, t + \alpha \rho \sqrt{1-\rho^2}$, we have
\begin{align*}
&~~~~\int_{|x| \le (1+\veps_0) t} \frac{1}{|t + \sqrt{1-\rho^2} \, x|} \exp\left[ - \frac{1}{2\rho^2}\left( x + b \right)^2 \right] \; \d x \\
&\stackrel{(i)}{=} \int_I \frac{1}{|\rho^2 t - \alpha \rho(1-\rho^2) + u\sqrt{1-\rho^2} | } \exp \left( -\frac{u^2}{2 \rho^2} \right) \; \d u \\
&= \int_{I_{\mathrm{in}}} \frac{1}{|\rho^2 t - \alpha \rho(1-\rho^2) + u\sqrt{1-\rho^2} | } \exp \left( -\frac{u^2}{2 \rho^2} \right) \; \d u  +\int_{I_{\mathrm{out}}} \frac{1}{|\rho^2 t - \alpha \rho(1-\rho^2) + u\sqrt{1-\rho^2} | } \exp \left( -\frac{u^2}{2 \rho^2} \right) \; \d u\\
&\stackrel{(ii)}{=} \int_{I_{\mathrm{in}}} \frac{1}{|\rho^2 t - o_t(1) \cdot \rho^2 t | } \exp \left( -\frac{u^2}{2 \rho^2} \right) \; \d u + o_t(1) \cdot  \int_{I_{\mathrm{out}}}  \exp \left( -\frac{u^2}{2 \rho^2} \right) \; \d u \\
&\stackrel{(iii)}{=} \frac{1+o_t(1)}{\rho^2 t }\int_{I_{\mathrm{in}}}  \exp \left( -\frac{u^2}{2 \rho^2} \right) \; \d u + o_t\left( \exp\Big(-\frac{t^{4/3}}{2\rho^2} \Big) \right) \\
%&- \int _{|x| > (1+\veps_0)|\bar \kappa|} \exp\left[ - \frac{1}{2\rho^2}\left( x - \sqrt{1-\rho^2}\, \bar \kappa + \alpha \rho \sqrt{1-\rho^2} \right)^2 \right]\; dx \\
%&\stackrel{(ii)}{=} \frac{\rho \sqrt{2\pi}}{(1+\breve{o}_{\kappa}(1)) \rho^2 |\bar \kappa|} - \breve{o}_{\kappa}\big( \exp(-\veps_0^2\kappa^2/2) \big) \\
&\stackrel{(iv)}{=} \frac{(1+o_t(1)) \sqrt{2\pi}}{\rho t},
\end{align*}
where in \textit{(i)} we used change of variable $u = x + b$, in \textit{(ii)} we used 
\begin{equation*}
\rho^2 t - \alpha \rho(1-\rho^2) + u\sqrt{1-\rho^2}  \ge t - (1-\veps_0)\sqrt{1-\rho^2}\, t \ge \eta_0^2 t / 4, \qquad \forall\, u \in I, \rho \ge \eta_0
\end{equation*}
 and in \textit{(iii)} and \textit{(iv)} we used the Gaussian tail probability (Lemma~\ref{lem:Gtail}). Here we have denoted by the interval 
\begin{equation*}
I = \left[-(1+\veps_0)t + b, (1+\veps_0)t + b \right], \quad I_{\mathrm{in}} = \left[-t^{2/3}, t^{2/3} \right], \quad I_{\mathrm{out}} = I \big\backslash \left[-t^{2/3}, t^{2/3} \right].
\end{equation*}
This leads to
\begin{align*}
\P \left( \rho Y G < - t - \sqrt{1-\rho^2}\, W \right) = \frac{(1+o_t(1)) \rho }{\pi } \exp\left( - \frac{t^2}{2} - \alpha \rho t + \frac{(1-\rho^2)\alpha^2}{2} \right) \cdot  \frac{\sqrt{2\pi}}{\rho t},
\end{align*}
which proves the first part of the tail probability.

\textbf{Step 2.} Now we consider the case $\rho \le -\eta_0$. Different from before, in the calculation of $\P(\rho Y G < - \tilde t)$ we will flip the inequality direction. To be specific,
\begin{align*}
\P \left( \rho Y G <  - \tilde t \right) &= \int^{\infty}_{- \rho^{-1} \tilde t} \P(Y = 1 | G = z ) \frac{1}{\sqrt{2\pi}} \exp\Big(-\frac{z^2}{2} \Big) \; \d z  \\
&+  \int^{\rho^{-1} \tilde t}_{-\infty} \P(Y= -1 | G = z ) \frac{1}{\sqrt{2\pi}} \exp\Big(-\frac{z^2}{2} \Big) \; \d z  \\
&\stackrel{(i)}{=} \frac{2(1+o_t(1))}{\sqrt{2\pi}} \int_{|\rho^{-1}  \tilde t | }^\infty \exp\Big(-\frac{z^2}{2} \Big) \; \d z \\
&= \frac{2(1+o_t(1)) \rho }{\sqrt{2\pi}\, | \tilde t | } \exp\Big( - \frac{\tilde t^2}{2 \rho^2}   \Big) 
\end{align*}
where \textit{(i)} is because by the tail assumption
\begin{equation*}
\varphi(z) + 1 - \varphi(-z) = 2 - o_t(1) \cdot 2 e^{-\alpha z} = 2 - o_t(1), \qquad \forall\, z \ge | \rho^{-1} \tilde t|
\end{equation*}
Then, following a similar argument as in Step~1 (essentially replacing $\alpha$ by $0$), we get
\begin{align}
\P \left( \rho Y G < - t - \sqrt{1-\rho^2}\, W \right) = \frac{(1+o_t(1)) \rho }{\pi } \exp\left( - \frac{t^2}{2} \right) \cdot  \frac{\sqrt{2\pi}}{\rho t }. \label{eq:negrho}
\end{align}
This proves the second part of the tail probability.

\textbf{Step 3:} Finally, let us consider the case $|\rho| < \eta_0$. We use a different way to do the conditional probability calculation.
\begin{align}
&~~~~\P \left( \sqrt{1-\rho^2} \, W < - t - \rho Y G \right) \notag \\
&= \int_\R \P \left( \sqrt{1-\rho^2} \,W < - t - \rho Y G \,\big| \, G = x \right) \frac{1}{\sqrt{2\pi}} \exp\Big( -\frac{x^2}{2} \Big) \; \d x \notag \\
&= \int_\R \P \left( \sqrt{1-\rho^2} \, W < -t - \rho x \right)  \varphi(x) \frac{1}{\sqrt{2\pi}} \exp\Big( -\frac{x^2}{2} \Big) \; \d x \\
&+ \int_\R \P \left( \sqrt{1-\rho^2} \, W < -t + \rho x \right)  \big[ 1-\varphi(x) \big] \frac{1}{\sqrt{2\pi}} \exp\Big( -\frac{x^2}{2} \Big) \; \d x \label{expr:long}
\end{align}
We notice that 
\begin{equation*}
\int_{|x| > 1.5 t} \P \left( \sqrt{1-\rho^2} \, W < -t - \rho x \right)  \varphi(x) \frac{1}{\sqrt{2\pi}} \exp\Big( -\frac{x^2}{2} \Big) \; \d x < \int_{|x| > 1.5 t} \exp\Big( -\frac{x^2}{2} \Big) \; \d x
\end{equation*}
which is at most $o_{t}(\exp(-2.25 t^2/2))$; and a similar upper bound holds if we replace $-t - \rho x$ by $-t + \rho x$. If $|x| \le 1.5 t$, then $-t - \rho x < -0.5 t$ so we use the Gaussian tail asymptotics (Lemma~\ref{lem:Gtail}) to get 
\begin{equation*}
\P \left( \sqrt{1-\rho^2} \, W < - t - \rho x \right)  = \frac{(1+o_t(1))\sqrt{1-\rho^2}}{\sqrt{2\pi}\, |-t - \rho x|} \exp \Big(- \frac{(t+\rho x)^2}{2(1-\rho^2)} \Big)
\end{equation*}
and thus
\begin{align*}
& ~~~~ \int_{|x| \le 1.5 t} \P \left( \sqrt{1-\rho^2} \, W < -t - \rho x \right)  \varphi(x) \frac{1}{2\pi} \exp\Big( -\frac{x^2}{2} \Big) \; \d x \\
&= \frac{(1+o_t(1))\sqrt{1-\rho^2}}{2\pi} \int_{|x| \le 1.5 t} \frac{1}{|-t - \rho x|}  \varphi(x) \exp    \left(-\frac{(t+ \rho x)^2}{2(1-\rho^2)} - \frac{x^2}{2} \right) \; \d x \\
&= \frac{(1+o_t(1))\sqrt{1-\rho^2}}{2\pi} \exp\Big(-\frac{t^2}{2} \Big)  \int_{|x| \le 1.5 t} \frac{1}{|t + \rho x|} \varphi(x) \exp \left(-\frac{(x + \rho t)^2}{2(1-\rho^2)}  \right) \;  \d x \\
& = \frac{(1+o_t(1))\sqrt{1-\rho^2}}{2\pi} \exp\Big(-\frac{t^2}{2} \Big)  \int_{u \in J} \frac{1}{(1-\rho^2) t + \rho u} \varphi(u - \rho t) \exp \left(-\frac{u^2}{2(1-\rho^2)}  \right) \; \d u
%&\stackrel{(ii)}{=} \frac{1+\breve{o}_{\kappa}(1)}{\sqrt{2\pi} |\bar \kappa|} \exp\Big(-\frac{\kappa^2}{2} \Big)
\end{align*}
%where (i) is due to $\varphi(x) \le 1$ and (ii) follows from a similar argument as in our proof of the first part.
where we used change of variable $u = x  + \rho t$ in the last equality, and where we define the intervals
\begin{equation*}
J=[-1.5 t + \rho t, 1.5 t + \rho t], \quad J_{\mathrm{in}} = \left[-t^{2/3}, t^{2/3} \right], \quad J_{\mathrm{out}} = J \big\backslash \left[-t^{2/3}, t^{2/3} \right]. 
\end{equation*}
Note that $(1-\rho^2)t + \rho u \ge t/2$, so using the Gaussian tail probability inequality (Lemma~\ref{lem:Gtail}), we have
\begin{align*}
& \int _{J_{\mathrm{out}}} \frac{1}{(1-\rho^2) t + \rho u} 
\varphi(u - \rho t) \exp \left(-\frac{u^2}{2(1-\rho^2)}  \right) \; \d u = o_t\left( \exp\Big(-\frac{t^{4/3}}{2} \Big) \right), \\
& \int_{J_{\mathrm{out}}} \varphi(u - \rho t) \exp \left(-\frac{u^2}{2(1-\rho^2)}  \right) \; \d u = o_t\left( \exp\Big(-\frac{t^{4/3}}{2} \Big) \right).
\end{align*}
Thus we obtain
\begin{align*}
& ~~~~ \int_{|x| \le 1.5 t} \P \left( \sqrt{1-\rho^2} \, W < -t  - \rho x \right)  \varphi(x) \frac{1}{2\pi} \exp\Big( -\frac{x^2}{2} \Big) \; \d x \\ 
&= \frac{1+o_t(1)}{2\pi\sqrt{1-\rho^2}\, t} \exp\Big(-\frac{t^2}{2} \Big) \int_{J_{\mathrm{in}}} \varphi(u - \rho t) \exp\Big(-\frac{u^2}{2(1-\rho^2)} \Big) \; \d u+ o_{t}\left(\exp\Big( - \frac{t^2 + t^{4/3}}{2} \Big) \right) \\
&= \frac{1+o_t(1)}{2\pi\sqrt{1-\rho^2}\, t} \exp\Big(-\frac{t^2}{2} \Big) \int_{\R} \varphi(u-\rho t) \exp\Big(-\frac{u^2}{2(1-\rho^2)} \Big) \; \d u+ o_t\left(\exp\Big(-\frac{t^2 + t^{4/3}}{2} \Big) \right)% \\
%&= \frac{(1+\breve{o}_{\kappa}(1))}{2\pi\sqrt{1-\rho^2}\, |\bar \kappa|} \exp\Big(-\frac{\kappa^2}{2} \Big) 
\end{align*}
A similar term is obtained for \eqref{expr:long}, and thus we have
\begin{align}
\P \left( \sqrt{1-\rho^2}\,  W < -t - \rho Y G \right) &= \frac{1+o_t(1)}{2\pi\sqrt{1-\rho^2}\, t} \exp\Big(-\frac{t^2}{2} \Big) \int_{\R} \Big[\varphi(u- \rho t) + 1 - \varphi(u+\rho t) \Big] \exp\Big(-\frac{u^2}{2(1-\rho^2)} \Big) \; \d u \notag \\
&+ o_{t}\left(\exp\Big(-\frac{t^2 + t^{4/3}}{2} \Big) \right).
\label{expr:longlong}
\end{align}
We claim that 
\begin{equation}
\min\Big\{1, e^{-\alpha \eta_0 t} \Big\} \sqrt{2\pi(1-\rho^2)} \le \int_{\R} \Big[\varphi(u- \rho t) + 1 - \varphi(u+\rho t) \Big] \exp\Big(-\frac{u^2}{2(1-\rho^2)} \Big) \; \d u \le 2\sqrt{2\pi(1-\rho^2)}.\label{ineq:twoway}
\end{equation}
In fact, it is obvious that $\varphi(u- \rho t) + 1 - \varphi(u+\rho t) \le 2$ so the second inequality follows. To see why the first inequality holds, we notice that if $\rho \le 0$, by monotonicity of $\varphi$ we must have $\varphi(u-\rho t) + 1 - \varphi(u+\rho t) \ge 1$. And if $\rho > 0$, we notice
\begin{align*}
& ~~~\int_{\R} \Big[\varphi(u- \rho t) + 1 - \varphi(u+\rho t) \Big] \exp\Big(-\frac{u^2}{2(1-\rho^2)} \Big) \; \d u \\
&= \int_0^\infty \Big[\varphi(u- \rho t) + 1 - \varphi(u+\rho t) + \varphi(-u - \rho t) + 1 - \varphi(-u +\rho t) \Big] \exp\Big(-\frac{u^2}{2(1-\rho^2)} \Big) \; \d u \\
& \ge \int_0^\infty \Big[ \varphi(u- \rho t) + 1 - \varphi(-u +\rho t) \Big] \exp\Big(-\frac{u^2}{2(1-\rho^2)} \Big) \; \d u
\end{align*}
and by monotonicity,
\begin{align*}
\varphi(u-\rho t) + 1 - \varphi(-u+\rho t) \ge \varphi(-\eta_0 t) + 1 - \varphi(\eta_0 t) = (1+o_t(1)) \cdot 2e^{-\eta_0 \alpha t}.
\end{align*}
Combining \eqref{expr:longlong} and \eqref{ineq:twoway} yields the third part of the tail probability.

\textbf{Step 4:} We prove the ``as a consequence'' part. Using the identity $\E [X_+^2] = 2\int_0^\infty t \P(X > t) \; \d t$, we have
\begin{align*}
& ~~~ \E \Big[ \big(\rho Y G + \sqrt{1-\rho^2}\, W - \kappa\big)_-^2 \Big] = 2\int_0^\infty t \P\big( \rho Y G + \sqrt{1-\rho^2}\, W < \kappa - t \big) \; \d t = (1+\breve{o}_{\kappa}(1)) \cdot 2 \int_0^\infty t A_{\rho,|\kappa - t|} \; \d t.
\end{align*}
For $\rho \ge \eta_0$, 
\begin{align*}
2 \int_0^\infty t A_{\rho,|\kappa - t|} \; \d t &= 2 \sqrt{\frac{2}{\pi}} \, \frac{1}{|\kappa|} \exp\left( -\frac{\kappa^2}{2} - \alpha \rho |\kappa| + \frac{(1-\rho^2)\alpha^2}{2} \right) \int_0^\infty \frac{t}{(t/|\kappa|) + 1 } \exp \left( -\frac{t^2}{2} - t|\kappa| - \alpha \rho t \right) \; \d t \\
&\stackrel{(i)}{=} 2 \sqrt{\frac{2}{\pi}} \, \frac{1}{|\kappa|^3} \exp\left( -\frac{\kappa^2}{2} - \alpha \rho |\kappa| + \frac{(1-\rho^2)\alpha^2}{2} \right) \int_0^\infty \frac{u}{(u/|\kappa|^2) + 1 }  \exp \left( -\frac{u^2}{2\kappa^2} - u - \frac{\alpha \rho}{|\kappa|} \right) \; \d u \\
&\stackrel{(ii)}{=} 2 \sqrt{\frac{2}{\pi}} \, \frac{1}{|\kappa|^3} \exp\left( -\frac{\kappa^2}{2} - \alpha \rho |\kappa| + \frac{(1-\rho^2)\alpha^2}{2} \right) \cdot (1+\breve{o}_{\kappa}(1)) = \frac{2(1+\breve{o}_{\kappa}(1))}{|\kappa|^2} A_{\rho, |\kappa|}
\end{align*}
where (i) is due to a change of variable $u = t|\kappa|$ and (ii) is due to the dominated convergence theorem. For $\rho \le -\eta_0$, following a similar argument,
\begin{align*}
2 \int_0^\infty t A_{\rho,|\kappa - t|} \; \d t &= 2 \sqrt{\frac{2}{\pi}} \, \frac{1}{|\kappa|} \exp\left( -\frac{\kappa^2}{2} \right) \int_0^\infty \frac{t}{(t/|\kappa|) + 1 } \exp \left( -\frac{t^2}{2} - t|\kappa| \right) \; \d t = \frac{2(1+\breve{o}_{\kappa}(1))}{|\kappa|^2} A_{\rho, |\kappa|}.
\end{align*}
For $\rho \in (-\eta_0,\eta_0)$, we define $F(x) = \varphi(x) + 1 - \varphi(-x)$, so the formula \eqref{def:a} is simplified into
\begin{align*}
a_{\rho,t} &= \frac{1}{2\sqrt{2\pi(1-\rho^2)}} \int_0^\infty \big[ \varphi(u-\rho t) + 1 - \varphi(u+\rho t) + \varphi(-u - \rho t) + 1 - \varphi(-u + \rho t) \big] \exp\left( -\frac{u^2}{2(1-\rho^2)} \right) \; \d u\\
&= \frac{1}{2\sqrt{2\pi(1-\rho^2)}}  \int_0^\infty \big[ F(u-\rho t) + F(-u-\rho t) \big]  \exp\left( -\frac{u^2}{2(1-\rho^2)} \right) \; \d u
\end{align*}
Using this and Fubini's theorem, we derive
\begin{align*}
& ~~~ 2 \int_0^\infty t A_{\rho,|\kappa - t|} \; \d t \\
&= \sqrt{\frac{2}{\pi}} \frac{1}{\sqrt{2\pi(1-\rho^2)}}\, \int_0^\infty \frac{t}{|t-\kappa|} \exp\left( - \frac{(t-\kappa)^2}{2} \right) \;\d t \\
& \times \int_0^\infty \big[ F(u-\rho t - \rho |\kappa|) + F(-u-\rho t- \rho |\kappa|) \big] \exp\left( -\frac{u^2}{2(1-\rho^2)} \right)  \;\d u \\
&= \sqrt{\frac{2}{\pi}} \frac{\exp(-\kappa^2/2)}{\sqrt{2\pi(1-\rho^2)}} \int_0^\infty \exp\left( -\frac{u^2}{2(1-\rho^2)} \right)  \;\d u \\
& \times \int_0^\infty \frac{t}{|t-\kappa|} \exp\left( - \frac{t^2}{2} - t|\kappa| \right) \big[ F(u-\rho t- \rho |\kappa|) + F(-u-\rho t- \rho |\kappa|) \big] \;\d t \\
&= \sum_{k=1}^3 \sqrt{\frac{2}{\pi}} \frac{\exp(-\kappa^2/2)}{\sqrt{2\pi(1-\rho^2)}} \int_0^\infty \exp\left( -\frac{u^2}{2(1-\rho^2)} \right)  \;\d u \\
& \times \int_{J_k} \frac{t}{|t-\kappa|} \exp\left( - \frac{t^2}{2} - t|\kappa| \right) \big[ F(u-\rho t- \rho |\kappa|) + F(-u-\rho t- \rho |\kappa|) \big] \;\d t \\
&=: I_1 + I_2 + I_3
\end{align*}
where 
\begin{equation*}
J_1 = [0, |\kappa|^{-1/4}], \qquad J_2 = [|\kappa|^{-1/4}, |\kappa|^{1/4}], \qquad J_3 = [|\kappa|^{1/4}, +\infty).
\end{equation*}
A straightforward bound gives $I_3 = \breve{o}_{\kappa}(\exp(-|\kappa|^2/2 - |\kappa|^{5/4}))$. We claim that
\begin{equation*}
\sup_{u \in \R, |\rho| \le \eta_0, t \in J_1} \Big| \frac{F(u-\rho t - \rho |\kappa|)}{F(u- \rho |\kappa|)} - 1 \Big| = \breve{o}_{\kappa}(1), \qquad \sup_{u \in \R, |\rho| \le \eta_0, t \in J_2} \frac{F(u-\rho t - \rho |\kappa|)}{F(u- \rho |\kappa|)} = O(\exp(|\kappa|^{1/4})).
\end{equation*}
Once this is proved, we simply the inner integrals and obtain
\begin{align*}
I_1 &= \sqrt{\frac{2}{\pi}} \frac{\exp(-\kappa^2/2)}{\sqrt{2\pi(1-\rho^2)}}\frac{1+\breve{o}_{\kappa}(1)}{|\kappa|^3} \int_0^\infty  \big[ F(u-\rho |\kappa|) + F(-u-\rho |\kappa|) \big]  \exp\left( -\frac{u^2}{2(1-\rho^2)} \right)  \;\d u \\
&= \sqrt{\frac{2}{\pi}}\exp\left( - \frac{\kappa^2}{2} \right) \frac{1+\breve{o}_{\kappa}(1)}{|\kappa|^3} a_{\rho, |\kappa|} = \frac{1 + \breve{o}_{\kappa}(1)}{|\kappa|^2}  A_{\rho,|\kappa|}; \\
I_2 &= \sqrt{\frac{2}{\pi}} \frac{\exp(-\kappa^2/2)}{\sqrt{2\pi(1-\rho^2)}}\frac{1 +\breve{o}_{\kappa}(1)}{|\kappa|^3} \int_0^\infty  \big[ F(u-\rho |\kappa|) + F(-u-\rho |\kappa|) \big]  \exp\left( -\frac{u^2}{2(1-\rho^2)} \right)  \;\d u \\ & \times O\left( \exp\Big(- |\kappa|^{3/4} + |\kappa|^{1/4} \Big)\right) \\
&= \sqrt{\frac{2}{\pi}}\exp\left( - \frac{\kappa^2}{2} \right) \frac{\breve{o}_{\kappa}(1)}{|\kappa|^3} a_{\rho, |\kappa|} = \frac{ \breve{o}_{\kappa}(1)}{|\kappa|^2}  A_{\rho,|\kappa|}.
\end{align*}
The would lead to the conclusion. Thus it remains to prove the claim. 

First consider the case $t \in J_1$. By the tail assumption of $\varphi$, for any given $\veps_0 > 0$, we can find an $M = M(\veps_0)>0$ such that $|F(x) - 2e^{-\alpha x} | \le \veps_0 e^{-\alpha x}$ for all $x <- M$. If $u - \rho |\kappa| < -2M$, 
\begin{equation*}
|F(u - \rho t - \rho |\kappa|) - F(u - \rho |\kappa|) | \le (2+\veps_0) e^{-\alpha (u-\rho t - \rho|\kappa|)} - (2-\veps_0) e^{-\alpha(u - \rho |\kappa|)}
\end{equation*}
so 
\begin{equation*}
\sup_{u- \rho |\kappa| < -2M, t \in J_1}\Big| \frac{F(u-\rho t - \rho |\kappa|)}{F(u- \rho |\kappa|)} - 1 \Big| \le \sup_{t \in J_1} \frac{(2+\veps_0) e^{\alpha \rho t} - (2-\veps_0)}{2-\veps_0} \stackrel{\kappa \to -\infty}{\longrightarrow} \frac{2\veps_0}{2-\veps_0}
\end{equation*}
If $u - \rho |\kappa| > -2M$, we use uniform continuity of $F$ to obtain
\begin{equation*}
\sup_{u- \rho |\kappa| \ge  -2M, t \in J_1} \Big| \frac{F(u-\rho t - \rho |\kappa|)}{F(u- \rho |\kappa|)} - 1 \Big| \stackrel{\kappa \to -\infty}{\longrightarrow} 0.
\end{equation*}
Since $\veps_0$ is arbitrary, this proves the first half of the claim. 

Next consider the case $t \in J_2$. If $u - \rho |\kappa| > -M$, then the ratio is always bounded by $O(2/F(-M))$. If $u - \rho |\kappa| \le -M$ and $u - \rho|\kappa| - \rho t \le -M$, then
\begin{equation*}
\frac{F(u-\rho t - \rho |\kappa|)}{F(u- \rho |\kappa|)} \le \frac{(2+\veps_0)\exp\big( u - \rho t - \rho |\kappa| \big) }{(2-\veps_0) \exp\big( u - \rho |\kappa| \big)} = \frac{(2+\veps_0) \exp( |\kappa|^{1/4} )}{2-\veps_0}.
\end{equation*}
If $u - \rho |\kappa| \le -M$ and $u - \rho|\kappa| - \rho t > -M$, then
\begin{align*}
\frac{F(u-\rho t - \rho |\kappa|)}{F(u- \rho |\kappa|)} \le \frac{2}{(2-\veps_0) \exp\big( u - \rho |\kappa| \big)} = \frac{2 \exp(\rho t)}{(2-\veps_0) \exp\big( u - \rho |\kappa| - \rho t \big)} \le \frac{2 \exp(|\kappa|^{1/4})}{(2-\veps_0) \exp\big( -M \big)}.
\end{align*}
This proves the second half of the claim, thus completing the proof.
%
%second inequality is due 
%\begin{equation*}
%\int_{|x| \le 1.5 |\kappa|} \P \left( \sqrt{1-\rho^2} g < \bar \kappa + \rho x \right)  \varphi(-x) \frac{1}{2\pi} \exp\Big( -\frac{x^2}{2} \Big) \; dx  \le \frac{1+\breve{o}_{\kappa}(1)}{\sqrt{2\pi} |\bar \kappa|} \exp\Big(-\frac{\kappa^2}{2} \Big). 
%\end{equation*}
%This leads to 
%\begin{equation*}
%\P \left( \sqrt{1-\rho^2} g < \bar \kappa - \rho y_i G_i \right) \le (1+\breve{o}_{\kappa}(1)) \sqrt{\frac{2}{\pi}}  \frac{1}{|\bar \kappa|} \exp\Big(-\frac{\kappa^2}{2} \Big).
%\end{equation*}
\end{proof}

\begin{lem}\label{lem:Gtail}
Let $G \sim \cN(0,1)$ be a Gaussian variable. Then, for $t > 0$, 
\begin{equation*}
\frac{t}{t^2 + 1} e^{-t^2/2} < \int_t^\infty e^{-x^2/2}\; \d x < \frac{1}{t} e^{-t^2/2}.
\end{equation*}
Moreover, 
\begin{equation*}
\E \big[ (G - t)_+ \big] = \frac{1+o_t(1)}{\sqrt{2\pi}\, t^2} e^{-t^2/2}, \qquad \E \big[ (G - t)_+^2 \big] = \frac{2(1+o_t(1))}{\sqrt{2\pi}\,t^3} e^{-t^2/2}.
\end{equation*}
\end{lem}

\begin{proof}[\bf Proof of Lemma~\ref{lem:sig_lower_gordon}]
	This proof is similar to that of Theorem~\ref{thm:algpure}. Let $\vv \sim \Unif(\S^{d-1})$ be independent of the $(y_i, G_i, \zz_i)$'s. For any $c \ge 1$, consider the following convex optimization problem:
	\begin{equation}\label{eq:sig_lower_gordon}
		\text{maximize} \ \left\langle \vv, \ww \right\rangle, \ \text{subject to} \ \rho y_i G_i + \sqrt{1 - \rho^2} \langle \zz_i, \ww \rangle \ge \kappa, \ \forall i \in [n], \ \text{and} \ \Vert \ww \Vert_2 \le c.
	\end{equation}
	Denote by $M_n(c)$ the maximum of \eqref{eq:sig_lower_gordon}, then we can write
	\begin{align*}
		M_n(c) = & \max_{\norm{\ww}_2 \le c} \min_{\balpha \ge \bzero} \left\{ \left\langle \vv, \ww \right\rangle + \left\langle \balpha, \rho \yy \odot \bG + \sqrt{1 - \rho^2} \ZZ \ww - \kappa \bone \right\rangle \right\} \\
		= & \max_{\norm{\ww}_2 \le c} \min_{\balpha \ge \bzero} \left\{ \left\langle \vv, \ww \right\rangle + \left\langle \balpha, \rho \yy \odot \bG - \kappa \bone \right\rangle + \sqrt{1 - \rho^2} \left\langle \balpha, \ZZ \ww \right\rangle \right\}.
	\end{align*}
	As in the proof of Theorem~\ref{thm:algpure}, we can first condition on the $(y_i, G_i)$'s, then applying Gordon's inequality (Theorem~\ref{thm:Gordon}) reduces the analysis of $M_n (c)$ to that of $\tilde{M}_n (c)$, where
	\begin{equation*}
		\tilde{M}_n (c) = \max_{\norm{\ww}_2 \le c} \min_{\balpha \ge \bzero} \left\{ \left\langle \vv, \ww \right\rangle + \left\langle \balpha, \rho \yy \odot \bG - \kappa \bone \right\rangle + \sqrt{1 - \rho^2} \norm{\ww}_2 \langle \balpha, \hh \rangle + \sqrt{1 - \rho^2} \norm{\balpha}_2 \langle \ww, \bgg \rangle \right\},
	\end{equation*}
	where $\hh \sim \cN (\bzero, \bI_n)$, $\bgg \sim \cN (\bzero, \bI_{d - 1})$ are mutually independent, and further independent of other random variables.
	
	Proceeding similarly as in the proof of Theorem~\ref{thm:algpure}, we obtain that
	\begin{equation*}
		\tilde{M}_n (c) = \max_{r \in [0, c]} F(r) + o_{d, \P} (1),
	\end{equation*}
	where
	\begin{equation*}
		F(r) = \begin{cases}
			\sqrt{r^2 - \frac{\delta}{1 - \rho^2} \E \left[ \left( \kappa - \rho YG - \sqrt{1 - \rho^2} r W \right)_+^2 \right]}, \ & \text{if} \ r^2 (1 - \rho^2) \ge \delta \E \left[ \left( \kappa - \rho YG - \sqrt{1 - \rho^2} r W \right)_+^2 \right], \\
			- \infty, \ & \text{otherwise}.
		\end{cases}
	\end{equation*}
	Note that by Stein's identity, when $F(r) > - \infty$, we have
	\begin{align*}
		\frac{\d}{\d r} F(r)^2 = & \frac{\d}{\d r} \left( r^2 - \delta \E \left[ \left( \frac{\kappa}{\sqrt{1 - \rho^2}} - \frac{\rho YG}{\sqrt{1 - \rho^2}} - r W \right)_+^2 \right] \right) \\
		= & 2 r \left( 1 - \delta \E \left[ \Phi \left( \frac{\kappa - \rho YG}{r \sqrt{1 - \rho^2}} \right) \right] \right) \\
		= & 2r \left( 1 - \delta \P \left( \rho YG + \sqrt{1 - \rho^2} r W < \kappa \right) \right).
	\end{align*}
	Since the link function $\varphi$ is monotone increasing, we know that the map $r \mapsto \P ( \rho YG + \sqrt{1 - \rho^2} r W < \kappa )$ is increasing. According to Assumption~\ref{cond:sig_lower_a}, we have
	\begin{align*}
		\delta \P ( \rho YG + \sqrt{1 - \rho^2} W < \kappa ) < & 1, \\
		\delta \E \left[ \left( \frac{\kappa}{\sqrt{1 - \rho^2}} - \frac{\rho YG}{\sqrt{1 - \rho^2}} - W \right)_+^2 \right] < & 1,
	\end{align*}
	it follows that
	\begin{equation*}
		F(1) > - \infty, \ \frac{\d}{\d r} F(r)^2 > 0 \ \text{whenever} \ r \in [0, 1], \ F(r) > -\infty.
	\end{equation*}
	This proves that there exists $c > 1$, such that
	\begin{equation*}
		\max_{r \in [0, c]} F(r) > \max_{r \in [0, 1]} F(r),
	\end{equation*}
	which further implies that with high probability,
	\begin{equation*}
		\tilde{M}_n (c) > \tilde{M}_n (1) \implies M_n (c) > M_n (1).
	\end{equation*}
	Therefore, the maximizer of \eqref{eq:sig_lower_gordon} satisfies $\norm{\ww}_2 \ge 1$ with high probability. As a consequence, Eq.~\eqref{eq:logistic_lower_bound} holds and we finish the proof of Lemma~\ref{lem:sig_lower_gordon}.
\end{proof}

\begin{proof}[\bf Proof of Lemma~\ref{lem:good_sample_2nd_moment_method}]
To begin with, we rewrite Eq.~\eqref{eq:good_sample_2nd_moment_method} as
\begin{equation*}
	\rho y_i G_i + \sqrt{1 - \rho^2} \langle \zz_{i1}, \ww_1 \rangle \ge \sqrt{1 - \rho^2} \kappa_2, \ \forall i \in S_G,
\end{equation*}
where $\ww_1 \in \S^{d - 2 - d_0}$, and
\begin{equation*}
	i \in S_G \iff y_i G_i \ge \frac{\sqrt{1-\rho^2} \kappa_1}{\rho}.
\end{equation*} 
It will be more convenient if we adjust our notation here: We recast $\zz_{i1}$ as $\zz_i$, $\ww_1$ as $\ww$, $d - d_0$ as $d$, $\sqrt{1 - \rho^2} \kappa_2$ as $\kappa$, and $\sqrt{1 - \rho^2} \kappa_1 / \rho$ as $\kappa_{\mathrm t}$. Note that under this new notation, according to Eq.~\eqref{eq:Z_partition_1} we have
\begin{equation}\label{eq:new_delta_sec}
	\delta := \lim_{n \to +\infty} \frac{n}{d} < \delta_{\sec} (\rho, \kappa, \kappa_{\mathrm t}).
\end{equation}
Now it suffices to show that with probability bounded away from $0$, $\exists \ww \in \S^{d-2}$ such that
\begin{equation*}
	\rho y_i G_i + \sqrt{1-\rho^2} \langle \zz_i, \ww \rangle \ge \kappa, \ \forall i \in S_G, \ \text{i.e.}, \ y_i G_i \ge \kappa_{\mathrm t}.
\end{equation*}
To this end we apply the second moment method. As in the proof of Theorem~\ref{thm:pure_noise_lower_bound}, let $\mu$ denote the uniform probability measure on $\S^{d-2}$, and choose $f: \R^2 \to [0, +\infty)$ to be
\begin{equation*}
	f \left( \rho s, \sqrt{1 - \rho^2} w \right) = \frac{1}{\sqrt{e(0, s)}} \exp (-c(s) w) \bone \{ w \ge \kappa(s) \},
\end{equation*}
where $\kappa(s), c(s)$ and $e(q, s)$ are defined in Definition~\ref{def:sig_lower}. In particular,
\begin{equation}\label{eq:expression_e(0, s)}
	e(0, s) = \E [\exp (-c(s) W) \bone \{ W \ge \kappa(s) \}]^2.
\end{equation}
Note that $f$ is supported on $\{ (x, y) \in \R^2: x+y \ge \kappa \}$. We further define the following random variable:
\begin{equation*}
	Z = \int_{\S^{d-2}} \prod_{i \in S_G} f \left( \rho y_i G_i, \sqrt{1 - \rho^2} \langle \zz_i, \ww \rangle \right) \mu (\d \ww),
\end{equation*}
then we want to show that $\E [Z^2] = O_d (\E[Z]^2)$. Below we calculate the first and second moments of $Z$, we omit some details here since the computation is almost the same as that in the proof of Theorem~\ref{thm:pure_noise_lower_bound}.

Let $W_i \sim_{\iid} \cN (0, 1)$, and independent of the $(y_i, G_i)$'s, then we have 
\begin{align*}
	\E [Z] =& \E \left[ \prod_{i \in S_G} f \left( \rho y_i G_i, \sqrt{1 - \rho^2} W_i \right) \right] \\
	=& \E \left[ \prod_{i=1}^{n} \exp \left( \bone \left\{ y_i G_i \ge \kappa_{\mathrm t} \right\} \log f \left( \rho y_i G_i, \sqrt{1 - \rho^2} W_i \right) \right) \right] \\
	=& \E \left[ \exp \left( \bone \left\{ YG \ge \kappa_{\mathrm t} \right\} \log f \left( \rho YG, \sqrt{1 - \rho^2} W \right) \right) \right]^n.
\end{align*}
Again, using the same kind of analysis as in the proof of Theorem~\ref{thm:pure_noise_lower_bound}, we obtain that
\begin{align*}
	\E[Z^2] = & \int_{\S^{d-2} \times \S^{d-2}} \E \left[ \prod_{i \in S_G} f \left( \rho y_i G_i, \sqrt{1 - \rho^2} \langle \zz_i, \ww_1 \rangle \right) f \left( \rho y_i G_i, \sqrt{1 - \rho^2} \langle \zz_i, \ww_2 \rangle \right) \right] \mu (\d \ww_1) \mu (\d \ww_2) \\
	= & \int_{-1}^{1} p(q) \E_q \left[ \exp \left( \bone \left\{ YG \ge \kappa_{\mathrm t} \right\} \log \left( f \left( \rho YG, \sqrt{1 - \rho^2} W_1 \right) f \left( \rho YG, \sqrt{1 - \rho^2} W_2 \right) \right) \right) \right]^n \d q \\
	= & \left( 1 + o_d (1) \right) \int_{-1}^{1} \sqrt{\frac{d}{2 \pi}} \left(1-q^{2}\right)^{(d-3)/2} \\
	& \times \E_q \left[ \exp \left( \bone \left\{ YG \ge \kappa_{\mathrm t} \right\} \log \left( f \left( \rho YG, \sqrt{1 - \rho^2} W_1 \right) f \left( \rho YG, \sqrt{1 - \rho^2} W_2 \right) \right) \right) \right]^n \d q,
\end{align*}
where $p(q)$ is the density of $\langle \ww_1, \ww_2 \rangle$ introduced in the proof of Theorem~\ref{thm:pure_noise_lower_bound}. Notice that
\begin{align*}
	& \E \left[ \exp \left( \bone \left\{ YG \ge \kappa_{\mathrm t} \right\} \log f \left( \rho YG, \sqrt{1 - \rho^2} W \right) \right) \right] \\
	= & 1 + \int_{\kappa_{\mathrm t}}^{+\infty} p_{YG} (s) \left( \E \left[ f \left( \rho s, \sqrt{1 - \rho^2} W \right) \right] - 1 \right) \d s \\
	= & 1 + \int_{\kappa_{\mathrm t}}^{+\infty} p_{YG} (s) \left( \frac{\E [\exp (-c(s) W) \bone \{ W \ge \kappa(s) \}]}{\sqrt{e(0, s)}} - 1 \right) \d s = 1,
\end{align*}
where the last equality is due to Eq.~\eqref{eq:expression_e(0, s)}. This implies $\E [Z] = 1$. Recall that $e(q), F(q)$ and $I(q)$ are defined in Definition~\ref{def:sig_lower}, a similar calculation then leads to
\begin{equation*}
	\E[Z^2] = \left( 1 + o_d (1) \right) \int_{-1}^{1} \sqrt{\frac{d}{2 \pi}} \exp \left( - n F(q) - (d-3) I(q) \right) \d q,
\end{equation*}
and
\begin{equation*}
	\exp \left( - n F(0) \right) = e(0)^n = 1.
\end{equation*}
Therefore, with the aid of Laplace's method, it only suffices to verify that $F''(0) + I''(0)/\delta > 0$, and that $F(q) + I(q)/\delta$ is uniquely minimized at $q = 0$. However, this is just a direct consequence of Eq.~\eqref{eq:new_delta_sec} and the definition of $\delta_{\sec} (\rho, \kappa, \kappa_{\mathrm t})$, thus proving Lemma~\ref{lem:good_sample_2nd_moment_method}.
\end{proof}

\begin{proof}[\bf Proof of Lemma~\ref{lem:bad_sample_gordon}]

\noindent For fixed $\ww_1 \in \R^{d - 1 - d_0}$, $\norm{\ww_1}_2 = 1$, we know that $\ZZ_1^B\ww_1 \sim_{\iid} \cN (0, 1)$. Denote $\zz = \ZZ_1^B\ww_1$, and define
\begin{equation*}
	\bb = \begin{bmatrix}
		(\kappa_0 - \kappa_2) \bone \\
		\kappa_0 \bone - \uu^B - \zz
	\end{bmatrix},
\end{equation*}
then we want to show that with high probability, $\exists \ww_2 \in \R^{d_0}$, $\norm{\ww_2}_2 \le c$ such that $\ZZ_2 \ww_2 \ge \bb$.

The proof strategy is similar to that of Theorem~\ref{thm:algpure} and Lemma~\ref{lem:sig_lower_gordon}, so technical details will be omitted. To be specific, let us consider the following random variable:
\begin{equation}\label{eq:bad_sample_gordon_optimization}
	\xi_{n, c} = \frac{1}{\sqrt{d_0}} \min_{\norm{\ww_2}_2 \le c} \max_{\blambda \ge \bzero, \norm{\blambda}_2 \le 1} \left\{ \blambda^\top \left( \bb - \ZZ_2 \ww_2 \right) \right\}.
\end{equation}
Then our desired result is equivalent to proving that $\xi_{n, c} = 0$ with high probability. Using Gordon's comparison inequality (Theorem~\ref{thm:Gordon}) and the Uniform Law of Large Numbers \cite[Lemma 2.4]{newey1994large}, we conclude that as $n \to +\infty$, $\xi_{n, c}$ converges in probability to the following constant:
\begin{equation*}
	\min_{ 0 \le \gamma \le c} \left( \sqrt{ \delta_0 \left( p_G \E \left[ \left( \kappa_0 - \kappa_2 - \gamma W \right)_+^2 \right] + p_B \E \left[ \left( \kappa_0 - u^B -  z - \gamma W \right)_+^2 \right] \right)} - \gamma \right)_+,
\end{equation*}
where $p_G = 1 -p_B$, $W \sim \cN (0, 1)$ is independent of $u^B$ and $z$, and
\begin{equation*}
	\delta_0 = \lim_{n \to +\infty} \frac{n}{d_0}.
\end{equation*}
Note that Eq.~\eqref{eq:Z_partition_2} implies
\begin{align*}
	1 > & \delta_0 \left( p_G \E \left[ \left( \frac{\kappa_0 - \kappa_2}{c} - W \right)_+^2 \right] + \frac{p_B}{c^2}\E \left[ \left( \kappa_0 - \frac{\rho YG}{\sqrt{1-\rho^2}} - \sqrt{1 + c^2} W \right)_+^2 \Bigg\vert YG < \frac{\sqrt{1-\rho^2}}{\rho} \kappa_1 \right] \right) \\
	= & \frac{\delta_0}{c^2} \left( p_G \E \left[ \left( \kappa_0 - \kappa_2 - c W \right)_+^2 \right] + p_B \E \left[ \left( \kappa_0 - u^B - \sqrt{1 + c^2} W \right)_+^2 \right] \right) \\
	= & \frac{\delta_0}{c^2} \left( p_G \E \left[ \left( \kappa_0 - \kappa_2 - c W \right)_+^2 \right] + p_B \E \left[ \left( \kappa_0 - u^B - z - c W \right)_+^2 \right] \right),
\end{align*}
where $z \sim \cN (0, 1)$ is independent of other random variables, thus leading to
\begin{align*}
	& \min_{ 0 \le \gamma \le c} \left( \sqrt{ \delta_0 \left( p_G \E \left[ \left( \kappa_0 - \kappa_2 - \gamma W \right)_+^2 \right] + p_B \E \left[ \left( \kappa_0 - u^B -  z - \gamma W \right)_+^2 \right] \right)} - \gamma \right)_+ \\
	\le &  \left( \sqrt{ \delta_0 \left( p_G \E \left[ \left( \kappa_0 - \kappa_2 - c W \right)_+^2 \right] + p_B \E \left[ \left( \kappa_0 - u^B -  z - c W \right)_+^2 \right] \right)} - c \right)_+ = 0.
\end{align*}
This concludes the proof.
\end{proof}

\begin{proof}[\bf Proof of Lemma~\ref{lem:transport_yG}]
	Denote $\mu = \operatorname{Law}(yG)$, then we know that $\mu$ has density
	\begin{equation}
	\mu (\d x) = (\varphi(x) + 1 - \varphi(-x)) \phi(x) \d x.	
	\end{equation}
    Define $\alpha (x) = (\varphi(x) + 1 - \varphi(-x))$ and $V(x) = - \log \alpha(x) - \log \phi(x)$, then by Assumption~\ref{ass:exponential_tail_link},
    \begin{align}\label{eq:V_pos_asymptotic}
    	V(x) = & -\log 2 + \log \sqrt{2 \pi} - \frac{x^2}{2} + o (1), \quad x \to +\infty, \\
    	\label{eq:V_neg_asymptotic}
    	V(x) = & -\log 2 - \alpha x + \log \sqrt{2 \pi} - \frac{x^2}{2} + o (1), \quad x \to -\infty.
    \end{align}
    Moreover, $\mu(\d x) = \exp(-V(x)) \d x$. According to Proposition 5.5 in \cite{cattiaux2006quadratic}, in order to show that $\mu$ satisfies the $T_2$-inequality, it suffices to verify the following two conditions:
    \begin{itemize}
    	\item [(a)] There exists $\veps > 0$ and $x_0 \in \R$ such that $\int_{\R} \exp(\veps (x - x_0)^2) \mu(\d x) < \infty$: 
    	
    	We can just take $\veps < 1/2$ and note that $\mu (\d x) = \exp(-V(x)) \d x = \exp(-x^2/2 + O(x)) \d x$, which immediately implies the desired inequality.
    	\item [(b)] The quantities $A^+$ and $A^-$ defined below are both finite:
    	    \begin{align*}
    	    	A^{+} = & \sup _{x \ge 0}\left(\int_x^{\infty} t^2 \mathrm{e}^{-V(t)} \mathrm{d} t \int_0^x \mathrm{e}^{V(t)} \mathrm{d} t\right), \\
    	    	A^{-} = & \sup _{x \le 0}\left(\int_{-\infty}^x t^2 \mathrm{e}^{-V(t)} \mathrm{d} t \int_x^0 \mathrm{e}^{V(t)} \mathrm{d} t\right).
    	    \end{align*}
    	    We first prove $A^+ < \infty$, it suffices to consider the supremum over all $x \ge M$ for $M$ a large enough constant. According to Eq.~\eqref{eq:V_pos_asymptotic}, there exists a constant $C > 0$ such that for all $x \ge 0$,
    	    \begin{equation*}
    	    	\exp(-V(x)) \le C \exp \left( -\frac{x^2}{2} \right), \quad \exp(V(x)) \le C \exp \left( \frac{x^2}{2} \right),
    	    \end{equation*}
    	    thus leading to the estimate
    	    \begin{align*}
    	    	\int_x^{\infty} t^2 \mathrm{e}^{- V(t)} \mathrm{d} t \le & C \int_x^{\infty} t^2 \mathrm{e}^{- t^2/2} \mathrm{d} t = - C \int_{x}^{\infty} t \d e^{-t^2/2} = C \left(x e^{-x^2/2} + \int_{x}^{\infty} e^{-x^2 / 2} \d x \right) \\
    	    	\stackrel{(i)}{\le} & C \left(x e^{-x^2/2} + \frac{2}{x} e^{-x^2 / 2} \right) \le C x e^{-x^2/2}.
    	    \end{align*}
    	    where $(i)$ follows from Mills ratio if we choose $M$ to be sufficiently large. This further implies that
    	    \begin{align*}
    	    	& \int_x^{\infty} t^2 \mathrm{e}^{-V(t)} \mathrm{d} t \int_0^x \mathrm{e}^{V(t)} \mathrm{d} t \le C^2 x \int_{0}^{x} \exp \left( (t^2 - x^2) / 2 \right) \d t \\
    	    	\le & C^2 x \int_{0}^{x} \exp \left( - xt/2 \right) \d t \le C^2 x \cdot \frac{2}{x} = 2 C^2.
    	    \end{align*}
    	    Hence, $A^+ < \infty$. To show that $A^- < \infty$, we note that Eq.~\eqref{eq:V_neg_asymptotic} implies $\forall x \le 0$:
    	    \begin{equation*}
    	    	\exp (-V(x)) \le C \exp \left( - \frac{(x+\alpha)^2}{2} \right), \quad \exp (V(x)) \le C \exp \left( \frac{(x+\alpha)^2}{2} \right).
    	    \end{equation*}
    	    Moreover, if $x \le -M$ then $x^2 \le C (x + \alpha)^2$, thus yielding the estimate
    	    \begin{align*}
    	    	& \int_{-\infty}^x t^2 \mathrm{e}^{-V(t)} \mathrm{d} t \int_x^0 \mathrm{e}^{V(t)} \mathrm{d} t \le C^3 \int_{-\infty}^x (t+\alpha)^2 \mathrm{e}^{-(t+\alpha)^2/2} \mathrm{d} t \int_x^0 \mathrm{e}^{(t+\alpha)^2/2} \mathrm{d} t \\
    	    	= & C^3 \int_{-\infty}^{x+\alpha} t^2 \mathrm{e}^{-t^2/2} \mathrm{d} t \int_{x+\alpha}^0 \mathrm{e}^{t^2/2} \mathrm{d} t \le 2C^3
    	    \end{align*}
            as long as $M$ is large enough. This shows $A^- < \infty$.
    \end{itemize}
    We thus have verified conditions (a) and (b) and completed the proof of Lemma~\ref{lem:transport_yG}.
\end{proof}

\begin{proof}[\bf Proof of Lemma~\ref{lem:norm_bd_yX}]
	By definition, we have
	\begin{equation*}
		s_{\rm min} (\yy \odot \XX) = \min_{\norm{\vv}_2 = 1} \norm{\yy \odot \XX \vv}_2 = \min_{\norm{\vv}_2 = 1} \norm{\XX \vv}_2 = s_{\rm min} (\XX).
	\end{equation*}
	The conclusion then follows from classical facts from random matrix theory: with high probability $s_{\rm min} (\XX) \ge (\sqrt{\delta} - 1 - o_d(1)) \sqrt{d}$. This completes the proof.
\end{proof}

\begin{proof}[\bf Proof of Lemma~\ref{lem:sig_lower_asymptotic}]
To verify Assumption~\ref{cond:sig_lower_b}, we discuss the choice of parameters $\rho, \kappa_1, \kappa_2, c$ in the following two parts:

\noindent \textbf{Part 1. The truncated second moment method.} 

\noindent Let $t, \eta >0$ be small constants independent of $\kappa$. Set
	\begin{equation*}
		\rho = 1 - \frac{t}{\vert \kappa \vert}, \ \kappa_2 = \kappa_0 + c \left( \vert \kappa \vert + \alpha + \eta \right), \ \kappa_1 = \kappa_2 + M (\kappa),
	\end{equation*}
	where $M(\kappa) = \vert \kappa \vert^{1/4}$. Then we have
	\begin{equation*}
		\sqrt{1 - \rho^2} \kappa_2 - \kappa = \breve{o}_{\kappa} (\vert \kappa \vert^{-1}).
	\end{equation*}
	According to Lemma~\ref{lem:tail}, it follows that
	\begin{align*}
		& \P \left( \rho YG + \sqrt{1 - \rho^2} W < \sqrt{1 - \rho^2} \kappa_2 \right) \\
		= & \frac{1 + \breve{o}_{\kappa} (1)}{\vert \kappa + \breve{o}_{\kappa} (\vert \kappa \vert^{-1}) \vert} \sqrt{\frac{2}{\pi}} \exp \left( - \frac{(\kappa + \breve{o}_{\kappa} (\vert \kappa \vert^{-1}))^2}{2} + \alpha \rho (\kappa + \breve{o}_{\kappa} (\vert \kappa \vert^{-1})) + \frac{(1 - \rho^2) \alpha^2}{2} \right) \\
		= & \frac{1 + \breve{o}_{\kappa} (1)}{\vert \kappa \vert} \sqrt{\frac{2}{\pi}} \exp \left( - \frac{\kappa^2}{2} + \alpha \rho \kappa + \frac{(1 - \rho^2) \alpha^2}{2} \right) = \frac{1 + \breve{o}_{\kappa} (1)}{\vert \kappa \vert} \sqrt{\frac{2}{\pi}} \exp \left( - \frac{\kappa^2}{2} + \alpha \kappa + \alpha t \right).
	\end{align*}
	Therefore, we can take $t > 0$ to be small enough so that
	\begin{equation}\label{eq:new_notation_threshold}
		 \left( 1 - \frac{\veps}{2} \right) \frac{\sqrt{\pi}}{2 \sqrt{2}} \vert \kappa \vert \log \vert \kappa \vert \exp \left( \frac{\kappa^2}{2} + \alpha \vert \kappa \vert \right) < \left( 1 - \frac{\veps}{4} \right) \frac{\log \big\vert \sqrt{1 - \rho^2} \kappa_2 \big\vert}{2 \P \left( \rho YG + \sqrt{1 - \rho^2} W < \sqrt{1 - \rho^2} \kappa_2 \right)}.
	\end{equation}
	Additionally, we have
	\begin{equation*}
		\rho = 1 - \frac{t_\kappa}{\big\vert \sqrt{1 - \rho^2} \kappa_2 \big\vert}, \quad t_\kappa \to t.
	\end{equation*}
	Next we aim to show that
	\begin{equation}\label{eq:delta_sec_lower_bound}
		\delta_{\sec} \left( \rho, \sqrt{1-\rho^2} \kappa_2, \frac{\sqrt{1 - \rho^2}}{\rho} \kappa_1 \right) \ge \left( 1 - \frac{\veps}{4} \right) \frac{\log \big\vert \sqrt{1 - \rho^2} \kappa_2 \big\vert}{2 \P \left( \rho YG + \sqrt{1 - \rho^2} W < \sqrt{1 - \rho^2} \kappa_2 \right)}.
	\end{equation}
	Under the new notation introduced in the proof of Lemma~\ref{lem:good_sample_2nd_moment_method}, we have
	\begin{equation*}
	\rho = 1 - \frac{t_\kappa}{\vert \kappa \vert}, \ t_\kappa \to t, \quad \kappa_{\mathrm t} = \frac{\kappa}{\rho} + \frac{\sqrt{1-\rho^2} M(\kappa)}{\rho}, M(\kappa) \to +\infty.
	\end{equation*}
    Recasting $\veps$ as $4 \veps$, it suffices to prove that as long as
	\begin{equation}\label{eq:logistic_threshold}
	\delta < \frac{(1 - \veps) \log \vert \kappa \vert}{2 \P \left( \rho YG + \sqrt{1 - \rho^2} W < \kappa \right)},
    \end{equation}
	then $F''(0) + I''(0)/\delta > 0$, and that $F(q) + I(q)/\delta$ achieves unique minimum at $q = 0$. We proceed with the following two steps:

\noindent \textbf{Step 1. Verify the first condition}

\noindent In fact, the first condition will be automatically implied by the second one. To see this we note that
\begin{equation*}
	e'(0) =  \int_{\frac{\kappa}{\rho} + \frac{\sqrt{1-\rho^2} M(\kappa)}{\rho}}^{+\infty} p_{YG} (s) \frac{e'(0, s)}{e(0, s)} \d s = 0,
\end{equation*}
where we use $e'(q, s)$ to denote the partial derivative of $e(q, s)$ with respect to the first variable $q$, and $e'(0, s) = 0$ is just a consequence of Lemma~\ref{lem:property_of_e_q} \textit{(b)}. Therefore, $F'(0) + I'(0)/\delta = 0$. Assume that for all $\delta$ satisfying Eq.~\eqref{eq:logistic_threshold}, $F(q) + I(q)/\delta$ is uniquely minimized at $q = 0$, then we must have $F''(0) + I''(0)/\delta \ge 0$. However, $F''(0) + I''(0)/\delta$ is strictly decreasing in $\delta$, so it must be the case that $F''(0) + I''(0)/\delta > 0$ for all $\delta$ satisfying Eq.~\eqref{eq:logistic_threshold}.
  
\noindent \textbf{Step 2. Verify the second condition}

\noindent Similarly as before, it suffices to show that (note $e(0) = 1$ by Definition~\ref{def:sig_lower})
\begin{align*}
	& \forall q \in [-1, 1] \backslash \{ 0 \} , \ \delta \log \left( \frac{e(q)}{e(0)} \right) < I(q) \impliedby \ \delta \left( e(q) - 1 \right) < I(q) \\
	\iff & \delta \int_{\frac{\kappa}{\rho} + \frac{\sqrt{1-\rho^2} M(\kappa)}{\rho}}^{+\infty} \frac{p_{YG} (s)}{e(0, s)} \left( e(q, s) - e(0, s) \right) \d s < I(q) \\
	\iff & \delta \int_{\frac{\kappa}{\rho} + \frac{\sqrt{1-\rho^2} M(\kappa)}{\rho}}^{+\infty} \frac{p_{YG} (s)}{e(0, s)} \d s\int_{0}^{q} e'(u, s) \d u < I(q).
\end{align*}
Consider the following three cases:

(1) $q \in [-1/2, 1/2]\backslash \{ 0 \}$. Since $\kappa(s) \le - M(\kappa)$, Lemma~\ref{lem:property_of_e_q} \textit{(b)} tells us that $e(0, s) = 1 +\breve{o}_{\kappa} (1)$, and further from Eq.~\eqref{eq:bound_on_e''(q)} we know that
\begin{equation*}
	\sup_{u \in [-q, q]} \vert e''(u, s) \vert \le \vert \kappa (s)\vert^C \exp \left( - \frac{2 \kappa(s)^2}{3} \right) = \breve{o}_{\kappa} \left( \exp \left( - \frac{3 \kappa(s)^2}{5} \right) \right),
\end{equation*}
therefore it follows that (note $e'(0, s) = 0$)
\begin{align*}
	 & \delta \int_{\frac{\kappa}{\rho} + \frac{\sqrt{1-\rho^2} M(\kappa)}{\rho}}^{+\infty} \frac{p_{YG} (s)}{e(0, s)} \d s\int_{0}^{q} e'(u, s) \d u = \delta \int_{\frac{\kappa}{\rho} + \frac{\sqrt{1-\rho^2} M(\kappa)}{\rho}}^{+\infty} \frac{p_{YG} (s)}{e(0, s)} \d s\int_{0}^{q} \d u \int_{0}^{u} e''(v, s) \d v \\
	 \stackrel{(i)}{=} & \breve{o}_{\kappa}(1) \delta \frac{q^2}{2} \int_{\frac{\kappa}{\rho} + \frac{\sqrt{1-\rho^2} M(\kappa)}{\rho}}^{+\infty} (\varphi(s) + 1 - \varphi (-s)) \frac{1}{\sqrt{2 \pi}} \exp \left( - \frac{s^2}{2} - \frac{3 \kappa(s)^2}{5} \right) \d s \\
	 = & \breve{o}_{\kappa}(1) \delta \frac{q^2}{2} \exp \left( - \frac{3 \kappa^2}{5 + \rho^2} \right) \int_{\frac{\kappa}{\rho} + \frac{\sqrt{1-\rho^2} M(\kappa)}{\rho}}^{+\infty} (\varphi(s) + 1 - \varphi (-s)) \frac{1}{\sqrt{2 \pi}} \exp \left( - \frac{5 + \rho^2}{10(1 - \rho^2)} \left( s - \frac{6 \rho \kappa}{5 + \rho^2} \right)^2 \right) \d s \\
	 \stackrel{(ii)}{\le} & \breve{o}_{\kappa}(1) \delta \frac{q^2}{2} \exp \left( - \frac{3 \kappa^2}{5 + \rho^2} \right) \sqrt{\frac{5(1 - \rho^2)}{5 + \rho^2}} \\
	 & \times \int_{\R} \left( \varphi \left( \sqrt{\frac{5(1 - \rho^2)}{5 + \rho^2}} x + \frac{6 \rho \kappa}{5 + \rho^2} \right) + 1 - \varphi \left( -\sqrt{\frac{5(1 - \rho^2)}{5 + \rho^2}} x - \frac{6 \rho \kappa}{5 + \rho^2} \right) \right) \frac{1}{\sqrt{2 \pi}} \exp \left( - \frac{x^2}{2} \right) \d x,
\end{align*}
where $(i)$ is due to the fact that $p_{YG} (s) = (\varphi(s) + 1 - \varphi(-s)) \phi (s)$, $(ii)$ results from the change of variable
\begin{equation*}
	x = \sqrt{\frac{5 + \rho^2}{5(1 - \rho^2)}} \left( s - \frac{6 \rho \kappa}{5 + \rho^2} \right) \iff s = \sqrt{\frac{5(1 - \rho^2)}{5 + \rho^2}} x + \frac{6 \rho \kappa}{5 + \rho^2}.
\end{equation*}
Now, recall that Assumption~\ref{ass:exponential_tail_link} implies (note $\rho = 1 - t_\kappa/\vert \kappa \vert \to 1$)
\begin{equation*}
	\frac{\varphi \left( \sqrt{\frac{5(1 - \rho^2)}{5 + \rho^2}} x + \frac{6 \rho \kappa}{5 + \rho^2} \right) + 1 - \varphi \left( -\sqrt{\frac{5(1 - \rho^2)}{5 + \rho^2}} x - \frac{6 \rho \kappa}{5 + \rho^2} \right)}{2 \exp\left(\frac{6 \alpha \rho \kappa}{5+\rho^2} \right)} \to 1,
\end{equation*}
hence applying Dominated Convergence Theorem yields that
\begin{align*}
	& \delta \int_{\frac{\kappa}{\rho} + \frac{\sqrt{1-\rho^2} M(\kappa)}{\rho}}^{+\infty} \frac{p_{YG} (s)}{e(0, s)} \d s\int_{0}^{q} e'(u, s) \d u \le \breve{o}_{\kappa}(1) \delta \frac{q^2}{2} \exp \left( - \frac{3 \kappa^2}{5 + \rho^2} \right) 2 \exp\left(\frac{6 \alpha \rho \kappa}{5+\rho^2} \right) \sqrt{\frac{5(1 - \rho^2)}{5 + \rho^2}} \\
	\le & q^2 \vert \kappa \vert^C \exp \left( - \frac{1 - \rho^2}{2(5 + \rho^2)} \kappa^2 - \frac{\alpha \rho(1 - \rho^2)}{5 + \rho^2} \vert \kappa \vert \right) \le q^2 \vert \kappa \vert^C \exp \left(- C t_\kappa \vert \kappa \vert \right) = \breve{o}_{\kappa} (1) q^2 < I(q),
\end{align*}
when $\vert \kappa \vert$ is large enough and $q \in [-1/2, 1/2] \backslash \{ 0 \}$.

(2) $q \in [-1, -1/2]$. Then similarly as the previous case, we deduce from Eq.~\eqref{eq:asymptotics_for_e'(q)} that
\begin{align*}
	& \delta \int_{\frac{\kappa}{\rho} + \frac{\sqrt{1-\rho^2} M(\kappa)}{\rho}}^{+\infty} \frac{p_{YG} (s)}{e(0, s)} \d s\int_{0}^{q} e'(u, s) \d u \le C \delta \int_{\frac{\kappa}{\rho} + \frac{\sqrt{1-\rho^2} M(\kappa)}{\rho}}^{+\infty} p_{YG} (s) \exp \left( - \kappa(s)^2 \right) \d s \\
	= & C \delta \exp \left( - \frac{\kappa^2}{1 + \rho^2} \right) \int_{\frac{\kappa}{\rho} + \frac{\sqrt{1-\rho^2} M(\kappa)}{\rho}}^{+\infty} (\varphi(s) + 1 - \varphi(-s)) \frac{1}{\sqrt{2 \pi}} \exp \left( - \frac{1+\rho^2}{2(1 - \rho^2)} \left( s - \frac{2 \rho \kappa}{1 + \rho^2} \right)^2 \right) \d s \\
	\le & C \delta \exp \left( - \frac{\kappa^2}{1 + \rho^2} \right) \sqrt{\frac{1 - \rho^2}{1 + \rho^2}} \\
	& \times \int_{\R} \left( \varphi \left(\sqrt{\frac{1 - \rho^2}{1 + \rho^2}} x + \frac{2 \rho \kappa}{1 + \rho^2} \right) + 1 - \varphi \left( - \sqrt{\frac{1 - \rho^2}{1 + \rho^2}} x -\frac{2 \rho \kappa}{1 + \rho^2} \right) \right) \frac{1}{\sqrt{2 \pi}} \exp \left( - \frac{x^2}{2} \right) \d x \\
	\le & 2 C (1 + \breve{o}_{\kappa} (1)) \delta \sqrt{\frac{1 - \rho^2}{1 + \rho^2}} \exp \left( - \frac{\kappa^2}{1 + \rho^2} + \frac{2  \alpha \rho \kappa}{1 + \rho^2} \right) \le \vert \kappa \vert^C \exp \left( - \frac{1-\rho^2}{2(1 + \rho^2)} \kappa^2 - \frac{\alpha \rho(1 - \rho^2)}{1 + \rho^2} \vert \kappa \vert \right) \\
	\le & \vert \kappa \vert^C \exp \left( - C t_\kappa \vert \kappa \vert \right) = \breve{o}_{\kappa} (1) < I(q),
\end{align*}
for sufficiently large $\vert \kappa \vert$ and $q \in [-1, -1/2]$.

(3) $q \in [1/2, 1]$. From Lemma~\ref{lem:property_of_e_q} \textit{(b)} and \textit{(c)} we know that for large $\vert \kappa \vert$, $e'(u, s) \ge 0$ when $u \ge 1/2$, and $e(1, s) - e(0, s) = (1 + \breve{o}_{\kappa}(1)) \Phi (\kappa(s))$, therefore
\begin{align*}
	& \delta \int_{\frac{\kappa}{\rho} + \frac{\sqrt{1-\rho^2} M(\kappa)}{\rho}}^{+\infty} \frac{p_{YG} (s)}{e(0, s)} \d s\int_{0}^{q} e'(u, s) \d u \le \delta \int_{\frac{\kappa}{\rho} + \frac{\sqrt{1-\rho^2} M(\kappa)}{\rho}}^{+\infty} \frac{p_{YG} (s)}{e(0, s)} (e(1, s) - e(0, s)) \d s \\
	\le & (1 + \breve{o}_{\kappa}(1)) \delta \int_{\frac{\kappa}{\rho} + \frac{\sqrt{1-\rho^2} M(\kappa)}{\rho}}^{+\infty} p_{YG} (s) \Phi (\kappa(s)) \d s = (1 + \breve{o}_{\kappa}(1)) \delta \int_{\frac{\kappa}{\rho} + \frac{\sqrt{1-\rho^2} M(\kappa)}{\rho}}^{+\infty} p_{YG} (s) \P \left( W \le \frac{\kappa - \rho s}{\sqrt{1 - \rho^2}} \right) \d s \\
	= & (1 + \breve{o}_{\kappa}(1)) \delta \P \left( \rho YG + \sqrt{1 - \rho^2} W \le \kappa, \rho YG \ge \kappa + \sqrt{1-\rho^2} M(\kappa) \right) \\
	\le & (1 + \breve{o}_{\kappa}(1)) \delta \P \left( \rho YG + \sqrt{1 - \rho^2} W \le \kappa \right) < \frac{(1+\breve{o}_{\kappa}(1)) (1-\veps)}{2} \log \vert \kappa \vert \\
	< & \left( \frac{1}{2} - \frac{\veps}{4} \right) \log \vert \kappa \vert \le I(q) = - \frac{1}{2} \log \left( 1 - q^2 \right),
\end{align*}
if $1 - q \le \vert \kappa \vert^{-1 + \veps/2}$. Now it remains to deal with the case when $1 - q > \vert \kappa \vert^{-1 + \veps/2}$, using Eq.~\eqref{eq:asymptotics_for_e'(q)} and a similar argument as in the proof of Theorem~\ref{thm:pure_noise_lower_bound}, we obtain that
\begin{align*}
	& \int_{0}^{q} e'(u, s) \d s \le (1 + \breve{o}_{\kappa}(1)) \int_{0}^{q} \frac{1}{2 \pi \sqrt{1 - u^2}} \exp \left( - \frac{\kappa(s)^2}{1 + u} \right) \d u \le C \exp \left( - \frac{\kappa(s)^2}{1 + q} \right) \\
	= & C \exp \left( - \frac{\kappa(s)^2}{2} - \frac{(1 - q)\kappa(s)^2}{2(1 + q)} \right) \le C \exp \left( - \frac{\kappa(s)^2}{2} - \frac{\kappa(s)^2}{4 \vert \kappa \vert^{1 - \veps/2}} \right) = C \exp \left( - a \kappa(s)^2 \right),
\end{align*}
where we denote $a = 1/2 + 1/ (4 \vert \kappa \vert^{1 - \veps/2})$. Similarly as the previous computations, it follows that
\begin{align*}
	& \delta \int_{\frac{\kappa}{\rho} + \frac{\sqrt{1-\rho^2} M(\kappa)}{\rho}}^{+\infty} \frac{p_{YG} (s)}{e(0, s)} \d s \int_{0}^{q} e'(u, s) \d u \le C \delta \int_{\frac{\kappa}{\rho} + \frac{\sqrt{1-\rho^2} M(\kappa)}{\rho}}^{+\infty} p_{YG} (s) \exp \left( - a \kappa(s)^2 \right) \d s \\
	= & C \delta \exp \left( - \frac{a \kappa^2}{1 + (2a-1) \rho^2} \right) \int_{\frac{\kappa}{\rho} + \frac{\sqrt{1-\rho^2} M(\kappa)}{\rho}}^{+\infty} (\varphi(s) + 1 - \varphi(-s)) \\
	& \times \exp \left( - \frac{1 + (2a-1) \rho^2}{2(1 - \rho^2)} \left( s - \frac{2a \rho \kappa}{1 + (2a-1) \rho^2} \right)^2 \right) \d s \\
	\le & C \delta \exp \left( - \frac{a \kappa^2}{1 + (2a-1) \rho^2} \right) \sqrt{\frac{1 - \rho^2}{1 + (2a-1) \rho^2}} \int_{\R} \frac{1}{\sqrt{2 \pi}} \exp \left( -\frac{x^2}{2} \right) \d x \\
	& \times \left( \varphi \left( \sqrt{\frac{1 - \rho^2}{1 + (2a-1) \rho^2}} x + \frac{2a \rho \kappa}{1 + (2a-1) \rho^2} \right) + 1 - \varphi \left( - \sqrt{\frac{1 - \rho^2}{1 + (2a-1) \rho^2}} x - \frac{2a \rho \kappa}{1 + (2a-1) \rho^2} \right) \right) \\
	\le & 2 C (1 + \breve{o}_{\kappa}(1)) \delta \exp \left( - \frac{a \kappa^2}{1 + (2a-1) \rho^2} \right) \sqrt{\frac{1 - \rho^2}{1 + (2a-1) \rho^2}} \exp \left( \frac{2 \alpha a \rho \kappa}{1 + (2a-1) \rho^2} \right) \\
	\le & \vert \kappa \vert^C \exp \left( - \frac{(2 a - 1) (1 - \rho^2)}{2(1 + (2a-1) \rho^2)} \kappa^2 - \frac{(2a-1) \alpha \rho (1 - \rho^2)}{1 + (2a-1) \rho^2} \vert \kappa \vert \right) \\
	\le & \vert \kappa \vert^C \exp \left( - C t_\kappa \vert \kappa \vert^{\veps/2} \right) = \breve{o}_{\kappa} (1) < I(q),
\end{align*}
when $\vert \kappa \vert$ is large and $q \in [1/2, 1]$. Here the last inequality just results from the following quick calculation:
\begin{equation*}
	a = \frac{1}{2} + \frac{1}{4 \vert \kappa \vert^{1 - \veps/2}} \implies \frac{(2 a - 1) (1 - \rho^2)}{2(1 + (2a-1) \rho^2)} \kappa^2 \ge C \frac{1}{\vert \kappa \vert^{1 - \veps/2}} \frac{t_\kappa}{\vert \kappa \vert} \kappa^2 = C t_\kappa \vert \kappa \vert^{\veps/2}.
\end{equation*}

To conclude, we have finished the verification of the above two conditions, thus proving Eq.~\eqref{eq:delta_sec_lower_bound}. Combining this with Eq.~\eqref{eq:new_notation_threshold} finally yields
\begin{equation}\label{eq:part_1_conclusion}
	\delta \cdot \delta_{\sec} \left( \rho, \sqrt{1-\rho^2} \kappa_2, \frac{\sqrt{1 - \rho^2}}{\rho} \kappa_1 \right)^{-1} < \frac{1 - \veps}{1 - \veps / 2}.
\end{equation}
	
\noindent \textbf{Part 2. The Gordon calculation.}

\noindent Based on the choice of $\rho, \kappa_2$ and $\kappa_1$ in   Part 1, we deduce that
	\begin{equation*}
		\frac{\sqrt{1 - \rho^2} \kappa_1}{\rho} = \frac{\sqrt{1 - \rho^2} \kappa_2}{\rho} + \frac{\sqrt{1 - \rho^2} M(\kappa)}{\rho} = \frac{\kappa}{\rho} + \breve{o}_{\kappa} \left( \frac{1}{\vert \kappa \vert} \right) + \frac{\sqrt{1 - \rho^2} M(\kappa)}{\rho} = \frac{\kappa}{\rho} + \breve{o}_{\kappa} (1) = \kappa - t +  \breve{o}_{\kappa} (1).
	\end{equation*} 
	Hence, by Lemma~\ref{lem:tail}, it follows that
	\begin{equation*}
		\delta p_B = \delta \P (u < \kappa_1) = \delta \P \left( YG < \frac{\sqrt{1 - \rho^2} \kappa_1}{\rho} \right) = \delta \P \left( YG < \kappa - t + \breve{o}_{\kappa} (1) \right) \le C \exp \left( - t \vert \kappa \vert \right).
	\end{equation*}
	Moreover, if we choose $c = 1 / \vert \kappa \vert^2$, then
	\begin{equation}\label{eq:condition_bad_sample_gordon}
		\frac{\delta p_B}{c^2} \le C \vert \kappa \vert^4 \exp \left( - t \vert \kappa \vert \right) = \breve{o}_{\kappa} (1).
	\end{equation}	
    Now we focus on the second term on the right hand side of Eq.~\eqref{eq:cond_b}. We obtain the following estimate:
    \begin{align*}
	& \delta \P \left( YG \ge \frac{\sqrt{1 - \rho^2}}{\rho} \kappa_1 \right) \E \left[ \left( \frac{\kappa_0 - \kappa_2}{c} - W \right)_+^2 \right] \le \delta \E \left[ \left( \frac{\kappa_0 - \kappa_2}{c} - W \right)_+^2 \right] \\
	\stackrel{(i)}{=} & \delta \frac{2(1 + \breve{o}_{\kappa}(1))}{\left( \frac{\kappa_0 - \kappa_2}{c} \right)^2} \Phi \left( \frac{\kappa_0 - \kappa_2}{c} \right) \le \frac{C \delta}{\kappa^2} \Phi \left( \kappa - \alpha - \eta \right) \le C \vert \kappa \vert^C \exp \left( - \eta \vert \kappa \vert \right)= \breve{o}_{\kappa} (1),
    \end{align*}
    where $(i)$ is due to Lemma~\ref{lem:tail}. As for the third term, we write
\begin{align*}
	 & \frac{\delta}{c^2} \E \left[ \left( \kappa_0 - \frac{\rho YG}{\sqrt{1 - \rho^2}} - \sqrt{1+c^2} W \right)_+^2 \bone \left\{ YG < \frac{\sqrt{1 - \rho^2} \kappa_1}{\rho} \right\} \right] \\
	 = & \frac{\delta}{c^2} \int_{-\infty}^{\frac{\sqrt{1 - \rho^2} \kappa_1}{\rho}} p_{YG} (s) \E \left[ \left( \frac{\kappa - \rho s}{\sqrt{1 - \rho^2}} - \sqrt{1 + c^2} W \right)_+^2 \right] \d s.
\end{align*}
Then it follows that
\begin{align*}
	& \int_{-\infty}^{\frac{\sqrt{1 - \rho^2} \kappa_1}{\rho}} p_{YG} (s) \E \left[ \left( \frac{\kappa - \rho s}{\sqrt{1 - \rho^2}} - \sqrt{1 + c^2} W \right)_+^2 \right] \d s \\
	= & \left( \int_{-\infty}^{\frac{\kappa}{\rho}} + \int_{\frac{\kappa}{\rho}}^{\frac{\sqrt{1 - \rho^2} \kappa_1}{\rho}} \right) p_{YG} (s) \E \left[ \left( \frac{\kappa - \rho s}{\sqrt{1 - \rho^2}} - \sqrt{1 + c^2} W \right)_+^2 \right] \d s \\
	\stackrel{(i)}{\le} & C \int_{-\infty}^{\frac{\kappa}{\rho}} p_{YG} (s) \left( 1 + \frac{(\kappa - \rho s)^2}{1 - \rho^2} \right) \d s + C \int_{\frac{\kappa}{\rho}}^{\frac{\sqrt{1 - \rho^2} \kappa_1}{\rho}} p_{YG} (s) \d s \\
	= & C p_B + \frac{C}{1 - \rho^2} \E \left[ \left( \kappa - \rho YG \right)_+^2 \right] \stackrel{(ii)}{\le} C p_B + \frac{C \rho^2}{1 - \rho^2} (1 + \breve{o}_{\kappa}(1)) \frac{\rho^2}{\kappa^2} \P \left( YG < \frac{\kappa}{\rho} \right) \\
	\le & C \left( p_B + \frac{1}{(1 - \rho^2) \kappa^2} \P \left( YG < \frac{\kappa}{\rho} \right) \right) \le \frac{C p_B}{(1 - \rho^2) \kappa^2},
\end{align*}
where $(i)$ is due to the fact that $\E [(a - \sqrt{1 + c^2} W)_+^2] \le C (a^2 +1)$, and further $\le C$ if $a \le 0$ (since $c = \breve{o}_{\kappa} (1)$), $(ii)$ results from Lemma~\ref{lem:tail}:
\begin{equation*}
	\E \left[ \left( \kappa - \rho YG \right)_+^2 \right] = \rho^2 \E \left[ \left( \frac{\kappa}{\rho} - YG \right)_+^2 \right] = (1 + \breve{o}_{\kappa}(1)) \frac{2 \rho^4}{\kappa^2} \P \left( YG < \frac{\kappa}{\rho} \right).
\end{equation*}
Therefore, we finally deduce that
\begin{equation*}
	\frac{\delta}{c^2} \E \left[ \left( \kappa_0 - \frac{\rho YG}{\sqrt{1 - \rho^2}} - \sqrt{1 + c^2} W \right)_+^2 \bone \left\{ YG < \frac{\sqrt{1 - \rho^2} \kappa_1}{\rho} \right\} \right] \le \frac{C \delta p_B}{c^2 (1 - \rho^2) \kappa^2} \le C \vert \kappa \vert^C \exp (- t \vert \kappa \vert),
\end{equation*}
where the last inequality follows from Eq.~\eqref{eq:condition_bad_sample_gordon}. Note that we already proved that
\begin{footnotesize}
	\begin{equation*}
	\delta \left( \P \left(YG \ge \frac{\sqrt{1-\rho^2}}{\rho} \kappa_1 \right) \E \left[ \left( \frac{\kappa_0 - \kappa_2}{c} - W \right)_+^2 \right] + \frac{1}{c^2}\E \left[ \bone \left\{ YG < \frac{\sqrt{1-\rho^2}}{\rho} \kappa_1 \right\} \left( \kappa_0 - \frac{\rho YG}{\sqrt{1-\rho^2}} - \sqrt{1 + c^2} W \right)_+^2 \right] \right) = \breve{o}_{\kappa} (1),
\end{equation*}
\end{footnotesize}
and further $< \veps/2$ if $\vert \kappa \vert$ is large enough. Combining this result with Eq.~\eqref{eq:part_1_conclusion} immediately gives Eq.~\eqref{eq:cond_b}.
\end{proof}

\begin{proof}[\bf Proof of Lemma~\ref{lem:continuity_of_f_delta}]
	Since the infimum of an arbitrary collection of upper semicontinuous functions is upper semicontinuous, we only need to show that for any $c > 0$,
\begin{equation*}
	\frac{c}{\delta} \left( \eta + \frac{\sqrt{1 - \rho^2}}{c \sqrt{1 - \rho^2} + \sqrt{c^2 (1 - \rho^2) + 4}} + \frac{1}{c} \log \frac{c \sqrt{1 - \rho^2} + \sqrt{c^2 (1 - \rho^2) + 4}}{2} - \inf_{u > 0} \left\{ \frac{c}{4 u} - \frac{\delta}{c} \log \psi_{\kappa, \rho} (-u) \right\} \right)
\end{equation*}
is upper semicontinuous in $\rho$. We will prove an even stronger statement, i.e.,
\begin{equation*}
	\inf_{u > 0} \left\{ \frac{c}{4 u} - \frac{\delta}{c} \log \psi_{\kappa, \rho} (-u) \right\}
\end{equation*}
is a continuous function of $\rho \in [-1, 1]$. Let $p_\rho (s)$ denote the p.d.f. of $\rho YG + \sqrt{1 - \rho^2} W$, then it's not hard to see that $\forall \rho, \rho' \in [-1, 1]$,
\begin{align*}
	& \bigg\vert \inf_{u > 0} \left\{ \frac{c}{4 u} - \frac{\delta}{c} \log \psi_{\kappa, \rho} (-u) \right\} - \inf_{u > 0} \left\{ \frac{c}{4 u} - \frac{\delta}{c} \log \psi_{\kappa, \rho'} (-u) \right\} \bigg\vert \le \frac{\delta}{c} \sup_{u > 0} \vert \log \psi_{\kappa, \rho} (-u) - \log \psi_{\kappa, \rho'} (-u) \vert \\
	\stackrel{(i)}{\le} & \frac{\delta C_\kappa}{c} \sup_{u > 0} \vert \psi_{\kappa, \rho} (-u) - \psi_{\kappa, \rho'} (-u) \vert = \frac{\delta C_\kappa}{c} \sup_{u > 0} \bigg\vert \int_\R \exp \left( - u (\kappa - s)_+^2 \right) \left( p_\rho(s) - p_{\rho'} (s) \right) \d s \bigg\vert \\
	\le & \frac{\delta C_\kappa}{c} \int_\R \big\vert p_\rho(s) - p_{\rho'} (s) \big\vert \d s = \frac{\delta C_\kappa}{c} \norm{p_\rho - p_{\rho'}}_{L^1(\R)},
\end{align*}
where $(i)$ is due to
\begin{equation*}
	\inf_{\rho \in [-1, 1], u > 0} \psi_{\kappa, \rho} (-u) = 1 - \sup_{\rho \in [-1, 1]} \P \left( \rho YG + \sqrt{1 - \rho^2} W < \kappa \right) > 0,
\end{equation*}
and the fact that $\log x$ is Lipschitz-continuous on any $[c, +\infty)$, $c > 0$. Hence, it reduces to proving $\norm{p_\rho - p_{\rho'}}_{L^1(\R)} \to 0$ when $\rho \to \rho'$, which according to Scheffe's Lemma will be implied by the a.e. convergence of $p_\rho$ to $p_{\rho'}$. Below we show that if $\varphi$ is continuous at both $s$ and $-s$, then $p_\rho (s) \to p_{\rho'} (s)$ as $\rho \to \rho'$. Consider the following two cases:

(1) $\rho' \in (-1, 1)$. Then we can assume $\rho \in (-1, 1)$ and write
\begin{equation*}
	p_\rho (s) = \frac{1}{\sqrt{1 - \rho^2}} \int_{\R} p_{YG} (x) \phi \left( \frac{s - \rho x}{\sqrt{1 - \rho^2}} \right) \d x.
\end{equation*}
Based on the continuity of $\phi$ and the dominated convergence theorem, we obtain that
\begin{equation*}
	\lim_{\rho \to \rho'} \int_{\R} p_{YG} (x) \phi \left( \frac{s - \rho x}{\sqrt{1 - \rho^2}} \right) \d x = \int_{\R} p_{YG} (x) \phi \left( \frac{s - \rho' x}{\sqrt{1 - \rho'^2}} \right) \d x,
\end{equation*}
thus leading to $\lim_{\rho \to \rho'} p_{\rho} (s) = p_{\rho'} (s)$.

(2) $\rho' = \pm 1$. In this case we have (suppose $\rho \ne 0$)
\begin{equation*}
	p_\rho (s) = \frac{1}{\vert \rho \vert} \int_{\R} p_{YG} \left( \frac{s - \sqrt{1 - \rho^2} x}{\rho} \right) \phi (x) \d x.
\end{equation*}
According to our assumption, $p_{YG}$ is continuous at $s/\rho'$, therefore as $\rho \to \rho'$ one has
\begin{equation*}
	p_{YG} \left( \frac{s - \sqrt{1 - \rho^2} x}{\rho} \right) \to p_{YG} \left( \frac{s}{\rho'} \right) = p_{\rho'} (s),
\end{equation*}
and again using dominated convergence theorem yields that
\begin{equation*}
	\lim_{\rho \to \rho'} p_{\rho} (s) = p_{\rho'} (s) \int_{\R} \phi (x) \d x = p_{\rho'} (s).
\end{equation*}

Now since $\varphi$ is almost everywhere continuous, for a.e. $s \in \R$ both $s$ and $-s$ are continuity points of $\varphi$, consequently we have $\lim_{\rho \to \rho'} \norm{p_\rho - p_{\rho'}}_{L^1(\R)} = 0$, which concludes the proof of this lemma.
\end{proof}

\begin{proof}[\bf Proof of Proposition~\ref{prop:properties_of_d(a)}]
	(a) $\lim_{a \to +\infty} 2 \mathcal{D} (a)/\log a = 1$. It suffices to prove that $\forall \veps > 0$:
\begin{enumerate}
	\item $\mathcal{D} (a) \ge (1 - \veps) (\log a)/2$ for large $a > 0$, i.e., $\forall c > 0$,
	\begin{equation*}
		\frac{1}{c + \sqrt{c^2 + 4}} + \frac{1}{c} \log \frac{c + \sqrt{c^2 + 4}}{2} \ge \inf_{t > 0} \left\{ \frac{c}{4 t a} +  \frac{(1 - \veps) \log a}{2 c} \int_{0}^{+\infty} 2 t s \exp \left( -t s^2 - s \right) \d s \right\}.
	\end{equation*}
	First, for sufficiently large $a > 0$ and $c \ge \sqrt{a}$, we have
	\begin{equation*}
		\LHS \ge (1 - \veps) \frac{\log c}{c} \ge \frac{(1 - \veps) \log a}{2 c} = \lim_{t \to +\infty} \left(  \frac{c}{4 t a} +  \frac{(1 - \veps) \log a}{2 c} \int_{0}^{+\infty} 2 t s \exp \left( -t s^2 - s \right) \d s \right) \ge \RHS.
	\end{equation*}
	Next, if $c < \sqrt{a}$, then we have
	\begin{equation*}
		\RHS \le \inf_{t > 0} \left\{ \frac{c}{4 t a} +  \frac{(1 - \veps) t \log a}{c} \right\} = \sqrt{\frac{(1 - \veps) \log a}{a}} \le \frac{(1 - \veps) \log \sqrt{a}}{\sqrt{a}} \le \LHS
	\end{equation*}
	for $a$ large enough. Hence we have proved the desired result.
	
	\item $\mathcal{D} (a) \le (1 + \veps) (\log a)/2$ for large $a > 0$, i.e., $\exists c > 0$,
	\begin{equation*}
		\frac{1}{c + \sqrt{c^2 + 4}} + \frac{1}{c} \log \frac{c + \sqrt{c^2 + 4}}{2} < \inf_{t > 0} \left\{ \frac{c}{4 t a} +  \frac{(1 + \veps) \log a}{2 c} \int_{0}^{+\infty} 2 t s \exp \left( -t s^2 - s \right) \d s \right\}.
	\end{equation*}
	Proof of this inequality is similar to the argument in the proof of Theorem~\ref{thm:pure_noise_upper_bound}.
\end{enumerate}

\noindent (b) $\mathcal{D} (a) = a/2$ when $a \le 2$. Again we need to show the following two things for any $\veps > 0$:
\begin{enumerate}
	\item $\mathcal{D} (a) \le (1 + \veps)a/2$ for any $a > 0$, i.e., $\exists c > 0$,
	\begin{equation*}
		\frac{1}{c + \sqrt{c^2 + 4}} + \frac{1}{c} \log \frac{c + \sqrt{c^2 + 4}}{2} < \inf_{t > 0} \left\{ \frac{c}{4 t a} +  \frac{(1 + \veps) a}{2 c} \int_{0}^{+\infty} 2 t s \exp \left( -t s^2 - s \right) \d s \right\}.
	\end{equation*}
	To this end, let $\eta > 0$ satisfy $(1 + \eta)^3 < 1 + \veps$, then we choose $c > 0$ small enough so that $\LHS < 1 + \eta$, and that
	\begin{equation*}
		\int_{0}^{+\infty} 2 t s \exp \left( -t s^2 - s \right) \d s \le \frac{2c(1 + \eta)}{a(1 + \veps)} \implies \int_{0}^{+\infty} 2 t s \exp \left( -t s^2 - s \right) \d s \ge \frac{2 t}{1 + \eta}.
	\end{equation*}
	Therefore, if
	\begin{equation*}
		\int_{0}^{+\infty} 2 t s \exp \left( -t s^2 - s \right) \d s > \frac{2c(1 + \eta)}{a(1 + \veps)},
	\end{equation*}
	then we have
	\begin{equation*}
		\frac{c}{4 t a} +  \frac{(1 + \veps) a}{2 c} \int_{0}^{+\infty} 2 t s \exp \left( -t s^2 - s \right) \d s > 1 + \eta \ge \LHS.
	\end{equation*}
	Otherwise, it follows that
	\begin{equation*}
		\frac{c}{4 t a} +  \frac{(1 + \veps) a}{2 c} \int_{0}^{+\infty} 2 t s \exp \left( -t s^2 - s \right) \d s \ge \frac{c}{4 t a} + \frac{(1 + \veps) a t}{c (1 + \eta)} \ge \sqrt{\frac{1 + \veps}{1 + \eta}} > 1 + \eta \ge \LHS.
	\end{equation*}
	Hence we have proved $\LHS < \RHS$ for small $c > 0$.
	
	\item $\mathcal{D} (a) \ge a/2$ for $0 < a \le 2$, i.e., $\forall c > 0$,
	\begin{equation*}
		\frac{1}{c + \sqrt{c^2 + 4}} + \frac{1}{c} \log \frac{c + \sqrt{c^2 + 4}}{2} \ge \inf_{t > 0} \left\{ \frac{c}{4 t a} +  \frac{a}{2 c} \int_{0}^{+\infty} 2 t s \exp \left( -t s^2 - s \right) \d s \right\}.
	\end{equation*}
	By taking derivatives we know that this infimum is obtained when
	\begin{equation*}
		\frac{c^2}{2a^2} = \sqrt{t} \int_{0}^{+\infty} x^2 \exp \left( -x^2 - \frac{x}{\sqrt{t}} \right) \d x.
	\end{equation*}
	Note that the above equation uniquely determines $t$ since the right hand side is increasing in $t$, now we make a change of variable $u = 1/\sqrt{2t}$ and integrate by part to get that
	\begin{equation*}
		\frac{2c^2}{a^2} + 1 = \frac{\left( u^2 + 1 \right) \left( 1 - \Phi(u) \right)}{u \phi (u)} := \frac{u^2+1}{u} R(u),
	\end{equation*}
	where $R(u)$ denotes the Mills ratio of standard normal distribution, thus leading to
	\begin{align*}
		& \frac{c}{4 t a} +  \frac{a}{2 c} \int_{0}^{+\infty} 2 t s \exp \left( -t s^2 - s \right) \d s = \frac{c}{a} \left( \frac{u^2}{2} + \frac{a^2}{2c^2} \left( 1 - u R(u) \right) \right) \\
		= & \frac{a}{2c} \left( \left( \frac{u^2+1}{u} R(u) - 1 \right) \frac{u^2}{2} + 1 - u R(u) \right) = \frac{a}{2c} \left( \frac{1 - u^2}{2} \left( 1 - u R(u) \right) + \frac{1}{2} \right).
	\end{align*}
	Hence, it suffices to prove that $\forall u > 0$,
	\begin{equation*}
		\frac{1 - u^2}{2} \left( 1 - u R(u) \right) + \frac{1}{2} \le \frac{2}{a} \left( \frac{c}{c + \sqrt{c^2 + 4}} + \log \frac{c + \sqrt{c^2 + 4}}{2} \right),
	\end{equation*}
	where
	\begin{equation*}
		c = \frac{a}{\sqrt{2}} \sqrt{\frac{u^2+1}{u} R(u) - 1}.
	\end{equation*}
	Since the right hand side is a decreasing function of $a$ for fixed $u > 0$, we only need to prove this inequality when $a = 2$.
	To this end we define
	\begin{equation*}
		h(u) = \frac{1 - u^2}{2} \left( 1 - u R(u) \right) + \frac{1}{2} - \left( \frac{c}{c + \sqrt{c^2 + 4}} + \log \frac{c + \sqrt{c^2 + 4}}{2} \right),
	\end{equation*}
	then simple calculation yields that
	\begin{equation*}
		h'(u) = \left( \frac{1}{2} - \frac{2}{u^2 c \left( c + \sqrt{c^2 + 4} \right) } \right) \left( \left( u^4 + 2u^2 - 1 \right) R(u) - u(u^2 + 1) \right).
	\end{equation*}
	According to Theorem 1 of \cite{gasull2014approximating}, we know that
	\begin{equation*}
		R(u) < \frac{u^2 + 2}{u^3 + 3u} < \frac{u^4 + 2u^2 + 2}{u (u^2 + 1) (u^2 + 2)},
	\end{equation*}
	which implies that
	\begin{equation*}
		\left( u^4 + 2u^2 - 1 \right) R(u) - u(u^2 + 1) < \left( u^4 + 2u^2 - 1 \right) \frac{u^2 + 2}{u^3 + 3u} - u(u^2 + 1) = - \frac{2}{u^3 + 3u} < 0,
	\end{equation*}
	and
	\begin{equation*}
		c = \sqrt{2 \left( \frac{u^2+1}{u} R(u) - 1 \right)} < \sqrt{2 \left( \frac{u^4 + 2u^2 + 2}{u^2 (u^2 + 2)} - 1 \right)} = \frac{2}{u \sqrt{u^2 + 2}}.
	\end{equation*}
	Therefore, we deduce that
	\begin{align*}
		& u^2 c \left( c + \sqrt{c^2 + 4} \right) < u^2 \frac{2}{u \sqrt{u^2 + 2}} \left( \frac{2}{u \sqrt{u^2 + 2}} + \frac{2 (u^2 + 1)}{u \sqrt{u^2 + 2}} \right) = 4 \\
		\implies & \frac{1}{2} - \frac{2}{u^2 c \left( c + \sqrt{c^2 + 4} \right)} < 0 \implies h'(u) > 0,
	\end{align*}
	i.e., $h(u)$ is increasing. Since $c \to 0$ as $u \to +\infty$, we know that $\lim_{u \to +\infty} h(u) = 0$, and consequently $h(u) \le 0$ for all $u > 0$.
\end{enumerate}

Combing parts (a) and (b) concludes the proof.
\end{proof}

\section{Linear programming algorithm in the linear signal model:\\ Proofs of Theorems
\ref{thm:signal} and \ref{thm:err}}
\label{sec:AlgorithmSignal}

As before, in our proofs for the linear signal model, we will treat $C_{\tail}=1$. 

Let us consider a slightly generalized setup. Suppose we are given $r_0 > 0$ and closed interval $\cI \subset [-1,1]$. Denote $\bP^\bot = \bI_d - \btheta^* (\btheta^*)^\top$ the projection matrix onto the orthogonal space of $\btheta^*$. Define
\begin{align}
& \Theta = \Big\{ \btheta \in \R^d: \langle \btheta ,\btheta^* \rangle \in \cI, \| \bP^\bot \btheta \| \le r_0 \sqrt{1 - \langle \btheta, \btheta^* \rangle^2}\Big\}, \label{def:Theta}\\
&\bar M_n = \max\Big\{ \langle \btheta, \vv \rangle:  \btheta \in \Theta,~ y_i \langle \xx_i, \btheta \rangle \ge  \kappa,~\forall\, i \in [n] \Big\}. \label{def:Mbar}
\end{align}
In particular, if $\cI = [-1,1]$ and $r_0=1$, then the maximization in \eqref{def:Mbar} is exactly the linear programming algorithm we introduced in Section~\ref{sec:AlgoNRLabels}. Recall that in \eqref{def:Z} we defined the random variable $Z := Z_{\rho,r}$ (we suppress the subscripts if no confusion arises). Define three disjoint sets
\begin{align*}
&\Omega_> (\cI, r_0) = \Big\{(\rho, r) \in \cI \times [0,r_0]: (1-\rho^2) r^2 \delta^{-1} > \E[Z^2 ; Z < 0 ] \Big\}, \\
&\Omega_<(\cI, r_0) = \Big\{(\rho, r) \in \cI \times [0,r_0]: (1-\rho^2) r^2 \delta^{-1} < \E[Z^2 ; Z < 0 ] \Big\}, \\
&\Omega_=(\cI, r_0) = \Big\{(\rho, r) \in \cI \times [0,r_0]: (1-\rho^2) r^2 \delta^{-1} = \E[Z^2 ; Z < 0 ] \Big\}.
\end{align*}
We also denote $\Omega_{\ge}(\cI, r_0) = \Omega_>(\cI, r_0) \cup \Omega_=(\cI, r_0)$. Using the notation in \eqref{def:omegage}, we have equivalence $\Omega_{\ge}(\cI, r_0) = \Omega_{\ge} \cap \{(\rho,r): \rho \in \cI, r \le r_0 \}$. 

We will prove the following crucial convergence result that serves as a cornerstone for proving Theorem~\ref{thm:signal} and~\ref{thm:err}.

\begin{thm}\label{thm:Mlimit}
Suppose that we are given $\cI$ and $r_0>0$ and that $\Omega_{>}(\cI, r_0)$ is nonempty. 
\begin{enumerate}
\item[(a)]{ The set $ \Omega_{\ge}(\cI, r_0)$ is compact. For $\rho \in \cI$, $r \in [0,r_0]$, the equation
\begin{equation}\label{def:s}
(1-\rho^2) r^2 \delta^{-1} = \E[ \max\{s, -Z\}^2 ], \qquad \text{s.t.} ~~~~ s \ge 0
\end{equation}
has a unique solution if and only if $(\rho, r) \in \Omega_{\ge}(\cI, r_0)$. Moreover, if $(\rho, r) \in \Omega_{\ge}(\cI, r_0)$, then the solution $s_* = s_*(\rho,r)$ is continuous in $(\rho,r) $. 
}
\item[(b)]{%Denote $m := \E[YG_1]$ where $Y,G_1$ are the same as in \eqref{def:Z}. 
The function
\begin{equation}
M(\rho,r) = \E \big[ ( Z + s)_+ \big] + \kappa.
\end{equation}
is continuous in $(\rho,r) \in  \Omega_{\ge}(\cI, r_0)$. Let
\begin{equation}\label{def:Mstar}
M^* = \max\big\{ M(\rho,r): (\rho,r) \in \Omega_{\ge}(\cI, r_0) \big\}.
\end{equation}
Then for any $\veps > 0$ independent of $n,d$, we have
\begin{equation*}
\lim_{n \to \infty} \P \big( M^* - \veps \le \bar M_n \le M^* + \veps \big) = 1.
\end{equation*}
}
\end{enumerate}
\end{thm}

\subsection{Reduction via Gordon's comparison theorem}

Let $\tilde \btheta_* \in \R^{d \times (d-1)}$ be the orthogonal complement of $\btheta_*$ (i.e., $(\btheta_*, \tilde \btheta_*)$ is an orthogonal matrix). Recall that in the beginning of Section~\ref{sec:signal-lower}, we write $\btheta$ into its projection to the space of $\btheta_*$ and its complement: denoting $\rho = \langle \btheta, \btheta_* \rangle$, we obtain an one-to-one map $\btheta \leftrightarrow (\rho, \ww)$ in the unit ball via the equivalence
\begin{equation}\label{map}
\btheta  = \rho\, \btheta_* +  \sqrt{1-\rho^2}\, \tilde \btheta_* \ww, \qquad \where~ \ww \in \R^{d-1}, ~\| \ww \| \le 1.
\end{equation}
In particular, if $\cI = [-1,1]$ and $r_0 = 1$, then $\Theta$ is simply the unit ball in $\R^d$. %Recall that in \eqref{} we derived (replacing the notations $\zz_i$ by $\tilde \xx_i$).
%Let $\tilde \btheta_* \in \R^{d \times (d-1)}$ be the orthogonal complement of $\btheta_*$ (i.e., $(\btheta_*, \tilde \btheta_*)$ is an orthogonal matrix). As before, we decompose $\btheta$ into its projection to the space of $\btheta_*$ and its complement. Recall that $\rho = \langle \btheta, \btheta_* \rangle$. The following decomposition
%\begin{equation}\label{map}
%\btheta  = \rho\, \btheta_* +  \sqrt{1-\rho^2}\, \tilde \btheta_* \ww, \qquad \where~ \ww \in \R^{d-1}, ~\| \ww \| \le 1.
%\end{equation}
%gives a one-to-one map $\btheta \leftrightarrow (\rho, \ww)$ in the unit ball. Thus, equivalently 
%\begin{equation*}
%\Theta = \big\{ (\rho, \ww): \rho \in \cI, \| \ww \| \le r_0 \big\}.
%\end{equation*}
%In particular, if $\cI = [-1,1]$ and $r_0 = 1$, then $\Theta$ is simply the unit ball in $\R^d$. Recall that in \eqref{} we derived (replacing the notations $\zz_i$ by $\tilde \xx_i$) \TODO{notations changed from Kangjie's version; to be unified and simplified}
%\begin{equation*}
%y_i \langle \xx_i, \btheta \rangle \stackrel{d}{=} \rho y_i G_i + \sqrt{1-\rho^2}\, \langle \tilde \xx_i, \ww \rangle, \qquad \where~G_i := \langle \xx_i, \btheta_*\rangle, ~\tilde \xx_i := \tilde \btheta_* \xx_i.
%\end{equation*}
By the definition of $\vv$, we have 
\begin{equation*}
\langle \btheta, \vv \rangle = n^{-1} \sum_{i=1}^n y_i \langle \xx_i, \btheta \rangle \stackrel{d}{=}  \frac{\rho}{n} \sum_{i=1}^n y_i G_i + \frac{\sqrt{1-\rho^2}}{n} \sum_{i=1}^n \langle \tilde \xx_i, \ww \rangle.
\end{equation*}
Thus, we can rewrite $\bar M_n$ as
\begin{align}
\bar M_n &= \max_{\btheta \in \Theta} \min_{\balpha \ge \bzero} \Big\{ \langle \btheta, \vv \rangle  + \sum_{i=1}^n \alpha_i \big( y_i \langle \xx_i, \btheta \rangle -  \kappa \big)  \Big\} \\
& \stackrel{d}{=} \max_{\rho \in \cI} \max_{\| \ww \| \le r_0} \min_{\balpha \ge \bzero} \left\{ \frac{ \rho}{n} \sum_{i=1}^n y_i G_i +  \frac{\sqrt{1-\rho^2}}{n} \, \sum_{i=1}^n \langle \tilde \xx_i, \ww \rangle +  \rho \sum_{i=1}^n \alpha_i y_i G_i +  \sqrt{1 - \rho^2}\, \sum_{i=1}^n \alpha_i \langle \tilde \xx_i, \ww \rangle -  \kappa \sum_{i=1}^n \alpha_i \right\} \notag \\
&=   \max_{\rho \in \cI} \max_{\| \ww \| \le r_0} \min_{\balpha \ge \bzero} \left\{  \rho \langle \balpha + n^{-1} \bone_n, \yy \odot \bG \rangle +  \sqrt{1 - \rho^2}\, \langle \balpha + n^{-1} \bone_n, \tilde \bX \ww  \rangle - \kappa\langle \balpha , \bone_n \rangle \right\} \notag \\
&=: \max_{\rho \in \cI} \bar M_{\rho,n}. \label{def:barM}
\end{align}
Let us denote $\tilde \balpha = \balpha + n^{-1} \bone_n$, which satisfies $\tilde \balpha \ge n^{-1} \bone_n$. We will apply Gordon's theorem. To this end, we define
\begin{align}
M_{\rho,n} &:= \max_{\| \ww \| \le r_0} \min_{\tilde \balpha \ge n^{-1} \bone_n} \Big\{  \rho \langle \tilde \balpha, \yy \odot \bG \rangle +  \sqrt{1 - \rho^2}\, \| \ww \| \langle \tilde \balpha, \bgg  \rangle +  \sqrt{1 - \rho^2}\, \| \tilde \balpha \| \langle \ww, \hh  \rangle  - \kappa\langle \tilde \balpha , \bone_n \rangle + \kappa \Big\}\notag \\
&=  \max_{\| \ww \| \le r_0} \min_{\tilde \balpha \ge n^{-1} \bone_n}  \Big\{ \langle \tilde \balpha, \underbrace{\rho \yy \odot \bG + \sqrt{1-\rho^2}\, \| \ww \| \bgg - \kappa \bone_n}_{\text{denoted by}~\zz} \rangle + \sqrt{1-\rho^2}\, \| \tilde \balpha \| \langle \ww, \hh \rangle +  \kappa \Big\} \label{expr:Mrho} \\
M_n &:= \max_{\rho \in \cI} M_{\rho,n}. \notag
\end{align}
where $\bgg \sim \cN(\bzero, \bI_n)$ and $\hh \sim \cN(\bzero, \bI_{d-1})$ are independent Gaussian vectors. 

By using a variant of Gordon's theorem \cite[Corollary~G.1]{miolane2018distribution}, we obtain the following result. A minor technicality is that the constraints $\tilde \balpha \ge n^{-1} \bone_n$ do not produce a compact set. Its proof is deferred to the appendix.
\begin{lem}\label{lem:Gordon2}
For all $t \in \R$, we have 
\begin{align*}
\P\left( \bar M_n \le t\right ) \le 2 \P\left(  M_n \le t\right ), \qquad \P\left( \bar M_n \ge t \right) \le 2 \P\left( M_n \ge t \right).
\end{align*}
\end{lem}

\subsection{Convergence to the asymptotic limit: Proof of Theorem~\ref{thm:signal}}

First, let us slightly simplify the expression of $M_{\rho,n}$ in \eqref{expr:Mrho}. 

\begin{lem}\label{lem:simpleMrho}
We can express $M_{\rho,n}$ as
\begin{equation}\label{eq:maxmin}
 M_{\rho,n} = \max_{r \in [0,r_0]} \min_{\tilde \balpha \ge  n^{-1} \bone_n} \Big\{  \langle \tilde \balpha, \zz \rangle + \sqrt{1-\rho^2} \, r\| \tilde \balpha \|  \| \hh \| \Big\} + \kappa.
\end{equation}
\end{lem}
\begin{proof}[{\bf Proof of Lemma~\ref{lem:simpleMrho}}]
%The maximization over $\ww$ can be done equivalently by first maximizing $\ww$ with fixed norm $r= \| \ww \|$ and then maximizing the norm, namely,
We only need to show that for every fixed $r \in [0,r_0]$, the following holds:
\begin{equation}\label{eq:3opt}
%M_{\rho,n} = \max_{r \in [0,r_0]} \max_{\| \ww \| = r} \min_{\tilde \balpha \ge  n^{-1} \bone_n} \Big\{  \langle \tilde \balpha, \zz \rangle + \sqrt{1-\rho^2} \, \| \tilde \balpha \|  \langle \ww, \hh \rangle \Big\} + b^{-1} \kappa.
I_0 := \max_{\| \ww \| = r} \min_{\tilde \balpha \ge  n^{-1} \bone_n} \Big\{  \langle \tilde \balpha, \zz \rangle + \sqrt{1-\rho^2} \, \| \tilde \balpha \|  \langle \ww, \hh \rangle \Big\} = \min_{\tilde \balpha \ge  n^{-1} \bone_n} \Big\{  \langle \tilde \balpha, \zz \rangle + \sqrt{1-\rho^2} \, r\| \tilde \balpha \|  \| \hh \| \Big\} =:I_1.
\end{equation}
Note that $\zz$ does not change on $\{ \ww : \| \ww \| = r\}$. Thus, by the Cauchy-Schwarz inequality, for any $\ww$ with $\| \ww \| = r$ and $\tilde \balpha$ with $\tilde \balpha \ge  n^{-1} \bone_n$, we have
\begin{equation}\label{wmax}
\langle \tilde \balpha, \zz \rangle + \sqrt{1-\rho^2} \, \| \tilde \balpha \|  \langle \ww, \hh \rangle \le \langle \tilde \balpha, \zz \rangle + \sqrt{1-\rho^2} \, r \| \tilde \balpha \|  \| \hh \|.
\end{equation}
%In the above inequality, first taking the minimum over $\tilde \balpha$ and then taking the maximum over $\ww$ and finally over $r$, we obtain
By first taking the minimum over $\tilde \balpha$ and then taking the maximum over $\ww$, we obtain $I_0 \le I_1$. By choosing $\ww = r \hh / \| \hh \|$ in the maximization in $I_0$, we obtain $I_0 \ge I_1$. Therefore, $I_0 = I_1$ and this lemma is proved.
%\begin{align*}
%M_{\rho, n} &\le  \max_{r \in [0,r_0]} \max_{\| \ww \| = r_0} \min_{\tilde \balpha \ge  n^{-1} \bone_n} \Big\{  \langle \tilde \balpha, \zz \rangle + \sqrt{1-\rho^2} \, r \| \tilde \balpha \|  \| \hh \| \Big\} + b^{-1} \kappa \\
%&= \max_{r \in [0,r_0]} \min_{\tilde \balpha \ge  n^{-1} \bone_n} \Big\{  \langle \tilde \balpha, \zz \rangle + \sqrt{1-\rho^2} \, r \| \tilde \balpha \|  \| \hh \| \Big\} + b^{-1} \kappa
%\end{align*}
%Moreover, in the intermediate maximization problem of \eqref{eq:3opt}, if we choose $\ww = r \hh / \| \hh \|$, then this gives a lower bound.
%\begin{equation*}
%M_{\rho, n} \ge \max_{r \in [0,r_0]} \min_{\tilde \balpha \ge  n^{-1} \bone_n} \Big\{  \langle \tilde \balpha, \zz \rangle + \sqrt{1-\rho^2} \, r \| \tilde \balpha \|  \| \hh \| \Big\} + b^{-1} \kappa.
%\end{equation*} 
%The lower bound and upper bound match each other. This proves the equivalent expression of $M_{\rho,n}$.
\end{proof}

For given $r \in [0,r_0]$, the minimization problem in \eqref{eq:maxmin} is convex. We will solve this minimization problem via the KKT conditions. To that end, we define 
\begin{equation*}
M_n(\rho,r) = \min_{\tilde \balpha \ge  n^{-1} \bone_n} \Big\{  \langle \tilde \balpha, \zz \rangle + \sqrt{1-\rho^2} \, r\| \tilde \balpha \|  \| \hh \| \Big\} + \kappa.
\end{equation*}
We allow $M_n(\rho,r) = -\infty$ in this definition. We also assume without loss of generality that $z_i \neq 0$ for all $i$ (this holds almost surely since $z_i$ is a continuous variable). For convenience, we denote
\begin{equation}\label{def:tau}
\tau_n := \tau_n(\rho,r) = \frac{ \sqrt{1-\rho^2}\, r \| \hh\|}{\sqrt{n}}, \qquad \tau := \tau(\rho,t) = \frac{ \sqrt{1-\rho^2}\, r }{\sqrt{\delta}}.
\end{equation}
We will use $\mathbb{Q}_n$ to denote the empirical distribution of $\{(y_i, G_i, g_i)\}_{i=1}^n$, namely, the probability measure on $\R^3$ given by
\begin{equation*}
\mathbb{Q}_n =  \frac{1}{n} \sum_{i=1}^n \delta_{(y_i, G_i, g_i)} \, .
\end{equation*}
Using this notation, we can write $n^{-1} \sum_{i=1}^n \psi(z_i)$ as $\E_{\mathbb{Q}_n} [\psi(Z)]$ for any function $\psi$. Sometimes, we will also use $\mathbb{Q}_\infty$ to denote the population measure of $(Y,G,g)$ (so that $\E_{\mathbb{Q}_\infty} = \E$).

\begin{lem}\label{lem:solucond}
Let $\mathbb{Q}$ be any probability measure (which can be random, including the empirical measure $\mathbb{Q}_n$), and let $\tau_{\mathbb{Q}} \ge 0$ be a scalar associated with $\mathbb{Q}$. Then, the following equation
\begin{equation}\label{def:news}
\tau_{\mathbb{Q}}^2 = \E_{\mathbb{Q}}\big[ \max\{ -Z, s \}^2 \big], \qquad \text{s.t.}~~~~s \ge 0
\end{equation}
has a unique solution $s_{\mathbb{Q}} = s_{\mathbb{Q}}(\rho,r)>0$ if the inequality
\begin{equation}\label{ineq:newcond1}
\tau_{\mathbb{Q}}^2 > \E_{\mathbb{Q}}\big[ \max\{ -Z, 0 \}^2 \big]
\end{equation}
holds; if the reverse inequality 
\begin{equation}\label{ineq:newcond2}
\tau_{\mathbb{Q}}^2 < \E_{\mathbb{Q}}\big[ \max\{ -Z, 0 \}^2 \big],
\end{equation}
holds, then there is no solution to the equation~\eqref{def:news}. Moreover, if $\mathbb{Q} = \mathbb{Q}_\infty$, then \eqref{def:news} has a unique solution $s(\rho,r) \ge 0$ under the relaxed condition
\begin{equation}\label{ineq:newcond3}
\tau^2 \ge \E\big[ \max\{ -Z, 0 \}^2 \big],
\end{equation}
under which $s(\rho,r)$ is continuous in $(\rho,r)$.
\end{lem}
\begin{proof}[{\bf Proof of Lemma~\ref{lem:solucond}}]
Denote by $\essinf \, Z$ the essential infimum of $Z$, i.e., $\essinf\,  Z = \sup\{a: \P(Z<a) = 0 \}$. Since $\max\{-Z, s\}$ is strictly increasing in $s$, the function $F(\rho,r,s) := \E\big[ \max\{-Z, s\} \big]$ is strictly increasing for $s \ge \essinf \, Z$ and is constant for $s < \essinf \, Z$. Note that $\lim_{s \to \infty} F(\rho,r,s) = \infty$. Thus, under the condition~\eqref{ineq:newcond1}, there exists a solution to  equation \eqref{def:news}, and under the condition \eqref{ineq:newcond2} there is no solution.

Under the condition~\eqref{ineq:newcond1}, we must have $s_{\mathbb{Q}} > 0$. Since $F(\rho,r,s_{\mathbb{Q}}) > F(\rho,r,0)$, we have $s_{\mathbb{Q}}>0 \ge \essinf \, Z$, so $s_{\mathbb{Q}}$ is the unique solution. 

To prove the `moreover' part, we derive
\begin{align*}
& \frac{\partial}{\partial s} F(\rho, r, s) = 2 \E \big[ \max\{s,-Z\} \bone\{s > -Z \} \big]  = 2s \P(s > -Z) \ge 0, \\
& \frac{\partial^2}{\partial s^2} F(\rho, r, s) = 2\P(s > -Z)  + 2sp_{-Z}(s),
\end{align*}
where $p_{-Z}$ is the probability density function of $-Z$. 
\begin{equation*}
\min_{s \ge 0} \frac{\partial^2}{\partial s^2} F(\rho, r, s) \ge 2 \min_{s \ge 0}\P(Z > -s) \ge 2 \P(Z > 0 ).
\end{equation*}
By Lemma~\ref{lem:quickfix}, we have $\P(Z > 0) \ge 1/2$, which gives $\frac{\partial^2}{\partial s^2} F(\rho, r, s) \ge 1$
%Note that $\rho Y G + \sqrt{1-\rho^2}\,r g$ has a symmetric distribution, so we must have $\P(Z > 0) \ge 1/2$
, which gives $\frac{\partial^2}{\partial s^2} F(\rho, r, s) \ge 1$. This implies that $F(\rho,r,s)$ is strictly increasing in $s$, so the solution to \eqref{def:news} is unique. In fact, we have the stronger inequality
\begin{equation}\label{ineq:Fconvex}
F(\rho,r,s) - F(\rho,r,s') \ge \frac{1}{2} (s-s')^2, \qquad \forall~s' > s.
\end{equation}
By the implicit function theorem, the solution $s(\rho,r)$ is also continuous in $(\rho,r)$.
\end{proof}

\begin{lem}\label{lem:alpha}
\hspace{2em}
\begin{enumerate}
\item[(a)]{ If the condition
\begin{equation}\label{ineq:cond1}
\tau_n^2 > \frac{1}{n} \sum_{i=1}^n \max\{-z_i, 0 \}^2
\end{equation}
holds, then the solution $\tilde \balpha$ and optimal objective value are given by
\begin{align}
M_n(\rho,r) &= \frac{1}{n} \sum_{i=1}^n (z_i + s_n)_+ +  \kappa , \label{eq:Mrhor} \\
\alpha_i & = \frac{1}{n} \max\left\{1, \frac{-z_i}{s_n}  \right\} \label{eq:alpha}
\end{align}
where $s_n$ is the solution to equation~\eqref{def:news} with $\mathbb{Q} = \mathbb{Q}_n$.
}
\item[(b)]{ If the reverse inequality 
\begin{equation}\label{ineq:cond2}
\tau_n^2 < \frac{1}{n} \sum_{i=1}^n \max\{-z_i, 0 \}^2
\end{equation}
holds, then $M(\rho,r) = -\infty$. 
}
\item[(c)]{ If, instead, the equality holds, then $M_n(\rho,r) = n^{-1} \langle \bone_n, \zz_+ \rangle + \kappa$, and the optimization problem does not have a finite minimizer. 
}
\end{enumerate}
\end{lem}

For convenience, we will henceforth denote the unique solution of~\eqref{def:news} by $s_n$ for the empirical measure $\mathbb{Q}_n$. If the solution of \eqref{def:news} does not exist, we will denote $s_n = -\infty$. We also denote the unique solution of~\eqref{def:news} by $s$ for the population measure $\mathbb{Q}_\infty$.

\begin{proof}[{\bf Proof of Lemma~\ref{lem:alpha}}]
\textbf{Step 1:} First we prove that the inequality \eqref{ineq:cond1} guarantees the existence of a (finite) minimizer, and reverse of the inequality produces $M(\rho,r) = -\infty$. 

By the Cauchy-Schwarz inequality, 
\begin{align}
\langle \tilde \balpha, \zz \rangle + \sqrt{1-\rho^2} \, r\| \tilde \balpha \|  \| \hh \| &= \langle \tilde \balpha, \zz_+ \rangle - \langle \tilde \balpha, \zz_- \rangle +  \sqrt{n}\, \tau_n \| \tilde \balpha \|    \notag \\
&\ge \langle \tilde \balpha, \zz_+ \rangle - \| \tilde \balpha \| \| \zz_-\| + \sqrt{n}\, \tau_n \| \tilde \balpha \|    \notag \\
& \ge  \frac{1}{n} \langle \bone_n, \zz_+ \rangle +  \| \tilde \balpha \| \cdot \big ( \sqrt{n}\, \tau_n   - \| \zz_-\|  \big). \label{ineq:objlbnd}
\end{align}
If the condition \eqref{ineq:cond1} holds, then the objective value goes to $+\infty$ as $\| \tilde \balpha \| \to \infty$, so a finite minimizer is guaranteed to exist. However, if $\sqrt{n}\, \tau_n   < \| \zz_- \|$, then we can find $\tilde \balpha$ with $\| \tilde \balpha \| \to \infty$ such that the objective value tends to $-\infty$. The reason is the following.

We choose $\tilde \alpha_i = n^{-1}$ if $z_i \ge 0$ and $\tilde \alpha_i =  (-z_i)\gamma$ if $z_i <0$ where $\gamma > 0$ is to be determined. Then, the objective value is
\begin{align*}
\frac{1}{n} \langle \bone_n, \zz_+ \rangle - \gamma \| \zz_- \|^2 + \sqrt{n}\, \tau_n   \big( \gamma^2 \| \zz_-\|^2 + n^{-2}k \big)^{1/2}
\end{align*}
where $k$ is the number of $i$ such that $z_i \ge 0$. As $\gamma \to \infty$, we have $ \gamma^{-1} \big( \gamma^2 \| \zz_-\|^2 + n^{-2}k \big)^{1/2} \to \| \zz_- \|$. Thus, 
\begin{equation*}
- \gamma \| \zz_- \|^2 + \sqrt{n}\, \tau_n \big( \gamma^2 \| \zz_-\|^2 + n^{-2}k \big)^{1/2} \to -\infty, \quad \text{as}~ \gamma \to \infty.
\end{equation*}
The objective value goes to $-\infty$ as $\| \tilde \balpha \| \to \infty$, so $M_n(\rho,r) = -\infty$. 

\textbf{Step 2:} If we have equality $\sqrt{n}\, \tau_n = \| \zz_-\| $, then the same choice of $\tilde \balpha$ leads to
\begin{equation*}
- \gamma \| \zz_- \|^2 + \sqrt{n}\, \tau_n \big( \gamma^2 \| \zz_-\|^2 + n^{-2}k \big)^{1/2} = \frac{n^{-2}k \| \zz_- \|}{\gamma \|\zz_- \| + \big( \gamma^2 \| \zz_-\|^2 + n^{-2}k \big)^{1/2}}
\end{equation*}
where we used the simple identity $a - b = (a^2-b^2)/(a+b)$. As $\gamma \to \infty$, the above expression tends to $0$, so the objective value tends to $n^{-1} \langle \bone_n, \zz_+ \rangle$. Also, the lower bound \eqref{ineq:objlbnd} shows that the objective value is at least  $n^{-1} \langle \bone_n, \zz_+ \rangle$. The equality condition for the Cauchy-Schwarz inequality is $\tilde \alpha_i = \tau (z_i)_-$ for some $\tau$, which is only possible if $z_i = 0$ whenever it is nonnegative. But $z_i \neq 0$ almost surely, so we must have a strict inequality in \eqref{ineq:objlbnd}. Therefore, there is no finite minimizer in this case.

\textbf{Step 3:} Now we suppose that the condition \eqref{ineq:cond1} holds. The KKT conditions (complementary slackness) give
\begin{equation*}
\begin{cases} \displaystyle\frac{\tilde \alpha_i}{\| \tilde \balpha \|} \sqrt{n}\, \tau_n  + z_i \ge 0, & \text{if } \displaystyle \tilde \alpha_i = \frac{1}{n} , \\ \displaystyle\frac{\tilde \alpha_i}{\| \tilde \balpha \|} \sqrt{n}\, \tau_n  + z_i  = 0 , & \text{if } \displaystyle\tilde \alpha_i > \frac{1}{n}.
\end{cases} 
\end{equation*}
This gives
\begin{equation*}
\tilde \alpha_i = \max\left\{ -(\sqrt{n}\, \tau_n )^{-1} z_i  \| \tilde \balpha \|, \frac{1}{n} \right\}.
\end{equation*}
The value of $\| \tilde \balpha \| $ can be obtained from solving
\begin{equation*}
\| \tilde \balpha \|^2 = \sum_{i=1}^n \max\left\{ -(\sqrt{n}\, \tau_n )^{-1}  z_i  \| \tilde \balpha \|, \frac{1}{n} \right\}^2.
\end{equation*}
For convenience, we $\tilde s_n = \frac{\sqrt{n}\, \tau_n}{n \| \tilde \balpha \|}$. Then, the solution $\tilde \balpha$ is given by
\begin{equation*}
\tilde \alpha_i = \frac{1}{n} \Big[ \max\left\{1, \frac{-z_i}{s_n}  \right\} \Big]
\end{equation*}
and $\tilde s_n > 0$ is determined by
\begin{equation*}
\tau_n^2 = \frac{1}{n} \sum_{i=1}^n \max\{ \tilde s_n, -z_i \}^2.
\end{equation*}
This $\tilde s_n$ is exactly the solution to equation \eqref{def:news}, namely $\tilde s_n = s_n$.

The above analysis proves the expression \eqref{eq:alpha}. It also leads to
\begin{align}
M_n(\rho,r) &= \frac{1}{n} \sum_{i=1}^n z_i \max\left\{1, \frac{-z_i}{s_n} \right\} + \frac{\tau_n^2}{s_n} +  \kappa.
\end{align}
After simplification, we obtain
\begin{align*}
M_n(\rho,r) = \frac{1}{n} \sum_{i=1}^n (z_i + s_n)_+ + \kappa,
\end{align*}
which is exactly the expression for $M_n(\rho,r)$ in \eqref{eq:Mrhor}.
\end{proof} 

Next, we establish a convergence result, which is key to the asymptotic limit theorem. 

\begin{lem}\label{lem:alpha2}
For $(\rho, r) \in  \Omega_{\ge}$, recall that $s_* = s_*(\rho,r) \ge 0$ is the solution to the equation \eqref{def:s}. For $(\rho, r) \in  \Omega_<$, let $s_*(\rho,r) = 0$. Recall the definition $Z = Z_{\rho,r}$ in \eqref{def:Z} and
\begin{equation*}
M(\rho,r) = \E \big[ (Z + s_* )_+ \big] + \kappa.
\end{equation*}
\begin{enumerate}
\item[(a)]{
 Let $K \subset \Omega_>$ be any compact set. Then,
\begin{equation*}
\sup_{(\rho, r) \in K} \big| s_n - s_*  \big| = o_{n,\P}(1), \quad \sup_{(\rho, r) \in K} \big| M_{n}(\rho,r) - M(\rho,r) \big| =  o_{n,\P}(1).
\end{equation*} 
}
\item[(b)]{
 Let $K \subset \Omega_<$ be any compact set. Then,
\begin{equation*}
\P \big( s_n = -\infty, ~ \forall\, (\rho,r) \in K\big) = 1-o_{n,\P}(1), \quad \P \big( M_{n}(\rho,r) = -\infty, ~ \forall\, (\rho,r) \in K\big) = 1-o_{n,\P}(1).
\end{equation*} 
}
\item[(c)]{
%More generally, let $K \subset \cI \times [0,r_0]$ be any compact set, and $\veps > 0$ be any constant that does not depend on $n,d,\rho,r$. Then,
Let $\veps > 0$ be any constant that does not depend on $n,d,\rho,r$. Then, There exists an open set $U \supset \Omega_=$ such that 
\begin{align*}
&\P \big( s_n \in [s_* - \veps, s_*+\veps] \cup \{ -\infty\} , ~\forall \, (\rho,r) \in U \big) = 1 -  o_n(1), \\
&\P \big( M_{n}(\rho,r) \in [- \infty, \E[Z_+] + \kappa + \veps ] , ~\forall \, (\rho,r) \in U \big) = 1 -  o_n(1).
\end{align*}
(Note that $\E[Z_+] + \kappa = M(\rho,r)$ if $(\rho,r) \in \Omega_=$.)
}
\end{enumerate}
\end{lem}
\begin{proof}[{\bf Proof of Lemma~\ref{lem:alpha2}}]
For convenience, we will use `with high probability' or simply w.h.p.~to refer to an event that happens with probability $1-o_n(1)$. To emphasize the limit, we will write $s_\infty$ and $s_*$ interchangeably. We notice that $s_n$ has a trivial upper bound: in fact, from \eqref{def:news}, we always have
\begin{equation}\label{ineq:strivialbnd}
\tau_n^2 \ge \frac{1}{n} \sum_{i=1}^n s_n^2 = s_n^2 \qquad \tau^2 \ge \E[s_\infty^2] = s_\infty^2.
\end{equation}
Since $\| \hh \|^2 / n \stackrel{p}{\to} \delta^{-1}$, we get $s_n \le r_0\sqrt{2(1-\rho^2) \delta^{-1}} \le \sqrt{2}r_0\, \delta^{-1/2}$ w.h.p.. 
Now we define
\begin{align*}
&\tilde F_n(\rho, r, s) = \frac{1}{n} \sum_{i=1}^n (z_i + s)_+ , \quad  F_n(\rho, r, s) = \frac{1}{n}  \sum_{i=1}^n \max\{s, -z_i\}^2 \\
&\tilde F(\rho, r, s) = \E[ (Z + s)_+], \quad F(\rho, r, s) = \E \big[ \max\{s, -Z\}^2 \big].
\end{align*}
Denote the set $\cA = \{ (\rho, r, s): |\rho| \le 1, r \in [0,1+\veps_b], s \in [0, \sqrt{2 \delta^{-1}}] \}$. Then, by the uniform law of large numbers \cite[Lemma 2.4]{newey1994large}, we have 
\begin{equation}\label{eq:lln}
\max_{(\rho,r,s) \in \cA} \big| \tilde F_n (\rho, r, s) - \tilde F(\rho, r, s) \big| = o_{n,\P}(1), \quad \max_{(\rho,r,s) \in \cA} \big| F_n (\rho, r, s) - F(\rho, r, s) \big| = o_{n,\P}(1).
\end{equation}

(a) Suppose that $(\rho, r) \in K \subset \Omega_>$. To prove $\sup_{(\rho, r) \in K} \big| s_n - s_\infty  \big| = o_{n,\P}(1)$, we note that %there exists a unique solution $s_n$ to the equation \eqref{def:sn}. This is because
\begin{equation*}
\tau_n^2 \stackrel{p}{\to} \tau^2, \qquad \frac{1}{n} \sum_{i=1}^n \max\{-z_i, 0 \}^2 \stackrel{p}{\to}  \E \big[ Z^2; Z<0 \big].
\end{equation*}
%Thus $(\rho, r) \in \Omega_>$ implies $\frac{(1-\rho^2)r^2 \| \hh\|^2}{n} > \frac{1}{n} \sum_{i=1}^n \max\{-z_i, 0 \}^2 $ with high probability. %Since $\sum_{i=1}^n \max\{-z_i, s \}^2$ is non-decreasing in $s$ and diverges to $\infty$ as $s \to \infty$. 
%By Lemma~\ref{lem:alpha}, there must be a unique solution $s_n$ to the equation \eqref{def:sn}.
%To prove $\sup_{(\rho, r) \in K} \big| s_n - s_\infty  \big| = o_{n,\P}(1)$, 
Fix an arbitrary $\veps'>0$ that is independent of $n$. We set $s = s_n$ in \eqref{eq:lln} and use the definition of $s_n$ to obtain
\begin{equation*}
\tau_n^2 \le \E \big[ \max\{s_n, -Z\}^2 \big] + \veps, \quad \forall \; (\rho,r) \in K
\end{equation*}
w.h.p.. Since $\tau_n^2 \ge \tau^2 - \veps$ w.h.p., we get
\begin{equation}\label{ineq:objcomp1}
\tau^2 - \veps' \le \E \big[ \max\{s_n, -Z\}^2 \big] + \veps', \quad \forall \; (\rho,r) \in K.
\end{equation}
Similarly, 
\begin{equation}\label{ineq:objcomp2}
\tau^2  + \veps' \ge \E \big[ \max\{s_n, -Z\}^2 \big] - \veps', \quad \forall \; (\rho,r) \in K.
\end{equation}
The above two inequalities, namely \eqref{ineq:objcomp1}--\eqref{ineq:objcomp2}, imply
\begin{equation*}
\max_{(\rho,r) \in K} \big| F(\rho, r, s_n) - F(\rho, r, s_\infty) \big| \le 2 \veps'.
\end{equation*}
By the inequality in~\eqref{ineq:Fconvex}, we get $\max_{(\rho, r) \in K} | s_n - s_\infty | \le 2 \sqrt{\veps'}$. This proves $\max_{(\rho, r) \in K} | s_n - s_\infty| = o_{n,\P}(1)$. 

Next, we bound the difference $|\tilde F_n(\rho, r, s_n) - \tilde F(\rho,r,s_\infty)|$. By the triangle inequality,
\begin{align*}
\sup_{(\rho,r) \in K} \big| \tilde F_n(\rho, r, s_n) - \tilde F(\rho,r,s_\infty) \big|& \le\sup_{(\rho,r) \in K} \big| \tilde F_n(\rho, r, s_n) - \tilde F(\rho,r,s_n) \big| + \sup_{(\rho,r) \in K}\big| \tilde F(\rho, r, s_n) - \tilde F(\rho,r,s_\infty) \big|
\end{align*}
Since $s_n \in [0,\sqrt{2\delta^{-1}}]$ w.h.p., the first term on the right-hand side is $o_{n,\P}(1)$ due to \eqref{eq:lln}. We will show that the second term is also $o_{n,\P}(1)$.
Using the inequality $|a_+ - b_+| \le |a-b|$, we have, for any fixed constant $\veps'>0$,
\begin{align*}
\big| \tilde F(\rho, r, s_n) - \tilde F(\rho,r,s_\infty)  \big| \le \E \big[ | s_n - s_\infty| \big]  \le \veps' \P\big( | s_n - s_\infty| \le \veps') + (\tau_n + \tau) \P\big( | s_n - s_\infty| > \veps'\big)
\end{align*}
where we used \eqref{ineq:strivialbnd}. Since $\max_{(\rho, r) \in K} | s_n - s_\infty| = o_{n,\P}(1)$, $\tau_n \stackrel{p}{\to} \tau$ and $\veps'$ is arbitrary, we obtain 
\begin{equation*}
\sup_{(\rho,r) \in K}\big| \tilde F(\rho, r, s_n) - \tilde F(\rho,r,s_\infty) \big| = o_{d,\P}(1).
\end{equation*}
This then proves $\sup_{(\rho,r)\in K} | M_{n}(\rho,r) - M(\rho,r) | = o_{n,\P}(1)$.

(b) Suppose $K \subset \Omega_{<}$. Our goal is to show that w.h.p.~ for all $(\rho, r) \in K$, there is no solution $s_n$ to \eqref{def:news} with the empirical measure. Since $\tau(\rho,r)$ and $\E[Z^2; Z<0]$ are both continuous in $(\rho,r)$, by compactness we must have
\begin{equation*}
\E[Z^2; Z<0] > \tau^2 + 2\veps', \quad \text{for all}~ (\rho,r) \in K
\end{equation*}
for certain $\veps'>0$. We set $s=0$ in \eqref{eq:lln} and get $\E[Z^2; Z<0] \le n^{-1} \sum_{i=1}^n (-z_i)_+^2 + \veps'$ for all $(\rho, r) \in K$ w.h.p. Using $\tau_n \stackrel{p}{\to} \tau$, we obtain, w.h.p.,
\begin{equation*}
\tau_n^2 < \frac{1}{n} \sum_{i=1}^n \max\{ 0, -z_i \}^2 \le \frac{1}{n} \sum_{i=1}^n \max\{ s, -z_i \}^2, \qquad \text{for all}~ s \ge 0.
\end{equation*}
Therefore, w.h.p., there is no nonnegative solution to $s_n$ for all $(\rho, r) \in K$.

(c) Since $s_\infty(\rho,r)$ is continuous in $(\rho,r)$ and $\Omega_=$ is compact, we can find an open set $U \supset \Omega_=$ relative to $[-1,1] \times [0,1+\veps_b]$ such that $s_\infty \le \veps/2$. %If $(\rho,r) \in \Omega_> \cap U^c =: \tilde K$, then we apply (a) and find $\sup_{(\rho,r) \in \tilde K}|s_n - s_\infty|= o_{n,\P}(1)$ and $\sup_{(\rho,r) \in \tilde K} |M_n(\rho,r) - M(\rho,r)| = o_{n,\P}(1)$ w.h.p. If $(\rho,r) \in \Omega_< \cap U^c$, then by (b) we know $s_n = -\infty$ w.h.p. It remains to consider $(\rho,r) \in U$. 

Suppose $(\rho,r) \in U$. By (strict) monotonicity of $F(\rho,r,s)$ in $s$ from the proof of Lemma~\ref{lem:solucond}, we have
\begin{equation*}
\tau^2 < \E\big[ \max\{ \veps, -Z \}^2 \big].
\end{equation*}
Then, w.h.p., we also have
\begin{equation*}
\tau_n^2  < \frac{1}{n}\sum_{i=1}^n \big[ \max\{ \veps, -z_i \}^2 \big].
\end{equation*}
Since $F_n(\rho, r,s)$ is increasing in $s$, we must have w.h.p., $s_n = - \infty$ (no solution) or $s_n \le \veps$, where the latter implies $|s_n - s_\infty| \le \veps$. 

Also,  we note that if $s_n = -\infty$ then by definition $M_{n}(\rho,r) = -\infty$, and if $|s_n - s_\infty| \le \veps$, then 
\begin{align*}
M_n(\rho,r) = \frac{1}{n} \sum_{i=1}^n (z_i+s_n)_+ + \kappa \le \frac{1}{n} \sum_{i=1}^n (z_i)_+ + s_n + \kappa  \le \frac{1}{n} \sum_{i=1}^n (z_i)_+ + \frac{3}{2}\veps + \kappa.
\end{align*}
Since $\sup_{(\rho,r) \in U} \big| \frac{1}{n} \sum_{i=1}^n (z_i)_+ - \E [ Z_+] \big| =o_{n,\P}(1)$ by \eqref{eq:lln}, we obtain $M_{n}(\rho,r) \le \E[Z_+] + 2\veps + \kappa$. By replacing $2\veps$ with $\veps$, we prove claim (c) in this lemma.
\end{proof}

\begin{proof}[{\bf Proof of Theorem~\ref{thm:Mlimit}}]
Note that Part (a) is already proved in Lemma~\ref{lem:solucond}. Below we prove Part (b).

\textbf{Step 1.} Note that $Z$ is continuous in $(\rho,r)$. We also proved that $s(\rho,r)$ is continuous in $(\rho,r)$. This implies (by dominated convergence theorem) that $M(\rho,r)$ is uniformly continuous in $(\rho,r) \in \Omega_{\ge}$. By Lemma~\ref{lem:alpha2} (c) and continuity of $M(\rho,r)$, we can choose an open set $U \subset  \cI \times [0,r_0]$ such that the we have the conclusions of Lemma~\ref{lem:alpha2} (c) and $\E[Z_+] + \kappa \le \sup_{\Omega_{\ge}(\cI, r_0) \cap U} M(\rho,r) + \veps$.

Setting $K_1 = \Omega_{\ge}(\cI, r_0) \cap U^c$, $K_2 = \Omega_{\le}(\cI, r_0) \cap U^c$, and apply Lemma~\ref{lem:alpha2} (a)(b)(c) and find the following inequalities hold with high probability.
\begin{equation*}
\begin{array}{lll}
& M_n(\rho,r) \le M(\rho,r) + \veps, \qquad &\forall\, (\rho,r) \in K_1, \\
& M_n(\rho,r) = -\infty, \qquad &\forall\, (\rho,r) \in K_2, \\
& M_n(\rho,r) \le \sup_{\Omega_{\ge}(\cI, r_0) \cap U} M(\rho,r) + \veps, & \forall (\rho,r) \in U.
\end{array}
\end{equation*}
Combining all three inequalities, we get that with high probability, 
\begin{equation*}
\sup_{\rho \in \cI, r \in [0,r_0] } M_n(\rho,r) \le \max_{(\rho,r) \in \Omega_{\ge}(\cI, r_0)} M(\rho,r) + \veps,
\end{equation*}
which is exactly $\bar M_n \le M^* + \veps$ with high probability.

\textbf{Step 2.} Now suppose that $(\rho^*, r^*) \in \Omega_{\ge}(\cI, r_0)$ is a maximizer. We claim that we can find a sequence $(\rho^k, r^k) \in \Omega_>(\cI, r_0)$ such that 
\begin{equation}\label{claim:converge}
\lim_{k \to \infty} (\rho^k, r^k) = (\rho^*, r^*).
\end{equation}
Once this claim is proved, then by continuity of $M(\rho,r)$, we can find a sufficiently large $k$ such that 
\begin{equation*}
M(\rho^k,r^k) \ge M(\rho^*,r^*) - \veps/2 = M^*  - \veps/2.
\end{equation*}
We can find a compact set $K \subset \Omega_>(\cI, r_0)$ such that $(\rho^k, r^k) \in K$. Applying Lemma~\ref{lem:alpha2} (a), we find that with high probability,
\begin{equation*}
\sup_{\rho \in \cI, r \in [0,r_0] } M_n(\rho,r) \ge \sup_{(\rho,r) \in K} M_n(\rho,r) \ge \sup_{(\rho,r) \in K} M(\rho,r) - \veps/2 \ge M^* - \veps.
\end{equation*}
which is $\bar M_n \ge M^* - \veps$. Combining with the conclusions in Step~1, we will obtain the result in Part (b).

\textbf{Step 3.} To prove the claim \eqref{claim:converge}, let us consider the continuous map $\mathcal{M}: (\rho,r) \to (\rho, \omega)$ where $\omega = \sqrt{1-\rho^2}\, r$. The image of $\Omega_{\ge}(\cI, r_0)$ under $\mathcal{M}$ is denoted by $\tilde \Omega_{\ge}(\cI, r_0)$, which is given by
\begin{equation*}
\tilde \Omega_{\ge}(\cI, r_0) = \Big\{ (\rho, \omega): \rho \in \cI, \frac{\omega^2}{r_0^2} + \rho^2 \le 1, \frac{\omega}{\sqrt{\delta}} \ge \| Z_- \|_{L^2} \Big\}
\end{equation*}
where $\| Z_- \|_{L^2} = \big[\E[Z_-^2] \big]^{1/2}$ with $Z = \rho YG + \omega W - \kappa$ similarly as before. Since any $(\rho,r) \in \Omega_{\ge}(\cI, r_0)$ satisfies $|\rho| \neq 1$, the map $\mathcal{M}$ has a continuous inverse.

We observe that $\| Z_- \|_{L^2}$ is convex in $(\rho,\omega)$: in fact, taking any $(\rho_1,\omega_1)$ and $(\rho_2,\omega_2)$ and denoting $Z_k = \rho_k YG + \omega_k W - \kappa$ ($k= 1,2$), for any $\lambda \in [0,1]$ we have
\begin{align*}
\lambda \big\|\max\{-Z_1,0\} \big\|_{L^2} + (1-\lambda) \big\|\max\{-Z_1,0\} \big\|_{L^2} &\ge \big\| \max\{- \lambda Z_1,0\} + \max\{- (1-\lambda)Z_2,0\} \big\|_{L^2}  \\ &\ge \big\|\max\{- \lambda Z_1- (1-\lambda)Z_2,0\} \big\|_{L^2}  
\end{align*}
where we used $a_- + b_- \ge (a+b)_-$ in the second inequality. Hence, $\tilde \Omega_{\ge}$ is a convex set. For any $(\rho,\omega) \in \tilde \Omega_{\ge}$ with $\omega/\sqrt{\delta} = \| Z_- \|_{L^2}$, we can thus find a sequence $(\rho^k, \omega^k) \in \tilde \Omega_{\ge}$ with $\omega/\sqrt{\delta} > \| Z_- \|_{L^2}$ such that $\lim_{k \to \infty}(\rho^k, \omega^k) = (\rho,\omega) $. Using the map $\mathcal{M}^{-1}$ we then obtain the claim \eqref{claim:converge}.
\end{proof}

\begin{proof}[{\bf Proof of Theorem~\ref{thm:signal}: first claim}]
By the definition of $\delta_{\mathrm{lin}}$, for given $\delta < \delta_{\mathrm{lin}}$, we have $\Omega_{>}([-1,1],1) \neq \emptyset$, and $M_*(1+\veps') > M_*(1)$ for certain $\veps'>0$. That $\Omega_{>}([-1,1],1)$ is nonempty implies that $\Omega_{>}([-1,1],1+\veps')$ is also nonempty. So we can apply Theorem~\ref{thm:Mlimit} twice where we set $\cI = [-1,1]$, and $r_0=1$ and $r_0 = 1+\veps'$ respectively. The two different choices of $r_0$ are associated with two optimization problems given in \eqref{thm:Mlimit}; correspondingly, let $M_n$ and $M_n'$ be the maximum. Also, let $\hat \btheta'$ be a maximizer associated with $M_n'$. (Recall that $\hat \btheta$ defined in the algorithm is associated with $M_n$.)

By Theorem~\ref{thm:Mlimit}(b), we have $M_n = M_*(1)+ o_{n,\P}(1)$ and $M_n' = M_*(1+\veps')+ o_{n,\P}(1)$. Since $M_*(1+\veps') > M_*(1)$, we have $M_n' > M_n$ with high probability, which also implies that $\| \hat \btheta' \|>1$ with high probability. For any $\lambda \in [0,1]$, consider the interpolant 
\begin{equation*}
\hat \btheta_\lambda = \lambda \hat \btheta + (1-\lambda) \hat \btheta'.
\end{equation*}
By linearity, we must have $y_i \langle \hat \btheta_\lambda, \vv \rangle \ge \kappa$. If $\| \hat \btheta \| < 1$, then we can choose appropriate $\lambda \in (0,1)$ such that $\| \hat \btheta_\lambda \| = 1$ and $\langle \hat \btheta_\lambda, \vv \rangle > M_n$, which contradicts the definition of $M_n$. Hence, we must have $\| \hat \btheta \| = 1$ with high probability. This proves that $\hat \btheta$ is a $\kappa$-margin solution with high probability.
\end{proof}

\subsection{Analysis of maximization of $M(\rho,r)$: Proof of Theorem~\ref{thm:signal} continued}

From the defining equation \eqref{def:s} of $s_*$, it is easy to derive an upper bound 
\begin{equation}\label{ineq:naive}
s_*^2 \le \E \big[ \max\{ s_*, -Z \}^2 \big] = (1-\rho^2) r^2 \delta^{-1}, \qquad \text{thus} \qquad 0 \le s_* \le \sqrt{1-\rho^2}\, \delta^{-1/2}.
\end{equation}
Denote 
\begin{equation}\label{def:q}
q(\kappa) = \sqrt{\frac{2}{\pi}}\, \frac{1}{|\kappa|} \exp \left( -\frac{\kappa^2}{2} - \alpha |\kappa| \right), \qquad q_+(\kappa) = q(\kappa) \exp\left(\alpha|\kappa| + \frac{\alpha^2}{2}\right).
\end{equation}
Also denote the set 
\begin{equation}\label{def:omegarho}
\Omega_{\ge, \rho}(\cI, r_0) = \big\{ r \in [0, r_0]: (\rho,r) \in \Omega_{\ge}(\cI, r_0) \big\}.
\end{equation}

It is useful to derive a rough range of the maximizer $\rho^*$ when optimizing $M(\rho,r)$. The following lemma gives a lower bound.
\begin{lem}\label{lem:rholbnd}
Assume that $\delta$ satisfies $ |\kappa|^8  \le \delta \le (1-\veps) \sqrt{\pi/2} |\kappa| \exp(\kappa^2/2 + \alpha |\kappa|)$. %Denote $\Omega_{\ge} = \Omega_{\ge}([-1,1],1)$ for simplicity. 
Then, there exists a sufficiently negative $\underline \kappa(\alpha, m, \veps)$ such that for all $\kappa < \underline \kappa$ the following holds. The particular choice $(\rho, r) = (1 - |\kappa|^{-2}, 1 )$ satisfies $(1-\rho^2) r^2 \delta^{-1} > \E [Z^2; Z < 0]$,  and thus
\begin{equation}\label{ineq:Mbnd}
\max_{(\rho, r) \in \Omega_{\ge}(\cI, r_0)} M(\rho,r) \ge \left(1 - \frac{1}{|\kappa|^2} \right) m, \qquad \where~m:= \E[YG]
\end{equation}
and consequently, if $(\rho^*, r^*)$ maximizes $M(\rho,r)$ over $\Omega_{\ge}(\cI, r_0)$, then we must have $\rho^* \ge 1- (1+\veps/2)|\kappa|^{-2}$.
\end{lem}
\begin{proof}[{\bf Proof of Lemma~\ref{lem:rholbnd}}]
First we prove that under the choice $(\rho, r) = (1 - |\kappa|^{-2}, 1 )$, we have $(1-\rho^2) r^2 \delta^{-1} > \E [Z^2; Z < 0]$ for sufficiently negative $\kappa$. Note
\begin{equation*}
(1-\rho^2) r^2 \delta^{-1}  = (2 - |\kappa|^{-2}) |\kappa|^{-2} \delta^{-1} = 2(1+\breve{o}_{\kappa}(1)) |\kappa|^{-2} \delta^{-1}.
\end{equation*}
The condition $\delta \le (1-\veps) \sqrt{\pi/2} |\kappa| \exp(\kappa^2/2 + \alpha |\kappa|)$ is equivalent to $q(\kappa) \le (1-\veps)\delta^{-1}$. So by Lemma~\ref{lem:xi}, 
\begin{equation*}
\E \big[ Z^2; Z < 0 \big] = \frac{2(1+\breve{o}_{\kappa}(1))}{|\kappa|^2} q(\kappa) \le \frac{2(1-\veps) \cdot (1+\breve{o}_{\kappa}(1))}{|\kappa|^2} \delta^{-1}.
\end{equation*}
This proves the first claim. The bound \eqref{ineq:Mbnd} then follows from 
\begin{equation*}
M(\rho,r) = \E[(Z+s_*)_+] - \E\big[ (Z + s_*)_- \big]  \ge \E[Z] = \rho m.
\end{equation*} 
Moreover, by the condition $\delta \ge |\kappa|^8$ and the naive bound \eqref{ineq:naive} on $s_*$, we have $s_* \le \delta^{-1/2} \le |\kappa|^{-4}$. The upper bound on \eqref{ineq:Ezbnd} is further bounded by $\breve{o}_{\kappa}(1) \cdot |\kappa|^{-2}$. Thus,
\begin{equation*}
M(\rho^*,r^*) \le \rho^* m + 2|\kappa|^{-4} + \breve{o}_{\kappa}(|\kappa|^{-2}).
\end{equation*}
Combining this with the lower bound \eqref{ineq:Mbnd}, we must have $\rho^* \ge 1 - (1+\veps/2) |\kappa|^{-2}$ for sufficiently negative $\kappa$.
\end{proof}

Let $L^2 := L^2(\mathbb{Q}_\infty)$ be the space of square integrable functions on the population measure $\mathbb{Q}_\infty$. For $A \in L^2(\mathbb{Q}_\infty)$, $\rho \in [-1,1]$, and $r \ge 0$, we define
\begin{equation}\label{def:Psi}
\Psi(A, \rho, r) = \E[ A Z ] + \frac{r \sqrt{1-\rho^2}}{\sqrt{\delta}} \| A \|_{L^2} + \kappa \, .
\end{equation}
Below we make use of this definition to prove monotonicity of $M(\rho,r)$ in $r$.
\begin{prop}\label{prop:Mmonotone}
Suppose that either of the two holds: (1) $\delta$ satisfies $ |\kappa|^8  \le \delta \le (1-\veps) \sqrt{\pi/2} |\kappa| \exp(\kappa^2/2 + \alpha |\kappa|)$ and $\rho \ge 1 - (1+\veps/2) |\kappa|^{-2}$, (2) $\delta \le |\kappa|^8 $. Also suppose that $\Omega_{\ge, \rho}(\cI, r_0)$ is nonempty. Then, there exists a sufficiently negative $\underline \kappa = \underline \kappa(\alpha, m, \veps)$ such that the following holds for all $\kappa < \underline \kappa$:
\begin{enumerate}
\item[(a)]{
$\Psi$ is strictly convex in $A$ and
\begin{equation*}
M(\rho,r) = \min_{A \in L^2, A \ge 1} \big\{ \Psi(A, \rho, r)  \big\}
\end{equation*}
and the unique minimizer $A^*:= A^*(\rho,r)$ is given by $A^* = s_*^{-1} \max\{s, -Z\}$ where recall $s_*$ is the solution in \eqref{lem:solucond} (with $\mathbb{Q} = \mathbb{Q}_{\infty}$).
}
\item[(b)]{There exists $\veps' := \veps'(\kappa, \delta) >0$ such that for all $r \in (0, 1+\veps')$, we have
\begin{align*}
\frac{\partial \Psi}{\partial r} \Big|_{A = A^*} = (1-\rho^2)r \Big[ \frac{1}{\delta} - \P(s_* + Z < 0) \Big] > 0.
\end{align*}
}
\item[(c)]{The function $M(\rho,r)$ is strictly increasing in $r \in [0,1+\veps']$.
}
\item[(d)]{Consequently, together with Lemma~\ref{lem:rholbnd}, we have that %if $(\rho^*, r^*) \in \Omega_{\ge}$ is a maximizer, then $r^* = 1 + \veps_b$.
$M_*(1+\veps') > M_*(1)$.
}
\end{enumerate}
\end{prop}
\begin{proof}[{\bf Proof of Proposition~\ref{prop:Mmonotone}}]
First, we note that $r \sqrt{1-\rho^2} \neq 0$, because otherwise $\Omega_{\ge}$ is empty. 

(a) Since $\| A \|_{L^2}$ is strictly convex in $A$ and $\E[AZ]$ is linear in $A$, $\Psi$ must also be strictly convex in $A$. Similar to the proof of Lemma~\ref{lem:alpha}, we apply the KKT conditions and find 
\begin{equation*}
A^* = \max\big\{ - \tau^{-1} Z \| A^* \|_{L^2}, 1 \big\}.
\end{equation*}
The norm $\| A^* \|_{L^2}$ satisfies 
\begin{equation*}
\| A^* \|_{L^2} = \Big\| \max\big\{ - \tau^{-1} Z \| A^* \|_{L^2}, 1 \big\} \Big\|_{L^2}.
\end{equation*}
Comparing this with the definition of $s_*$, we find $s_* = \tau / \| A^* \|_{L^2}$, and thus $A^* = \max\{s_*, -Z\} / s_*$. Using this to simplify $\Psi(A^*,\rho,r)$, we get $\Psi(A^*,\rho,r) = \E[(Z+s_*)_+] + \kappa = M(\rho,r)$ as desired.

(b) We calculate 
\begin{equation*}
\frac{\partial \Psi}{\partial r} = \sqrt{1-\rho^2}\, \E[A W] + \frac{\sqrt{1-\rho^2}}{\sqrt{\delta}} \| A \|_{L^2}.
\end{equation*}
At the minimizer $A = A^*$, we use Stein's identity and $\|A^*\|_{L^2} = \tau / s_*$ to obtain
\begin{equation*}
\frac{\partial \Psi}{\partial r} \Big|_{A = A^*} = (1-\rho^2)\cdot (-r) \cdot \P(s_*<-Z) + \frac{\sqrt{1-\rho^2}}{s_* \sqrt{\delta}} \tau = (1-\rho^2)r \Big[ \frac{1}{\delta} - \P(s_* + Z < 0) \Big].
\end{equation*}
We claim that, under the conditions of this proposition, for all $\rho \in [-1,1]$,
\begin{equation*}
\frac{1}{\delta}  > \P\big(\rho Y G + \sqrt{1-\rho^2}\, W < \kappa \big) =: q_0.
\end{equation*}
Once this claim is proved, we can find a sufficiently small constant $\veps'>0$ (due to continuity) such that 
\begin{equation*}
\frac{1}{\delta}  > \P\big(\rho Y G + \sqrt{1-\rho^2}\, rW < \kappa \big), \qquad \text{for all}~\rho\in[-1,1], r \in [0, 1+\veps'],
\end{equation*}
which will imply $\frac{\partial \Psi}{\partial r} \Big|_{A = A^*} > 0$ as desired. To prove this claim, first consider the case $ |\kappa|^8  \le \delta \le (1-\veps) \sqrt{\pi/2} |\kappa| \exp(\kappa^2/2 + \alpha |\kappa|)$ and $\rho \ge 1 - (1+\veps/2) |\kappa|^{-2}$. By Lemma~\ref{lem:xi}, 
\begin{equation*}
q_0  \le \big(1+\breve{o}_{\kappa}(1)) \cdot  q(\kappa) \le \big(1+\breve{o}_{\kappa}(1)) \cdot (1-\veps) \cdot \delta^{-1} < \delta^{-1}.
\end{equation*}
For the case $\underline c_0 \le \delta < |\kappa|^8$,
\begin{equation*}
q_0 \le (1+\breve{o}_{\kappa}(1)) \cdot q(\kappa) \exp\big(\alpha |\kappa| + \alpha^2/2 \big),
\end{equation*}
Since $\delta^{-1} \ge |\kappa|^{-8}$, which vanishes only polynomially as $ \kappa \to -\infty$, we must have $q_0 < 1/\delta$ as well. 

(c) First we fix $\rho \in [-1,1]$ and $r \in [0,1+\veps']$. By the strict monotonicity from (b), we can find $\eta_r > 0$ such that $\Psi(A,\rho,r) > \Psi(A,\rho,r')$ at $A = A^*(\rho,r)$ for all $r' \in (r-\eta_r, r)$. Thus, we derive
\begin{align*}
M(\rho,r) &= \Psi(A^*(\rho,r),\rho,r) > \Psi(A^*(\rho,r),\rho,r')  \\
&\ge \Psi(A^*(\rho,r'),\rho,r') = M(\rho,r').
\end{align*}
Similarly, we have $M(\rho,r) < M(\rho,r')$ for all $r' \in (r, r+\eta_r')$ where $\eta_r'>0$. Thus, $M(\rho,r)$ is strictly increasing in $r$ on every compact set in $(0,1+\veps')$. By continuity, we conclude that $M(\rho,r)$ is strictly increasing in $r$ on $[0,1+\veps']$.

(d) By Lemma~\ref{lem:rholbnd}, the maximizer $(\rho^*, r^*) \in \Omega_{\ge}(\cI, r_0)$ satisfies $\rho^* \ge 1 - (1+\veps/2) |\kappa|^{-2}$ if $ |\kappa|^8  \le \delta \le (1-\veps) \sqrt{\pi/2} |\kappa| \exp(\kappa^2/2 + \alpha |\kappa|)$. By (c), we must have $r^*=1$; moreover, $M(\rho^*, 1+\veps') >  M(\rho^*, r^*)$. Thus, we must have $M_*(\infty) \ge M_*(1+\veps') > M_*(1)$.
\end{proof}

Once this proposition is proved, it leads to Theorem~\ref{thm:signal}.
\begin{proof}[{\bf Proof of Theorem~\ref{thm:signal}: second claim}]
First, we note that under the condition \eqref{cond:delta}, the set $\Omega_>([-1,1],1)$ is nonempty due to Lemma~\ref{lem:rholbnd}. Using $\veps'$ defined in Proposition~\ref{prop:Mmonotone}, we have $M_*(\infty) \ge  M_*(1)$. Thus, $\delta_{\mathrm{lin}} \ge (1-\veps)\sqrt{\pi/2}\, |\kappa| \exp(\kappa^2/2 + \alpha |\kappa|)$ and so the second claim is proved.
\end{proof}

\subsection{Estimation error: Proof of Theorem~\ref{thm:err}}

We note that part (a) of Theorem~\ref{thm:err} is a direct consequence of Theorem~\ref{thm:Mlimit} (by setting $r_0 = 1$ and taking $\cI = [-1, \rho^* - \veps_0]$ and $\cI = [\rho^* + \veps_0]$ for any small constant $\veps_0>0$). Below we prove part (b) of Theorem~\ref{thm:err}.

Recall the definitions of $q(\kappa)$ and $q_+(\kappa)$ in \eqref{def:q}.
Let us start with approximating $s_*(\rho,r)$ and $M(\rho,r)$ with simpler functions, as stated in the next two lemmas. 
\begin{lem}\label{lem:sapprox}
Let $\veps >0 $ be a constant, and recall the definition of $s_* = s_*(\rho,r)$ in \eqref{def:s}. Then, there exists a sufficiently negative $\underline \kappa = \underline \kappa(\alpha, m, \veps)$ such that for all $\kappa < \underline \kappa$ the following holds for all $(\rho,r) \in \Omega_{\ge}(\cI, 1)$.
\begin{equation}
0 \le s_* - \sqrt{(1-\rho^2)r^2 \delta^{-1} - \E[ Z^2; Z < 0 ] } \le 3 \delta^{-1/2} \sqrt{(1-\rho^2) q_+(\kappa)},
\end{equation}
and  
\begin{equation}\label{ineq:Ezbnd}
0 \le \E[(Z+s_*)_- ] \le \frac{2}{|\kappa|} q_+(\kappa).
\end{equation}
If, in addition, $\rho \ge 1 - (1+\veps/2)|\kappa|^{-2}$, then $q_+(\kappa)$ can be replaced by $q(\kappa)$ from the above bound.
\end{lem}

\begin{proof}[{\bf Proof of Lemma~\ref{lem:sapprox}}]
In the proof, we will write $s = s_*$ for simplicity. We will use the results from Lemma~\ref{lem:xi}: 
\begin{equation*}
\sup_{\rho \in [-1,1], r \in [0,1]} \P( Z < 0 ) \le (1+o_{\kappa}(1)) \cdot q_+(\kappa), \qquad \sup_{\rho \in [1-(1+\veps/2)|\kappa|^{-2},1], r \in [0,1]} \P( Z < 0 ) \le (1+o_{\kappa}(1)) \cdot q(\kappa).
\end{equation*}
We can rewrite the equation~\eqref{def:s} as
\begin{align*}
s^2 \P(Z \ge -s ) &=  (1-\rho^2) r^2 \delta^{-1} - \E[Z^2; Z < -s ]\\
&= (1-\rho^2) r^2 \delta^{-1} - \E[Z^2; Z < 0] + \E[Z^2 ; -s \le Z < 0].
\end{align*}
Note that the third term on the last expression is bounded by
\begin{equation*}
0 \le \E[Z^2 ; -s \le Z < 0] \le s^2 \P(Z<0)
\end{equation*}
where we used the naive bound \label{ineq:naives} on $s$ in the last inequality. Thus, by nonnegativity of $\E[Z^2; Z<0]$ and the simple inequality $\sqrt{a+b} \le \sqrt{a} + \sqrt{b}$ for $a,b \ge 0$,  we have
\begin{align*}
& s \sqrt{ \P(Z \ge -s ) } \ge  \sqrt{(1-\rho^2)\, r^2 \delta^{-1} - \E[Z^2 ; Z<0]}, \\
&s \sqrt{ \P(Z \ge -s ) } \le  \sqrt{(1-\rho^2)\, r^2 \delta^{-1} - \E[Z^2 ; Z<0]}\,  + \sqrt{1-\rho^2}\, \delta^{-1/2} \sqrt{\P(Z<0)}.
\end{align*}
We also have
\begin{equation*}
\P(Z \ge -s) = 1 - \P(Z \le -s) \ge 1 - \P(Z \le 0) .
\end{equation*}
and thus
\begin{equation*}
1 \le \big[ \P(Z \ge -s ) \big]^{-1/2} \le \big[1 - \P(Z <0 ) \big]^{-1/2} = 1 +  (1+\breve{o}_{\kappa}(1)) \cdot \xi(\rho,r,b^{-1} \kappa) / 2 \le  1 + q_+(\kappa).
\end{equation*}
Thus for sufficiently negative $\kappa$ we derive
\begin{align*}
s & \ge \sqrt{(1-\rho^2)\, r^2 \delta^{-1} - \E[Z^2 ; Z<0]}, \\
s & \le \big(1 + q_+(\kappa)\big) \cdot \Big( \sqrt{(1-\rho^2)\, r^2 \delta^{-1} - \E[Z^2 ; Z<0]}\,  + \sqrt{1-\rho^2}\, \delta^{-1/2} \sqrt{2 q_+(\kappa)} \Big) \\
&\le \sqrt{(1-\rho^2)\, r^2 \delta^{-1} - \E[Z^2 ; Z<0]} + q_+(\kappa) \sqrt{1-\rho^2} \, \delta^{-1/2} + (1+q_+(\kappa)) \cdot  \sqrt{1-\rho^2} \delta^{-1/2} \sqrt{2 q_+(\kappa)} \\
&\stackrel{(i)}{\le} \sqrt{(1-\rho^2)\, r^2 \delta^{-1} - \E[Z^2 ; Z<0]} + \sqrt{(1-\rho^2) q_+(\kappa) }\, \delta^{-1/2} + 2 \sqrt{(1-\rho^2) q_+(\kappa) }\, \delta^{-1/2} \\
&= \sqrt{(1-\rho^2)\, r^2 \delta^{-1} - \E[Z^2 ; Z<0]}  + 3\sqrt{(1-\rho^2) q_+(\kappa) }\, \delta^{-1/2}
\end{align*}
where (i) is because $q_+(\kappa) \le \sqrt{q_+(\kappa)}$ and $1+q_+(\kappa) \le \sqrt{2}$ for sufficiently negative $\kappa$. If $\rho \ge 1 - (1+\veps/2)|\kappa|^{-2}$ additionally, then by Lemma~\ref{lem:xi}, we can replace $q_+(\kappa)$ with $q(\kappa)$.
\end{proof}

Hereafter we will focus on $r=1$, since by Proposition~\ref{prop:Mmonotone} implies that $M(\rho,r)$  is always maximized at $r=1$ for sufficiently negative $\kappa$. By Lemma~\ref{lem:xi} (d)(e), the expectation term $\E[Z^2;Z<0]$ has an asymptotic expression (as $\kappa \to -\infty$) if $\rho \ge 1 - (1+\veps/2) |\kappa|^{-2}$, namely
\begin{equation}\label{ineq:Eqdiff}
\sup_{\rho \ge 1 - (1+\veps/2) |\kappa|^{-2}} \Big| \E \big[ (\rho YG_1 + \sqrt{1-\rho^2}\, G_2 - b^{-1}\kappa)_-^2 \big] -  \frac{2}{|\kappa|^2} q(\kappa) \Big| \le  \frac{2\alpha_\kappa}{|\kappa|^2} q(\kappa) 
\end{equation}
for certain $\alpha_\kappa \ge 0$ and $\alpha_\kappa = \breve{o}_{\kappa}(1)$. Let $(\rho^*, 1)$ be any maximizer of $M(\rho,r)$.

Using the above inequality and Lemma~\ref{lem:sapprox}, we can approximate the function $M(\rho,r)$. Recall that $m = \E[YG]$. Define functions $M^0$, $M_+$, and $M_-$ as follows.
\begin{align}
M^0(\rho)& = \rho m + \sqrt{(1-\rho^2)  \delta^{-1} - \E[Z^2; Z<0 ] }, \qquad \forall\, \rho: (\rho,1) \in \Omega_{\ge}(\cI,1), \label{def:M0} \\
&\begin{cases} \displaystyle  M_{\pm}^1(\rho) =\rho m + \sqrt{(1-\rho^2)  \delta^{-1} - (1 \pm \alpha_\kappa) \frac{2}{|\kappa|^2} q(\kappa) }, & \text{if}~\delta \ge |\kappa|^8 \label{def:M1} \\
\displaystyle M^1(\rho) = \rho m + \sqrt{1-\rho^2}\,  \delta^{-1/2}, & \text{if}~\delta < |\kappa|^8.
\end{cases}
\end{align}
Let $J_+, J_- \subset [-1,1]$ be the subsets in which $M_+(\rho)$ and $M_-(\rho)$, respectively, are well defined (i.e., the expressions in the square root are nonnegative).

The next simple lemma states that $M_\pm^1(\rho)$ and $M^1(\rho)$ are strictly concave functions, and their  maximizers have explicit expressions. %and the maximizer satisfies $(\rho,1) \in \Omega_{\ge}$ due to \eqref{ineq:Eqdiff}.
\begin{lem}\label{lem:rho0}
There exists a sufficiently negative $\underline \kappa = \underline \kappa(\alpha, m, \veps)$ such that for all $\kappa < \underline \kappa$ the following holds.
\begin{enumerate}
\item[(a)]{ Suppose that $\delta$ satisfies $|\kappa|^8 \le \delta \le (1-\veps)q(\kappa)^{-1}$. Then $M^1(\rho)$ is a strictly concave function for $\rho$. Its unique maximizer is given by
\begin{equation*}
\rho = \Big(1 - \frac{2(1 \pm \alpha_\kappa) \delta}{|\kappa|^2} q(\kappa) \Big)^{1/2} \cdot \Big( \frac{m^2 \delta}{1 + m^2 \delta} \Big)^{1/2},
\end{equation*}
or equivalently,
\begin{equation}\label{def:rho0}
1 - \rho^2 = \frac{1}{1+m^2 \delta} + \frac{m^2 \delta}{1+ m^2 \delta} \cdot \frac{2(1 \pm \alpha_\kappa) \delta}{|\kappa|^2} q(\kappa),
\end{equation}
}
and it holds that $\rho \ge 1 - (1+\veps/2) |\kappa|^{-2}$.
\item[(b)]{Suppose that $\delta$ satisfies $\underline c \le \delta \le |\kappa|^8$. Then $M^1(\rho)$ is a strictly concave function and its unique maximizer is given by
\begin{equation*}
\rho =  \Big( \frac{m^2 \delta}{1 + m^2 \delta} \Big)^{1/2}
\end{equation*}
or equivalently,
\begin{equation}\label{def:rho0-2}
1 - \rho^2 = \frac{1}{1+m^2 \delta}.
\end{equation}
}
\end{enumerate}
%For both cases, we have $(1-\rho_1^2) \delta^{-1} \ge 2(1+\tilde \alpha_\kappa)q(\kappa)/|\kappa|^2$ for certain $\tilde \alpha_\kappa \ge 0$ with $\tilde \alpha_\kappa = o_\kappa(1)$.
\end{lem}
\begin{proof}[{\bf  Proof of Lemma~\ref{lem:rho0}}]
First consider the scenario in (a). Taking the derivative of $M_\pm^1(\rho)$, we get
\begin{equation}\label{def:M1deriv}
\frac{d}{d \rho} M_\pm^1(\rho) = m - \frac{\rho \delta^{-1}}{\big( (1-\rho^2) \delta^{-1} - (1 \pm \alpha_\kappa)\frac{2}{|\kappa|^2} q(\kappa) \big)^{1/2} }. 
\end{equation}
Taking a further derivative, we have
\begin{equation}\label{def:M1deriv2}
\frac{d^2}{d \rho^2} M_\pm^1(\rho) = -\frac{\delta^{-2} \big[ 1 - (1 \pm \alpha_\kappa) \frac{2 \delta}{|\kappa|^2} q(\kappa)\big]}{\big( (1-\rho^2) \delta^{-1} - (1 \pm \alpha_\kappa)\frac{2}{|\kappa|^2}q(\kappa) \big)^{3/2} }
\end{equation}
which is negative for sufficiently negative $\kappa$. This proves that $M_\pm^1(\rho)$ is strictly concave. Solving the equation $\d M_\pm^1(\rho)/\d \rho = 0$ for $\rho$, we obtain the desired expression \eqref{def:rho0}. The derivation for the scenario (b) is similar.
\end{proof}

\begin{lem}\label{lem:residual2new}
Then there exists a sufficiently negative $\underline \kappa = \underline \kappa(\alpha, m, \veps)$ such that for all $\kappa < \underline \kappa$ the following holds.
\begin{enumerate}
\item[(a)]{ Suppose that $\delta$ satisfies $|\kappa|^8 \le \delta \le (1-\veps)q(\kappa)^{-1}$. Let $\rho_1$ be the unique minimizer of $M_-^1(\rho)$. Then there exists some $\beta_\kappa \ge 0$ satisfying $\beta_\kappa = \breve{o}_{\kappa}(1)$, such that
\begin{equation}\label{ineq:funcval1}
M^1_-(\rho_1) - M^1_-(\rho^*) \le \frac{\beta_\kappa}{|\kappa|} \sqrt{q(\kappa)}.
\end{equation}
}
\item[(b)]{ Suppose that $\delta$ satisfies $\underline c \le \delta \le |\kappa|^8$. Let $\rho_1$ be the unique minimizer of $M^1(\rho)$. Then,
\begin{equation}\label{ineq:funcval2}
M^1(\rho_1) - M^1(\rho^*) \le \frac{3}{|\kappa|} \sqrt{q_+(\kappa)}.
\end{equation}
}
\end{enumerate}
\end{lem}

\begin{proof}[{\bf Proof of Lemma~\ref{lem:residual2new}}]
Let $\tilde \rho_1$ be the unique minimizer of $M_+^1(\rho)$. By Lemma~\ref{lem:rho0} and the inequality \eqref{ineq:Eqdiff}, we have $(\tilde \rho_1,1) \in \Omega_{\ge}$ with sufficiently negative $\kappa$. By the definition of $\rho^*$, the optimality of $\rho^*$ implies $M(\rho^*,1) \ge M(\tilde \rho_1,1)$. 

By Lemma~\ref{lem:sapprox}, we can approximate the function
\begin{equation*}
M(\rho,1) = \E[(Z + s_*)] + \E[(Z + s_*)_-]   + \kappa = \rho m + s_* + \E[(Z + s_*)_-] 
\end{equation*}
using $M^0(\rho)$; that is, we write
\begin{equation*}
M(\rho,1) = M^0(\rho) + M^{\res}(\rho), \qquad \where~ |M^{\res}(\rho)| \le 3 \delta^{-1/2} \sqrt{(1-\rho^2) q_+(\kappa)} + 2|\kappa|^{-1} q_+(\kappa)
\end{equation*}
and $q_+(\kappa)$ can be replaced by $q(\kappa)$ if $\rho \ge 1 - (1+\veps/2) |\kappa|^{-2}$.

(a) Since $M^1_+(\rho) \le M^0(\rho) \le M^1_-(\rho)$ for any $\rho$ such that these functions are well defined (namely, the expression inside the square root is nonnegative), we derive
\begin{align*}
M_-^1(\rho^*)&\ge M^0(\rho^*) \ge M(\rho^*,1)- |M^{\res}(\rho^*)|   \ge M^0(\tilde \rho_1) - |M^{\res}(\rho^*)|  - |M^{\res}(\tilde \rho_1)| \\
&\stackrel{(i)}{\ge} M_-^1(\tilde \rho_1) - \frac{\sqrt{4\alpha_\kappa q(\kappa)}}{|\kappa|} - |M^{\res}(\rho^*)|  - |M^{\res}(\tilde \rho_1)| \\
&\stackrel{(ii)}{\ge} M_-^1(\tilde \rho_1) - (2\sqrt{\alpha_\kappa}\, + o_{\kappa}(1) ) \frac{\sqrt{ q(\kappa)}}{|\kappa|}
\end{align*}
where in \textit{(i)} we used the simple inequality $|\sqrt{a} - \sqrt{b} | \le \sqrt{a-b}$, and in  \textit{(ii)} we used $\min\{ \rho^*, \tilde \rho_1\} \ge 1 - (1+\veps/2) |\kappa|^{-2}$. Also, we have 
\begin{equation*}
M_-^1(\tilde \rho_1) \ge M_-^1(\rho_1) - (4\sqrt{\alpha_\kappa}\, + o_{\kappa}(1) ) \frac{\sqrt{ q(\kappa)}}{|\kappa|}.
\end{equation*}
We conclude that 
\begin{equation*}
M_-^1(\rho^*) \ge M_-^1(\rho_1) - (6\sqrt{\alpha_\kappa}\, + o_{\kappa}(1) ) \frac{\sqrt{ q(\kappa)}}{|\kappa|}.
\end{equation*}

(b) By Lemma~\ref{lem:xi}, we get $\sup_{\rho \in [-1,1]}|M^1(\rho) - M^0(\rho)| \le \sqrt{\E[Z^2; Z<0]} \le  1.1 \sqrt{q_+(\kappa)} / |\kappa|$ for sufficiently negative $\kappa$, so 
\begin{align*}
M^1(\rho^*) &\ge M^0(\rho^*) - \frac{1.1\sqrt{q_+(\kappa)}}{|\kappa|} \ge M^0(\rho_1) - \frac{1.1\sqrt{q_+(\kappa)}}{|\kappa|} - |M^{\res}(\rho^*)|  - |M^{\res}(\tilde \rho_1)| \\
&\ge M^1(\rho_1) - \frac{2.2\sqrt{q_+(\kappa)}}{|\kappa|} - o_{\kappa}(1)\frac{\sqrt{q_+(\kappa)}}{|\kappa|},
\end{align*}
which completes the proof.
\end{proof}

The above lemma allows us to bound the difference $\rho^* - \rho_1$ through the concave function $M_-^1(\rho)$ (or $M^1(\rho)$).

\begin{lem}\label{lem:residual2}
\begin{enumerate}
\item[(a)]{Suppose that $\delta$ satisfies $|\kappa|^8 \le \delta \le (1-\veps)q(\kappa)^{-1}$. Then there exists some $\gamma_\kappa \ge 0$ satisfying $\gamma_\kappa = \breve{o}_{\kappa}(1)$ such that
\begin{equation*}
\big| \rho_1 - \rho^* \big| < \gamma_\kappa \cdot \Big( \frac{1}{\delta} + \frac{ \delta q(\kappa)}{|\kappa|^2}\Big).
\end{equation*}
}
\item[(b)]{Suppose that $\delta$ satisfies $\underline c \le \delta \le |\kappa|^8$. Then there exists some $\gamma_\kappa \ge 0$ satisfying $\gamma_\kappa = \breve{o}_{\kappa}(1)$ such that
\begin{equation*}
\big| \rho_1 - \rho^* \big| < \frac{\gamma_\kappa}{\delta}.
\end{equation*}
}
\end{enumerate}
\end{lem}
\begin{proof}[{\bf Proof of Lemma~\ref{lem:residual2}}]
In this proof, for simplicity we define
\begin{equation*}
a:= a(\kappa, \delta)= \frac{\sqrt{\delta q(\kappa)}}{|\kappa|}, \qquad b:= b(\kappa, \delta)= 1 - \frac{(1-\alpha_\kappa) 2 \delta q(\kappa)}{|\kappa|^2}.
\end{equation*}
Note that $b$ satisfies $b \le 1 - |\kappa|^{-2} = 1 - o_\kappa(1)$. Recall that the expression of $\frac{d^2}{d \rho^2} M_-^1(\rho)$ is given in \eqref{def:M1deriv2}. By optimality of $\rho_1$, for any $\rho \in J_-$ we have
\begin{align*}
M_-^1(\rho_1) - M_-^1(\rho^*)  \ge \inf_{\rho \in [\rho_1, \rho^*]} \left| \frac{d^2}{d \rho^2} M_-^1(\rho) \right| \cdot \frac{(\rho_1 - \rho^*)^2}{2}.
\end{align*}
Since
\begin{align*}
(1-\rho^2) - (1-\alpha_\kappa) \frac{2\delta}{|\kappa|^2} q(\kappa) &\le (1-\rho_1^2) - (1-\alpha_\kappa) \frac{2\delta}{|\kappa|^2} q(\kappa) + 2|\rho_1 - \rho| \\
&\le \frac{1}{1+m^2 \delta}\Big(1 - \frac{2(1-\alpha_\kappa)\delta}{|\kappa|^2} q(\kappa) \Big) + 2|\rho_1 - \rho| \\
&= \frac{b}{1+m^2 \delta}  + 2|\rho_1 - \rho|,
\end{align*}
we obtain
\begin{align*}
\inf_{\rho \in [\rho_1, \rho^*]} \left| \frac{d^2}{d \rho^2} M_-^1(\rho) \right|  \ge \frac{\delta^{-1/2} b }{\big[ b (1+m^2 \delta)^{-1}  + 2|\rho_1 - \rho^*| \big]^{3/2}}.
\end{align*}
Therefore, we derive
\begin{equation*}
2(M^1_-(\rho_1) - M^1_-(\rho^*)) \ge \delta^{-1/2} b  f(|\rho_1 - \rho^*|), \qquad \where~ f(t) := \frac{t^2}{\big[b (1+m^2 \delta)^{-1} + 2t \big]^{3/2}}.
\end{equation*}
Note that $f(t)$ is strictly increasing in $t$, since
\begin{equation*}
f'(t) = \frac{t(t+2b (1+m^2 \delta)^{-1})}{2\big[b (1+m^2 \delta)^{-1} + 2t \big]^{5/2}} > 0.
\end{equation*}
By Lemma~\ref{lem:residual2new}, we get
\begin{equation*}
\frac{2\beta_\kappa \delta^{1/2}}{b |\kappa|} \sqrt{q(\kappa)} \ge f(|\rho_1 - \rho^*|) \quad \Longrightarrow \quad |\rho_1 - \rho^*| \le f^{-1} \left(\frac{2\beta_\kappa a}{b} \right).
\end{equation*}
Let $\gamma_\kappa := \beta_\kappa^{1/3}$ so it satisfies $\gamma_\kappa = o_\kappa(1)$. We claim that
\begin{equation}\label{claim:f}
f \big(\gamma_\kappa\max\{ \delta^{-1} , a^2 \} \big) > \frac{2\beta_\kappa a}{b},
\end{equation}
which, once proved, implies $|\rho_1 - \rho^*| \le  \gamma_\kappa\max\{ \delta^{-1} , a^2 \}$ and thus the conclusion in \textit{(a)}. To prove the claim~\eqref{claim:f}, observe that
\begin{align*}
f \big(\gamma_\kappa\max\{ \delta^{-1} , a^2 \} \big) &\ge \frac{\gamma_\kappa^2 \max\{ \delta^{-1}, a^2\}^2 }{\max\{ (1+o_\kappa(1))(1+m^2 \delta)^{-1} , \gamma_\kappa a^2 \}^{3/2}} \\
&\ge \frac{(1-o_\kappa(1))\gamma_\kappa^{2}m^3\max\{ \delta^{-1}, a^2\}^2}{\max\{ \delta^{-1}, \gamma_\kappa a^2\}^{3/2}} \\
&\ge (1-o_\kappa(1))m^3 \Big( \gamma_\kappa^{2} \bone\{ \delta^{-1} \ge \gamma_\kappa a^2 \} \delta^{-1/2} + \gamma_\kappa^{1/2} \bone\{ \delta^{-1} < \gamma_\kappa a^2 \} a \Big) \\
&\ge (1-o_\kappa(1))m^3 \Big(\gamma_\kappa^{5/2} \bone\{ \delta^{-1} \ge \gamma_\kappa a^2 \} a + \gamma_\kappa^{1/2} \bone\{ \delta^{-1} < \gamma_\kappa a^2 \} a \Big) \\
&\ge  1-o_\kappa(1))m^3 \gamma_\kappa^{5/2}a
\end{align*}
which is smaller than $2\beta_\kappa a / b$ when $\kappa$ is sufficiently negative. This proves the claim and thus proves part \textit{(a)}.

For part \textit{(b)} Following the same strategy, we first derive
\begin{equation*}
\inf_{\rho \in [\rho_1, \rho^*]} \left| \frac{d^2}{d \rho^2} M^1(\rho) \right|  \ge \frac{\delta^{-1/2}  }{\big[  (1+m^2 \delta)^{-1}  + 2|\rho_1 - \rho^*| \big]^{3/2}}
\end{equation*}
and so we have
\begin{align*}
\frac{6 \delta^{1/2}}{|\kappa|} \sqrt{q(\kappa)} \ge 2\delta^{1/2} \big( M^1(\rho_1) - M^1(\rho^*)\big) \ge \frac{(\rho_1 - \rho^*)^2}{\big[(1+m^2 \delta)^{-1} + 2|\rho_1 - \rho^* | \big]^{3/2}}.
\end{align*}
Similarly, the function $\tilde f(t) = \frac{t^2}{[(1+m^2\delta)^{-1} + 2t]^{3/2}}$ is strictly increasing in $t$. Taking $\gamma_\kappa = 1/|\kappa|$, for sufficiently negative $\kappa$ we have 
\begin{equation*}
f(\gamma_\kappa / \delta) \ge (1-o_\kappa(1)) \gamma_\kappa^2 m^3 \delta^{-1/2} \ge \frac{6 \delta^{1/2}}{|\kappa|} \sqrt{q(\kappa)}.
\end{equation*} 
This implies that $|\rho_1-\rho^*| \le \gamma_\kappa / \delta$ as desired.
\end{proof}

Finally, we put the pieces together to prove Theorem~\ref{thm:err}.

\begin{proof}[{\bf Proof of Theorem~\ref{thm:err}}]
%First notice that the condition \eqref{cond:closure} is automatically satisfied for sufficiently negative $\kappa$, so that we can apply Corollary~\ref{cor:Mlimit}. 
As discussed in the beginning of this subsection, we only need to prove part (b) of this theorem. We set 
\begin{align*}
& \bar \rho_{\min} = \displaystyle \rho_1 - \gamma_\kappa \Big( \frac{1}{\delta} + \frac{\delta q(\kappa)}{|\kappa|^2} \Big), &&\bar \rho_{\max} = \rho_1 + \gamma_\kappa \Big( \frac{1}{\delta} + \frac{\delta q(\kappa)}{|\kappa|^2} \Big), \\
&\cI_1 = [-1, \bar \rho_{\min}],  & &\cI_2 = [\bar \rho_{\max}, 1]
\end{align*}
where $\gamma_\kappa$ is given by Lemma~\ref{lem:residual2}. We apply Theorem~\ref{thm:Mlimit} with $r_0=1$, and we set $\cI = \cI_1$, $\cI = \cI_2$, and $\cI = [-1,1]$ respectively. Theorem~\ref{thm:Mlimit}(c) implies that the maximizer $\hat \btheta$ must satisfy $\langle \hat \btheta, \btheta_* \rangle \in \cI \setminus (\cI_1 \cap \cI_2)$ w.h.p. Thus, w.h.p.,
\begin{equation*}
\langle  \hat \btheta , \btheta_* \rangle = \rho_1 + o_{\kappa}(1) \cdot \left( \frac{1}{\delta} + \frac{\delta q(\kappa)}{|\kappa|^2} \right).
\end{equation*}
We note that 
\begin{equation*}
\mathcal{E}(\kappa) = (1+o_\kappa(1)) \cdot \left( \frac{1}{2m^2\delta} + \frac{\delta q(\kappa)}{|\kappa|^2} \right).
\end{equation*}  
For case \textit{(a)}, if $\lim_{\kappa \to -\infty} \delta(\kappa) = \infty$, then we have
\begin{equation*}
\rho_1 = \Big( 1 - \frac{(1+o_\kappa(1))\delta}{|\kappa|^2} q(\kappa) \Big) \cdot \Big(1 - \frac{1+o_\kappa(1)}{2m^2\delta}  \Big) = 1 - (1+o_\kappa(1)) \cdot \mathcal{E}(\kappa).
\end{equation*}
Therefore, w.h.p.,
\begin{equation}\label{eq:Ek}
\left| \langle  \hat \btheta , \btheta_* \rangle - (1-\mathcal{E}(\kappa)) \right| = o_\kappa(1) \cdot \mathcal{E}(\kappa).
\end{equation}
For case \textit{(b)}, we have $\mathcal{E}(\kappa) = (1+o_\kappa(1)) / (m^2 \delta)$ and $\rho_1 = 1 - (1+o_\kappa(1))/(2m^2\rho)$. This gives the same high probability bound as \eqref{eq:Ek}.
\end{proof}

\subsection{Additional proofs}\label{sec:append-signal-alg}

\begin{proof}[{\bf Proof of Lemma~\ref{lem:Gordon2}}]
For convenience, we introduce $ \uu := \sqrt{1-\rho^2}\, \ww \in \R^{d-1}$. Consider the convex set 
\begin{equation*}
\cC = \big\{ (\rho, \uu): \rho \in \cI, r_0^{-2} \| \uu \|^2 + \rho^2 \le 1 \big\}.
\end{equation*}
Denote $\bar \uu = (\rho, \uu)$. Define two Gaussian processes as
\begin{align*}
Q_1 \big(\tilde \bX; \bar \uu, \tilde \balpha \big) &= \rho\langle \tilde \balpha, \yy \odot \bG \rangle + \langle \tilde \balpha, \tilde \bX \uu \rangle - b^{-1}\kappa \langle \tilde \balpha, \bone_n \rangle + b^{-1} \kappa, \\
Q_2 \big(\bgg, \hh; \bar \uu, \tilde \balpha \big) &=\rho\langle \tilde \balpha, \yy \odot \bG \rangle + \| \uu \| \langle \tilde \balpha, \bgg \rangle + \| \tilde \balpha \| \langle \uu, \hh \rangle - b^{-1}\kappa \langle \tilde \balpha, \bone_n \rangle + b^{-1} \kappa.
\end{align*}
For every integer $k \ge 1$, define
\begin{align*}
& \bar A_k = \min_{\substack{\tilde \balpha \ge n^{-1}\bone_n \\ \| \tilde \balpha \| \le k }} \max_{\bar \uu \in \cC} Q_1\big( \tilde \bX; \bar \uu, \tilde \balpha \big), \qquad A_k = \min_{\substack{\tilde \balpha \ge n^{-1}\bone_n \\ \| \tilde \balpha \| \le k }} \max_{\bar \uu \in \cC} Q_2 \big(\bgg, \hh; \bar \uu, \tilde \balpha \big).
\end{align*}
Note that both minimization and maximization above are defined over compact and convex constraint sets. This allows us to apply \cite[Corollary~G.1]{miolane2018distribution}, which yields
\begin{equation*}
\P \big( \bar A_k \le t \big) \le  2 \P \big( A_k \le t \big), \qquad \P \big( \bar A_k \ge t \big) \le  2 \P \big( A_k \ge t \big), \qquad \text{for all}~t \in \R.
\end{equation*}
We take $k \to \infty$ in the both two inequalities (by the monotone convergence theorem), and obtain
\begin{align*}
& \P \left(  \min_{ \tilde \balpha \ge n^{-1}\bone_n} \max_{\bar \uu \in \cC} Q_1\big( \tilde \bX; \bar \uu, \tilde \balpha \big) \le t \right) \le 2 \P \left(  \min_{ \tilde \balpha \ge n^{-1}\bone_n} \max_{\bar \uu \in \cC} Q_2\big(  \bgg, \hh ; \bar \uu, \tilde \balpha \big) \le t  \right), \\
 &\P \left(  \min_{ \tilde \balpha \ge n^{-1}\bone_n} \max_{\bar \uu \in \cC} Q_1\big( \tilde \bX; \bar \uu, \tilde \balpha \big) \ge t \right) \le 2 \P \left(  \min_{ \tilde \balpha \ge n^{-1}\bone_n} \max_{\bar \uu \in \cC} Q_2\big(  \bgg, \hh ; \bar \uu, \tilde \balpha \big) \ge t  \right).
\end{align*}
By Sion's minimax theorem \cite{sion1958general}, we can exchange the $\min$ with $\max$ for $Q_1$ and obtain
\begin{align*}
\min_{\tilde \balpha \ge n^{-1} \bone_n} \max_{\bar \uu \in \cC} Q_1 \big( \tilde \bX; \bar \uu, \tilde \balpha \big) =  \max_{\bar \uu \in \cC} \min_{\tilde \balpha \ge n^{-1} \bone_n} Q_1 \big( \tilde \bX; \bar \uu, \tilde \balpha \big) = \max_{\rho \in \cI} \max_{\| \ww \| \le r_0}  \min_{\tilde \balpha \ge n^{-1} \bone_n} Q_1 \big( \tilde \bX; \bar \uu, \tilde \balpha \big)  = \bar M_n.
\end{align*}
It suffices to prove that we can also exchange the $\min$ with $\max$ for $Q_2$, namely
\begin{equation}\label{eq:minmaxexchange}
\min_{ \tilde \balpha \ge n^{-1}\bone_n} \max_{\bar \uu \in \cC} Q_2\big(  \bgg, \hh ; \bar \uu, \tilde \balpha \big) =  \max_{\bar \uu \in \cC} \min_{ \tilde \balpha \ge n^{-1}\bone_n} Q_2\big(  \bgg, \hh ; \bar \uu, \tilde \balpha \big),
\end{equation}
since the right-hand side above is exactly $M_n$. It is clear that given the norm $\| \uu \| = r$, we have
\begin{align}
\max_{\uu: \| \uu \| = r} Q_2(\bgg, \hh; \bar \uu, \tilde \balpha) &= \rho\langle \tilde \balpha, \yy \odot \bG \rangle + r \langle \tilde \balpha, \bgg \rangle + r\| \tilde \balpha \| \| \hh \| - b^{-1}\kappa \langle \tilde \balpha, \bone_n \rangle + b^{-1} \kappa \label{eq:Q2rhor}\\
&=: Q_2\big(  \bgg, \hh ; \rho, r, \tilde \balpha \big). \notag
\end{align}
The right-hand side of \eqref{eq:Q2rhor} is linear in $(\rho, r)$ over the compact and convex set $\cC' := \{ (\rho,r): \rho \in \cI, (r/r_0)^2 + \rho^2 \le 1\}$. Thus, 
\begin{align*}
\min_{ \tilde \balpha \ge n^{-1}\bone_n} \max_{\bar \uu \in \cC} Q_2\big(  \bgg, \hh ; \bar \uu, \tilde \balpha \big) &=  \min_{\tilde \balpha \ge n^{-1} \bone_n} \max_{(\rho,r) \in \cC'} Q_2\big(  \bgg, \hh ; \rho, r, \tilde \balpha \big) \\
&= \max_{(\rho,r) \in \cC'} \min_{\tilde \balpha \ge n^{-1} \bone_n} Q_2\big(  \bgg, \hh ; \rho, r, \tilde \balpha \big)
\end{align*}
where the the second equality is due to Sion's minimax theorem. In order to prove \eqref{eq:minmaxexchange}, it remains to show, for every $\rho \in \cI$, 
\begin{equation*}
\max_{\| \uu \| \le \sqrt{1-\rho^2}\, r_0} \min_{ \tilde \balpha \ge n^{-1}\bone_n} Q_2\big(  \bgg, \hh ; \bar \uu, \tilde \balpha \big) = \max_{0 \le r \le \sqrt{1-\rho^2}\, r_0} \min_{\tilde \balpha \ge n^{-1} \bone_n} Q_2\big(  \bgg, \hh ; \rho, r, \tilde \balpha \big).
\end{equation*}
It suffices to show that for every fixed $r \in [0, \sqrt{1-\rho^2}\, r_0]$, 
\begin{equation*}
\max_{\| \uu \| = r} \min_{ \tilde \balpha \ge n^{-1}\bone_n} Q_2\big(  \bgg, \hh ; \bar \uu, \tilde \balpha \big) =  \min_{\tilde \balpha \ge n^{-1} \bone_n} Q_2\big(  \bgg, \hh ; \rho, r, \tilde \balpha \big).
\end{equation*}
%On the left-hand side, in order to do maximization over $\uu$, we can fix its norm at $r \in [0, \sqrt{1-\rho^2}\, r_0]$ first, and then maximize $Q_2$ over $\{ \uu: \| \uu \| = r\}$. The second maximization will always produce the maximizer $\uu = r \hh / \| \hh \|$ regardless of $\tilde \balpha$. 
This is a consequence of the Cauchy-Schwarz inequality: the maximizer is given by $\uu = r \hh / \| \hh \|$ regardless of $\tilde \balpha$. (A similar argument has already appeared in the proof of Lemma~\ref{lem:simpleMrho}.) Thus this lemma is proved. 
\end{proof}

\begin{lem}\label{lem:quickfix}
Recall that we defined $Z_{\rho,r} = \rho Y G + \sqrt{1-\rho^2}\, r W$ in \eqref{def:Z}. Then, we have $\P(Z_{\rho,r} > 0) \ge 1/2$. 
\end{lem}
\begin{proof}[{\bf Proof of Lemma~\ref{lem:quickfix}}]
We denote the density function of a normal variable by $\phi$ and the density function of $YG$ by $p_{YG}$. First, we claim that $p_{YG}(u) \ge p_{YG}(-u)$ holds for $u \ge 0$. In fact, for $u \ge 0$,
\begin{align*}
p_{YG}(u) &= \phi(u) \varphi(u) + \phi(-u)(1 - \varphi(-u)) = \phi(u) \big( 1 + \varphi(u) - \varphi(-u) \big) \\
&\stackrel{(i)}{\ge} \phi(-u) \big(1 + \varphi(-u) - \varphi(u) \big) \\
&= p_{YG}(-u)
\end{align*}
where in $(i)$ we used symmetry of $\phi$ and monotonicity of $\varphi$. Now using the density $p_{YG}$ we write $\P(Z_{\rho,r} > 0)$ as
\begin{align*}
\P( Z_{\rho,r} > 0 ) &= \int_0^\infty \P \big(\rho u + \sqrt{1-\rho^2}\, r W > 0 \big) p_{YG}(u) + \P \big(-\rho u + \sqrt{1-\rho^2}\, r W > 0 \big) p_{YG}(-u) \; \d u.
\end{align*}
We use a similar decomposition and $W \stackrel{d}{=} -W$ for $\P(Z_{\rho,r} < 0)$, and find
\begin{align*}
\P( Z_{\rho,r} > 0 ) - \P( Z_{\rho,r} < 0 )  &= \int_0^\infty \Big[ \P \big(\rho u + \sqrt{1-\rho^2}\, r W > 0 \big) - \P \big(-\rho u + \sqrt{1-\rho^2}\, r W > 0 \big) \Big] \big(p_{YG}(u) - p_{YG}(-u) \big) \; \d u
\end{align*}
which is nonnegative. This proves $\P( Z_{\rho,r} > 0 ) \ge \P( Z_{\rho,r} < 0 )$, which implies $\P(Z_{\rho,r} > 0) \ge 1/2$ since $\P( Z_{\rho,r} = 0 ) = 0$.
\end{proof}

For $\rho \in [-1,1], r\in [0,1], t \ge 0$, let us denote by $\xi(\rho,r,-t)$ and $\zeta(\rho,r)$
\begin{align}
&\xi(\rho,r,-t) = \P \big( \rho Y G + \sqrt{1-\rho^2}\, r W < -t \big), \label{def:xi} \\
&\zeta(\rho,r) = \E \big[\big(\rho Y G + \sqrt{1-\rho^2}\, r W -   \kappa \big)_-^2 \big].
\end{align}
The probability $\xi(\rho,r,-t)$ is related to Lemma~\ref{lem:tail} in the following way. Denote
\begin{equation*}
\rho' = \frac{\rho}{\sqrt{\rho^2 + (1-\rho^2)r^2} }, \qquad t' = \frac{t}{\sqrt{\rho^2 + (1-\rho^2)r^2} }.
\end{equation*}
Then we can rewrite $\xi(\rho,r,-t)$ as 
\begin{align}
\xi(\rho,r,-t) &= \P \Big( \frac{\rho}{\sqrt{\rho^2 + (1-\rho^2)r^2}} Y G + \frac{\sqrt{1-\rho^2} \, r}{\sqrt{\rho^2 + (1-\rho^2)r^2}} W < \frac{-t}{\sqrt{\rho^2 + (1-\rho^2)r^2}} \Big)\notag \\
&= \P\big(\rho' Y G + \sqrt{1-(\rho')^2}\, W < -t' \big). \label{eq:xiequiv}
\end{align}
If $t \to \infty$ (and thus $t' \to \infty$), we can apply the tail probability Lemma~\ref{lem:tail} to find an asymptotic expression for $\xi(\rho,r,-t)$.

\begin{lem}\label{lem:xi}
Let $c>0$ be any constant. Then, there exists a sequence $(\beta_\kappa)$ with $\beta_\kappa>0$, and $\lim_{\kappa \to -\infty} \beta_\kappa = 0$ such that the following holds. 
\begin{enumerate}
\item[(a)]{ For all $\rho \in [-1,1], r \in [0,1]$, and $t \ge   |\kappa|$,
\begin{equation*}
(1- \beta_\kappa) \cdot q(-t')  \le \xi(\rho, r, -t) \le (1+ \beta_\kappa) \cdot q(-t) \exp\Big( \alpha t + \frac{\alpha^2}{2}\Big);
\end{equation*}
}
\item[(b)]{ For $\rho \in [1-c|\kappa|^{-2},1]$, $r \in [0,1]$, and $2  |\kappa| \le t \le   | \kappa |$, 
\begin{equation*}
\xi(\rho, r, -t) \le (1+ \beta_\kappa) \cdot q(-t);
\end{equation*}
}
\item[(c)]{ For all $\rho \in [-1,1], r \in [0,1]$,
\begin{equation*}
\int_0^{\infty} \xi(\rho,r,  \kappa - t) \; \d t \le  (1+ \beta_\kappa) \cdot \frac{q(\kappa)}{|\kappa|} \exp\Big( \alpha |\kappa| + \frac{\alpha^2}{2}\Big); 
\end{equation*}
}
\item[(d)]{ For all $\rho \in [-1,1]$, 
\begin{equation*}
\zeta(\rho,1) \ge \frac{1 - \beta_\kappa}{|\kappa|^2} q(\kappa);
\end{equation*}
}
\item[(e)]{ For all $\rho \in [1-c|\kappa|^{-2},1]$, 
\begin{equation*}
\zeta(\rho,1) \le \frac{1 + \beta_\kappa}{|\kappa|^2} q(\kappa);
\end{equation*}
}
\end{enumerate}
\end{lem}
\begin{proof}[{\bf Proof of Lemma~\ref{lem:xi}}] %Throughout this proof, we use $  \kappa = (1 + \breve{o}_{\kappa}(1)) \kappa$ by the definition of $b$; also, 
In different parts of this proof, we may choose different sequence $(\beta_\kappa)$ for proving inequalities (because in the end we can always take the maximum). \\
(a) In the tail probability Lemma~\ref{lem:tail}, we have
\begin{equation*}
A_{\rho,t} \le \frac{1}{t} \sqrt{\frac{2}{\pi}} \, \exp\Big( -\frac{t^2}{2} + \frac{\alpha^2}{2} \Big).
\end{equation*}
We use the relationship \eqref{eq:xiequiv} and apply Lemma~\ref{lem:tail} to obtain 
\begin{equation*}
\xi(\rho,r,-t) = (1+o_t(1)) \cdot A_{\rho', t'} \le \frac{1+o_t(1)}{t'} \sqrt{\frac{2}{\pi}} \, \exp\Big( -\frac{(t')^2}{2} + \frac{\alpha^2}{2} \Big).
\end{equation*}
Since $t \le t'$, we must have
\begin{equation*}
\xi(\rho,r,-t)  \le \frac{1+o_t(1)}{t} \sqrt{\frac{2}{\pi}} \, \exp\Big( -\frac{t^2}{2} + \frac{\alpha^2}{2} \Big) = (1+o_t(1)) \cdot q(-t) \exp\Big(\alpha t + \frac{\alpha^2}{2} \Big).
\end{equation*}
Note that $\kappa \to -\infty$ necessarily implies $t \to \infty$, so we obtain the upper bound. %For sufficiently large $t$ (guaranteed by $\kappa \to -\infty$), 
We minimizer $A_{\rho,t}$ over $\rho$ and find
\begin{equation*}
A_{\rho,t} \ge  \frac{1}{t} \sqrt{\frac{2}{\pi}} \, \exp\Big( -\frac{t^2}{2} - \alpha t\Big) = q(-t). 
\end{equation*}
This gives the lower bound.\\
(b) If $\rho \ge 1 - c|\kappa|^{-2}$, then for sufficiently negative $\kappa$ we have $\rho \ge \eta_0$, so by Lemma~\ref{lem:tail}, 
\begin{equation*}
\xi(\rho,r,-t) \le \frac{1+o_t(1)}{t'} \sqrt{\frac{2}{\pi}} \, \exp\Big( -\frac{(t')^2}{2}  - \alpha \rho' t' + \frac{\alpha^2(1-(\rho')^2)}{2} \Big).
\end{equation*}
Note that $t' \ge t$, $\rho' \ge \rho$, so
\begin{align*}
&\frac{\alpha^2(1-(\rho')^2)}{2} \le \frac{\alpha^2(1-\rho^2)}{2} \le \alpha^2 (1-\rho) \le \frac{c\alpha^2}{|\kappa|^2} = o_{\kappa}(1). \\
& \rho' t \ge \rho t \ge t - \frac{ct}{|\kappa|^2} \ge t - o_{\kappa}(1).
\end{align*}
Thus we get 
\begin{equation*}
\xi(\rho,r,-t) \le \frac{1+o_{\kappa}(1)}{t} \sqrt{\frac{2}{\pi}} \, \exp\Big( -\frac{t^2}{2}  - \alpha t\Big) = \big(1+o_{\kappa}(1)\big) \cdot q(-t). 
\end{equation*}
(c) We use the conclusion of (a) and Gaussian tail probability inequality (Lemma~\ref{lem:Gtail}) to derive
\begin{align*}
\int_0^\infty \xi(\rho,r, \kappa - t) \; \d t &\le (1+o_{\kappa}(1)) \cdot \int_{ |\kappa|}^\infty \exp\Big( -\frac{u^2}{2} \Big) \; \d u \cdot \exp\Big( \frac{\alpha^2}{2} \Big) \\
&\le \frac{1+o_{\kappa}(1)}{|\kappa|} \exp\Big( -\frac{|\kappa|^2}{2} + \frac{\alpha^2}{2} \Big) \\
& = (1+o_{\kappa}(1)) \cdot \frac{q(\kappa)}{|\kappa|} \exp\Big( \alpha |\kappa| + \frac{\alpha^2}{2}\Big).
\end{align*}
(d) We can express $\zeta(\rho,1)$ as
\begin{equation}\label{eq:zeta}
\zeta(\rho,1) = \int_0^\infty 2t \xi(\rho,1, \kappa -t) \; \d t.
\end{equation}
Note that with $r=1$, we have $t' = t$. Using the lower bound from (a), we derive
\begin{align*}
\zeta(\rho,1) & \ge (1 - \beta_\kappa) \cdot \int_0^\infty 2t q( \kappa-t) \; \d t \\
&= (1 - \beta_\kappa)\sqrt{\frac{2}{\pi}}\, \cdot \int_0^\infty \frac{2t}{ |\kappa|+t} \exp\left( -\frac{( |\kappa|+t)^2}{2} -\alpha(t+ |\kappa|) \right) \; \d t \\
&= 2(1 - \beta_\kappa)\sqrt{\frac{2}{\pi}} \exp\left(-\frac{|\kappa|^2}{2b^2}-\alpha \frac{|\kappa|}{b} \right) \cdot \int_0^\infty \frac{t}{ |\kappa|+t}  \exp\left( -\frac{t^2}{2} - t |\kappa| - \alpha t \right) \; \d t \\
&= \frac{2b^3(1 - \beta_\kappa)}{|\kappa|^3} \sqrt{\frac{2}{\pi}} \exp\left(-\frac{|\kappa|^2}{2b^2 }-\alpha \frac{|\kappa|}{b} \right) \cdot \int_0^\infty \frac{u}{1 + b^2 u |\kappa|^{-2}} \exp \left( -\frac{b^2u^2}{|\kappa|^2} - u - \frac{\alpha b u}{|\kappa|} \right) \; \d u
\end{align*}
where we used change of variable $u =  t|\kappa|$ in the last equality. Since
\begin{equation*}
\lim_{\kappa \to -\infty} \int_0^\infty \frac{u}{1 + b^2 u |\kappa|^{-2}}  \exp \left( -\frac{b^2u^2}{|\kappa|^2} - u - \frac{\alpha b u}{|\kappa|} \right) \; \d u = \int_0^\infty u \exp(-u) \; \d u = 1
\end{equation*}
by the dominated convergence theorem. Thus,
\begin{equation*}
\zeta(\rho,1) \ge \frac{2(1-o_{\kappa}(1))}{|\kappa|^3} \sqrt{\frac{2}{\pi}} \exp\left( - \frac{|\kappa|^2}{2} - \alpha |\kappa| \right) = \frac{2(1-o_{\kappa}(1))}{|\kappa|} q(\kappa).
\end{equation*}
(e) We use the identity \eqref{eq:zeta} and will show
\begin{align*}
&\zeta_1 := \int_0^{ |\kappa|} 2t \xi(\rho,1, \kappa -t) \; \d t \le \frac{2(1+o_{\kappa}(1))}{|\kappa|} q(\kappa). \\
&\zeta_2 := \int_{ |\kappa|}^\infty 2t \xi(\rho,1, \kappa -t) \; \d t \le \frac{o_{\kappa}(1)}{|\kappa|} q(\kappa).
\end{align*}
Using the upper bound in (b) and following a similar derivation as in (d), we have
\begin{align*}
\zeta_1 &\le \frac{2b^3(1 + \beta_\kappa)}{|\kappa|^3} \sqrt{\frac{2}{\pi}} \exp\left(-\frac{|\kappa|^2}{2b^2 }-\alpha \frac{|\kappa|}{b} \right) \cdot \int_0^{b^{-2}|\kappa|^2} \frac{u}{1 + b^2 u |\kappa|^{-2}} \exp \left( -\frac{b^2u^2}{|\kappa|^2} - u - \frac{\alpha b u}{|\kappa|} \right) \; \d u \\
&\le \frac{2(1+o_{\kappa}(1))}{|\kappa|^3} \sqrt{\frac{2}{\pi}} \exp\left( - \frac{|\kappa|^2}{2} - \alpha |\kappa| \right) = \frac{2(1+o_{\kappa}(1))}{|\kappa|} q(\kappa).
\end{align*}
Moreover, because
\begin{equation*}
 \frac{2t}{ |\kappa|+t} \exp\left( -\frac{( |\kappa|+t)^2}{2} -\alpha(t+ |\kappa|) \right) \le 2 \exp\left( -\frac{( |\kappa|+t)^2}{2} \right),
\end{equation*}
together with the Gaussian tail probability inequality (Lemma~\ref{lem:Gtail}), we obtain
\begin{align*}
\zeta_2 &\le (1 + \beta_\kappa)\sqrt{\frac{2}{\pi}} \, \int_{2 |\kappa|}^\infty \exp\left( -\frac{u^2}{2} \right) \; \d u \le (1 + o_{\kappa}(1))\sqrt{\frac{2}{\pi}} \,  \exp\left(-2 |\kappa|^2 \right) = \frac{o_{\kappa}(1)}{|\kappa|} q(\kappa).
\end{align*}
This completes the proof.
\end{proof}

\section{Gradient descent: Proof of Theorem \ref{thm:grad_descent}}
\label{sec:GD_proof}

\begin{proof}[\bf Proof of Theorem~\ref{thm:grad_descent}]
Without loss of generality, we may assume $y_i = 1$ and recast $y_i \xx_i$ as $\xx_i$. To begin with, we calculate the Hessian of $\hR_{n,\kappa}(\btheta)$:
\begin{align*}
	\nabla^2 \hR_{n,\kappa}(\btheta) = & \frac{1}{n} \sum_{i=1}^{n} \ell'' \left( \left\langle \xx_i, \btheta \right\rangle - \kappa \norm{\btheta}_2 \right) \left( \xx_i - \kappa \frac{\btheta}{\norm{\btheta}_2} \right) \left( \xx_i - \kappa \frac{\btheta}{\norm{\btheta}_2} \right)^\top \\
	& - \frac{\kappa}{\norm{\btheta}_2} \frac{1}{n} \sum_{i=1}^{n} \ell' \left( \left\langle \xx_i, \btheta \right\rangle - \kappa \norm{\btheta}_2 \right) \left( \bI_d - \frac{\btheta \btheta^\top}{\norm{\btheta}_2^2} \right).
\end{align*}
Hence, if $\norm{\btheta}_2 \ge r > 0$, one gets the following estimate:
\begin{align*}
    \norm{\nabla^2 \hR_{n,\kappa}(\btheta)}_{\mathrm{op}} \le & \frac{\beta}{n} \norm{\sum_{i=1}^{n} \left( \xx_i - \kappa \frac{\btheta}{\norm{\btheta}_2} \right) \left( \xx_i - \kappa \frac{\btheta}{\norm{\btheta}_2} \right)^\top}_{\mathrm{op}} + \frac{\vert \kappa \vert}{r} \sup_{u \in \R} \vert \ell'(u) \vert \norm{\bI_d - \frac{\btheta \btheta^\top}{\norm{\btheta}_2^2}}_{\mathrm{op}} \\
    \le & \beta \left( \frac{1}{n} \norm{\XX^\top \XX}_{\mathrm{op}} + 2 \frac{\vert \kappa \vert}{n} \norm{\sum_{i=1}^{n} \xx_i}_2 + \kappa^2 \right) + \frac{\vert \kappa \vert}{r} \sup_{u \in \R} \vert \ell'(u) \vert.
\end{align*}
Now if $M$ is chosen to be greater than the above quantity, then $\hR_{n,\kappa}(\btheta)$ is $M$-smooth in $\R^d \backslash \mathsf{B}_d(0, r)$, where $\mathsf{B}_d(0, r)$ is the ball of radius $r$ in $\R^d$. Assume $\eta < 2/M$ and $\norm{\btheta^t}_2 \to \infty$, we have $\norm{\btheta^t}_2 \ge r$ for large enough $t$ and therefore Lemma~10 from \cite{soudry2018implicit} tells us $\nabla \hR_{n,\kappa}(\btheta^t) \to \bzero$ as $t \to \infty$.

Denote $R_t = \norm{\btheta^t}_2, m_{it} = \langle \btheta^t/\norm{\btheta^t}_2, \xx_i \rangle - \kappa$, then we have
\begin{equation*}
	\lim_{t \to \infty} \left\langle \frac{\btheta^t}{\norm{\btheta^t}_2}, \nabla \hR_{n,\kappa}(\btheta^t) \right\rangle = \lim_{t \to \infty} \frac{1}{n} \sum_{i=1}^{n} \ell'(R_t m_{it}) m_{it} = 0.
\end{equation*}
Suppose there exists $\veps > 0$ such that for some $1 \le i \le n$, $m_{i t_k} \le - \veps$ for a subsequence $t_k \to \infty$. It follows that
\begin{align*}
	0 = & \lim_{t_k \to \infty} \sum_{i=1}^{n} \ell'(R_{t_k} m_{i t_k}) m_{i t_k} = \lim_{t_k \to \infty} \left( \sum_{m_{i t_k} \ge 0} \ell'(R_{t_k} m_{i t_k}) m_{i t_k} + \sum_{m_{i t_k} < 0} \ell'(R_{t_k} m_{i t_k}) m_{i t_k} \right) \\
	\ge & \limsup_{t_k \to \infty} \left( - \veps \ell'(-R_{t_k} \veps) - n \sup_{m \ge 0} \left\vert m \ell'(R_{t_k} m) \right\vert \right) \stackrel{(i)}{=} - \veps \lim_{u \to -\infty} \ell'(u) > 0,
\end{align*}
where $(i)$ is due to our assumption on $\ell(x)$, Definition \ref{ass:tight_exp_tail} and the fact $R_{t_k} \to \infty$. Therefore, a contradiction occurs and we complete the proof.
\end{proof}

\end{document}